\documentclass[12pt, twosided]{ociamthesis}

\usepackage{xcolor}
\pagecolor{white}

\usepackage[
  backend=biber,
  style=authoryear,
  giveninits=true,
  maxcitenames=1,
  maxbibnames=99,
  dashed=false,
  uniquename=false,
  uniquelist=false,
  labeldateparts=true
]{biblatex}
\let\citep\parencite
\let\citet\textcite
\usepackage{amssymb}
\usepackage{amsmath}
\renewbibmacro{in:}{,}
\usepackage[bottom]{footmisc}
\usepackage{appendix}
\usepackage{quoting}
\usepackage[T2A]{fontenc}
\usepackage[utf8]{inputenc}
\usepackage[russian,english]{babel}

\usepackage{algorithm}      
\usepackage{algpseudocode}  

\quotingsetup{
  font=epigraphfont,   
  leftmargin=0.075\textwidth, 
  rightmargin=0.075\textwidth 
}

\newcommand{\bI}{{\bf I}}
\newcommand{\bII}{{\bf II}}

\newcommand{\twopartdef}[4]
{ 
        \left\{
                \begin{array}{ll}
                        #1 & \mbox{if } #2 \\
                        #3 & \mbox{if } #4
                \end{array}
        \right.
}

\newcommand{\ignore}[1]{}
\usepackage{amsthm}
\usepackage{tikz}

\usepackage{graphicx}
\usepackage{caption}
\usepackage[skip=0.5ex]{subcaption}

\usepackage{listings}
\usepackage{subcaption}
\usepackage{float}

\definecolor{mydarkblue}{rgb}{0,0.08,0.45}
 
\definecolor{Code}{rgb}{0,0,0} 
\definecolor{Decorators}{rgb}{0.5,0.5,0.5} 
\definecolor{Numbers}{rgb}{0.5,0,0} 
\definecolor{MatchingBrackets}{rgb}{0.25,0.5,0.5} 
\definecolor{Keywords}{rgb}{0,0,1} 
\definecolor{self}{rgb}{0,0,0} 
\definecolor{Strings}{rgb}{0,0.63,0} 
\definecolor{Comments}{rgb}{0,0.63,1} 
\definecolor{Backquotes}{rgb}{0,0,0} 
\definecolor{Classname}{rgb}{0,0,0} 
\definecolor{FunctionName}{rgb}{0,0,0} 
\definecolor{Operators}{rgb}{0,0,0} 
\definecolor{Background}{rgb}{0.98,0.98,0.98} 
\lstdefinelanguage{Python}{ 
numbers=left, 
numberstyle=\scriptsize, 
numbersep=1em, 
xleftmargin=1em, 
framextopmargin=2em, 
framexbottommargin=2em, 
showspaces=false, 
showtabs=false, 
showstringspaces=false, 
frame=l, 
tabsize=4, 
basicstyle=\small,
backgroundcolor=\color{Background}, 
commentstyle=\color{Comments}\slshape, 
stringstyle=\color{Strings} 
morecomment=[s][\color{Strings}]{"""}{"""}, 
morecomment=[s][\color{Strings}]{'''}{'''}, 
morecomment=[s][\color{Strings}]{`}{'}, 
morecomment=[l][\color{Comments}]{\#},
morekeywords={import,from,class,def,for,while,if,is,in,elif,else,not,and,or,print,break,continue,return,True,False,None,access,as,,del,except,exec,finally,global,import,lambda,pass,print,raise,try,assert,with}, 
keywordstyle={\color{Keywords}\bfseries}, 
morekeywords={[2]@invariant,pylab,numpy,np,scipy}, 
keywordstyle={[2]\color{Decorators}\slshape}, 
emph={self}, 
emphstyle={\color{self}\slshape}, 
}  
\usepackage{courier}
\lstnewenvironment{python}[1][]{%
  \lstset{language=Python,#1}%
  \lstset{basicstyle=\footnotesize\ttfamily}
}{}

\def\dsqcup{\small\sqcup\mathchoice{\mkern-7mu}{\mkern-7mu}{\mkern-3.2mu}{\mkern-3.8mu}\sqcup}
\usepackage{enumitem}

\makeatletter

\usepackage{caption}
\usepackage{newfloat} 

\DeclareFloatingEnvironment{code}
\newrobustcmd*{\parenttexttrack}[1]{%
  \begingroup
  \blx@blxinit
  \blx@setsfcodes
  \blx@bibopenparen#1\blx@bibcloseparen
  \endgroup}

\AtEveryCite{%
  }

\makeatother

\usepackage{booktabs}
\usepackage{graphicx}
\usepackage{comment}
\usepackage{hyperref}
\usepackage{pifont}
\usepackage{multirow}
\usepackage{makecell}

\usepackage{amsthm}
\newtheorem{theorem}{Theorem}[chapter]
\newtheorem{lemma}[theorem]{Lemma}
\newtheorem{definition}[theorem]{Definition} 
\newtheorem{remark}[theorem]{Remark}
\newtheorem{corollary}[theorem]{Corollary}

\newtheorem{assumption}[theorem]{Assumption}

\title{Advances in Neural Controlled Differential Equations}   

\author{Benjamin Walker}             
\college{Balliol College}  

\renewcommand{\submittedtext}{A thesis submitted for the degree of}
\degree{Doctor of Philosophy}     
\degreedate{November 2025}

\begin{document}

\pagenumbering{roman}
\setlength{\oddsidemargin}{0.14in}
\setlength{\evensidemargin}{-0.0in}
\maketitle
\setlength{\oddsidemargin}{0.33in}
\setlength{\evensidemargin}{-0.08in}
\thispagestyle{empty}

\begin{dedication}
    To Mum and Dad, for everything.
\end{dedication}

\begin{acknowledgements}

    Terry, it has been a privilege spending an hour (and often much longer) each week discussing mathematics, machine learning, finance, politics, philosophy, and every other topic we wandered into. Although you ``cannot discuss mathematics with someone who doesn’t understand differential geometry,'' you certainly did your best with me. I would not be the researcher I am today without you.

    Mum and Dad, I will never be able to express the gratitude I have for your continuous love, encouragement, and guidance. From Mum spending hours revising with me to Dad patiently reminding me for the hundredth time that there are two solutions to $x^2=4$, I owe every step of this journey to you both.

    Sophie, Naomi, and all of my friends, thank you. From adventures Down Under, to European hiking trips and weekend getaways in London, Selsey, Wales, and Milton Keynes, from late-night board games (and those who had to put up with them) to some truly wonderful Christmas meals. I cannot imagine having done this without your support and friendship.
    
    Thank you to my collaborators and colleagues, Lingyi Yang, Nicola Mu\c{c}a Cirone, Cris Salvi, Christian Bayer, Andrew McLeod, Felix Krones, Adam Mahdi, Tiexin Qin, Haoliang Li, Cora Cartis, Kate Zhu, Ammar Naseer, Torben Berndt, Alexandre Bloch, Sam Morley, and Elena Gal. This thesis, and I personally, owe much to the time, effort, and enthusiasm you have so generously given.

    Thank you also to my examiners, Marc Deisenroth and Stephen Roberts, for your time, care, and thoughtful feedback on this thesis.

    Finally, Nathalie, the best part of this journey has been sharing it with you. I cannot wait for our next adventure.
\end{acknowledgements}

\clearpage  
\thispagestyle{plain}
    \vspace*{4cm}
    \begin{quoting}[leftmargin=1.2cm, rightmargin=1.8cm]
        There are many human behaviours which unfold over time. It would be folly to try to understand those behaviours without taking into account their temporal nature.
    \end{quoting}
    \vspace{0.5cm}
    \noindent\large---Jeffrey Elman, \emph{Finding Structure in Time} (1990)
    \normalsize
\clearpage

\begin{abstract}
    Many real-world systems evolve continuously, yet most machine learning models interpret time series as discrete sequences. Continuous-time approaches instead treat time series as samples from an underlying input path, a formulation that naturally accommodates irregularly sampled or oversampled data. Among these, Neural Controlled Differential Equations (NCDEs) are a maximally expressive class of models that parametrise a vector field using a neural network and evolve their hidden state by solving a dynamical system driven by the input path. NCDEs typically use a non-linear vector field, so their expressive power and continuous-time flexibility come at the cost of a forward pass that is both computationally expensive and inherently sequential, limiting their scalability and practical applicability.

    This thesis advances the training and scalability of NCDEs through three complementary contributions. First, building on neural rough differential equations, Log-NCDEs apply the Log-ODE method to efficiently approximate an NCDE's solution during training, improving both computational speed and empirical performance. Second, Linear NCDEs replace the non-linear vector field with a linear one, enabling closed-form solutions and parallel-in-time computation without sacrificing theoretical expressivity. Third, Structured Linear NCDEs use structured linear vector fields to further enhance efficiency while maintaining theoretical expressiveness and empirical performance.

    Collectively, these methods reduce the time per training step for an NCDE by up to three orders of magnitude while achieving state-of-the-art performance across diverse time series benchmarks.
\end{abstract}

\newpage
\setcounter{tocdepth}{2}
\tableofcontents
\newpage
\pagenumbering{arabic}

\setcounter{page}{1}

\chapter{Introduction}

\begin{quoting}
    Paths --- simply everywhere.
\end{quoting}
\noindent\large---Terry Lyons, \emph{Rough Paths, Signatures and the Modelling of Functions on Streams} (2014)
\normalsize

\section{Paths}

\subsection{The PhysioNet Challenge 2022}

\label{sec:intro_path_intro}

In the first year of my PhD, I took part in the George B.\ Moody PhysioNet Challenge 2022, an international machine learning competition focused on developing open-source solutions to clinically relevant problems in healthcare \citep{goldberger2000physiobank, reyna2023heart}. The task was to detect heart murmurs from phonocardiogram (PCG) recordings. Each patient contributed up to six recordings, with between $20{,}000$ and $180{,}000$ samples per recording. Our team, \emph{PathToMyHeart}, placed fourth out of 78 total teams \citep{walker2022dual}.
\\ \\
Our approach converted PCGs into log-mel spectrograms and applied a fine-tuned ResNet-50, which had been pre-trained on ImageNet \citep{deng2009imagenet, He2016DeepRL, walker2022dual}. Notably, the five best-performing teams all used spectrograms, and four, including the winner, applied convolutional neural networks to those spectrograms \citep{Xu2022HMSNet, Lee2022FreqTime, Lu2022LightweightRobust, McDonald2022ParallelHSMM}. Despite heart-sound recordings being time series, the dominant approach among the best-performing teams was to change the modality: move from time series to images and apply a convolutional neural network. This is particularly interesting given the concurrent success of Transformer-based sequence models \citep{bahdanau2015neural, vaswani2017attention, brown2020language}.
\\ \\
There are many reasons why a model operating on spectrograms could outperform one operating on the raw signal.
One that stands out is the application of discrete sequence models to samples from a continuous process.
For example, assuming the signal is band-limited, then a frequency-cropped spectrogram is theoretically independent of the sampling rate $r$ once it exceeds twice the signal's highest frequency \citep{nyquist1928certain, shannon1949communication}.
Therefore, past this sampling rate, the downstream model's performance is independent of $r$.
In contrast, applying a recurrent neural network (RNN) to the raw signal becomes increasingly unstable as $r$ increases due to exploding or vanishing gradients, while applying a Transformer incurs a computational cost that scales as $\mathcal{O}(r^2)$ \citep{Hochreiter1991UntersuchungenZD, hochreiter1997long, vaswani2017attention}.
Both challenges stem from treating samples of a continuous signal as a sequence of discrete observations.
A more natural approach is to model the data as a path, where the evolution is described continuously.

\subsection{From Paths to Signatures}

\label{sec:intro_sig}

This thesis is certainly not the first work to advocate for a path-based approach to time series modelling. 
Differential equations have been used to model epidemics \citep{kermack1927contribution}, economics \citep{solow1956contribution}, and ecology \citep{lotka1925elements, volterra1926fluctuations}, to name some of the applications beginning with the letter e.
In fact, neural networks were used to parametrise the vector field of a differential equation and trained to output paths before  the foundation of modern recurrent neural networks \citep{pearlmutter1989learning, elman1990finding}.
\\ \\
The approach to path-based modelling this thesis builds upon can be traced back to \citet{Chen1954Iterated}, who introduced an infinite collection of iterated integrals of a path, now known as the path's signature. 
This object has several important mathematical properties, which are discussed in detail in Section \ref{sec:signature}. 
One important feature for time series modelling is that, modulo an equivalence relation discussed in Section \ref{sec:sig_prop}, the signature determines a path uniquely \citep{hambly2010uniqueness, BOEDIHARDJO2016720}. 
This uniqueness, together with the natural grading of the signature terms, motivates the truncated signature as a finite feature set that provides a high-level description of the path over an interval.
Early applications of the truncated signature include handwritten character recognition \citep{graham2013sparse, Weixin2016} and the extraction of information from financial data streams \citep{gyurkó2014extractinginformationsignaturefinancial}. 
Another important feature for time series modelling is that any real-valued continuous function defined on a compact set of signatures can be approximated arbitrarily well by a linear function. 
Combined with uniqueness, this property implies that continuous functions on compact subsets of path space can be approximated by linear functions of the signature.
This motivates the use of truncated signature features in linear regression, as formalised by \citet{levin2016learningpastpredictingstatistics}.
A more recent development has been the introduction of the signature kernel, which allows one to operate with the complete signature \citep{kiraly2019kernels, salvi2021signature, salvi2021rough, lemercier2021distribution, lemercier2021siggpde, manten2025signature}.
In healthcare, signature methods have been applied to distinguishing between bipolar disorder and borderline personality disorder \citep{Perez_Arribas2018-xp}, diagnosing Alzheimer's disease \citep{Moore2019-tp}, speech emotion recognition \citep{wang19e_interspeech}, early detection of sepsis \citep{Morrill2019, cohen2023subtle}, and heart failure prediction from electronic health records \citep{Vauvelle2022NeuralSignatureEHR}. 
In finance, signature methods have been applied to derivative pricing \citep{Arribas2018DerivativesPU}, path-dependent stochastic models \citep{arribas2020sigsdes}, optimal stopping \citep{horvath2023optimal}, hedging \citep{cirone2025rough}, and macroeconomic nowcasting \citep{cohen2023nowcasting}.
Additionally, signature methods have proven valuable in information theory \citep{salvi2023structure, shmelev2024sparse}, cybersecurity \citep{cochrane2021sk}, and computational neuroscience \citep{holberg2024exact}.
\\ \\
The signature has typically been used as a shallow learning tool. 
A predefined, universal transformation is applied to a path, with the feature extractor being independent of any specific task or dataset. 
Consequently, learning is confined to a final map from the signature features to the desired output.
Even when signatures are incorporated into deep learning architectures, the transformation itself is often kept static, with learnable components occurring before or after the signature computation \citep{Bonnier2019DeepST, Liao2021LogsigRNN, MorenoPino2024RoughTransformers}.
As will be discussed in Section \ref{sec:cde_def}, the truncated signature is the solution to a certain controlled differential equation (CDE).
This perspective naturally raises the question: rather than using the fixed signature as a feature extractor, can we instead learn a task-specific CDE? 

\subsection{Neural Controlled Differential Equations}

A CDE describes the relationship between the increments of a control path and the evolution of a solution path using a vector field. 
Neural CDEs (NCDEs) treat time series as observations from a control path, parametrise a CDE's vector field using a neural network, and use the solution path as a continuously evolving hidden state \citep{kidger2020neuralcde}. 
An immediate benefit of this path-based approach is detaching the sampling rate of the data from how we evolve the hidden state, which is now controlled by the choice of differential equation solver.
This makes NCDEs naturally suited to handling irregularly sampled or oversampled data \citep{kidger2020neuralcde, Walker2024LogNCDE}.
They have been applied to counterfactual prediction in healthcare \citep{seedat2022continuous}, economic nowcasting \citep{seonkyu2024bridging}, survival prediction \citep{zeng2025trajsurv}, anomaly detection for driving assistance \citep{lee2024gdflowanomalydetectionncdebased}, and dynamic graphs \citep{qin2023learningdynamicgraphembeddings, berndt2025permutation}, with particular success in traffic forecasting \citep{choi2022STGNCDE, choi2023graphneuralroughdifferential}.
\\ \\
However, despite being aware of this path-based approach, our team did not attempt to train an NCDE to detect heart murmurs from PCG recordings during the PhysioNet Challenge 2022. 
Each forward pass involves using a differential equation solver to approximate the solution of a non-linear CDE. 
This process is not only computationally expensive, but also inherently sequential, preventing parallel-in-time training.
Furthermore, although some progress has been made on time series with many samples using neural rough differential equations (NRDEs) \citep{morrill2021neuralrough}, NCDEs continue to struggle on these tasks.
\\ \\
Addressing these limitations is the central focus of the work presented in this thesis.

\section{Thesis Outline}

\subsection{Contributions}

This thesis makes three key contributions to the study and application of NCDEs:

\begin{itemize}
    \item First, building on the work of NRDEs \citep{morrill2021neuralrough}, we use the Log-ODE method to approximate the solutions of NCDEs during training. 
    Rather than working directly with the raw path, this approach uses the signature to summarise the path over intervals, yielding a more efficient approximation. 
    This reduces training times and improves empirical performance, but it does not remove the core computational bottleneck: the forward pass still requires solving a non-linear differential equation sequentially.

    \item Second, we introduce Linear NCDEs, a variant where the non-linear vector field is replaced by a linear one. 
    Linear NCDEs retain the theoretical expressivity of NCDEs while admitting explicit solutions, removing the need for a differential equation solver and enabling parallel-in-time training. 
    Linearity is the key structural innovation that makes NCDEs scalable.

    \item Third, we develop Structured Linear NCDEs (SLiCEs), which replace the dense matrices of a Linear NCDE with structured variants that preserve the model's expressivity while further improving computational efficiency.
\end{itemize}

These three advances are complementary: the Log-ODE method provides an efficient path-based approximation, Linear NCDEs provide a scalable model architecture, and SLiCEs improve the computational efficiency.
Together, they reduce the time per training step for NCDEs by up to three orders of magnitude, making continuous-time models applicable at scales that were previously infeasible. 
The core of this thesis brings together and extends three publications, with my specific contributions to each detailed below.

\begin{itemize}
    \item \fullcite{Walker2024LogNCDE}
    \begin{itemize}
        \item Developed the methodology, implementation, and empirical validation of Log-NCDEs presented in Section \ref{sec:log_ncde}.
        \item The proof of Lemma \ref{lem:normcomplip2} presented in Section \ref{sec:1_2_case}, which explicitly bounds the $\mathrm{Lip}(\gamma)$-norm for the composition of two $\mathrm{Lip}(\gamma)$ functions when $1 < \gamma \leq 2$.
        This proof was completed collaboratively with Dr.\ Andrew McLeod.
        \item The proof of Theorem \ref{thm:lipnn} presented in Section \ref{sec:lipgammaNN}, which uses Lemma \ref{lem:normcomplip2} to bound the $\mathrm{Lip}(\gamma)$-norm for a specific class of neural networks when $1 < \gamma \leq 2$.
    \end{itemize}
    \item \fullcite{cirone2024deepSSM}
    \begin{itemize}
        \item Developed the methodology in collaboration with Nicola Mu\c{c}a Cirone.
        \item Solely contributed the implementation and empirical validation of Linear NCDEs presented in Sections \ref{sec:empirical} and \ref{sec:ssm_results}.
    \end{itemize}
    \item \fullcite{walker2025structuredlinearcdesmaximally} \emph{(Spotlight)}
    \begin{itemize}
        \item Developed the concept, methodology, implementation, and empirical validation of SLiCEs presented in Section \ref{sec:slice}.
    \end{itemize}
\end{itemize}

All of my published work has been completed in close collaboration with my co-authors. 
In particular, I want to highlight Nicola Mu\c{c}a Cirone, who proved Theorems \ref{thm:max_prob_lin_ncde}, \ref{thm:linear_CDEs_closure}, and \ref{thm:diagonal_expr}, originally presented in \citep{cirone2024deepSSM}, and Theorem \ref{thm:slice_max_prob}, originally presented in \citep{walker2025structuredlinearcdesmaximally}.
These theorems are presented in this thesis to give a complete picture, with original overviews of the proof techniques provided where appropriate.
This thesis also presents the following previously unpublished contributions:

\begin{itemize}
    \item To the best of our knowledge, Section \ref{sec:lipgamma} provides the first formal treatment of the Lie bracket for two $\mathrm{Lip}(\gamma)$ functions defined on arbitrary subsets of potentially infinite-dimensional Banach spaces.
    \item Section \ref{sec:opt} provides an explicit example that lower bounds the norm of the composition of two $\mathrm{Lip}(\gamma)$ functions when $1 < \gamma \leq 2$. 
    This complements the upper bound in Lemma \ref{lem:normcomplip2} originally presented in \cite{Walker2024LogNCDE}. We also show that these two bounds converge as $\gamma \to 1$.
    \item Theorem \ref{thm:lipnn} generalises \citep[Theorem 3.1]{Walker2024LogNCDE} by extending the proof that certain neural networks are $\mathrm{Lip}(\gamma)$ from the specific case of Sigmoid Linear Unit (SiLU) activation functions and $\gamma=2$ to a broader class of activation functions and $1 < \gamma \leq 2$ \citep{ELFWING20183}.
\end{itemize}

\subsection{Organisation}

This thesis is structured into two parts. 
The first half establishes the mathematical foundations for the continuous-time models developed later.
\begin{itemize}
    \item Chapter 2 introduces the prerequisite mathematics. 
    Topics covered include the tensor algebra, the signature of a path, existence and uniqueness of solutions to CDEs, and the Log-ODE method for approximating solutions to CDEs. 
    These concepts provide the foundation for the path-based approach adopted in this thesis: paths are the fundamental objects, signatures give a principled description of a path over an interval, and the Log-ODE method shows how such descriptions can be used to efficiently approximate continuous-time dynamics.

    \item Chapter 3 introduces $\mathrm{Lip}(\gamma)$ regularity, which specifies the assumptions on the vector field of a CDE needed to establish existence and uniqueness of solutions. 
    The chapter also provides a formal treatment of the Lie bracket for $\mathrm{Lip}(\gamma)$ functions, a key component of the Log-ODE method. 
    Finally, a novel explicit bound on the $\mathrm{Lip}(\gamma)$ norm of the composition of two $\mathrm{Lip}(\gamma)$ functions when $1 < \gamma \leq 2$ is presented. 
    These results provide the regularity theory needed to justify the application of the Log-ODE method to NCDEs.
\end{itemize}
The second half of the thesis uses these mathematical foundations to develop scalable continuous-time models.
\begin{itemize}
    \item Chapter 4 begins by introducing NCDEs. 
    The results of Chapter 3 are used to ensure that the neural network parametrising the NCDE's vector field satisfies $\mathrm{Lip}(\gamma)$ regularity, allowing the application of the Log-ODE method. 
    The chapter then shows that using the Log-ODE method during training improves the empirical performance and computational efficiency of NCDEs across a range of real-world time series benchmarks. 
    The chapter concludes by highlighting a key remaining limitation of this approach: because the dynamics are governed by a non-linear differential equation, the forward pass remains inherently sequential.

    \item Chapter 5 introduces Linear NCDEs, a subclass of NCDEs with vector fields that depend linearly on the hidden-state.
    Linear NCDEs are shown to retain the full theoretical expressivity of NCDEs whilst admitting explicit solutions, removing the need for a differential equation solver and enabling parallel-in-time computation. 
    The chapter also shows that modern Structured State-Space Models (SSMs) can be viewed as a restrictive subclass of Linear NCDEs, allowing a formal characterisation of their expressive limitations.
    Finally, the chapter introduces SLiCEs, which replace the dense matrices of a Linear NCDE with structured variants that reduce the computational burden whilst retaining full theoretical expressivity and achieving equivalent empirical performance.

    \item Chapter 6 concludes the thesis by summarising the key contributions, reflecting on their implications for scalable continuous-time models, and discussing promising avenues for future research.
\end{itemize}

\chapter{Controlled Differential Equations}

\label{chap:cde}

\begin{quoting}
    The winding roads must be made straight and the rough paths made smooth.
\end{quoting}
\noindent\large{---Luke~3{:}5, \textit{Good News Translation}}
\normalsize

\section{Introduction}

\label{sec:cde_intro}
Realising the benefits of continuous-time machine learning requires a mathematical framework for understanding continuous paths.
This chapter develops that framework by introducing CDEs and the signature of a path.
These objects are closely connected.
The signature arises naturally when deriving the flow of a linear CDE, as shown in Theorem \ref{thm:linear_cde_solution}, and is itself the solution of a CDE.
Furthermore, the signature has several properties that make it central to the methods developed in this thesis.
In particular, it provides a graded representation of a path, its components form a universal feature set, and it composes associatively under concatenation of intervals, enabling parallel-in-time computation via an associative scan.
\\ \\
These ideas already appear in the simplest one-dimensional example, the signature of the time path,
\begin{equation}
    X_t = X(t) = t \in \mathbb{R},
\end{equation}
for $t \in [a,b]$.
The signature of a one-dimensional path over an interval $[a,t]$ is given by an infinite collection of real-valued numbers,
\begin{equation}
    S_{[a,t]} = [S^0_{[a,t]}, S^1_{[a,t]}, S^2_{[a,t]}, \ldots, S^i_{[a,t]}, \ldots],
\end{equation}
where the superscript $i \in \mathbb{N}_0 = \{0\}\, \cup\, \mathbb{N}$ denotes the $i^{\text{th}}$ component of the signature and the subscript $[a,t]$ records the interval over which it is computed.
For the time path, the components of the signature are defined recursively by the system of ordinary differential equations
\begin{equation}
    \label{eq:flat_exp_ode}
    \mathrm{d}S^i_{[a,t]} = S^{i-1}_{[a,t]}\,\mathrm{d}t,
\end{equation}
where $S^0_{[a,t]} = 1$ for all $t \in [a,b]$ and $S^i_{[a,a]} = 0$ for $i \geq 1$.
The solution to this system is
\begin{equation}
    S_{[a,t]} = \left[1, t-a, \frac{(t-a)^2}{2!}, \ldots, \frac{(t-a)^i}{i!}, \ldots\right]
\end{equation}
for $t \in [a,b]$.
The components of the time path's signature are scaled monomials corresponding to the terms in the exponential series. Hence, the sequence of partial sums $\sum_{i=0}^{N} S^i_{[a,t]}$ converges uniformly to the exponential,
\begin{equation}
    \sup_{t \in [a,b]} \left| \mathrm{e}^{t-a} - \sum_{i=0}^{N} S^i_{[a,t]} \right| \xrightarrow{N\to\infty} 0.
\end{equation}
More generally, the signature retains this graded, exponential-like structure for arbitrary paths, although in higher dimensions its components no longer take values in $\mathbb{R}$, but in different tensor spaces, as will be seen in Section~\ref{sec:signature}.
\\ \\
Since the components of the time path's signature are scaled monomials, a straightforward application of the Weierstrass Approximation Theorem shows that any continuous function on $[a,b]$ can be approximated to arbitrary precision by a linear combination of the terms of $S_{[a,t]}$. 

\begin{theorem}[Weierstrass Approximation Theorem \citep{Weierstrass1885Uber}]
    \label{thm:weierstrass}
    Let $f: [a, b] \to \mathbb{R}$ be a continuous function. 
    Then, for every $\epsilon > 0$, there exists a polynomial $P$ such that
    \begin{equation}
        \sup_{t \in [a, b]} |f(t) - P(t)| < \epsilon.
    \end{equation}
\end{theorem}

\begin{corollary}
    \label{cor:E_t_approximation}
    Let $S_{[a, t]} = [1, (t-a), \frac{1}{2!}(t-a)^2, \ldots, \frac{1}{i!}(t-a)^i, \ldots]$ for $t \in [a, b]$. 
    Then, for any continuous function $f: [a, b] \to \mathbb{R}$ and any $\epsilon > 0$, there exists a finite sequence of coefficients $L = [L_0, L_1, \ldots, L_n]$ such that
    \begin{equation}
        \sup_{t \in [a, b]} \left| f(t) - \sum_{i=0}^{n} L_i S^i_{[a, t]} \right| < \epsilon.
    \end{equation}
\end{corollary}
\begin{proof}
    The elements $S^i_{[a, t]}$ are scaled monomials, and their finite linear combinations form the polynomials. 
    Therefore, by the Weierstrass Approximation Theorem, any continuous function can be uniformly approximated by a finite sum of $S^i_{[a,t]}$ to arbitrary precision.
\end{proof}
Hence, linear maps of the time path's signature are dense in the space of continuous functions of time.
Section \ref{sec:signature} will extend this same universality from continuous functions of time to continuous functions of paths.
\\ \\
We can define a product of $S_{[a, b]}$ and $S_{[b, c]}$, which corresponds to concatenating the intervals $[a,b]$ and $[b,c]$.
\begin{lemma}[One-Dimensional Concatenation Product]
    \label{lem:concatenation}
    For intervals $[\alpha,\beta]$ and $[\gamma,\delta]$, let $*$ denote the concatenation product defined by
    \begin{equation}
        \label{eq:oned_sig_prod_statement}
        (S_{[\alpha,\beta]} * S_{[\gamma,\delta]})^k
        = \sum_{j=0}^k S^j_{[\alpha,\beta]} S^{k-j}_{[\gamma,\delta]},
    \end{equation}
    for $k\in\mathbb{N}_0$. Then, for any $a \leq b \leq c$,
    \begin{equation}
        S_{[a,c]} = S_{[a,b]} * S_{[b,c]}.
    \end{equation}
\end{lemma}

\begin{proof}
    Using \eqref{eq:oned_sig_prod_statement},
    \begin{align}
    \label{eq:oned_sig_prod}
        (S_{[a, b]} * S_{[b, c]})^k &= \sum_{j=0}^k S^j_{[a, b]} S^{k - j}_{[b, c]} \\
        &= \sum_{j=0}^k \frac{(b - a)^j (c - b)^{k - j}}{j!(k - j)!}, \\
        &= \frac{1}{k!} \sum_{j=0}^k \binom{k}{j} (b - a)^j (c - b)^{k - j}
    \end{align}
    Applying the Binomial Theorem,
    \begin{equation}
        (S_{[a, b]} * S_{[b, c]})^k = \frac{(c - a)^k}{k!} = S^k_{[a, c]}.
    \end{equation}
    Since the $k$-th components of $S_{[a, c]}$ and $S_{[a, b]} * S_{[b, c]}$ are equal for all $k \geq 0$,
    \begin{equation}
        S_{[a, c]} = S_{[a, b]} * S_{[b, c]}.
    \end{equation}
\end{proof}
This simple setting already demonstrates an important computational advantage of signatures.
To compute the signature over a long interval, one may first partition the interval into smaller subintervals, compute the signature on each subinterval independently, and then combine the results using the product.
Since this product is associative, these local computations can be aggregated in parallel via an associative scan \citep{blelloch1993prefix}.
The same idea will also be used in Chapter~\ref{chap:lin_ncde} to make Linear NCDEs parallel-in-time.
\\ \\
The one-dimensional case already illustrates the core properties of the signature that make it a powerful tool.
It provides a universal, graded representation of a path that composes naturally under concatenation of intervals.
For general paths, these components are no longer real-valued, but instead take values in tensor spaces.
Formulating this precisely requires a brief detour into the tensor algebra.

\section{The Tensor Algebra}

\subsection{Tensor Product Space}
\label{sec:tensor_prod_space}

For paths taking values in a vector space, the higher levels of the signature are defined by iterated integrals of the path increments. At the second level, these terms involve pairs of increments and so naturally define bilinear quantities. Higher levels similarly involve several increments and hence give rise to multilinear quantities. To treat such objects within a linear framework, we seek a vector space through which these bilinear interactions may be represented linearly. This motivates the following universality property.

\begin{definition}[Universal for Bilinearity \citep{roman2007advanced14}]
    Let $U$ and $V$ be vector spaces. A vector space $T$
    with a bilinear map $t : U \times V \to T$ is universal for bilinearity if for every vector space $W$ and every bilinear map
    $\kappa : U \times V \to W$, there exists a unique linear map
    $\tau : T \to W$ satisfying $\kappa = \tau \circ t$.
\end{definition}

We now construct a vector space with this property. 
Let $U$ and $V$ be vector spaces over $F$, and let $F_{U \times V}$ be the vector space with basis $\{(u,v) : u \in U, v \in V\}$. 
Let $R$ be the subspace of $F_{U\times V}$ spanned by the expressions
\begin{equation}
\begin{aligned}
(u+u', v)   &- (u, v) - (u', v), \\
(r u, v)    &- r (u, v), \\
(u, v+v')   &- (u, v) - (u, v'), \\
(u, r v)    &- r (u, v),
\end{aligned}
\end{equation}
where $r\in F$, $u,u'\in U$, and $v,v' \in V$.
These expressions measure the failure of bilinearity, so passing to the quotient identifies them with zero and thereby forces bilinearity to hold.
The quotient space
\begin{equation}
    U\otimes V = \frac{F_{U\times V}}{R},
\end{equation}
is called the tensor product of $U$ and $V$. 
The map $\iota :U \times V\rightarrow U \otimes V$ is defined by projecting $(u,v)$ onto $U\otimes V$, with the result denoted $\iota(u,v)=u\otimes v$.

\begin{theorem}
    \label{thm:tensor_universal}
    Let $U$ and $V$ be vector spaces. 
    Up to isomorphism, the tensor product space $U\otimes V$
    with the bilinear map $\iota : U \times V \to U \otimes V$ is the unique space that is universal for bilinearity \citep{roman2007advanced14}.
\end{theorem}
Theorem \ref{thm:tensor_universal} shows that the space constructed above is, up to isomorphism, the only vector space with the required universality property. Thus the tensor product is the natural space for bilinear quantities. Repeating this construction yields the higher tensor powers needed for the multilinear terms appearing in the signature.
\\ \\
Figure \ref{fig:tensor_commute} is a schematic representation of the universal bilinearity satisfied by the tensor product space. To illustrate the tensor product and universal property, take $U=\mathbb{R}^2$ and $V=\mathbb{R}^3$ with bases $(e_1,e_2)$ and $(f_1,f_2,f_3)$, respectively. 
Let $\Phi: U\otimes V \rightarrow \mathbb{R}^{2\times 3}$ be the isomorphism defined by mapping $e_i \otimes f_j$ to the matrix $E_{ij}$, which has a $1$ at $(i,j)$ and zeros elsewhere. 
Let $u=[u_1,u_2]^T\in \mathbb{R}^2$ and $v=[v_1,v_2,v_3]^T\in \mathbb{R}^3$. The image of their tensor product under $\Phi$ is the outer product of the vectors,
\begin{equation}
    \Phi(u\otimes v) = \begin{bmatrix} u_1 \\ u_2 \end{bmatrix} \begin{bmatrix} v_1 & v_2 & v_3 \end{bmatrix} = \begin{bmatrix} u_1v_1 & u_1v_2 & u_1v_3 \\ u_2v_1 & u_2v_2 & u_2v_3 \end{bmatrix}.
\end{equation}
Every bilinear expression $\kappa(u,v)$ depends on products $u_i v_j$ of components of
$u$ and $v$. 
In $U\otimes V$, these products are represented linearly by the basis elements $e_i \otimes f_j$. 
Thus, given a bilinear map $\kappa: U \times V \to W$, the universal property yields a unique linear map $\tau: U \otimes V \to W$ such that $\tau(u \otimes v) = \kappa(u, v)$.

\begin{figure}
\centering
\includegraphics[width=0.5\textwidth]{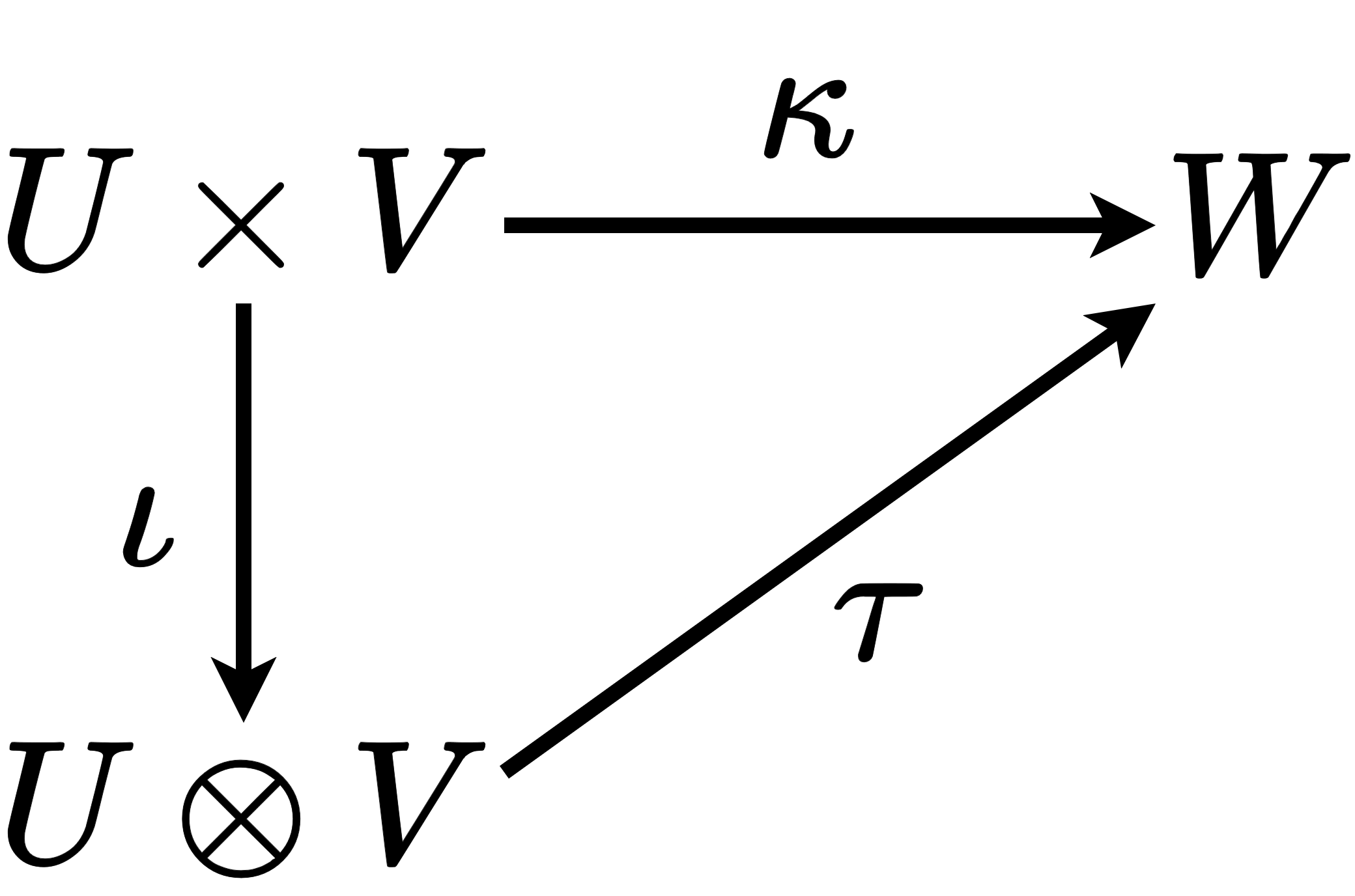}
\caption{Commutative diagram expressing the universal bilinearity of the tensor product, where each bilinear map $\kappa: U \times V \rightarrow W$ factors uniquely through a linear map $\tau: U\otimes V \rightarrow W$.}
\label{fig:tensor_commute}
\end{figure}

\subsection{Definition}

\label{sec:tens_alg}
The previous subsection introduced tensor product spaces as the natural linear spaces for the multilinear terms in the signature. 
To study the signature analytically, we now assume that the path takes values in a Banach space $V$. 
The higher levels of the signature then take values in the iterated tensor products of $V$. 
To analyse these terms, we must equip the tensor powers of $V$ with suitable norms.
\\ \\
Let $V$ be a Banach space. For each $n\geq1$, let $V^{\otimes n}$ denote the completion of the $n-$fold tensor product of $V$ with respect to a norm $\|\cdot\|_{V^{\otimes n}}$.
Throughout this thesis, we assume that the family of norms $\{\|\cdot\|_{V^{\otimes n}}\}_{n=1}^\infty$ satisfies the following conditions for all $m,n\geq 1$, $v_1,\ldots,v_n\in V$, $v\in V^{\otimes n}$, and $w\in V^{\otimes m}$.
\begin{enumerate}
\item For every $\sigma$ in the symmetric group over $\{1,\ldots,n\}$, let $\pi_{\sigma}$ denote the permutation operator defined on simple tensors by $\pi_{\sigma}(v_1\otimes \cdots \otimes v_n) = v_{\sigma(1)} \otimes \cdots \otimes v_{\sigma(n)}$ and extended to generic $v$ by linearity and completion. Then $\pi_\sigma$ is an isometry.
\item The norm is sub-multiplicative, $|| v \otimes w ||_{V^{\otimes (n+m)}} \leq ||v||_{V^{\otimes n}}\;||w||_{V^{\otimes m}}$.
\item For any bounded linear functional $f$ on $V^{\otimes n}$ and $g$ on $V^{\otimes m}$, there exists a unique bounded linear functional $f \otimes g$ on $V^{\otimes (m+n)}$ such that $(f \otimes g)(v \otimes w) = f(v)g(w)$. 
\item For a bounded linear functional $h$ on a Banach space $X$, let
\begin{equation}
    \|h\|_{\operatorname{op}} = \sup_{x\neq 0}\frac{|h(x)|}{\|x\|_X}
\end{equation}
be the operator norm. Then for bounded linear functionals $f$ on $V^{\otimes n}$ and $g$ on $V^{\otimes m}$, let $f \otimes g$ denote the linear functional defined by
\begin{equation}
    (f \otimes g)(v \otimes w) = f(v)g(w)
\end{equation}
for simple tensors and extended by linearity. Then
\begin{equation}
    \label{eq:lipgamma_norm_assum1}
    \|f \otimes g\|_{\operatorname{op}} \leq \|f\|_{\operatorname{op}}\|g\|_{\operatorname{op}}.
\end{equation}
\end{enumerate}
Families of norms which satisfy conditions $1$ and $2$ are called admissible tensor norms \citep{lyons2007differential}.
Requiring conditions $1$, $2$, and $3$ is the setting of \citet{lyons2002} and \citet{BOEDIHARDJO2016720}.
Families of norms satisfying conditions $1$, $2$, and $4$ are known as reasonable tensor algebra norms \citep{ChangLyonsNi2018}.
In fact, condition $4$ implies condition $3$ by continuity, so reasonable tensor algebra norms automatically satisfy all four properties.
Families of reasonable tensor algebra norms will be sufficient for the contents of this thesis \citep{ChangLyonsNi2018}.
Moreover, conditions $2$ and $4$ ensure that for each decomposition $n=r+s$ with $r,s\geq1$, the norm on $V^{\otimes n}$ induces a reasonable cross norm on $V^{\otimes r}\otimes V^{\otimes s}$.
Therefore, the bounds in conditions $2$ and $4$ hold with equality \citep{ryan2002introduction, lyons2025}.
Equipping each $V^{\otimes n}$ with either the projective or injective tensor norm yields a family of reasonable tensor algebra norms \citep{ChangLyonsNi2018}.

\begin{definition}[The Tensor Algebra \citep{lyons2007differential}]
    Let $\{V^{\otimes n}\}_{n=1}^{\infty}$ be equipped with a family of norms $\{\|\cdot\|_{V^{\otimes n}}\}_{n=1}^\infty$ satisfying the above conditions, and define $V^{\otimes 0}=\mathbb{R}$. The tensor algebra is the set
    \begin{equation}
        T((V)) = \{\mathbf{x} = (x^0, x^1, \ldots ) | x^n \in V^{\otimes n}\}
    \end{equation}
    with product $\mathbf{z}=\mathbf{x}\otimes\mathbf{y}$ defined by
    \begin{equation}
    \label{eq:tensor_prod}
        z^k = (\mathbf{x} \otimes \mathbf{y})^k = \sum_{j=0}^k x^j \otimes y^{k-j}.
    \end{equation}
\end{definition} 
The tensor algebra product is associative and has unit $\mathbf{1} = (1,0,0,\ldots)$. 
As $T((V))$ is an associative algebra, it has a Lie algebra structure, with Lie bracket
\begin{equation}
[\mathbf{x}, \mathbf{y}] = \mathbf{x} \otimes \mathbf{y} - \mathbf{y} \otimes \mathbf{x}
\end{equation}
for $\mathbf{x},\mathbf{y} \in T((V))$ \citep{reutenauer1993free}. 
When $V=\mathbb{R}$, each tensor power $V^{\otimes n}$ is canonically identified with $\mathbb{R}$, and an element of $T((V))$ may be viewed simply as a sequence of real numbers.
In this case, the tensor algebra product is exactly the one-dimensional concatenation product introduced in Lemma \ref{lem:concatenation}.
Four substructures of the tensor algebra which will be relevant in this thesis are:
\begin{enumerate}
    \item The space with only finitely many non-zero terms, denoted $T(V)$.
    \item The truncated tensor algebra, denoted $T^{n}(V)$, whose elements are of the form $\mathbf{x} = (x^0, x^1, \ldots, x^n)$.
    \item The space with $x^0=1$, denoted $\tilde{T}((V))$. 
    This space is a group with inverse
    \begin{equation}
        \mathbf{x}^{-1} = 1 - (\mathbf{x} - 1) + (\mathbf{x} - 1)^{\otimes 2} - (\mathbf{x} - 1)^{\otimes 3} + \cdots.
    \end{equation}
    \item The space of Lie series generated by $V$, 
    \begin{equation}
        \mathfrak{L}((V)) = \{(l_0,l_1,\ldots):l_i\in L_i\},
    \end{equation}
    where $L_0=0$, $L_1=V$, and $L_{i+1}$ is the span of $[v,l]$ for $v\in V$ and $l\in L_i$, which is denoted $[V,L_i]$.
\end{enumerate}

\subsection{The Tensor Algebra and $l^2(\mathbb{W}_d)$}
\label{sec:tens_alg_l2}
Let $V$ be a $d-$dimensional Banach space and $(e_1,\ldots,e_d)$ be a basis of $V$. 
A basis of $V^{\otimes n}$ is given by 
\begin{equation}
    \left\{e_I = e_{i_1}\otimes \cdots \otimes e_{i_n} | i_j \in \{1,\ldots, d\}\right\},
\end{equation}
where $I = (i_1, i_2, \dots, i_n)$ is a multi-index.
Let $\mathbb{W}_d$ denote the set of all finite sequences, or words, composed of the letters $\{1, 2, \dots, d\}$, with the empty word $\emptyset$ included. 
There is a natural identification between the basis of $V^{\otimes n}$ and the words of length $n$ via the map $\phi(e_{i_1} \otimes \cdots \otimes e_{i_n}) = i_1\cdots i_n$. 
Furthermore, elements of the tensor algebra can be identified with real-valued functions $A: \mathbb{W}_d \to \mathbb{R}$. 

\begin{definition}[$l^2(\mathbb{W}_d)$ Space \citep{KreyszigFunctionalAnalysis}]

The space $l^2(\mathbb{W}_d)$ consists of all real-valued functions $A: \mathbb{W}_d \to \mathbb{R}$ such that the sum of the squares of their values is finite:
\begin{equation}
\| A \|_{l^2} = \left( \sum_{I \in \mathbb{W}_d} A_I^2 \right)^{\frac{1}{2}} < \infty.
\end{equation}
This space is a Hilbert space when equipped with the inner product:
\begin{equation}
\langle A, B \rangle_{l^2} = \sum_{I \in \mathbb{W}_d} A_I B_I.
\end{equation}
\end{definition}

Using the above identification of $T((V))$ with real-valued functions on $\mathbb{W}_d$ and considering each as purely sets, we have the following inclusions: $T(V)\subset l^2(\mathbb{W}_d)\subset T((V))$.

\section{The Signature}

\label{sec:signature}

\subsection{Definition}

Armed with the tensor product, we can now extend the notion of a signature to paths taking values in a Banach space.
As in the one-dimensional example of Section \ref{sec:cde_intro}, the construction is based on iterated integrals.
Therefore, we begin by specifying a suitable notion of integration for the setting of this thesis, the Young integral.

\begin{definition}[$p$-variation \citep{Young1936AnIO, lyons2007differential}]
    Let $V$ be a Banach space and $p \ge 1$. 
    The $p$-variation of a path $X : [a,b] \to V$ is defined by
    \begin{equation}
        \|X\|_{p,[a,b]} = \left(\sup_{\mathcal{D}}\sum_{j=0}^{r-1} \|X_{t_{j+1}} - X_{t_j}\|^p\right)^{1/p},
    \end{equation}
    where $\mathcal{D}$ is the set of all finite partitions $(t_j)_{j=0}^r$ satisfying 
    $a = t_0 < t_1 < \cdots < t_r = b$.
\end{definition}
The set of paths from $[a,b]$ to $V$ with finite $p-$variation is denoted $\mathcal{V}^p([a,b], V)$.
\begin{definition}[Bounded Linear Maps \citep{KreyszigFunctionalAnalysis}]
    Let $V$ and $W$ be Banach spaces. A linear map $A:V\to W$ is bounded if there exists a constant $c>0$ such that
    \begin{equation}
        \|A(v)\|_W \leq c\|v\|_V
    \end{equation}
    for all $v\in V$.
\end{definition}
We denote the set of all bounded linear maps from $V$ to $W$ by $\mathbf{L}(V, W)$.
\begin{theorem}[Young Integration \citep{Young1936AnIO, lyons2007differential}]
    \label{thm:young_integral}
    Let $V$ and $W$ be Banach spaces, $p,q \geq 1$ satisfy $\frac{1}{p}+\frac{1}{q}>1$,  $X\in\mathcal{V}^p([a,b], V)$, and $Y\in\mathcal{V}^q([a,b], \mathbf{L}(V, W))$. Let $\{d^n=(t^n_0,\ldots,t^n_{N_n})\}_{n=0}^\infty$ be a sequence of finite partitions of $[a,b]$ with $\sup_{j}|t^n_{j+1}-t^n_j|\to 0$ as $n\rightarrow \infty$. Then the Young integral defined by the limit
    \begin{equation}
        \label{eq:young_integral}
        \int_a^b Y_s\mathrm{d}X_s = \lim_{n \to \infty} \sum_{i=0}^{N_n-1}Y_{t^n_i}[X_{t_{i+1}^n}-X_{t_i^n}]
    \end{equation}
    exists and does not depend on the choice of partitions $\{d^n\}_{n=0}^\infty$. 
\end{theorem}
\begin{proof}
    See \citep{Young1936AnIO} for real-valued paths and \citep{lyons2007differential} for paths taking values in Banach spaces.
\end{proof}
The integral defined by the limit in \eqref{eq:young_integral} also exists when both $X$ and $Y$ are continuous and at least one of them has finite $1-$variation. 
In this case, it is known as a Riemann-Stieltjes integral \citep{Stieltjes1894, Apostol1974MathematicalAnalysis}.
\begin{definition}[The Signature \citep{lyons2007differential}]
    Let $X\in\mathcal{V}^p([a,b], V)$ with $p<2$ and define 
    \begin{equation}
        X^n_{[a,b]} = \underbrace{\int\cdots\int}_{\substack{u_1<\cdots<u_n \\ u_1,\ldots,u_n\in [a,b]}}\mathrm{d}X_{u_1}\otimes \cdots \otimes \mathrm{d}X_{u_n} \in V^{\otimes n},
    \end{equation}
    where the integral is defined in the Young sense \citep{Young1936AnIO}. 
    The signature of the path $X$ is defined as
    \begin{equation}
        S_{[a,b]}(X)=(1,X^1_{[a,b]},\ldots,X^n_{[a,b]},\ldots) \in \tilde{T}((V)).
    \end{equation}
\end{definition}
To illustrate this definition, consider the case $V=\mathbb{R}^d$ with standard basis $(e_1,\ldots,e_d)$ and $X_t=(X^1_t,\ldots,X^d_t)$. 
Then
\begin{equation}
    X^n_{[a,b]} = \sum_{i_1,\ldots,i_n=1}^d S^{(i_1,\ldots,i_n)}_{[a,b]}\, e_{i_1}\otimes\cdots\otimes e_{i_n},
\end{equation}
where
\begin{equation}
    \label{eq:signature_element}
    S^{(i_1,\ldots,i_n)}_{[a,b]} = \underbrace{\int\cdots\int}_{\substack{u_1<\cdots<u_n \\ u_1,\ldots,u_n\in [a,b]}}\mathrm{d}X^{i_1}_{u_1}\cdots \mathrm{d}X^{i_n}_{u_n}.
\end{equation}
In particular, the first-level terms $S^{(i)}_{[a,b]} = X^i_b - X^i_a$ are the coordinate increments of the path, while higher-order terms capture ordered interactions between coordinates.
\\ \\
If $X:[a,b]\to\mathbb{R}$ is one-dimensional, then for each $n\geq1$,
\begin{equation}
    \label{eq:one_d_sig_terms}
    X^n_{[a,b]} = \frac{(X_b-X_a)^n}{n!}.
\end{equation}
As for the time path, the signature of a one-dimensional path depends only on the total increment $X_b-X_a$. 
In dimension two and higher, the order of the increments matters. 
For example, the piecewise linear paths $(0,0)\to(1,0)\to(1,1)$ and $(0,0)\to(0,1)\to(1,1)$ have the same first level, since they have the same increment, but different second levels. 
In $\mathbb{R}^2$, this difference is captured by the antisymmetric part of the second level,
\begin{equation}
    \mathcal{A} = \frac{1}{2}\left(S^{(1,2)}_{[a,b]} - S^{(2,1)}_{[a,b]}\right),
\end{equation}
which is the signed area enclosed by the path and its chord. 
So, beyond one dimension, the signature records not only where a path starts and ends, but also information about the route it traverses, as illustrated in Figure~\ref{fig:signature_intuition}.
\\ \\
The condition $\frac{1}{p}+\frac{1}{q}>1$ in Theorem \ref{thm:young_integral} restricts our definition of the signature to paths with finite $p-$variation for $p$ less than $2$.
Although this is sufficient for the paths considered in this thesis, many important examples, such as Brownian motion, have infinite $p-$variation for all $p<2$.
Rough path theory provides a generalisation of the construction presented here to include paths that have finite $p$-variation for some $p\geq2$, but for no $p<2$ \citep{Lyons1998}.
An accessible introduction to rough path theory can be found in \citep{lyons2007differential}.

\begin{figure}
\centering
\includegraphics[width=\textwidth]{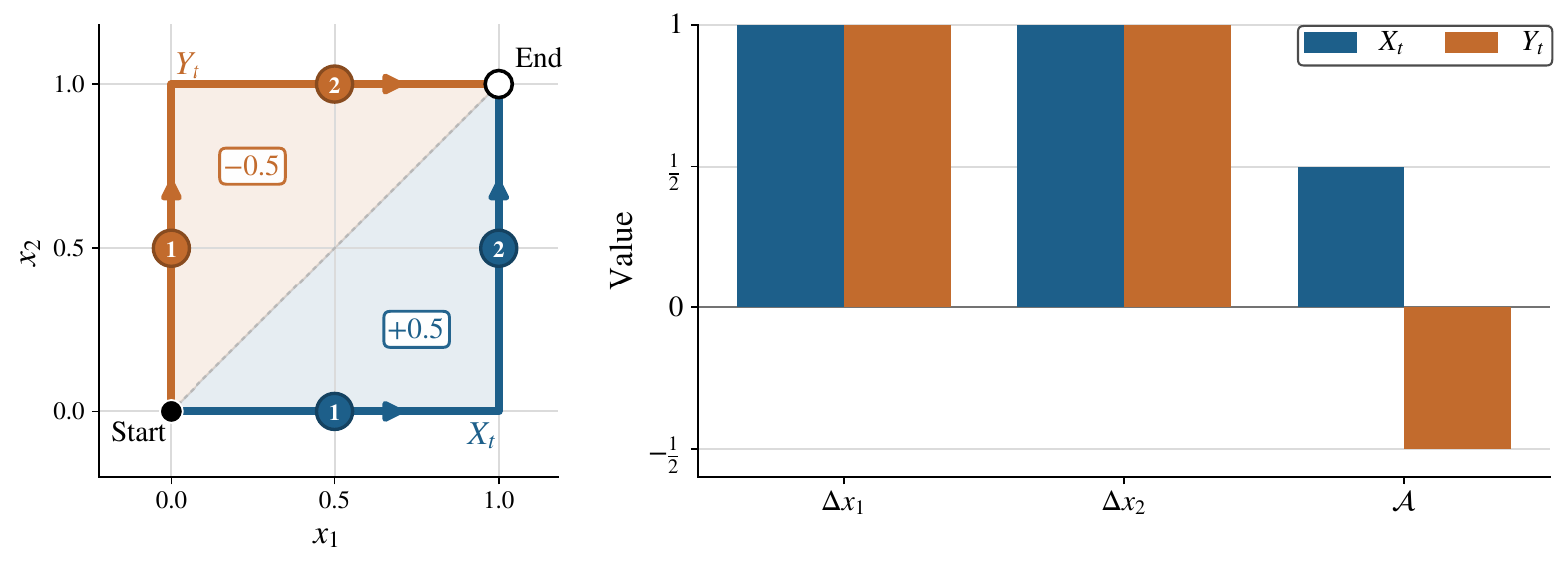}
\caption{Two piecewise linear paths from Start to End with the same net displacement but different ordering: the blue path $X_t$ moves first in the $x_1$ direction and then in $x_2$, while the orange path $Y_t$ moves first in the $x_2$ direction and then in $x_1$, as indicated by the numbered circles. The shaded triangles show the corresponding signed areas and the bar chart compares the common first-level terms $S^{(1)}$ and $S^{(2)}$ with the signed area $\mathcal{A} = \frac{1}{2}\big(S^{(1,2)} - S^{(2,1)}\big)$ for the two paths.}
\label{fig:signature_intuition}
\end{figure}

\subsection{Properties}

\label{sec:sig_prop}

In his seminal work on the signature of finite-dimensional bounded-variation paths, \citet{Chen1954Iterated} introduced a number of important properties, including invariance to reparametrisation and translation. \citet{lyons2007differential} provide an exposition of these properties for paths in $\mathcal{V}^p([a,b],V)$ with $p<2$.

\begin{lemma}[Invariance to Reparametrisation \citep{Chen1954Iterated, lyons2007differential}]
    Let $X\in\mathcal{V}^p([a,b], V)$ with $p<2$ and $\alpha:[a,b]\rightarrow[a,b]$ be a continuous, strictly increasing function satisfying $\alpha(a)=a$ and $\alpha(b)=b$. 
    Then 
    \begin{equation}
        S_{[a,b]}(X) = S_{[a,b]}(X \circ \alpha).
    \end{equation}
\end{lemma}

\begin{lemma}[Invariance to Translation \citep{Chen1954Iterated, lyons2007differential}]
    Let $X \in \mathcal{V}^p([a,b], V)$ with $p < 2$, and $v \in V$.
    Define $Y:[a,b] \to V$ by $Y_t = X_t + v$ for all $t \in [a,b]$.
    Then
    \begin{equation}
        S_{[a,b]}(X) = S_{[a,b]}(Y).
    \end{equation}
\end{lemma}

For paths that do not self-intersect, in the sense that $s \neq t$ implies $X_s \neq X_t$, translation and reparametrisation are the only transformations under which the signature is invariant. 
This makes the signature a natural representation when the object of interest is the shape of a path, since it treats the infinitely many paths related by translation or reparametrisation as equivalent.
However, when a path does self-intersect, further invariances appear. 
A basic example is that the signature is unchanged by a portion of the path that traces out an excursion and then retraces it exactly. 
In the bounded-variation setting, \citet{hambly2010uniqueness} showed that the full class of such invariances is captured by tree-like equivalence. \citet{BOEDIHARDJO2016720} subsequently extended this result to the $p<2$ setting.
For our purposes, it is enough to note that translation invariance is removed by working with paths that share a common initial point, while reparametrisation invariance and tree-like equivalence are removed by augmenting the path with time. 
In particular, the added time channel is strictly increasing, so time-augmented paths cannot self-intersect.

\begin{definition}
    Let $\mathcal{V}^p_0([a,b],V)$ denote the set of paths $X\in\mathcal{V}^p([a,b],V)$ where $X_a=0$.
\end{definition}

\begin{definition}[Time augmentation]
    For $X\in\mathcal{V}^p([a,b],V)$, define the time-augmented path $\hat X:[a,b]\to \mathbb{R}\times V$ by
    \begin{equation}
        \hat X_t = (t,X_t).
    \end{equation}
\end{definition}

Another fundamental property established by \citet{Chen1954Iterated} is that the signature is multiplicative under concatenation of intervals.

\begin{theorem}[Chen's Identity \citep{Chen1954Iterated, lyons2007differential}]
    \label{thm:chen}
    Let $X\in\mathcal{V}^p([a,b], V)$ with $p<2$. 
    Then for all $t\in[a,b]$,
    \begin{equation}
        S_{[a,b]}(X) = S_{[a,t]}(X) \otimes S_{[t,b]}(X).
    \end{equation}
\end{theorem}
Theorem~\ref{thm:chen} is the general analogue of the concatenation property in Lemma~\ref{lem:concatenation} for the signature of the time path. 
It shows that the signature over a long interval may be decomposed into signatures over smaller subintervals and then recombined via the tensor product.
This makes the signature amenable to memory-efficient recurrent computation, since the current signature encodes all necessary information from previous intervals, and to parallel computation via an associative scan \citep{blelloch1993prefix}.
\\ \\
As shown in \eqref{eq:one_d_sig_terms}, the terms of the signature of a one-dimensional path decay factorially. 
The following theorem shows that there is an analogous decay estimate for general paths.

\begin{definition}
    Let $\Gamma$ denote the gamma function, defined for $x>0$ by
\begin{equation}
\Gamma(x)=\int_0^\infty t^{x-1}e^{-t}\,\mathrm{d}t.
\end{equation}
\end{definition}

The gamma function extends the factorial function to the positive real line, in the sense that $\Gamma(1+m)=m!$ for $m\in\mathbb{N}_0$.

\begin{theorem}[Factorial decay {\citep[Theorem~2.2.1]{Lyons1994DIFFERENTIALED}}]
\label{thm:factorial-decay}
Let $X \in \mathcal{V}^p([a,b],V)$ with $p < 2$. Then there exists a finite constant $C_p > 1$, depending only on $p$, such that for every $[s,t]\subset[a,b]$ and every $n \geq 1$,
\begin{equation}
\label{eq:young_factorial}
\big\|X^n_{[s,t]}\big\|
\leq
C_p \frac{\|X\|_{p\text{-var};[s,t]}^n}{\Gamma\!\left(1+\frac{n}{p}\right)}.
\end{equation}
\end{theorem}

Theorem~\ref{thm:factorial-decay} shows that the signature decays rapidly with tensor level, and hence defines an element of the completed tensor algebra. 
Letting $S(\mathcal{V}^p([a,b],V))$ denote the set of signatures of paths in $\mathcal{V}^p([a,b],V)$ and using the identification introduced in Section~\ref{sec:tens_alg_l2}, we have
\begin{equation}
    S(\mathcal{V}^p([a,b],V)) \subset \ell^2(\mathbb{W}_d),
\end{equation}
when $V$ is $d-$dimensional.
Equation \eqref{eq:one_d_sig_terms} also shows that each term of the signature of a one-dimensional path is a monomial in the first term. 
For general paths, an analogous algebraic structure persists in a more subtle form. 
As first shown by \citet{Ree1958LieEA}, products of linear functionals applied to lower-order signature terms can be rewritten as a single linear functional applied to a higher-order term. 
This result is known as the shuffle product identity.

\begin{definition}
    Let $\mathrm{Sh}(\alpha,\beta)$ denote the set of permutations $\sigma \in S_{\alpha+\beta}$ such that $\sigma(1)<\cdots<\sigma(\alpha)$ and $\sigma(\alpha+1)<\cdots<\sigma(\alpha+\beta)$.
\end{definition}

\begin{definition}[Shuffle of functionals]
    Let $V$ be a Banach space and let $\alpha,\beta \in \mathbb{N}_0$. 
    Let $f$ be a bounded linear functional on $V^{\otimes \alpha}$ and $g$ on $V^{\otimes \beta}$. 
    Their shuffle $f \dsqcup g$ is the bounded linear functional on $V^{\otimes (\alpha+\beta)}$ defined by
    \begin{equation}
        (f \dsqcup g)(w)
        = \sum_{\sigma \in \mathrm{Sh}(\alpha,\beta)}
           (f \otimes g)\big(\pi_\sigma w\big),
        \qquad w \in V^{\otimes (\alpha+\beta)}.
    \end{equation}
\end{definition}

\begin{theorem}[Shuffle Product Identity {\citep{Ree1958LieEA, lyons2007differential}}]
\label{th:shuff}
    Let $V$ be a Banach space and $X \in \mathcal{V}^p([a,b],V)$ with $p<2$. 
    Then for all $\alpha,\beta \in \mathbb{N}_0$ and for every pair of bounded linear functionals $f$ on $V^{\otimes \alpha}$ and $g$ on $V^{\otimes \beta}$,
    \begin{equation}
        f\left(X^\alpha_{[a,b]}\right) g\left(X^\beta_{[a,b]}\right)=
        (f \dsqcup g)\left(X^{\alpha+\beta}_{[a,b]}\right).
    \end{equation}
\end{theorem}

The shuffle product identity is the key algebraic result underlying the expressivity of linear functionals of the signature, since it shows that polynomial functions of the signature can be re-expressed as linear functions of the signature. 
This is the direct analogue of linear maps on the signature of the time path generating all polynomials in $t$. 
Moreover, although these functionals are linear on the tensor algebra, when composed with the signature they define highly non-linear functions of the underlying path. 
Furthermore, the shuffle product identity shows that these linear functionals form an algebra, placing them in the setting of the Stone--Weierstrass theorem.

\begin{theorem}[Stone–Weierstrass \citep{StoneWeierstrass}]
\label{th:SW}
    Let $(X,d)$ be a compact metric space and $A$ be a subalgebra of $C(X,\mathbb{R})$ which contains the constant functions. 
    Then $A$ is dense in $C(X,\mathbb{R})$ if and only if it separates the points. 
\end{theorem}

To apply Theorem~\ref{th:SW}, we first introduce the family of linear functionals on the tensor algebra obtained by lifting bounded linear functionals from each tensor level.

\begin{definition}
\label{def:coord-on-signature}
    For $\alpha\in\mathbb{N}_0$ and a bounded linear functional $f$ on $V^{\otimes \alpha}$, define $\ell^{(\alpha,f)}:T((V))\to\mathbb{R}$ by
    \begin{equation}
    \ell^{(\alpha,f)}(\mathbf{x}) = f(x^\alpha),
    \end{equation}
    where $\mathbf{x}=(x^0,x^1,\ldots)\in T((V))$.
\end{definition}

Given an interval $[s,t]\subset[a,b]$ and a path $X\in \mathcal{V}^p([a,b],V)$ with $p<2$, we view the corresponding function on paths as the pullback of an $\ell^{(\alpha,f)}$.

\begin{definition}
\label{def:pullback}
Let $[s,t]\subset[a,b]$, $\alpha\in\mathbb{N}_0$, and $f$ be a bounded linear functional on $V^{\otimes \alpha}$. Define $\phi^{(\alpha,f)}_{[s,t]} = \ell^{(\alpha,f)}\circ S_{[s,t]} \;:\; \mathcal{V}^p([a,b],V)\to\mathbb{R}$ by
\begin{equation}
    \phi^{(\alpha,f)}_{[s,t]}(X)= f\!\big(X^\alpha_{[s,t]}\big).
\end{equation}
\end{definition}

We now verify that these linear functionals satisfy the hypotheses of the Stone--Weierstrass theorem on compact subsets of signature space. 

\begin{lemma}
\label{lem:separate-signatures}
Let $\mathcal{K}_S\subset S(\mathcal{V}^p([a,b],V))$ be compact and $\mathcal{F}_{\alpha}$ be the space of bounded linear functionals on $V^{\otimes \alpha}$. The family
$\{\ell^{(\alpha,f)}|_{\mathcal{K}_S} : \alpha\in\mathbb{N}_0,\ f\in \mathcal{F}_\alpha\}$
separates the points of $\mathcal{K}_S$.
\end{lemma}

\begin{proof}
If $\mathbf{x},\mathbf{y}\in\mathcal{K}_S$ are distinct, then there exists a minimal $n\ge 0$ with $x^n\neq y^n$ in $V^{\otimes n}$. By the Hahn–Banach Theorem, there exists $f\in \mathcal{F}_n$ with $f(x^n)\neq f(y^n)$, so $\ell^{(n,f)}(\mathbf{x})\neq \ell^{(n,f)}(\mathbf{y})$ \citep[Theorem 3.3]{Rudin1991}.
\end{proof}

\begin{lemma}
\label{prop:algebra-on-signature}
Let $\mathcal{K}_S\subset S(\mathcal{V}^p([a,b],V))$ be compact. The set of finite linear combinations of products of the restrictions $\ell^{(\alpha,f)}|_{\mathcal{K}_S}$ is a subalgebra of $C(\mathcal{K}_S,\mathbb{R})$ which contains the constants.
\end{lemma}

\begin{proof}
    Constants are obtained from $\alpha=0$, since $V^{\otimes 0}=\mathbb{R}$ and $X^0_{[a,b]}=1$. Closure under pointwise multiplication on $\mathcal{K}_S$ follows from the shuffle product: if $\mathbf{s}=S_{[a,b]}(X)\in \mathcal{K}_S$, then by Theorem~\ref{th:shuff}
    \begin{equation}
        \ell^{(\alpha,f)}(\mathbf{s})\,\ell^{(\beta,g)}(\mathbf{s})
        = f\!\big(X^\alpha_{[a,b]}\big)\,g\!\big(X^\beta_{[a,b]}\big)
        = (f\dsqcup g)\!\big(X^{\alpha+\beta}_{[a,b]}\big)
        = \ell^{(\alpha+\beta,\,f\dsqcup g)}(\mathbf{s}).
    \end{equation}
\end{proof}

\begin{theorem}[Stone–Weierstrass on signature space \citep{levin2016learningpastpredictingstatistics}]
\label{th:density-on-signature}
    Let $\mathcal{K}_S\subset S(\mathcal{V}^p([a,b],V))$ be compact with $p<2$. Then the algebra generated by $\{\ell^{(\alpha,f)}|_{\mathcal{K}_S}\}$ is dense in $C(\mathcal{K}_S,\mathbb{R})$.
\end{theorem}

\begin{proof}
    By Lemma~\ref{prop:algebra-on-signature}, we have a subalgebra of $C(\mathcal{K}_S,\mathbb{R})$ which contains the constants. By Lemma~\ref{lem:separate-signatures}, it separates points. Therefore, we can apply Stone–Weierstrass, Theorem~\ref{th:SW}.
\end{proof}

Theorem~\ref{th:density-on-signature} shows that, on any compact subset of signature space, functions of the signature can be approximated uniformly by finite linear combinations of the functionals $\ell^{(\alpha,f)}$. In other words, linear functionals of the signature are rich enough to approximate arbitrary continuous functions on compact subsets of signature space. However, our primary interest is in functions defined on paths rather than on their signatures. Therefore, we now transfer this approximation result from signature space to path space by pulling these functionals back along the signature map.
\\ \\
Let $\mathcal{K}\subset\mathcal{V}^p([a,b],V)$ be compact in a standard path topology, such as the $p$-variation topology, and let $\mathcal{K}_S=S_{[a,b]}(\mathcal{K})$. Since $S_{[a,b]}$ is continuous for $p<2$, the set $\mathcal{K}_S$ is compact.
\begin{corollary}[Universality on path space]
    \label{cor:universality-on-paths}
    Let $\mathcal{K}\subset\mathcal{V}^p([a,b],V)$ be compact with $p<2$. Define the equivalence relation
    \begin{equation}
        X \sim Y \iff S_{[a,b]}(X) = S_{[a,b]}(Y).
    \end{equation}
    Then:
    \begin{enumerate}
    \item The algebra generated by $\{\phi^{(\alpha,f)}_{[a,b]}|_{\mathcal{K}}\}$ is dense in the space
    \begin{equation}
        \{F \in C(\mathcal{K}, \mathbb{R}) \;|\; X \sim Y \implies F(X) = F(Y)\}.
    \end{equation}
    \item If $S_{[a,b]}$ is injective on $\mathcal{K}$, then the same algebra is dense in $C(\mathcal{K},\mathbb{R})$.
    \end{enumerate}
\end{corollary}

\begin{proof}
(1) The induced map $\bar S:\mathcal{K}/\!\sim \to \mathcal{K}_S$ is a homeomorphism. By Theorem~\ref{th:density-on-signature}, polynomials in $\{\ell^{(\alpha,f)}\}$ are dense on $\mathcal{K}_S$; pulling back along $\bar S$ yields density of polynomials in $\{\phi^{(\alpha,f)}_{[a,b]}\}$ on $\mathcal{K}/\!\sim$.

(2) If $S_{[a,b]}$ is injective on $\mathcal{K}$, then $\mathcal{K}\cong \mathcal{K}_S$ via $S_{[a,b]}$, and the conclusion follows directly from Theorem~\ref{th:density-on-signature}.
\end{proof}

Corollary~\ref{cor:universality-on-paths} shows that, on a general compact set of paths, linear functionals of the signature approximate exactly those continuous functionals that depend only on the signature. 
To strengthen this to universality for all continuous functionals on path space, it suffices to work in a setting where the signature map is injective. 
By the discussion earlier in this section, this can be achieved by restricting to paths with a common initial point, which removes translation invariance, and by augmenting with time, which removes reparametrisation and tree-like equivalence. 

\begin{lemma}[Injectivity via augmentation]
\label{lem:injectivity-time}
    Let $\mathcal{K}\subset \mathcal{V}^p_0([a,b],V)$ be compact. Then $S_{[a,b]}(\hat X)$ is injective on $\mathcal{K}$, and hence the algebra generated by $\{\phi^{(\alpha,f)}_{[a,b]}(\hat X)\}$ is dense in $C(\mathcal{K},\mathbb{R})$.
\end{lemma}

\begin{proof}
    Since each $X\in\mathcal{K}$ begins at the same point and the time channel of $\hat{X}$ is strictly monotone, no two distinct time-augmented paths can be tree-like equivalent \citep{levin2016learningpastpredictingstatistics}. Therefore $S_{[a,b]}(\hat X)$ determines $X$ uniquely \citep{BOEDIHARDJO2016720,hambly2010uniqueness} and  Corollary~\ref{cor:universality-on-paths} can be applied.
\end{proof}

\begin{remark}
    The definition of a compact set $\mathcal{K}\subset\mathcal{V}^p([a,b], V)$ depends on the chosen path topology \citep{Bonnier2019DeepST}. 
    In most data-driven applications, it suffices to take $\mathcal{K}$ to be a finite, and hence compact, set.
\end{remark}

Lemma~\ref{lem:injectivity-time} is a desirable property when using the signature of a path as a feature set for linear regression \citep{levin2016learningpastpredictingstatistics}. 
However, this linearisation of non-linear functions leads to algebraic redundancy in the signature, as shown in Theorem~\ref{th:shuff}. 
For one-dimensional paths, every higher-order term is a monomial in the first-level term. 
For general paths, not all higher-order terms are determined by lower-order ones, but the shuffle product identity shows that certain linear functionals of the higher-order terms are. 
Thus part of the information appearing at higher tensor levels is merely a linear encoding of polynomial functions of the lower-order terms. 
The transformation which removes these algebraic redundancies is the logarithm.

\subsection{The Log-Signature}
\label{sec:log_sig}

We begin with the formal definitions of the logarithm and exponential on the tensor algebra.

\begin{definition}
    For $\mathbf{x}\in\tilde{T}((V))$, the logarithm is defined by 
    \begin{equation}
        \log(\mathbf{x}) = \log(1+\mathbf{t}) = \sum_{n=1}^\infty \frac{(-1)^{n-1}}{n}\mathbf{t}^{\otimes n},
    \end{equation}
    where $\mathbf{t}=(0,x^1,x^2,\ldots)$.
\end{definition}

Since $\mathbf{x}\in\tilde{T}((V))$, all non-zero elements of $\mathbf{t}^{\otimes n}$ have degree at least $n$. Therefore, the logarithm converges degree-wise.
\begin{definition}
    For $\mathbf{x}\in T((V))$ with $x^0=0$, the exponential is defined by
    \begin{equation}
        \exp(\mathbf{x}) = \sum_{n=0}^\infty \frac{\mathbf{x}^{\otimes n}}{n!}.
    \end{equation}
\end{definition}
Similarly to the logarithm, the exponential converges degree-wise. 
Furthermore, for $\mathbf{x}\in \tilde{T}((V))$,
\begin{equation}
    \exp(\log(\mathbf{x}))=\mathbf{x}
\end{equation}
and for $\mathbf{x}\in T((V))$ with $x^0=0$,
\begin{equation}
    \log(\exp(\mathbf{x}))=\mathbf{x}.
\end{equation}
The effect of the logarithm is already visible in one dimension.
If $V=\mathbb{R}$, then by \eqref{eq:one_d_sig_terms},
\begin{equation}
    S_{[a,b]}(X) = \left(1,X_b-X_a,\frac{(X_b-X_a)^2}{2!},\frac{(X_b-X_a)^3}{3!},\ldots\right) = \exp(X_b-X_a),
\end{equation}
and hence
\begin{equation}
    \log(S_{[a,b]}(X)) = \left(0,X_b-X_a,0,0,\ldots\right).
\end{equation}
In one dimension the signature records all monomials in the increment as separate tensor levels, whereas the log-signature retains only the increment itself. 
This makes the log-signature a more economical representation, but it also removes the linearisation property of the signature, since linear functionals of the signature can recover polynomials in $X_b-X_a$, whereas linear functionals of the log-signature cannot.
\\ \\
For general paths, the same idea first appears at second level. 
Substituting $S_{[a,b]}(X) = (1,X^1_{[a,b]},X^2_{[a,b]},\ldots)$ into the logarithm gives
\begin{equation}
    \log(S_{[a,b]}(X)) = \sum_{n=1}^\infty \frac{(-1)^{n-1}}{n}\left(0,X^1_{[a,b]},X^2_{[a,b]},\ldots\right)^{\otimes n},
\end{equation}
so at the first two levels,
\begin{equation}
    \log(S_{[a,b]}(X)) = \left(0,X^1_{[a,b]},X^2_{[a,b]} - \frac{1}{2}X^1_{[a,b]}\otimes X^1_{[a,b]},\ldots\right).
    \label{eq:logsig_second_level}
\end{equation}
The shuffle product identity shows exactly why this removes redundancy. 
If $f$ and $g$ are bounded linear functionals on $V$, then Theorem~\ref{th:shuff} gives
\begin{equation}
    f(X^1_{[a,b]})g(X^1_{[a,b]}) = (f \dsqcup g)(X^2_{[a,b]}) = (f \otimes g + g \otimes f)(X^2_{[a,b]}).
\end{equation}
On the other hand,
\begin{equation}
    (f \otimes g + g \otimes f)\left(\frac{1}{2}X^1_{[a,b]}\otimes X^1_{[a,b]}\right) = f(X^1_{[a,b]})g(X^1_{[a,b]}).
\end{equation}
Therefore
\begin{equation}
    (f \otimes g + g \otimes f)\left(X^2_{[a,b]} - \frac{1}{2}X^1_{[a,b]}\otimes X^1_{[a,b]}\right) = 0
\end{equation}
for every pair of bounded linear functionals $f$ and $g$ on $V$. 
So every second-level term detected by a shuffle product of first-level functionals vanishes after taking the logarithm. 
In other words, the part of the second level that was already determined by the first level has been removed.
\\ \\
This is particularly transparent when $V=\mathbb{R}^2$. 
Writing
\begin{equation}
    X^2_{[a,b]} = S^{(1,1)}_{[a,b]}e_1\otimes e_1 + S^{(1,2)}_{[a,b]}e_1\otimes e_2 + S^{(2,1)}_{[a,b]}e_2\otimes e_1 + S^{(2,2)}_{[a,b]}e_2\otimes e_2,
\end{equation}
the shuffle product identity gives
\begin{equation}
    S^{(1,1)}_{[a,b]} = \frac{1}{2}\big(S^{(1)}_{[a,b]}\big)^2,
\end{equation}
\begin{equation}
    S^{(2,2)}_{[a,b]} = \frac{1}{2}\big(S^{(2)}_{[a,b]}\big)^2,
\end{equation}
and
\begin{equation}
    S^{(1,2)}_{[a,b]} + S^{(2,1)}_{[a,b]} = S^{(1)}_{[a,b]}S^{(2)}_{[a,b]}.
\end{equation}
Substituting these identities into \eqref{eq:logsig_second_level} shows that the coefficients of $e_1\otimes e_1$ and $e_2\otimes e_2$ become
\begin{equation}
    S^{(1,1)}_{[a,b]} - \frac{1}{2}\big(S^{(1)}_{[a,b]}\big)^2 = 0
\end{equation}
and
\begin{equation}
    S^{(2,2)}_{[a,b]} - \frac{1}{2}\big(S^{(2)}_{[a,b]}\big)^2 = 0.
\end{equation}
Moreover, the sum of the coefficients of $e_1\otimes e_2$ and $e_2\otimes e_1$ becomes
\begin{equation}
    \left(S^{(1,2)}_{[a,b]} - \frac{1}{2}S^{(1)}_{[a,b]}S^{(2)}_{[a,b]}\right) + \left(S^{(2,1)}_{[a,b]} - \frac{1}{2}S^{(2)}_{[a,b]}S^{(1)}_{[a,b]}\right) = 0.
\end{equation}
So all the second-level terms determined by the first level are cancelled by the logarithm. 
The only part that remains is
\begin{equation}
    \frac{1}{2}\left(S^{(1,2)}_{[a,b]} - S^{(2,1)}_{[a,b]}\right)(e_1\otimes e_2 - e_2\otimes e_1),
\end{equation}
which is exactly the signed area term. 
Thus, at second level, the log-signature removes the algebraic redundancy in the signature and retains only the genuinely new information.
\\ \\
\citet{Chen1957IntegrationOP} formalised this for all levels of the signature by showing that signatures are group-like, and hence that their logarithms lie in the free Lie algebra.
As noted in Section \ref{sec:tens_alg}, $\tilde{T}((V))$ is a group. 
We define
\begin{equation}
    \mathcal{G} = \left\{\mathbf{x}\in \tilde{T}((V)) \mid \log(\mathbf{x}) \in \mathfrak{L}((V)) \right\}.
\end{equation}
This is a subgroup of $\tilde{T}((V))$, and its elements are called group-like \citep{reutenauer1993free}. 
Chen's result shows that, for bounded-variation paths,
\begin{equation}
    S_{[a,b]}(X) \in \mathcal{G},
\end{equation}
or equivalently,
\begin{equation}
    \log(S_{[a,b]}(X)) \in \mathfrak{L}((V)).
\end{equation}
Thus the signature may be represented equivalently by its logarithm in $\mathfrak{L}((V))$, a strict subspace of $T((V))$. 
This shows that no information is lost by passing to the log-signature. 
The trade-off is that the linearisation property is no longer present, because the higher-order tensor terms which encode polynomial functions of the lower-order terms have been removed.
\\ \\
If a path is linear on an interval with increment $v\in V$, then its log-signature on that interval is simply $(0,v,0,0,\ldots)$, and so its signature is
\begin{equation}
    \exp(0,v,0,0,\ldots) = \left(1,v,\frac{v^{\otimes 2}}{2!},\frac{v^{\otimes 3}}{3!},\ldots\right).
\end{equation}
Hence, if a path is represented by a piecewise linear interpolation with increments $v_1,\ldots,v_m$, then the signature on each subinterval is obtained from the corresponding exponential, and the full signature is computed as
\begin{equation}
    S_{[a,b]}(X) = \exp(0,v_1,0,0,\ldots)\otimes\cdots\otimes\exp(0,v_m,0,0,\ldots).
\end{equation}
Thus signatures of discrete data can be computed efficiently by working interval by interval and then combining the results with Chen's identity.
\\ \\
As the signature and log-signature are infinite-dimensional, it can be useful to consider the truncated signature
\begin{equation}
S_{[a,b]}^N(X)=(1,X^1_{[a,b]},\ldots,X^N_{[a,b]}) \in \tilde{T}^N(V),
\end{equation}
or truncated log-signature, 
\begin{equation}
\log(S_{[a,b]}^{N}(X)) \in \mathfrak{L}^N(V).
\end{equation}
As discussed in Section \ref{sec:intro_sig}, the truncated signature has natural applications in machine learning, as it is a vector embedding of a multivariate path which captures interactions between the channels of the path.
Reviews of the machine learning applications can be found in \citep{SigPrimer}, \citep{cass2024lecturenotesroughpaths}, and \citep{SigPrimer2}. 
\\ \\
So far, the signature has been treated as a representation of a path built from iterated integrals. 
However, it also admits a dynamical interpretation, since it arises as the solution of a linear CDE driven by the path. 
More generally, CDEs provide a framework for describing how a state evolves in response to a driving signal, with the signature appearing naturally when one studies their flow over an interval. 
The continuous-time machine learning models developed in this thesis adopt this perspective by treating time series data as the driving signal and a CDE's evolving state as the hidden state.

\section{Controlled Differential Equations}

This section introduces CDEs, states the basic well-posedness results, and then specialises to the linear case, where the solution can be written explicitly in terms of the signature. 
It then returns to the log-signature and shows how it leads to the Log-ODE method, an efficient and accurate method for approximating the solution of a CDE.

\subsection{Definition}
\label{sec:cde_def}
Let $V$ and $W$ be Banach spaces, $X: [a,b] \rightarrow V$ and $Y:  [a,b] \rightarrow W$ be continuous paths, and $f:W\rightarrow \mathbf{L}(V, W)$ be a continuous function. 
The vector field $f$ can equivalently be viewed as a linear map from $v\in V$ to a vector field on $W$ denoted $f(\cdot)v$. 
This second formulation will prove useful when defining the Log-ODE method.
Assume that $X$, $Y$, and $f$ are sufficiently regular for the integral 
\begin{equation}
    \int_{a}^tf(Y_s)\mathrm{d}X_s
\end{equation}
to be defined for all $t\in [a,b]$ in the Young sense \citep{Young1936AnIO}. 
The path $Y$ is said to obey a CDE if 
\begin{equation}
\label{eq:cde}
    Y_t = Y_{a} + \int_{a}^tf(Y_s)\mathrm{d}X_s,
\end{equation}
for $t\in [a,b]$, where $Y_{a}\in W$ is the initial condition and $X$ is the control \citep{lyons2007differential}.
\\ \\
Just as the signature of time is the solution to an infinite set of differential equations \eqref{eq:flat_exp_ode}, the signature of a path $X\in\mathcal{V}^p([a,b], V)$ for $p<2$ is the solution to a tensor CDE,
\begin{equation}
    \mathrm{d}S_{[a,s]}(X) = S_{[a,s]}(X) \otimes \mathrm{d}X_s,
\end{equation}
where $S_{[a,s]}(X):[a,b]\rightarrow T((V))$ and $S_{[a,a]}(X)=(1,0,0,\ldots)\in T((V))$ \citep{salvi2021signature}. 
When $V$ is finite-dimensional with dimension $d$, we may identify $T((V))$ with the space of real-valued functions on $\mathbb{W}_d$, as described in Section~\ref{sec:tens_alg_l2}, and hence represent elements of $T((V))$ as vectors in $\mathbb{R}^{\mathbb{W}_d}$. 
In this representation,
\begin{equation}
    \mathrm{d}S_{[a,s]}(X) = AS_{[a,s]}(X)\mathrm{d}X_s,
\end{equation}
where $A\in\mathbf{L}(\mathbb{R}^{\mathbb{W}_d}, \mathbf{L}(\mathbb{R}^{d}, \mathbb{R}^{\mathbb{W}_d}))$. This can be equivalently represented by a set of $d$ matrices $A^i\in\mathbb{R}^{\mathbb{W}_d \times \mathbb{W}_d}$ satisfying
\begin{equation}
\label{eq:sig_cde}
    AS_{[a,s]}(X)\mathrm{d}X_s = \sum_{i=1}^dA^iS_{[a,s]}(X)\mathrm{d}X^i_s.
\end{equation}
Ordering the elements of $\mathbb{W}_d$ first by size and then alphabetically, the $(j,k)^{\text{th}}$ element of $A^i$ satisfies
\begin{equation}
    A^i_{jk} = \begin{cases}
        1, \quad &j=1+i+d(k-1), \\
        0, &\text{otherwise}.
    \end{cases}
\end{equation}
This representation is the linear algebraic analogue of the recursive structure of the iterated integrals defining the signature. 
As seen in \eqref{eq:signature_element}, the differential of the signature term corresponding to the word $(i_1,\ldots,i_n)$ is determined by the lower-order term indexed by $(i_1,\ldots,i_{n-1})$ together with the increment in the final channel. 
The matrices $A^i$ encode exactly this operation on the word basis by appending the letter $i$ to a word, thereby mapping each basis element to the corresponding basis element at the next tensor level.

\subsection{Existence and Uniqueness}

Whether a CDE is well posed depends on the regularity of both the control path $X$ and the vector field $f$. 
Rougher driving signals require stronger regularity of the vector field in order for the dynamics to remain well defined. 
We measure the regularity of a path by the smallest $p\geq 1$ for which its $p$-variation is finite, and the regularity of a vector field by the largest $\gamma>0$ such that it is $\mathrm{Lip}(\gamma)$. 
We defer the formal definition of $\mathrm{Lip}(\gamma)$ functions to Chapter~\ref{chap:lipgamma}, which is devoted to them.

\begin{theorem}[CDE existence  \citep{Lyons1994DIFFERENTIALED, lyons2007differential}]
\label{th:existence}
    Let $1\leq p <2$ and $p-1 < \gamma$.
    If $W$ is finite-dimensional, $X$ has finite $p-$variation, and $f$ is $\text{Lip}(\gamma)$, then (\ref{eq:cde}) admits a solution for every $Y_{a} \in W$.
\end{theorem}
\begin{theorem}[CDE uniqueness  \citep{Lyons1994DIFFERENTIALED, lyons2007differential}]
\label{th:uniqueness}
    Let $1\leq p <2$ and $p<\gamma$.
    If $X$ has finite $p-$variation and $f$ is $\text{Lip}(\gamma)$, then (\ref{eq:cde}) admits a unique solution for every $Y_{a} \in W$.
\end{theorem}

Theorems~\ref{th:existence} and \ref{th:uniqueness} extend the classical existence and uniqueness theory for ordinary differential equations to controls with unbounded variation but finite $p$-variation for $p<2$. 
They show that as the driving path becomes less regular, stronger assumptions on $f$ are needed to ensure that the CDE remains well posed.

\subsection{Solutions of Linear CDEs}
\label{sec:lin_cde}

Linear CDEs play a central role in this thesis for two reasons. 
First, when the driving path is piecewise linear, the flow on each linear segment can be found explicitly and independently. 
This will later underpin the scalable, parallel-in-time computational methods developed for Linear NCDEs. 
Second, their solutions can be expressed in terms of the signature of the driving path, allowing their expressivity to be understood through this explicit solution and the expressivity of the signature.
\\ \\
For a driving path $X:[a,b]\to V$ and output path $Y:[a,b]\to W$, a linear CDE takes the form
\begin{equation}
\label{eq:lin_cde}
    Y_t = Y_{a} + \int_{a}^tAY_s\mathrm{d}X_s,
\end{equation}
where $A\in \mathbf{L}(W, \mathbf{L}(V, W))$. 
Alternatively,
\begin{equation}
\label{eq:lin_cde_alt}
    Y_t = Y_{a} + \int_{a}^tA(\mathrm{d}X_s)Y_s,
\end{equation}
with $A\in \mathbf{L}(V, \mathbf{L}(W, W))$. 
As shown in \eqref{eq:sig_cde}, the signature of a path is an example of a linear CDE. 
However, the existence and uniqueness Theorems \ref{th:existence} and \ref{th:uniqueness} do not apply, as the value of the vector fields is unbounded, and hence they are not $\mathrm{Lip}(\gamma)$. 
Corresponding existence and uniqueness results for linear vector fields are given as Theorems 3.7 and 3.8 in \citep{friz_victoir_2010}. 
Here, we present the explicit form of the unique solution when the norm on each $V^{\otimes n}$ is the projective tensor norm. 

\begin{definition}[Projective tensor norm \citep{BOEDIHARDJO2016720}]
\label{def:proj_tens_norm}
    Let $V$ be a Banach space. The projective tensor norm on $V^{\otimes n}$ is defined by
    \begin{equation}
        \|v\| = \inf\left\{\sum_k \|v^1_k\|\cdots\|v^n_k\| : v=\sum_k v^1_k \otimes \cdots \otimes v^n_k \right\}
    \end{equation}
    for $v\in V^{\otimes n}$ and $v^i_k \in V$.
\end{definition}

Equipping each $V^{\otimes n}$ with the projective tensor norm yields a family of tensor norms satisfying the assumptions outlined in Section \ref{sec:tens_alg}. 

\begin{definition}[Operator tensor powers]
Let $V,W$ be Banach spaces. For $n\ge 1$ and $A\in\mathbf{L}(V,\mathbf{L}(W,W))$, let $A^{\otimes n}\in \mathbf{L}\big(V^{\otimes n},\,\mathbf{L}(W,W)\big)$ be the operator defined on simple tensors by
\begin{equation}
A^{\otimes n}(v^1\otimes\cdots\otimes v^n) = A(v^n)\cdots A(v^1),
\end{equation}
and extended by linearity and continuity to $V^{\otimes n}$. Set $A^{\otimes 0}=I_W$.
\end{definition}

Since $V^{\otimes n}$ is equipped with the projective tensor norm, $\|A^{\otimes n}(v)\|\le \|A\|^{n}\|v\|$ for $v\in V^{\otimes n}$.

\begin{theorem}[Linear CDE solution]
\label{thm:linear_cde_solution}
Let $V,W$ be Banach spaces, $A\in\mathbf{L}(V,\mathbf{L}(W,W))$, $X\in\mathcal{V}^p([a,b],V)$ with $p<2$, and $V^{\otimes n}$ be equipped with the projective tensor norm. Define the evolution operator
\begin{equation}
\label{eq:Phi_series}
    \Phi_{t,s} = \sum_{n=0}^{\infty} A^{\otimes n}\big(X^n_{[s,t]}\big) \in \mathbf{L}(W,W).
\end{equation}
Then the series \eqref{eq:Phi_series} converges absolutely in operator norm and the unique solution to \eqref{eq:lin_cde} is
\begin{equation}
\label{eq:banach_solution}
    Y_t = \Phi_{t,a}\,Y_{a}.
\end{equation}
\end{theorem}

\begin{proof}
This proof has four key steps.
\begin{enumerate}
    \item Apply Picard iteration and show that the iterates correspond to truncated versions of the sum in \eqref{eq:Phi_series}.
    \item Use the factorial decay of the signature to show that the Picard iterates converge.
    \item Upgrade the convergence of the Picard iterates from uniform to $p$-variation and pass the limit through the Young integral to obtain $Y_t = \Phi_{t,a}\,Y_{a}$.
    \item Demonstrate uniqueness by reapplying the Picard iteration with initial datum zero.
\end{enumerate}

\paragraph{1) Picard iteration}

Define $Y^{(0)}_t=Y_a$ and, recursively,
\begin{equation}
    Y^{(m+1)}_t = Y_a + \int_a^t A(\mathrm{d}X_s)Y^{(m)}_s,
\end{equation}
for $m\geq0$.
We claim that for every $m\ge 0$ and all $t\in[a,b]$,
\begin{equation}\label{eq:finite_neumann}
Y^{(m)}_t \;=\; \sum_{n=0}^{m} A^{\otimes n}\!\big(X^n_{[a,t]}\big)\,Y_a.
\end{equation}
This is proved by induction on $m$. For $m=0$, the identity is $Y^{(0)}_t= A^{\otimes 0}(X^0_{[a,t]})Y_a=Y_a$.
Assume \eqref{eq:finite_neumann} holds for some $m\ge 0$. Then,
\begin{equation}
\begin{aligned}
Y^{(m+1)}_t
&= Y_a + \int_a^t A(\mathrm{d}X_s)Y^{(m)}_s, \\
&= Y_a + \sum_{n=0}^{m}\int_a^t A(\mathrm{d}X_s)A^{\otimes n}(X^n_{[a,s]})Y_a,
\end{aligned}
\end{equation}
and,
\begin{equation}
\int_a^t A(\mathrm{d}X_s)A^{\otimes n}\big(X^n_{[a,s]}\big)
= A^{\otimes n+1}\left(\int_a^t X^n_{[a,s]} \otimes \mathrm{d}X_s\right) = A^{\otimes (n+1)}\big(X^{n+1}_{[a,t]}\big).
\end{equation}
Hence,
\begin{equation}
Y^{(m+1)}_t \;=\; Y_a + \sum_{n=0}^{m} A^{\otimes (n+1)}\!\big(X^{n+1}_{[a,t]}\big)Y_a
\;=\; \sum_{n=0}^{m+1} A^{\otimes n}\!\big(X^{n}_{[a,t]}\big)Y_a,
\end{equation}
which completes the induction and proves \eqref{eq:finite_neumann}. 

\paragraph{2) Convergence of the Picard iterates}

By Theorem \ref{thm:factorial-decay}, there exists a finite constant $C_p>1$ depending only on $p$ such that for $n\geq 1$
\begin{equation}\label{eq:young_factorial}
\big\|X^n_{[s,t]}\big\| \le C_p\frac{\|X\|_{p\text{-var};[s,t]}^{n}}{\Gamma(1+\frac{n}{p})}.
\end{equation}
Since $\|A^{\otimes n}(v)\|\le \|A\|^{n}\|v\|$ for $v\in V^{\otimes n}$,
\begin{equation}
\big\|A^{\otimes n}\big(X^n_{[a,t]}\big)\big\|
\le \|A\|^{n}C_p\frac{\|X\|_{p\text{-var};[a,t]}^{\,n}}{\Gamma(1+\frac{n}{p})},
\end{equation}
for $n\geq 1$.
Thus the series
\begin{equation}
\Phi_{t,a}=\sum_{n=0}^{\infty} A^{\otimes n}\!\big(X^n_{[a,t]}\big)\ \in\ \mathbf{L}(W,W)
\end{equation}
converges absolutely and uniformly for $t\in[a,b]$. In particular, $t\mapsto\Phi_{t,a}$ is continuous and
\begin{equation}
Y_t=\Phi_{t,a}Y_a=\sum_{n=0}^{\infty} A^{\otimes n}\!\big(X^n_{[a,t]}\big)Y_a
\end{equation}
is well-defined and continuous. 

\paragraph{3) Convergence in $p$-variation}

By \eqref{eq:finite_neumann}, we have $Y^{(m)}\to Y$ uniformly on $[a,b]$. 
Furthermore, by Chen's identity, for $u<v$,
\begin{equation}
    \begin{aligned}
        (Y-Y^{(m)})_v-(Y-Y^{(m)})_u
        =
        \sum_{\substack{k\geq 0,\ l\geq 1 \\ k+l\geq m+1}}
        A^{\otimes l}\!\big(X^l_{[u,v]}\big)
        A^{\otimes k}\!\big(X^k_{[a,u]}\big)Y_a .
    \end{aligned}
\end{equation}
The factorial decay estimate \eqref{eq:young_factorial} implies that this double series is absolutely convergent and that there exists a sequence $\eta_m\rightarrow 0$ such that
\begin{equation}
    \|(Y-Y^{(m)})_v-(Y-Y^{(m)})_u\|
    \leq \|Y_a\|\eta_m\|X\|_{p\text{-var};[u,v]}.
\end{equation}
Therefore, for any partition $\{t_i\}_{i=0}^N$ of $[s,t]$,
\begin{equation}
    \begin{aligned}
        \sum_{i=0}^{N-1}
        \|(Y-Y^{(m)})_{t_{i+1}}-(Y-Y^{(m)})_{t_i}\|^p
        &\leq
        \|Y_a\|^p\eta_m^p
        \sum_{i=0}^{N-1}
        \|X\|_{p\text{-var};[t_i,t_{i+1}]}^p,
        \\
        &\leq
        \|Y_a\|^p\eta_m^p
        \|X\|_{p\text{-var};[s,t]}^p.
    \end{aligned}
\end{equation}
Taking the supremum over partitions gives
\begin{equation}
    \|Y-Y^{(m)}\|_{p\text{-var};[s,t]}
    \leq
    \|Y_a\|\eta_m\|X\|_{p\text{-var};[s,t]}
    \longrightarrow 0.
\end{equation}
Therefore, $Y^{(m)}$ converges to $Y$ in the $p-$variation sense.
Furthermore, by the bilinear continuity of the Young integral \citep[Theorem 1.16]{lyons2007differential},
\begin{equation}
\int_a^t AY^{(m)}_s\mathrm{d}X_s \longrightarrow \int_a^t AY_s\mathrm{d}X_s
\end{equation}
for all $t\in[a,b]$. 
Passing to the limit in $Y^{(m+1)}_t = Y_a + \int_a^t A\,Y^{(m)}_s\,\mathrm{d}X_s$ yields
\begin{equation}\label{eq:mild_solution}
Y_t \;=\; Y_a \;+\; \int_a^t A\,Y_s\,\mathrm{d}X_s,
\end{equation}
so $Y$ solves \eqref{eq:lin_cde}.

\paragraph{4) Uniqueness}

Suppose $\tilde Y$ is another solution of \eqref{eq:lin_cde} with the same initial condition $Y_a$.
Then $Z=Y-\tilde Y$ satisfies $Z_a=0$ and
\begin{equation}
Z_t = \int_a^t AZ_s\mathrm{d}X_s.
\end{equation}
Applying the Picard expansion of Step~1 with initial datum $0$ gives $Z_t\equiv 0$ on $[a,b]$.
Thus the solution is unique.
\end{proof}

Theorem~\ref{thm:linear_cde_solution} shows that the solution of a linear CDE is determined by the signature of the driving path. 
More generally, the same holds for any CDE satisfying the uniqueness conditions of Theorem~\ref{th:uniqueness}  \citep{hambly2010uniqueness, BOEDIHARDJO2016720}. 
This motivates the construction of numerical methods based on the signature, and the Log-ODE method is one such approach.

\section{The Log-ODE Method}
\label{sec:logode}

The Log-ODE method is an efficient and accurate method for approximating the solution of a CDE \citep{CASTELL199513, Boutaib2013}. On a given interval, it replaces the original CDE by an autonomous ODE whose vector field is constructed from the log-signature of the driving path over that interval together with differential information about the vector field.
This construction is motivated by the fact that the log-signature $\log(S_{[a,b]}(X))$ belongs to the Lie series $\mathfrak{L}((V))$.

\subsection{The Free Lie Algebra}

\begin{definition}[Free Lie Algebra \citep{reutenauer1993free}]
\label{def:freeliealgebra}
    Let $X$ be a non-empty set, $L_0$ and $L$ be Lie algebras, and $\phi:X\rightarrow L_0$ be a map. 
    The Lie algebra $L_0$ is said to be the free Lie algebra generated by $X$ if for all maps $f:X\rightarrow L$, there exists a unique Lie algebra homomorphism $g:L_0\rightarrow L$ such that $g \circ \phi = f$.
\end{definition}
\begin{theorem}
    \label{thm:liealgbera}
    The space of elements of $\mathfrak{L}((V))$ with only finitely many nonzero terms, denoted $\mathfrak{L}(V)$, is the free Lie algebra generated by $V$ \citep{reutenauer1993free}.
\end{theorem}
For each $N\geq 1$, the truncated log-signature $\log(S^N_{[a,b]}(X))\in\mathfrak{L}^N(V)$ may be viewed as an element of $\mathfrak{L}(V)$ by setting all terms of degree greater than $N$ equal to $0$. 
Furthermore, as will be discussed in detail in Section \ref{sec:smooth_vec_fields}, the smooth vector fields on $W$ form a Lie algebra with Lie bracket defined pointwise by
\begin{equation}
    \label{eq:smooth_lie_bracket}
    [f,g](p)=Dg(p)[f(p)]-Df(p)[g(p)]
\end{equation}
for smooth $f,g:W\rightarrow W$ and all $p\in W$, where $Df$ is the Fr\'echet derivative of $f$.
For example, if $W=\mathbb{R}^d$, $f(p)=Bp$, and $g(p)=Cp$ with $B,C\in\mathbb{R}^{d\times d}$, then
\begin{equation}
    [f,g](p) = (CB-BC)p = -[B,C]\,p,
\end{equation}
where the left hand side Lie bracket is for vector fields and the right hand side Lie bracket is for matrices.
Therefore, assuming that $f(\cdot)v$ is a smooth vector field on $W$ for all $v\in V$, Theorem~\ref{thm:liealgbera} implies there exists a Lie algebra homomorphism $\bar{f}$ from $\mathfrak{L}(V)$ to the smooth vector fields on $W$.
Since the map $\bar{f}$ is a Lie algebra homomorphism, it is determined recursively by  
\begin{equation}
\label{eq:logoderecurs1}
\bar{f}(\cdot)v = f(\cdot)v, \;\; v\in V
\end{equation}
and 
\begin{equation}
\label{eq:logoderecurs2}
\bar{f}(\cdot)[v_1,v_2] = [\bar{f}(\cdot)v_1,\bar{f}(\cdot)v_2]
\end{equation}
for $v_1,v_2\in\mathfrak{L}(V)$ \citep{Lyons2014}. 
For each $N\geq 1$,
\begin{equation}
    \label{eq:log_ode_vf}
    F^N(\cdot) = \bar{f}\left(\cdot\right)\log(S^N(X)_{[a,b]})
\end{equation}
is a well-defined vector field on $W$. 
To illustrate the structure, consider a finite-dimensional driving path $X_t=(X^1_t,\ldots,X^d_t)$.
Then $F^2(\cdot)$ can be written explicitly as
\begin{equation}
\label{eq:F2_logsig}
\begin{aligned}
    F^2(\cdot)
    &= \bar{f}(\cdot)\log(S^2(X)_{[a,b]}), \\
    &= \bar{f}(\cdot)\left(X^1_{[a,b]} + X^2_{[a,b]} - \frac{1}{2}X^1_{[a,b]}\otimes X^1_{[a,b]}\right), \\
    &= \bar{f}(\cdot)\left(\sum_{i=1}^d S^{(i)}_{[a,b]}e_i
    + \frac{1}{2}\sum_{1\leq i<j\leq d}
    \left(S^{(i,j)}_{[a,b]}-S^{(j,i)}_{[a,b]}\right)[e_i,e_j]\right), \\
    &= \sum_{i=1}^d S^{(i)}_{[a,b]}f(\cdot)e_i
    + \frac{1}{2}\sum_{1\leq i<j\leq d}
    \left(S^{(i,j)}_{[a,b]}-S^{(j,i)}_{[a,b]}\right)
    [f(\cdot)e_i,f(\cdot)e_j],
\end{aligned}
\end{equation}
where the second line uses the explicit form of the log-signature up to level $2$ derived in \eqref{eq:logsig_second_level}, the third line rewrites the first and second levels in the basis $(e_1,\ldots,e_d)$ of $V$, and the final line applies the Lie algebra homomorphism $\bar f$.

\subsection{The Log-ODE Approximation}

Given appropriate conditions on the convergence of $F^N$ as $N\to\infty$, the Log-ODE method recovers the solution of the CDE
\begin{equation}
    Y_b = Y_{a} + \int_{a}^bf(Y_s)\mathrm{d}X_s
\end{equation}
by solving the autonomous ODE
\begin{equation}
    \bar{Y}_1 = \bar{Y}_0 + \int_0^1 \bar{f}\left(\bar{Y}_s\right)\log(S(X)_{[a,b]})\mathrm{d}s
\end{equation}
with initial condition $\bar Y_0 = Y_a$, and then setting $Y_b = \bar Y_1$.
This is represented schematically in Figure \ref{fig:logsig}.
\begin{figure}
\centering
\includegraphics[width=0.8\textwidth]{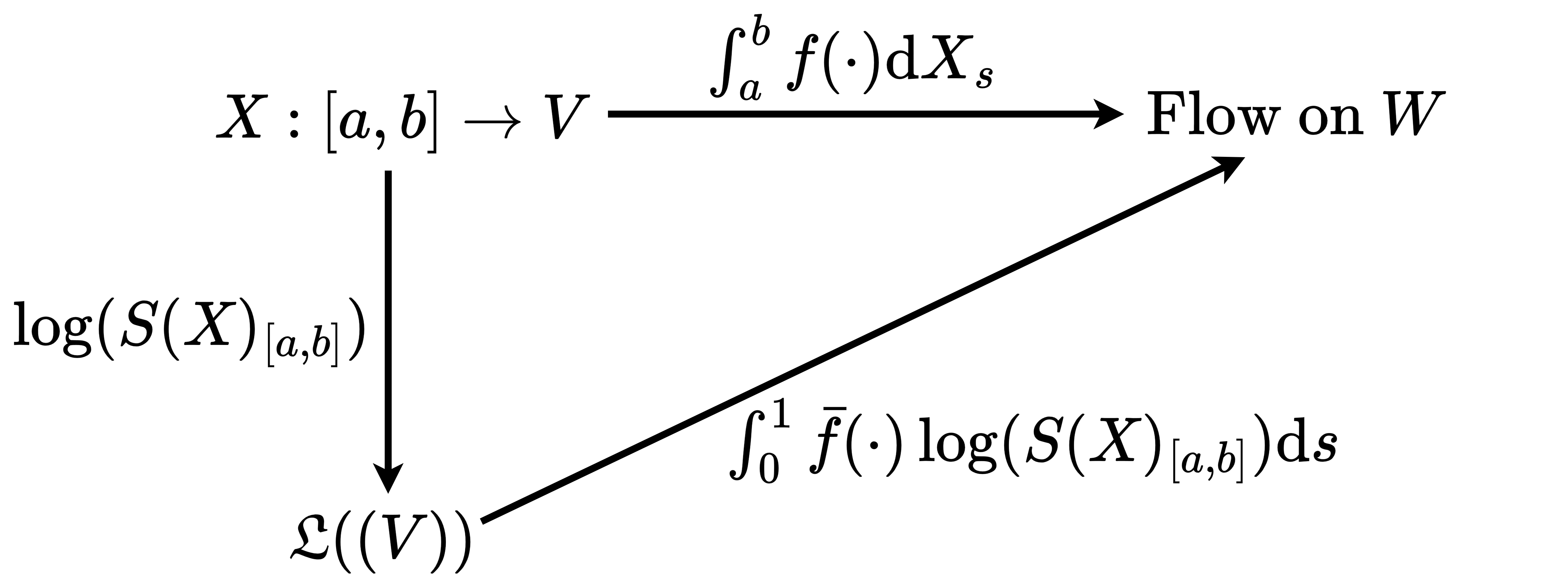}
\caption{A schematic diagram of the Log-ODE method, where $\int_a^b f(\cdot)\mathrm{d}X$ is the CDE to be solved, the log-signature of $X$ lives in $\mathfrak{L}((V))$, and the Log-ODE method constructs an approximating autonomous ODE using the unique Lie algebra homomorphism extending $f$ from $V$ to the free Lie algebra generated by $V$, denoted $\bar f$.}
\label{fig:logsig}
\end{figure}
\\ \\
When $X$ is continuously differentiable and $f(\cdot)v$ is a linear vector field for all $v\in V$, this construction coincides with the classical Magnus expansion \citep{Magnus1954}. 
To see this, let $\dot X_t$ denote the time derivative of $X_t$ and let $A\in\mathbf{L}(V,\mathbf{L}(W,W))$ be such that
\begin{equation}
    f(y)v = A(v)y
\end{equation}
for all $y\in W$ and $v\in V$. 
Then the CDE may be written as
\begin{equation}
    \frac{\mathrm{d}Y_t}{\mathrm{d}t} = A(\dot X_t)Y_t,
\end{equation}
and the classical Magnus expansion expresses the solution in the form
\begin{equation}
    \label{eq:magnus_expansion}
    Y_b = \exp\!\left(\Omega^{(1)}_{b,a} + \Omega^{(2)}_{b,a} + \cdots \right)Y_a,
\end{equation}
Equivalently, if $\bar Y:[0,1]\to W$ solves
\begin{equation}
    \bar Y_1 = \bar Y_0 + \int_0^1 \left(\Omega^{(1)}_{b,a} + \Omega^{(2)}_{b,a} + \cdots \right)\bar Y_s\,\mathrm{d}s,
\end{equation}
with $\bar Y_0 = Y_a$, then $\bar Y_1 = Y_b$.
Since $A$ is linear and $X^1_{[a,b]} = \int_a^b \dot X_{u_1}\,\mathrm{d}u_1$, the first term in the infinite series of \eqref{eq:magnus_expansion} is simply
\begin{equation}
    \Omega^{(1)}_{b,a}
    = \int_a^b A(\dot X_{u_1})\,\mathrm{d}u_1
    = A(X^1_{[a,b]}) 
    = \sum_{i=1}^d S^{(i)}_{[a,b]}A(e_i),
\end{equation}
aligning with the first term of \eqref{eq:F2_logsig}.
The second term is
\begin{equation}
\label{eq:omega2_bracket}
\begin{aligned}
    \Omega^{(2)}_{b,a}
    &= \frac{1}{2}\int_a^b\int_a^{u_1}
    [A(\dot X_{u_1}),A(\dot X_{u_2})]\,\mathrm{d}u_2\,\mathrm{d}u_1 \\
    &= \frac{1}{2}\int_a^b\int_a^{u_1}
    \left(A(\dot X_{u_1})A(\dot X_{u_2}) - A(\dot X_{u_2})A(\dot X_{u_1})\right)\,\mathrm{d}u_2\,\mathrm{d}u_1.
\end{aligned}
\end{equation}
Expanding $A(\dot X_{u_1})$ and $A(\dot X_{u_2})$ in the basis $(e_1,\ldots,e_d)$ and using bilinearity of the Lie bracket gives
\begin{equation}
\begin{aligned}
    \Omega^{(2)}_{b,a}
    &= \frac{1}{2}\int_a^b\int_a^{u_1}
    \left(A(\dot X_{u_1})A(\dot X_{u_2}) - A(\dot X_{u_2})A(\dot X_{u_1})\right)\,\mathrm{d}u_2\,\mathrm{d}u_1 \\
    &= \frac{1}{2}\sum_{i,j=1}^d \int_a^b\int_a^{u_1}
    \dot X^i_{u_1}\dot X^j_{u_2}\,[A(e_i),A(e_j)]\,\mathrm{d}u_2\,\mathrm{d}u_1 \\
    &= \frac{1}{2}\sum_{i,j=1}^d S^{(j,i)}_{[a,b]}\,[A(e_i),A(e_j)].
\end{aligned}
\end{equation}
We now group together the terms indexed by $(i,j)$ and $(j,i)$. Since $[A(e_i),A(e_i)]=0$ and $[A(e_j),A(e_i)]=-[A(e_i),A(e_j)]$, this becomes
\begin{equation}
\begin{aligned}
    \Omega^{(2)}_{b,a}
    &= \frac{1}{2}\sum_{1\leq i<j\leq d}
    \left(S^{(i,j)}_{[a,b]} - S^{(j,i)}_{[a,b]}\right)[A(e_j),A(e_i)],
\end{aligned}
\end{equation}
which is exactly the second sum in \eqref{eq:F2_logsig}, since for the linear vector fields $f(\cdot)e_i(y)=A(e_i)y$,
\begin{equation}
    [f(\cdot)e_i,f(\cdot)e_j](y) = [A(e_j),A(e_i)]y,
\end{equation}
where the first Lie bracket is for vector fields and the second Lie bracket is for matrices. 
Hence, at depth $2$, the Log-ODE vector field reproduces exactly the first two terms of the Magnus expansion.
In fact, this agreement persists term by term at every order, with each term built from iterated integrals of the driving path and iterated Lie brackets of the matrices $A(v)$.
\\ \\
For finite-dimensional $V$ and $W$, a sufficient condition ensuring local convergence of the Magnus expansion is \citep{MoanNiesen2008},
\begin{equation}
    \int_a^b \|A(\dot X_s)\|_2 \mathrm{d}s < \pi.
\end{equation}
When $X$ is continuously differentiable and $f(\cdot)v$ is non-linear, the same construction yields the Chen-Strichartz series \citep{Strichartz1987}. 
For finite-dimensional $V$ and $W$, local convergence of Chen-Strichartz holds under an analytic growth condition on the non-linear vector fields \citep{Strichartz1987}. 
The approach was first extended to the rough-path setting by \citet{CASTELL199513}, where it is known as the Log-ODE method.
\\ \\
Assuming it converges, exponentiating the infinite series in \eqref{eq:magnus_expansion} recovers the exact flow, which agrees with the direct solution formula of Theorem~\ref{thm:linear_cde_solution},
\begin{equation}
    \Phi_{t,a} = \sum_{n=0}^{\infty} A^{\otimes n}\big(X^n_{[a,t]}\big).
\end{equation}
Truncating the signature expansion of the flow
\begin{equation}
    \label{eq:truncated_flow}
    \Phi^N_{t,a} = \sum_{n=0}^{N} A^{\otimes n}\big(X^n_{[a,t]}\big),
\end{equation}
gives an alternative approach to approximating the solution of a CDE.
This approach truncates an expansion of the flow, whereas the Log-ODE method truncates the flow in its logarithmic, or generator, representation.
Equivalently, the Log-ODE method may be interpreted as replacing the original driving path on each interval by a rough path whose truncated log-signature agrees with that of the original path up to degree $N$, and whose higher-order terms are zero. 
The approximation is then obtained by solving the same class of controlled system against this approximating path, rather than by truncating the flow map after it has been computed. 
This is often desirable, since it keeps the approximation within the class of solutions of the same CDE, better preserving any structural or physical properties encoded by the model.
\\ \\ 
Given a set of intervals $a=r_0<\cdots<r_m=b$, a depth$-N$ Log-ODE method approximates the solution of a CDE via
\begin{equation}
    \tilde{Y}_{r_{i+1}} = \tilde{Y}_{r_i} + \int_{r_i}^{r_{i+1}} \bar{f}\left(\tilde{Y}_s\right)\frac{\log(S^N(X)_{[r_i,r_{i+1}]})}{r_{i+1}-r_i}\mathrm{d}s,
\end{equation}
where $\tilde{Y}_a=Y_a$ and the integral's time has been rescaled to match that of the original CDE.
Since this approach truncates to the depth$-N$ log-signature, the vector field no longer needs to be smooth, but only $\text{Lip}(\gamma)$ for $\gamma > N-1$. 
Conversely, the smoothness of the vector field determines the highest truncation depth usable in the Log-ODE method. 
For a given set of intervals, this approximation is not guaranteed to converge as $N\rightarrow \infty$.
However, the error $\|Y_b-\tilde{Y}_b\|$ can be quantified, even for rough driving paths and infinite dimensional Banach spaces \citep{Boutaib2013}. 
A recent development has been the introduction of an algorithm which adaptively updates $N$ and $\{r_i\}_{i=0}^m$ \citep{bayer2023adaptive}. 

\section{Conclusion}

This chapter developed the mathematical framework used throughout this thesis for modelling continuously evolving data. 
It introduced CDEs as a way to describe dynamics driven by a continuous path, showed how the signature provides a graded representation of that path with strong algebraic and approximation properties, and explained how the log-signature removes the algebraic redundancy of the signature while retaining the genuinely new geometric information. 
In the linear setting, the solution of a CDE was shown to be determined explicitly by the signature of the driving path.
The chapter then showed how the log-signature leads naturally to the Log-ODE method, which replaces the original CDE on each interval by an autonomous ODE built from the truncated log-signature of the driving path and iterated Lie brackets of the vector field. 
Together, these ideas establish the main mathematical objects and constructions used in the remainder of the thesis: paths are the fundamental data type, signatures and log-signatures summarise them over intervals, and CDEs describe how states evolve in response to them. 
\\ \\
The existence and uniqueness theorems, Theorems~\ref{th:existence} and~\ref{th:uniqueness}, together with the Log-ODE method, show that $\mathrm{Lip}(\gamma)$ regularity plays an important role in the theory of CDEs. 
The next chapter studies this notion of regularity in detail.
Furthermore, it formalises the Lie bracket of two $\mathrm{Lip}(\gamma)$ functions, a key ingredient in the Log-ODE method.
\chapter{$\mathrm{Lip}(\gamma)$ Functions}

\begin{quoting}
    I am convinced that it is impossible to know the parts without knowing the whole, any more than we can know the whole without a detailed knowledge of the parts.
\end{quoting}
\noindent\large---Blaise Pascal, \emph{Pense\'es} (1670), translation by Martin Turnell (1962)
\normalsize

\label{chap:lipgamma}

\section{Introduction}

The analysis of CDEs, particularly when employing numerical techniques such as the Log-ODE method, depends on the regularity of the driving vector field. 
As established by \citet{Lyons1994DIFFERENTIALED}, $\mathrm{Lip}(\gamma)$ is the correct notion of regularity for guaranteeing the existence and uniqueness of solutions to CDEs.
Furthermore, the degree of regularity $\gamma$ directly determines the maximum truncation depth $N$ that can be used by the Log-ODE method. 
This chapter relates $\mathrm{Lip}(\gamma)$ to more classical notions of differentiability, develops the Lie bracket of $\mathrm{Lip}(\gamma)$ vector fields, and proves Lemma~\ref{lem:normcomplip2}, which gives a new explicit bound for the $\mathrm{Lip}(\gamma)$ norm of the composition of two $\mathrm{Lip}(\gamma)$ functions in the case $1<\gamma\leq 2$.
\\ \\
Unlike Chapter~\ref{chap:cde}, which introduced the main mathematical objects used throughout the thesis, this chapter is concerned primarily with the regularity theory needed to apply those constructions when the vector field of a CDE is parametrised by a neural network.
In particular, Chapter~\ref{chap:ncde} uses Lemma~\ref{lem:normcomplip2} to control the $\mathrm{Lip}(\gamma)$ norm of the neural network vector fields, thereby justifying the use of the Lie brackets of those vector fields when applying the Log-ODE method to NCDEs.
Although this chapter contains material of independent mathematical interest, it is not essential for readers primarily interested in the continuous-time machine learning models developed in Chapters~\ref{chap:ncde} and~\ref{chap:lin_ncde}.
\\ \\
The notion of $\mathrm{Lip}(\gamma)$ functions originates in the work of \citet{Whitney1934analytic}, who studied collections of derivative data on closed subsets of $\mathbb{R}^n$. 
The modern formulation was later given by \citet{stein1970singular}.
A core result is the Stein-Whitney extension theorem, given in Section~\ref{sec:extension} as Theorem~\ref{thm:stein-whitney}, which states that any $\mathrm{Lip}(\gamma)$ function defined on a closed subset $E$ of a finite-dimensional Banach space can be extended to the whole space, with a bound on the norm of the extension that is independent of $E$. 
\\ \\
The definition of $\mathrm{Lip}(\gamma)$ regularity given by \citet{stein1970singular} applies to functions defined on arbitrary subsets of a Banach space. 
Lemma~\ref{lem:CKalpha_Lipgamma} shows that on open convex subsets, the space of $\mathrm{Lip}(\gamma)$ functions coincides with the space of $k$ times Fr\'echet differentiable functions with bounded value and derivatives, whose $k^{\text{th}}$ derivative satisfies an $\alpha$-H\"older bound, denoted $C^{k,\alpha}_b$.
However, the natural inclusion of a $C^{k+1,\alpha}_b$ function in $C^{k,\alpha}_b$ does not hold globally, whereas \citet{boutaib2016lipschitz} proved that $\mathrm{Lip}(\gamma)$ spaces are globally nested with respect to $\gamma$.
Similarly, the Lie bracket of two $C^{k,\alpha}_b$ vector fields need not be globally $C^{k-1,\alpha}_b$.
\\ \\
To the best of the author’s knowledge, Section~\ref{sec:lipgamma} gives the first formal treatment of the Lie bracket for $\mathrm{Lip}(\gamma)$ vector fields on arbitrary subsets of potentially infinite-dimensional Banach spaces.
We prove that $[f,g]\in\mathrm{Lip}(\gamma-1)$ for $f,g\in\mathrm{Lip}(\gamma)$ with $\gamma>1$, and that the Jacobi identity holds in $\mathrm{Lip}(\gamma-2)$ for $\gamma>2$. 
This work builds on \citep[Chapter 3]{boutaib2016lipschitz}, which proved a number of fundamental results for $\mathrm{Lip}(\gamma)$ functions, including that the composition of two $\mathrm{Lip}(\gamma)$ functions is $\mathrm{Lip}(\gamma)$ and that these spaces satisfy the nesting property with respect to $\gamma$.
\\ \\
The final section of this chapter proves Lemma~\ref{lem:normcomplip2}, an explicit bound on the norm of the composition of two $\mathrm{Lip}(\gamma)$ functions for $1<\gamma\leq 2$. 
This builds on the work of \citet{cass2012new} and \citet{boutaib2016lipschitz}, who proved using different techniques that the composition of two $\mathrm{Lip}(\gamma)$ functions is $\mathrm{Lip}(\gamma)$, with a norm bounded up to a finite unknown constant $C_{\gamma}$.
Our result allows us to explicitly bound the norm of a certain class of neural networks in Section~\ref{sec:lipgammaNN}, laying the theoretical groundwork for applying the Log-ODE method to NCDEs.
\\ \\
For a Banach space $V$, this chapter assumes that the $\{V^{\otimes n}\}_{n=1}^\infty$ are equipped with a family of reasonable tensor algebra norms, as outlined in Section~\ref{sec:tens_alg}. 
As previously discussed, these conditions ensure that for all
$v\in V^{\otimes n}$ and $w\in V^{\otimes m}$,
\begin{equation}
    \| v \otimes w \|_{V^{\otimes (n+m)}} = \|v\|_{V^{\otimes n}}\,\|w\|_{V^{\otimes m}}.
\end{equation}
Therefore, the equivalence property of \citet{boutaib2016lipschitz},
\begin{equation}
    \|v\|_{V^{\otimes n}}\|w\|_{V^{\otimes m}} \leq c\|v\otimes w\|_{V^{\otimes (n+m)}},
\end{equation}
holds with equality and $c=1$, allowing us to directly use their results.

\section{Differentiable Functions}

This section reviews three standard notions of function regularity: smooth, $C^{k}$, and $C^{k,\alpha}$, together with the Lie bracket of two functions. 
We develop their basic properties and provide proofs in this familiar setting, both to fix notation and to introduce the ideas and techniques that will later be applied when establishing results for the Lie bracket of $\mathrm{Lip}(\gamma)$ functions.

\subsection{Smooth Vector Fields}

\label{sec:smooth_vec_fields}

Just as the ordered interactions captured by the signature give rise to multilinear objects, the iterated derivatives of a function naturally take values in spaces of multilinear maps.
The first derivative records the best linear approximation to a function, while the second and higher derivatives describe how this approximation changes under simultaneous perturbations in multiple directions. 
For sufficiently regular functions, these higher derivatives are symmetric, since the order in which the input directions are differentiated does not matter. 
\\ \\
Let $U$ and $V$ be Banach spaces. 
A symmetric $k$-linear map is a map
\begin{equation}
A:U\times\cdots\times U\to V
\end{equation}
which is linear in each argument and unchanged by permuting its inputs.
By the universal property of the tensor product, bounded symmetric $k$-linear maps may be identified with bounded linear maps $A:U^{\otimes k}\to V$ such that
\begin{equation}
A(u_{\sigma(1)}\otimes\cdots\otimes u_{\sigma(k)})
=
A(u_1\otimes\cdots\otimes u_k)
\end{equation}
for all $u_1,\ldots,u_k\in U$ and every permutation $\sigma$ of $\{1,\ldots,k\}$.
We write $\mathbf{L}_{\mathrm{sym}}(U^{\otimes k},V)$ for this space, and for $k=0$ set
\begin{equation}
\mathbf{L}_{\mathrm{sym}}(U^{\otimes 0},V)\cong V.
\end{equation}
This is the space in which the successive derivatives of a sufficiently regular map between Banach spaces take values.

\begin{definition}[Fréchet derivative \citep{Luenberger1969}]
  Let $U$ and $V$ be Banach spaces and $E\subset U$ open.
  A map $f:E\to V$ is Fréchet differentiable at a point $p\in E$ if there exists a bounded linear operator $Df(p)\in \mathbf{L}(U,V)$ such that
  \begin{equation}
    \lim_{\substack{h\to 0\\ p+h\in E}} \frac{\|f(p+h)-f(p)-Df(p)[h]\|}{\|h\|}=0.
  \end{equation}
  The map $Df(p)$ is called the Fréchet derivative of $f$ at $p$.
\end{definition}

\begin{definition}[Smooth function]
  Let $U$ and $V$ be Banach spaces and $E\subset U$ be an open set.
  A map $f:E\to V$ is called smooth if the iterated Fréchet derivatives
  \begin{equation}
    D^j f:E\to\mathbf{L}_{\mathrm{sym}}(U^{\otimes j},V)
  \end{equation}
  exist and are continuous for all $j\in\mathbb{N}_0$, where $D^0f=f$.
\end{definition}

Let $C^\infty(E,V)$ denote the space of all smooth functions $f:E\to V$.

\begin{definition}[Lie bracket on a Banach space]
\label{def:smooth_liebracket}
  Let $U$ be a Banach space and $E\subset U$ an open set.
  For $f,g\in C^\infty(E,U)$, their Lie bracket
  $[f,g]\in C^\infty(E,U)$ is defined pointwise by
  \begin{equation}
    [f,g](p)=Dg(p)[f(p)]-Df(p)[g(p)],
  \end{equation}
  for $p\in E$.
\end{definition}

\begin{remark}
  The chosen sign convention
  \begin{equation}
    [f,g](p) = Dg(p)[f(p)] - Df(p)[g(p)]
  \end{equation}
  agrees with the usual definition of the Lie bracket of vector fields in differential geometry, ensuring that
  \begin{equation}
    [f,g] = f\circ g - g\circ f
  \end{equation}
  when $f$ and $g$ are viewed as derivations acting on smooth functions \citep{Lee2013}.
\end{remark}

The space $C^\infty(E,U)$, endowed with the bracket operation defined above, is a Lie algebra \citep[Chapter 5]{Lang1999}:
\begin{enumerate}
\item The bracket is bilinear:
\begin{equation}
    [af + bg, h] = a[f,h] + b[g,h],
\end{equation}
and 
\begin{equation}
    [h, af + bg] = a[h,f] + b[h,g],
\end{equation}
hold for all $a,b\in\mathbb{R}$ and $f,g,h\in C^\infty(E, U)$.
\item The bracket is antisymmetric:
\begin{equation}
    [f,g] = -[g,f],
\end{equation}
holds for all $f,g\in C^\infty(E, U)$.
\item The bracket satisfies the Jacobi identity:
\begin{equation}
    [f,[g,h]] + [g,[h,f]] + [h,[f,g]] = 0,
\end{equation}
holds for all $f,g,h\in C^\infty(E, U)$.
\end{enumerate}

\subsection{$C^k$ Vector Fields}

We now move from smooth vector fields to the weaker setting of finite differentiability. 

\begin{definition}[$C^{k}$ function]
    Let $U$ and $V$ be Banach spaces, $E\subset U$ be open, and $k\in \mathbb{N}_0$. 
    A map $f:E\to V$ is called \emph{$C^{k}$} if the iterated Fréchet derivatives
    \begin{equation}
        D^{j}f:E\to\mathbf{L}_{\mathrm{sym}}\bigl(U^{\otimes j},V\bigr),\qquad j=0,\ldots,k,
    \end{equation}
    exist and are continuous.
\end{definition}

Let $C^k(E,V)$ denote the set of all $C^k$ functions from $E$ to $V$. 
A first basic observation is that a $C^k$ function is $C^l$ for all $l\in\mathbb{N}_0$ satisfying $l<k$. 
A second is that differentiating a $C^k$ function produces a $C^{k-1}$ function.
To compare the derivatives of $Df$ with the higher derivatives of $f$, we use the canonical currying identification
\begin{equation}
    \mathbf{L}_{\mathrm{sym}}\bigl(U^{\otimes j},\mathbf{L}(U,V)\bigr)
    \cong
    \mathbf{L}\bigl(U^{\otimes (j+1)},V\bigr),
\end{equation}
given by 
\begin{equation} 
    A(u_1,\ldots,u_j)(u_{j+1}) \longleftrightarrow \widetilde A(u_1,\ldots,u_{j+1}). 
\end{equation}

\begin{lemma}
    \label{lem:Ck_derivative}
    Let $U$ and $V$ be Banach spaces, $E\subset U$ open, and $k\in \mathbb{N}$.
    If $f\in C^{k}(E,V)$, then its Fréchet derivative
    \begin{equation}
        Df:E\longrightarrow\mathbf L(U,V)
    \end{equation}
    belongs to $C^{k-1}\!\bigl(E,\mathbf L(U,V)\bigr)$.
\end{lemma}

\begin{proof}
    By definition of $C^{k}$, the iterated derivatives $D^{j}f:E\to\mathbf L_{\mathrm{sym}}\bigl(U^{\otimes j},V\bigr)$ exist and are continuous for all $0\le j\le k$. 
    For every $0\le j\le k-1$,
    \begin{equation}
        D^{j}\bigl(Df\bigr)=D^{j+1}f,
    \end{equation}
    under the currying convention $\mathbf{L}_{\mathrm{sym}}\bigl(U^{\otimes j},\mathbf{L}(U,V)\bigr) \cong \mathbf{L}\bigl(U^{\otimes (j+1)},V\bigr)$.
    The right-hand side exists and is continuous by assumption, therefore $Df$ possesses continuous iterated
    derivatives up to order $k-1$, i.e.\ $Df\in C^{k-1}\bigl(E,\mathbf L(U,V)\bigr)$.
\end{proof}

To define the bounded $C^k$ spaces, we use a uniform norm on each derivative level. 
For $j\in\mathbb{N}_0$, let $\|\cdot\|_j$ denote the norm on $\mathbf{L}_{\mathrm{sym}}(U^{\otimes j},V)$ given by
\begin{equation}
    \|\cdot\|_j=
    \begin{cases}
        \|\cdot\|_V, & j=0,\\
        \|\cdot\|_{\operatorname{op}}, & j\in\mathbb{N}.
    \end{cases}
\end{equation}
Under the identification $\mathbf{L}_{\mathrm{sym}}(U^{\otimes 0},V)\cong V$, this treats the case $j=0$ uniformly with the higher-order cases.

\begin{definition}[$C^{k}_{b}$ space and norm]
    Let $U$ and $V$ be Banach spaces, let $E\subset U$ be open, and fix $k\in\mathbb{N}_0$.  
    For $f\in C^{k}(E,V)$ define
    \begin{equation}
        \label{eq:Ck_b_norm}
        \|f\|_{C^{k}}=\max_{0\le j\le k}\sup_{p\in E}\|D^{j}f(p)\|_{j}.
    \end{equation}
    The bounded $C^{k}$ space is
    \begin{equation}
        \label{eq:Ck_b_space}
        C^{k}_{b}(E,V)=\{f\in C^{k}(E,V):\|f\|_{C^{k}}<\infty\}.
    \end{equation}  
\end{definition}

\begin{definition}[$C^{k}$ Lie bracket]
    Let $f\in C^{k_f}(E,U)$ and $g\in C^{k_g}(E,U)$ with $k_f,k_g\in\mathbb{N}$.  
    Their Lie bracket is defined pointwise by
    \begin{equation}
        \label{eq:Ck_bracket}
        [f,g](p)=Dg(p)[f(p)]-Df(p)[g(p)],\qquad p\in E.
    \end{equation}
\end{definition}

The definition of the Lie bracket is the same as that for smooth functions, Definition \ref{def:smooth_liebracket}. However, for $C^k$ functions the Lie bracket reduces the regularity by one.

\begin{lemma}\label{lem:Ck_liebracket_order}
     Let $f\in C^{k_f}(E,U)$ and $g\in C^{k_g}(E,U)$ with $k_f,k_g\in\mathbb{N}$.  
     Then $[f,g]\in C^{\min(k_f,k_g)-1}(E,U)$.
\end{lemma}

\begin{proof}
    Let $k=\min(k_f,k_g)$. 
    Then $f,g \in C^{k}(E,U)$.
    By Lemma \ref{lem:Ck_derivative} and the definition of $C^k$, we have 
    \begin{equation}
        f,g\in C^{k-1}(E, U) \quad\text{ and }\quad Dg,Df\in C^{k-1}(E, \mathbf{L}(U, U)).
    \end{equation}
    Let $B : \mathbf{L}(U, U) \times U \rightarrow U$ be defined by
    \begin{equation}
        B(A, v) = A[v].
    \end{equation}
    Then $B$ is a continuous bilinear function, and hence smooth \parencite[Chapter 1, Proposition 3.8]{Lang1999}. 
    Since the functions 
    \begin{equation}
        (Dg,f)(p) = (Dg(p),f(p)) \quad\text{ and }\quad (Df,g)(p) = (Df(p),g(p))
    \end{equation}
    are $C^{k-1}(E, \mathbf{L}(U, U) \times U)$, then $B \circ (Dg,f)$ and $B \circ (Df,g)$ are both $C^{k-1}(E,U)$ \parencite[Chapter 1, Proposition 3.2]{Lang1999}. 
    Since all terms on the RHS of \eqref{eq:Ck_bracket} are $C^{k-1}$, then $[f,g]\in C^{k-1}(E,U)$.
\end{proof}

Lemma \ref{lem:Ck_liebracket_order} means the Lie bracket of two $C^k$ functions does not necessarily belong to $C^{k}$, and it is therefore not a Lie algebra. 
However, the Lie bracket does satisfy bilinearity and anti-symmetry, both of which follow directly from the definition. 
Furthermore, the Lie bracket satisfies the Jacobi identity.

\begin{lemma}\label{lem:Ck_liebracket_Jacobi}
     Let $f,g,h\in C^{k}(E,U)$ with $k\in \mathbb{N}$ satisfying $k\geq 2$. 
     Then 
     \begin{equation}
         [f, [g, h]] + [g, [h, f]] + [h, [f, g]] = 0.
     \end{equation}
\end{lemma}

\begin{proof}
    Fix $p \in E$ and define:
    \begin{equation}
        \begin{aligned}
            x = f(p), \quad y &= g(p), \quad z = h(p), \\
            A = Df(p), \quad B &= Dg(p), \quad C = Dh(p), \\
            \mathbb{A} = D^2f(p), \quad \mathbb{B} &= D^2g(p), \quad \mathbb{C} = D^2h(p),
        \end{aligned}
    \end{equation}
    where $A, B, C \in \mathbf{L}(U, U)$ and $\mathbb{A}, \mathbb{B}, \mathbb{C} \in \mathbf{L}_{\mathrm{sym}}(U^{\otimes 2}, U)$. 
    Then
    \begin{equation}
        \begin{aligned}
            [f, [g, h]](p) &= -\mathbb{B}[z,x] - B[C[x]] + \mathbb{C}[y,x] + C[B[x]] + A[B[z]] - A[C[y]], \\
            [g, [h, f]](p) &= -\mathbb{C}[x,y] - C[A[y]] + \mathbb{A}[z,y] + A[C[y]] + B[C[x]] - B[A[z]], \\
            [h, [f, g]](p) &= -\mathbb{A}[y,z] - A[B[z]] + \mathbb{B}[x,z] + B[A[z]] + C[A[y]] - C[B[x]].
        \end{aligned}
    \end{equation}
    
    By the symmetry of the second derivatives,
    \begin{equation}
        \mathbb{B}[x,z] - \mathbb{B}[z,x] + \mathbb{C}[y,x] - \mathbb{C}[x,y] + \mathbb{A}[z,y] - \mathbb{A}[y,z] = 0,
    \end{equation}
    and all the other terms cancel. 
    Therefore
    \begin{equation}
        [f, [g, h]] + [g, [h, f]] + [h, [f, g]] = 0
    \end{equation}
    for all $p$, concluding the proof.
\end{proof}

\subsection{$C^{k,\alpha}$ Vector Fields}

\label{sec:Ckalpha}

We now strengthen finite differentiability by requiring the highest derivative to vary in a controlled way. 
This leads to the classical Hölder spaces $C^{k,\alpha}$, which refine the $C^k$ spaces by measuring not only the existence of derivatives up to order $k$, but also the regularity of the $k$-th derivative itself. 
These spaces will be useful for comparison with $\mathrm{Lip}(\gamma)$ regularity in Section~\ref{sec:lipgamma}, since the two notions agree on open convex domains but behave differently at a global level.

\begin{definition}[$C^{k,\alpha}$ function]\label{def:Ckalpha}
    Let $U$ and $V$ be Banach spaces, $E\subset U$ be open, $k\in\mathbb{N}_0$, and $\alpha\in(0,1]$. 
    A map $f:E\to V$ is $C^{k,\alpha}$ if
    \begin{enumerate}
        \item $f\in C^{k}(E,V)$ and
        \item its $k$-th derivative is $\alpha$-Hölder continuous on $E$:
              \begin{equation}
                  [D^{k}f]_{\alpha} =\sup_{\substack{p,q\in E\\ p\neq q}} \frac{\|D^{k}f(p)-D^{k}f(q)\|_{k}}{\|p-q\|_{U}^{\alpha}} <\infty.
              \end{equation}
    \end{enumerate}
\end{definition}

The space of all $C^{k,\alpha}$ functions from $E$ to $V$ is denoted $C^{k,\alpha}(E,V)$.

\begin{lemma}\label{lem:Ckalpha_derivative}
    If $f\in C^{k,\alpha}(E,V)$ with $k\in \mathbb{N}$, then its Fréchet derivative
    \begin{equation}
        Df:E\longrightarrow\mathbf L(U,V)
    \end{equation}
belongs to $C^{k-1,\alpha}\!\bigl(E,\mathbf L(U,V)\bigr)$.
\end{lemma}

\begin{proof}
    Because $f\in C^{k}(E,V)$, Lemma~\ref{lem:Ck_derivative} gives
    $Df\in C^{k-1}\bigl(E,\mathbf L(U,V)\bigr)$. 
    Under the currying convention $\mathbf{L}_{\mathrm{sym}}\bigl(U^{\otimes j},\mathbf{L}(U,V)\bigr) \cong \mathbf{L}\bigl(U^{\otimes (j+1)},V\bigr)$, we have
    \begin{equation}
    \sup_{\substack{p,q\in E\\ p\neq q}} \frac{\|D^{k-1}(Df)(p)-D^{k-1}(Df)(q)\|_{k}}{\|p-q\|_{U}^{\alpha}} = \sup_{\substack{p,q\in E\\ p\neq q}} \frac{\|D^{k}f(p)-D^{k}f(q)\|_{k}}{\|p-q\|_{U}^{\alpha}} = [D^{k}f]_{\alpha}.
    \end{equation}
    Hence $[D^{k-1}(Df)]_{\alpha}= [D^{k}f]_{\alpha}$ and the claim follows.
\end{proof}

\begin{definition}[$C^{k,\alpha}_{b}$ space and norm]\label{def:Ckalpha_b}
For $f\in C^{k,\alpha}(E,V)$ set
\begin{equation}
    \|f\|_{C^{k,\alpha}}:=\max\left\{\max_{0\le j\le k}\sup_{p\in E}\|D^{j}f(p)\|_{j}, [D^{k}f]_{\alpha}\right\}.
\end{equation}
The bounded $C^{k,\alpha}$ space is
\begin{equation}
    C^{k,\alpha}_{b}(E,V)
    =\bigl\{f\in C^{k,\alpha}(E,V):\|f\|_{C^{k,\alpha}}<\infty\bigr\}.
\end{equation}
\end{definition}

\begin{definition}[$C^{k, \alpha}$ Lie bracket]
    Let $f,g\in C^{k,\alpha}(E,U)$ with $k\in \mathbb{N}$. 
    Their Lie bracket is defined pointwise by
    \begin{equation}
        \label{eq:Ckalpha_bracket}
        [f,g](p)=Dg(p)[f(p)]-Df(p)[g(p)],\qquad p\in E.
    \end{equation}
\end{definition}

However, unlike $C^k$ functions, the Lie bracket of two $C^{k,\alpha}$ functions is not necessarily $C^{k-1,\alpha}$. 
In fact, $C^{k,\alpha}$ functions don't have an inclusion with respect to $k$. 
For example, consider $E\subset \mathbb{R}$ with $E=(-\infty, 0) \cup (0, \infty)$ and $f:E \rightarrow \mathbb{R}$ defined by
\begin{equation}
    f(x) = \begin{cases}
        x-1, \quad &x \in (-\infty, 0), \\
        x+1, \quad &x \in (0, \infty),
    \end{cases}
\end{equation}
with $Df=1$. 
Then $[Df]_1=0$ and $f\in C^{1,1}(E, \mathbb{R})$. However, $[f]_1$ is infinite, due to the discontinuity at $0$. Therefore, $f\notin C^{0,1}(E, \mathbb{R})$. 
Furthermore, letting $g:E\rightarrow\mathbb{R}$ be defined by $g(x)=2x$, then
\begin{equation}
    [f,g] = \begin{cases}
        -2, \quad &x \in (-\infty, 0), \\
        +2, \quad &x \in (0, \infty), \\
    \end{cases}
\end{equation}
so $[f,g]\notin C^{0,1}$ despite $f,g\in C^{1,1}$. 
The same results hold for $C^{k,\alpha}_b$ functions, which can be seen by replacing $E$ with $(-1,0) \cup (0,1)$. 
\\ \\
In the next section, we introduce $\mathrm{Lip}(\gamma)$ regularity, a global notion of regularity defined on both open and closed subsets, which will allow us to define the Lie bracket globally.

\section{$\mathrm{Lip}(\gamma)$ Functions}

\label{sec:lipgamma}

\subsection{Definition}

\begin{definition}[$\mathrm{Lip}(\gamma, E, V)$ \parencite{stein1970singular}]
\label{def:lipgamma}
    Let $U$ and $V$ be two Banach spaces, $\gamma>0$ be a real number, $k$ be the non-negative integer such that $\gamma \in (k, k+1]$, $E$ be a subset of $U$, and $f^0:E\rightarrow V$ a function. 
    For $j=1,\ldots,k$, let $f^j:E\rightarrow \mathbf{L}_{\mathrm{sym}}(U^{\otimes j}, V)$ be functions. 
    The collection $(f^0, f^1, \ldots, f^k)$ is an element of $\mathrm{Lip}(\gamma,E,V)$ if there exists $M\geq0$ such that the following conditions hold:
    \begin{itemize}
    \item For $j=0,\ldots,k$,
    \begin{equation}
    \label{eq:lipbound1}
        \sup_{p\in E}||f^j(p)||_{j} \leq M,
    \end{equation}
    \item For $j=0,\ldots,k$, all $p,q \in E$, 
    \begin{equation}
    \label{eq:lipbound2}
        \bigg|\bigg|f^j(q)-\sum_{l=0}^{k-j}\frac{f^{j+l}(p)[(q-p)^{\otimes l}]}{l!}\bigg|\bigg|_{j} \leq M||q-p||_U^{\gamma-j}.
    \end{equation}
    \end{itemize}
    where 
    \begin{equation}
        ||\cdot||_j = \begin{cases}
            ||\cdot||_V, \quad &j=0, \\
            ||\cdot||_{\operatorname{op}}, \quad &j=1,\ldots,k.
        \end{cases}
    \end{equation}
\end{definition}
When there is no confusion over $E$ and $V$, the shorthand $\mathrm{Lip}(\gamma)$ will be used. 
If a collection $f=(f^0, f^1, \ldots, f^k)$ is $\mathrm{Lip}(\gamma)$, then the $\mathrm{Lip}(\gamma)-$norm is the smallest $M$ for which \eqref{eq:lipbound1} and \eqref{eq:lipbound2} hold. 
The $\mathrm{Lip}(\gamma)-$norm is denoted by $||f||_{\mathrm{Lip}(\gamma)}$. 
$\mathrm{Lip}(\gamma, E, V)$ functions can be equivalently defined by considering a function defined on $E$ which takes its values in the polynomials from $U$ to $V$.

\begin{definition}[$\mathrm{Lip}(\gamma, E, V)$ as polynomials]
\label{def:lipgamma_poly}
    Let $U$ and $V$ be two Banach spaces, $\gamma>0$ be a real number, $k$ be the non-negative integer such that $\gamma \in (k, k+1]$, $E$ be a subset of $U$, and $\mathbf{P}^k(U,V)$ be the space of $k^{\text{th}}$ order polynomial functions from $U$ to $V$. 
    The function $f:E \rightarrow \mathbf{P}^k(U,V)$ defined by 
    \begin{equation}
        f(x)(y) = P_x(y) = \sum_{l=0}^{k} \frac{f^l(x)[(y-x)^{\otimes l}]}{l!}.
    \end{equation}
    for $f^j:E\rightarrow \mathbf{L}_{\mathrm{sym}}(U^{\otimes j}, V)$ is an element of $\mathrm{Lip}(\gamma, E, V)$ if 
    \begin{equation}
        \label{eq:poly_lipbound1}
        \sup_{x\in E}||P^j_x(x)||_{j} \leq M
    \end{equation}
    for $j=0,\ldots,k$ and
    \begin{equation}
        \label{eq:poly_lipbound2}
        \bigg|\bigg|P^j_y(y)-P^j_x(y)\bigg|\bigg|_{j} \leq M\|x-y\|_U^{\gamma-j}
    \end{equation}
    for $j=0,\ldots,k$ and all $x,y \in E$, where 
    \begin{equation}
    \begin{aligned}
        P^j_x(y)[u_1\otimes\cdots\otimes u_j] &= D^j_yP_x(y)[u_1\otimes\cdots\otimes u_j], \\
        &= \sum_{l=0}^{k-j} \frac{f^{j+l}(x)[u_1\otimes\cdots\otimes u_j \otimes (y - x)^{\otimes l}]}{l!}.
    \end{aligned}
    \end{equation}
\end{definition}

To further illustrate the definition of $\mathrm{Lip}(\gamma)$, we will compare and contrast it with $C^{k,\alpha}_b$.

\subsection{Comparison with $C^{k,\alpha}_b$}

\label{sec:Ckalpha_lipgamma}

On open convex subsets of a Banach space, the definition of $\mathrm{Lip}(\gamma)$ coincides with the definition of $C^{k,\alpha}_b$, as will be shown in Lemma \ref{lem:CKalpha_Lipgamma}. 
This result will make use of the following lemma, which demonstrates that the constant in the bound of any Taylor remainder for a $C^{k,\alpha}_b$ function defined on an open convex set is bounded by the constant of the $\alpha-$H\"older bound on the $k^{th}$ derivative. 

\begin{lemma}\label{lem:Ck_taylor_bound}
    Let $U$ and $V$ be Banach spaces, $E\subseteq U$ be an open convex set, and $f\in C^{k,\alpha}_b(E,V)$. 
    Let $D^0f=f$ and $M = [D^kf]_{\alpha}$. Then
    \begin{equation}
        \sup_{p,q\in E} \frac{\big|\big|D^{k-i}f(q)-\sum_{l=0}^{i}\frac{D^{k-i+l}f(p)[(q-p)^{\otimes l}]}{l!}\big|\big|_{k-i}}{||q-p||_U^{\alpha+i}} \leq M,
    \end{equation}
    for $i=0,\ldots,k$.
\end{lemma}

\begin{proof}
We proceed by induction on $i$. 
Let $i\in\{0,\ldots,k-1\}$ and assume that
\begin{equation}
\label{eq:kderivbound_induct_assump1}
    \sup_{p,q\in E} \frac{\big|\big|D^{k-i}f(q)-\sum_{l=0}^{i}\frac{D^{k-i+l}f(p)[(q-p)^{\otimes l}]}{l!}\big|\big|_{k-i}}{||q-p||_U^{\alpha+i}} \leq M.
\end{equation}

We now show that the same estimate holds with $i$ replaced by $i+1$. 
By the definition of $\|\cdot\|_{k-(i+1)}$, we have
\begin{equation}
\label{eq:kderivbound_induct_result1_part1}
    \begin{aligned}
        \sup_{p,q\in E}& \frac{\big|\big|D^{k-(i+1)}f(q)-\sum_{l=0}^{i+1}\frac{D^{k-(i+1)+l}f(p)[(q-p)^{\otimes l}]}{l!}\big|\big|_{k-(i+1)}}{||q-p||_U^{\alpha+(i+1)}} \\
        &= \sup_{p,q\in E} \sup_{u\neq0} \frac{\big|\big|D^{k-(i+1)}f(q)[u]-\sum_{l=0}^{i+1}\frac{D^{k-(i+1)+l}f(p)[u\otimes (q-p)^{\otimes l}]}{l!}\big|\big|_V}{||u||_{U^{\otimes(k-(i+1))}}\;||q-p||_U^{\alpha+(i+1)}}.
    \end{aligned}
\end{equation}

We next use the fundamental theorem of calculus to rewrite the difference
$D^{k-(i+1)}f(q)[u]-D^{k-(i+1)}f(p)[u]$, giving
\begin{equation}
\begin{aligned}
    \sup_{p,q\in E} \sup_{u\neq0} &\frac{\big|\big|\int_0^1D^{k-i}f(p+t(q-p))[u\otimes (q-p)]\mathrm{d}t-\sum_{l=1}^{i+1}\frac{D^{k-(i+1)+l}f(p)[u\otimes (q-p)^{\otimes l}]}{l!}\big|\big|_V}{||u||_{U^{\otimes(k-(i+1))}}\;||q-p||_U^{\alpha+(i+1)}}.
\end{aligned}
\end{equation}

We then re-index the sum by replacing $l$ with $l+1$,
\begin{equation}
\begin{aligned}
    \sup_{p,q\in E}\sup_{u\neq0}  \frac{\big|\big|\int_0^1D^{k-i}f(p+t(q-p))[u \otimes (q-p)]\mathrm{d}t-\sum_{l=0}^{i}\frac{D^{k-i+l}f(p)[u\otimes (q-p)^{\otimes l+1}]}{(l+1)!}\big|\big|_V}{||u||_{U^{\otimes(k-(i+1))}}\;||q-p||_U^{\alpha+(i+1)}}.
\end{aligned}
\end{equation}

Letting $u'=u\otimes (q-p)\in U^{\otimes (k-i)}$, bringing the sum inside the integral, and using $\int_0^1 t^l\,dt=\frac{1}{l+1}$, we obtain
\begin{equation}
\label{eq:kderivbound_induct_result1_part2}
    \begin{aligned}       
        \sup_{p,q\in E}\sup_{u\neq0}  &\frac{\big|\big|\int_0^1D^{k-i}f(p+t(q-p))[u']-\sum_{l=0}^{i}\frac{1}{l!}D^{k-i+l}f(p)[u'\otimes(t(q-p))^{\otimes l}]\mathrm{d}t\big|\big|_V}{||u||_{U^{\otimes(k-(i+1))}}\;||q-p||_U^{\alpha+(i+1)}}\\
        & \stackrel{\eqref{eq:kderivbound_induct_assump1}}{\leq} \sup_{p,q\in E} \frac{M\int_0^1 ||u'||_{U^{\otimes(k-i)}}\,||t(q-p)||_U^{\alpha+i}\mathrm{d}t}{||u||_{U^{\otimes(k-(i+1))}}\;||q-p||_U^{\alpha+(i+1)}}.
    \end{aligned}
\end{equation}

Applying $||u'||_{U^{\otimes(k-i)}}\leq ||u||_{U^{\otimes(k-(i+1))}}\,||q-p||_U$ we deduce that
\begin{equation}
\begin{aligned}
    \sup_{p,q\in E} \frac{\big|\big|D^{k-(i+1)}f(q)-\sum_{l=0}^{i+1}\frac{D^{k-(i+1)+l}f(p)[(q-p)^{\otimes l}]}{l!}\big|\big|_{k-(i+1)}}{||q-p||_U^{\alpha+(i+1)}}
    &\leq \sup_{p,q\in E} \frac{M\int_0^1 ||t(q-p)||_U^{\alpha+i}\mathrm{d}t}{||q-p||_U^{\alpha+i}} \\
    &\leq M\int_0^1 t^{\alpha+i}\mathrm{d}t \\
    &= \frac{M}{\alpha+(i+1)} \leq M.
\end{aligned}
\end{equation}

The case $i=0$ is true by the assumption that $[D^kf]_{\alpha}=M$, completing the induction. 
\end{proof}

With this Taylor remainder estimate, we can now show that on open convex sets the classical $C^{k,\alpha}_b$ notion of regularity agrees exactly with $\mathrm{Lip}(\gamma)$ regularity.

\begin{lemma} \label{lem:CKalpha_Lipgamma}
    Let $U$ and $V$ be Banach spaces, $E\subseteq U$ be open and convex, $k\in\mathbb{N}_0$, $\alpha\in(0,1]$, and $\gamma=k+\alpha$. 
    Then $C^{k,\alpha}_b(E,V)\equiv \mathrm{Lip}(\gamma, E, V)$.
\end{lemma}

\begin{proof}
    \textbf{Part $1$, $C^{k,\alpha}_b$ implies $\mathrm{Lip}(\gamma)$:} Let $f\in C^{k,\alpha}_b(E,V)$, and choose $\hat{f}^0=f$ and $\hat{f}^i=D^if$ for $i=1,\ldots,k$. Then 
    \begin{equation}
        \|\hat{f}\|_{\mathrm{Lip}(\gamma)} = \max\left\{\max_{0\leq j\leq k}\sup_{p\in E}||\hat{f}^j(p)||_{j}, \max_{0\leq j\leq k}\sup_{p,q \in E}\frac{\|\hat{f}^j(q)-\sum_{l=0}^{k-j}\frac{\hat{f}^{j+l}(p)[(q-p)^{\otimes l}]}{l!}\|_{j}}{||q-p||_U^{\gamma-j}}\right\},
    \end{equation}
    where 
    \begin{equation}
        ||\cdot||_j = \begin{cases}
            ||\cdot||_V, \quad &j=0, \\
            ||\cdot||_{\operatorname{op}}, \quad &j=1,\ldots,k.
        \end{cases}
    \end{equation}
    Applying Lemma \ref{lem:Ck_taylor_bound}, this simplifies to
    \begin{equation}
        \begin{aligned}
            \|\hat{f}\|_{\mathrm{Lip}(\gamma)} &= \max\left\{\max_{0\leq j\leq k}\sup_{p \in E}||\hat{f}^j(p)||_{j}, \sup_{p,q\in E}\frac{\|\hat{f}^k(q)-\hat{f}^{k}(p)\|_{k}}{||q-p||_U^{\alpha}}\right\}, \\
            &= \max\left\{\max_{0\le j\le k}\sup_{p\in E}\|D^{j}f(p)\|_{j}, [D^{k}f]_{\alpha}\right\}, \\
            &= \|f\|_{C^{k, \alpha}},
        \end{aligned}
    \end{equation}
    which is finite by assumption.
    Therefore, $\hat{f}\in\mathrm{Lip}(\gamma, E, V)$ and $\|\hat{f}\|_{\mathrm{Lip}(\gamma)}=\|f\|_{C^{k, \alpha}}.$
    \\ \\
    \textbf{Part $2$, $\mathrm{Lip}(\gamma)$ implies $C^{k,\alpha}_b$:} Let $\hat{f}\in\mathrm{Lip}(\gamma, E, V)$, $v\in U$ be fixed and $q=p+tv$ for $t\in[0,T]$, where $T$ is small enough that $q\in E$ for all $t$. 
    From \eqref{eq:lipbound2} with $j=0$,
    \begin{equation}
        \left\|\hat{f}^0(p+tv)-\sum_{l=0}^{k}\frac{\hat{f}^{l}(p)[(tv)^{\otimes l}]}{l!}\right\|_{V} = \|R_k(p,tv)\| \leq M\|tv\|^{\gamma},
    \end{equation}
    for all $p\in E$. 
    Since 
    \begin{equation}
        \lim_{t\rightarrow 0} \frac{\|R_k(p,tv)\|}{\|tv\|^{k}} \leq \lim_{t\rightarrow 0} M\|tv\|^{\gamma-k} = 0
    \end{equation}
    uniformly in $p$, then $\hat{f}^0$ is $C^{k}$ with $\hat{f}^i=D^i\hat{f}^0$ for $i=1,\ldots,k$ by the converse of Taylor's Theorem \citep{albrecht1971}. 
    From \eqref{eq:lipbound2} with $j=k$,
    \begin{equation}
        \sup_{p,q\in E}\frac{\bigg|\bigg|D^k\hat{f}^0(q)-D^k\hat{f}^0(p)\bigg|\bigg|_{k}}{||q-p||_U^{\gamma-k}} = [D^k\hat{f}^0]_{\gamma-k} \leq M.
    \end{equation}
    Therefore, $\hat{f}^0\in C_b^{k,\alpha}(E,V)$ with $\alpha = \gamma-k$. 
    Applying the same argument as in part $1$, $\|\hat{f}^0\|_{C^{k, \alpha}}=\|\hat{f}\|_{\mathrm{Lip}(\gamma)}$.
\end{proof}

Although $\mathrm{Lip}(\gamma, E, V)$ and $C^{k,\alpha}_b(E, V)$ are equivalent on open convex subsets $E$, and therefore the entire space $U$, they can differ on generic open sets $E$. 
Taking our example from before, where $E=(-\infty,0) \cup (0,\infty)$ and $f:E \rightarrow \mathbb{R}$ is defined by
\begin{equation}
    f(x) = \begin{cases}
        x-1, \quad &x \in (-\infty, 0), \\
        x+1, \quad &x \in (0, \infty),
    \end{cases}
\end{equation}
then $f\in C^{1,1}(E, \mathbb{R})$, but $f\not\in\mathrm{Lip}(2, E, \mathbb{R})$, as \eqref{eq:lipbound2} diverges for $j=0$ as $x$ approaches $0$.
When $E$ is an open subset, the core difference between $\mathrm{Lip}(\gamma)$ and $C^{k, \alpha}_b$ is requiring a H\"older type bound on each level of the derivative, as opposed to only the highest derivative. 
The additional regularity given by assuming your function is $\mathrm{Lip}(\gamma)$ is sufficient to ensure nesting.

\begin{lemma}[$\mathrm{Lip}(\gamma)$ function nesting \citep{boutaib2016lipschitz, lyons2025}]
\label{lem:lipgamma_nest}
    Let $U$ and $V$ be Banach spaces, $E\subset U$, and $f\in\mathrm{Lip}(\gamma, E, V)$, and $0<\theta<\gamma$. 
    Then $f\in\mathrm{Lip}(\theta, E, V)$ and
    \begin{equation}
        \|f\|_{\mathrm{Lip}(\theta)} \leq (1+e)\|f\|_{\mathrm{Lip}(\gamma)}.
    \end{equation}
\end{lemma}

The nesting property is a key advantage of $\mathrm{Lip}(\gamma)$ regularity over $C^{k,\alpha}_b$ regularity on general domains, since it allows the Lie bracket of two $\mathrm{Lip}(\gamma)$ vector fields to be defined globally as a $\mathrm{Lip}(\gamma-1)$ vector field.

\subsection{Lie Bracket}

To define the Lie bracket, we first introduce the derivative of a $\mathrm{Lip}(\gamma)$ function and the composition $g[f]:E\to W$ of $f\in\mathrm{Lip}(\gamma_f,E,V)$ and $g\in\mathrm{Lip}(\gamma_g,E,\mathbf{L}(V,W))$. 
These are the basic ingredients from which the Lie bracket will be constructed.

\begin{definition}[$\mathrm{Lip}(\gamma)$ Derivative]
    Let $f=(f^0,\ldots,f^k)\in \mathrm{Lip}(\gamma, E, V)$ with $\gamma>1$. 
    The derivative of $f$ is defined by
    \begin{equation}
        Df = ((Df)^0, (Df)^1, \ldots, (Df)^{k-1}) = (f^1,f^2,\ldots,f^k).
    \end{equation}
\end{definition}

\begin{lemma}\label{lem:lipgamma_deriv}
    If $f\in \mathrm{Lip}(\gamma, E, V)$ with $\gamma>1$ and $\|f\|_{\mathrm{Lip}(\gamma)}=M_f$, then $Df\in \mathrm{Lip}(\gamma-1, E, \mathbf{L}(U, V))$ with $\|Df\|_{\mathrm{Lip}(\gamma-1)}\leq M_f$.
\end{lemma}

\begin{proof}
    Trivially, $Df$ satisfies \eqref{eq:lipbound1} with bound $M_f$. 
    Under the currying convention $\mathbf{L}_{\mathrm{sym}}\bigl(U^{\otimes j},\mathbf{L}(U,V)\bigr) \cong \mathbf{L}\bigl(U^{\otimes (j+1)},V\bigr)$, we have
    \begin{equation}
        (Df)^j=f^{j+1}
    \end{equation}
    for $0\leq j\leq k-1$.
    Hence, for all $x,y\in E$,
    \begin{equation}
        \begin{aligned}
            &\left\|(Df)^j(y)-\sum_{l=0}^{k-1-j}\frac{(Df)^{j+l}(x)[(y-x)^{\otimes l}]}{l!}\right\|_{j} \\
            &\qquad\qquad\qquad\quad= \left\|f^{j+1}(y)-\sum_{l=0}^{k-(j+1)}\frac{f^{j+1+l}(x)[(y-x)^{\otimes l}]}{l!}\right\|_{j+1} \\ 
            &\qquad\qquad\qquad\quad\leq M_f\|x-y\|_U^{\gamma-1-j}.
        \end{aligned}
    \end{equation}
    Therefore, $Df$ satisfies \eqref{eq:lipbound2} with bound $M_f$ and $Df=(f^1,\ldots,f^k)\in \mathrm{Lip}(\gamma-1, E, \mathbf{L}(U, V))$.
\end{proof}

\begin{definition}\label{def:comp_lipgamma}
    Let $U,V,W$ be Banach spaces, $E\subset U$, $f\in\mathrm{Lip}(\gamma_f, E, V)$, $g\in\mathrm{Lip}(\gamma_g, E, \mathbf{L}(V, W))$, $\gamma=\min(\gamma_f, \gamma_g)$, and $k$ be the integer such that $k<\gamma\leq k+1$. 
    Then
    \begin{equation}
        h = (h^0, \ldots, h^k) = g[f]
    \end{equation}
    is defined pointwise by
    \begin{equation}\label{eq:comp_faadibruno}
        h^i(p)[u_1\otimes \cdots \otimes u_i] = \sum_{r=0}^i \frac{1}{r!(i-r)!} \sum_{\sigma \in S_i} g^{r}(p)[u_{\sigma(1)}\otimes \cdots \otimes u_{\sigma(r)}]\bigl[f^{i-r}(p)[u_{\sigma(r+1)}\otimes \cdots \otimes u_{\sigma(i)}] \bigr]
    \end{equation}
    for $p\in E$ and $i=0,\ldots, k$.
\end{definition}

This definition is an application of the definition from \parencite[Proposition 3.29]{boutaib2016lipschitz} and is analogous to applying Fa\`a-di-Bruno to a bilinear map applied to two $C^{k,\alpha}$ functions \citep{FaaDiBruno1855}.

\begin{lemma}\label{lem:comp_lipgamma}
    If $f\in\mathrm{Lip}(\gamma_f, E, V)$ and $g\in\mathrm{Lip}(\gamma_g, E, \mathbf{L}(V, W))$ then $h=g[f]\in\mathrm{Lip}(\gamma, E, W)$, where $\gamma=\min(\gamma_f, \gamma_g)$.
\end{lemma}

\begin{proof}
    By Lemma \ref{lem:lipgamma_nest}, $f\in\mathrm{Lip}(\gamma_f)$ and $g\in\mathrm{Lip}(\gamma_g)$ implies $f,g\in\mathrm{Lip}(\gamma)$ where $\|f\|_{\mathrm{Lip}(\gamma)}\leq(1+e)\|f\|_{\mathrm{Lip}(\gamma_f)}$ and $\|g\|_{\mathrm{Lip}(\gamma)}\leq(1+e)\|g\|_{\mathrm{Lip}(\gamma_g)}$, with at least one bound holding with equality. 
    Let $B:\mathbf{L}(V,W)\times V \rightarrow W$ be defined by
    \begin{equation}
        B(A, v) = A[v],
    \end{equation}
    for $A\in\mathbf{L}(V, W)$ and $v\in V$. 
    By \parencite[Proposition 3.29]{boutaib2016lipschitz}, the function $g[f]=B(g,f)=(B^0, \ldots, B^k)=(h^0,\ldots,h^k)$, where the $h^i$ are defined by \eqref{eq:comp_faadibruno}, satisfies 
    \begin{equation}
        \|B(g, f)\|_{\mathrm{Lip}(\gamma)} \leq C_{\gamma}\|g\|_{\mathrm{Lip}(\gamma)}\|f\|_{\mathrm{Lip}(\gamma)} \leq C_{\gamma}(1+e)\|g\|_{\mathrm{Lip}(\gamma_g)}\|f\|_{\mathrm{Lip}(\gamma_f)}, 
    \end{equation}
    where $C_{\gamma}$ is a constant depending only on $\gamma$. 
    Therefore, $h=g[f]\in\mathrm{Lip}(\gamma, E, W)$.
\end{proof}

\begin{definition}[$\mathrm{Lip}(\gamma)$ Lie Bracket]
    \label{def:lipgamma_liebracket}
    Let $f=(f^0,\ldots,f^{k_f})\in \mathrm{Lip}(\gamma_f, E, U)$ and $g=(g^0,\ldots,g^{k_g})\in \mathrm{Lip}(\gamma_g, E, U)$ with $\gamma_f,\gamma_g > 1$. 
    Their Lie bracket is defined by
    \begin{equation}
        [f, g] = Dg[f] - Df[g].
    \end{equation}
\end{definition}

\begin{lemma}\label{lem:lipgamma_liebracket_deriv}
    If $f=(f^0,\ldots,f^{k_f})\in \mathrm{Lip}(\gamma_f, E, U)$ and $g=(g^0,\ldots,g^{k_g})\in \mathrm{Lip}(\gamma_g, E, U)$ for $\gamma_f,\gamma_g > 1$, then 
    \begin{equation}
        [f, g] \in \mathrm{Lip}(\min(\gamma_f,\gamma_g)-1, E, U).
    \end{equation}
\end{lemma}

\begin{proof}
    Let $\gamma=\min(\gamma_f,\gamma_g)$. 
    By Lemma \ref{lem:lipgamma_nest}, $f,g\in \mathrm{Lip}(\gamma,E,U)$. 
    Then Lemma \ref{lem:lipgamma_deriv} gives
    \begin{equation}
        Df,Dg\in \mathrm{Lip}(\gamma-1,E,\mathbf{L}(U,U)).
    \end{equation}
    Applying Lemma \ref{lem:comp_lipgamma}, we obtain
    \begin{equation}
        Dg[f],Df[g]\in \mathrm{Lip}(\gamma-1,E,U).
    \end{equation}
    Therefore,
    \begin{equation}
        [f,g]=Dg[f]-Df[g]\in \mathrm{Lip}(\gamma-1,E,U).
    \end{equation}
\end{proof}

\begin{remark}
    The proof of Lemma \ref{lem:comp_lipgamma} relies on Lemma \ref{lem:lipgamma_nest}, the nesting property of $\mathrm{Lip}(\gamma)$ functions. 
    This is the crucial difference between $\mathrm{Lip}(\gamma)$ functions and $C^{k,\alpha}_b$ functions that allows for a well defined global Lie bracket of $\mathrm{Lip}(\gamma)$ functions.
\end{remark}

The polynomial view point of $\mathrm{Lip}(\gamma)$ functions, Definition \ref{def:lipgamma_poly}, motivates an alternative definition of the Lie bracket of two $\mathrm{Lip}(\gamma)$ functions.

\begin{definition}[$\mathrm{Lip}(\gamma)$ Polynomial Lie Bracket]
    \label{def:lipgamma_polyliebracket}
    Let $f=(f^0,\ldots,f^k)\in \mathrm{Lip}(\gamma, E, U)$ and $g=(g^0,\ldots,g^k)\in \mathrm{Lip}(\gamma, E, U)$ with $\gamma > 1$. 
    Let 
    \begin{equation}
        P^f_p(q) = \sum_{l=0}^{k} \frac{f^l(p)[(q-p)^{\otimes l}]}{l!}.
    \end{equation} 
    and
    \begin{equation}
        P^g_p(q) = \sum_{l=0}^{k} \frac{g^l(p)[(q-p)^{\otimes l}]}{l!}
    \end{equation}
    for $p\in E$ and $q\in U$. 
    The Polynomial Lie bracket of $f$ and $g$ is defined pointwise by 
    \begin{equation}
        [P^f_p, P^g_p](q) = \sum_{l=0}^{k-1} \frac{g^{l+1}(p)\left[P^f_p(q) \otimes (q-p)^{\otimes l}\right] - f^{l+1}(p)\left[P^g_p(q) \otimes (q-p)^{\otimes l}\right]}{l!}.
    \end{equation}
    for $p\in E$ and $q\in U$.
\end{definition}

\begin{remark}\label{rem:poly_lie_algebra}
    Definition \ref{def:lipgamma_polyliebracket} relies on the pointwise Lie bracket of two polynomial vector fields.
    Since polynomial vector fields are smooth and closed under this bracket, they form a Lie subalgebra of the space of smooth vector fields.
\end{remark}

Initially, the two definitions seem to produce different Lie brackets. 
In particular, the polynomial at each point in $E$ when using Definition \ref{def:lipgamma_liebracket} is of order $k-1$, whereas the polynomial at each point when using Definition \ref{def:lipgamma_polyliebracket} is of order up to $2k-1$. 
However, the two Lie brackets agree with each other in the $\mathrm{Lip}(\gamma-1)$ sense. 

\begin{lemma}\label{lem:Lipgamma_liebracket_agreement}
    Let $f=(f^0,\ldots,f^k)\in \mathrm{Lip}(\gamma, E, U)$ and $g=(g^0,\ldots,g^k)\in \mathrm{Lip}(\gamma, E, U)$ with $\gamma > 1$ and $k$ the non-negative integer such that $\gamma\in(k,k+1]$. 
    For each $p\in E$, the polynomial defined at each point by their Lie Bracket $P^{[f, g]}_p \in \mathbf{P}^{k-1}(U, U)$ and the polynomial Lie bracket $[P^f_p, P^g_p]$ agree up to the $(k-1)^{\text{th}}$ term. 
\end{lemma}

\begin{proof}
    Let $s\leq k$. 
    Then
    \begin{equation}
        D^s_qP^f_p(q) \bigg|_{q=p} = f^s(p),
    \end{equation}
    and
    \begin{equation}
        D^s_qP^g_p(q) \bigg|_{q=p} = g^s(p).
    \end{equation}
    Additionally,
    \begin{equation}
        D^s_q (q-p)^{\otimes l}\bigg|_{q=p}[u_1 \otimes \cdots \otimes u_s] = \begin{cases} \sum_{\sigma \in S_l} u_{\sigma(1)} \otimes \cdots \otimes u_{\sigma(l)}, \quad &s=l, \\ 0, &\text{otherwise.} \end{cases}
    \end{equation}
    Letting
    \begin{equation}
        F^l_g(q) = \frac{1}{l!} g^{l+1}(p)\left[P^f_p(q) \otimes (q-p)^{\otimes l}\right],
    \end{equation}
    and
    \begin{equation}
        \bigotimes_{\alpha=1}^mu_{\alpha} = u_1 \otimes \cdots \otimes u_{m},
    \end{equation}
    then for $l\leq m \leq k-1$, we first apply the product rule to obtain
    \begin{equation}
        \begin{aligned}
        &D^m_qF^l_g(q) \bigg|_{q=p}\left[\bigotimes_{\alpha=1}^mu_{\alpha}\right] \\
        &= \frac{1}{l!}\sum_{i=0}^m \frac{1}{i!(m-i)!} \sum_{\sigma \in S_m} g^{l+1}(p)\left(D_q^iP^f_p(q)\left[\bigotimes_{\alpha=1}^i u_{\sigma(\alpha)}\right] \otimes D_q^{m-i}(q-p)^{\otimes l}\left[\bigotimes_{\alpha=i+1}^m u_{\sigma(\alpha)}\right]\right).
        \end{aligned}
    \end{equation}
    At $q=p$, the term $D_q^{m-i}(q-p)^{\otimes l}$ vanishes unless $m-i=l$. 
    Therefore only the term $i=m-l$ remains, and so
    \begin{equation}
        \begin{aligned}
        &D^m_qF^l_g(q) \bigg|_{q=p}\left[\bigotimes_{\alpha=1}^mu_{\alpha}\right] \\
        &= \frac{1}{(m-l)!(l!)^2} \sum_{\sigma \in S_m} g^{l+1}(p)\left(D_q^{m-l}P^f_p(q)\left[\bigotimes_{\alpha=1}^{m-l} u_{\sigma(\alpha)}\right] \otimes D_q^{l}(q-p)^{\otimes l}\left[\bigotimes_{\alpha=m-l+1}^{m} u_{\sigma(\alpha)}\right] \right).
        \end{aligned}
    \end{equation}
    We now evaluate both derivatives at $q=p$. 
    Using $D_q^{m-l}P^f_p(q)|_{q=p}=f^{m-l}(p)$ and the formula above for $D_q^l(q-p)^{\otimes l}|_{q=p}$, we obtain
    \begin{equation}
        \begin{aligned}
            &D^m_qF^l_g(q) \bigg|_{q=p}\left[\bigotimes_{\alpha=1}^mu_{\alpha}\right] \\
            &= \frac{1}{(m-l)!(l!)^2} \sum_{\sigma \in S_m}\sum_{\tau\in S_l}
            g^{l+1}(p)\left(
            f^{m-l}(p)\left[\bigotimes_{\alpha=1}^{m-l} u_{\sigma(\alpha)}\right]
            \otimes
            \bigotimes_{\beta=1}^{l} u_{\sigma(m-l+\tau(\beta))}
            \right).
        \end{aligned}
    \end{equation}
    Finally, the sum over $\tau\in S_l$ contributes a factor of $l!$, and hence
    \begin{equation}
        \begin{aligned}
        &D^m_qF^l_g(q) \bigg|_{q=p}\left[\bigotimes_{\alpha=1}^mu_{\alpha}\right] \\
        &= \frac{1}{(m-l)!l!} \sum_{\sigma \in S_m} g^{l+1}(p)\left(f^{m-l}(p)\left[\bigotimes_{\alpha=1}^{m-l} u_{\sigma(\alpha)}\right]\otimes \bigotimes_{\alpha=m-l+1}^{m} u_{\sigma(\alpha)}\right).
        \end{aligned}
    \end{equation}
    Similarly, for
    \begin{equation}
        F^l_f(q) = \frac{1}{l!} f^{l+1}(p)\left[P^g_p(q) \otimes (q-p)^{\otimes l}\right],
    \end{equation}
    the same argument gives
    \begin{equation}
        \begin{aligned}
        D^m_q&F^l_f(q) \bigg|_{q=p}\left[\bigotimes_{\alpha=1}^mu_{\alpha}\right]\\
        &= \frac{1}{(m-l)!l!} \sum_{\sigma \in S_m} f^{l+1}(p)\left(g^{m-l}(p)\left[\bigotimes_{\alpha=1}^{m-l} u_{\sigma(\alpha)}\right] \otimes \bigotimes_{\alpha=m-l+1}^{m} u_{\sigma(\alpha)}\right).
        \end{aligned}
    \end{equation}
    Since $D^m_qF^l_g(q)\bigg|_{q=p}=0$ and $D^m_qF^l_f(q)\bigg|_{q=p}=0$ for $l>m$,
    \begin{equation}
        D^m_q[P^f_p,P^g_p](q)\bigg|_{q=p} = \sum_{l=0}^{m}\left(D^m_qF^l_g(q)\bigg|_{q=p} - D^m_qF^l_f(q)\bigg|_{q=p}\right),
    \end{equation}
    for $0\leq m \leq k-1$. 
    Therefore,
    \begin{equation}
        D^m_q[P^f_p,P^g_p](q) \bigg|_{q=p} = [f, g]^m(p),
    \end{equation}
    for $0\leq m \leq k-1$, where $[f, g]^m(p)$ is obtained by combining Definition \ref{def:comp_lipgamma} and Definition \ref{def:lipgamma_liebracket}. 
    By the definition of $P_p^{[f,g]}$, we also have
    \begin{equation}
        D_q^mP_p^{[f,g]}(q)\bigg|_{q=p} = [f,g]^m(p),
    \end{equation}
    for $0\leq m\leq k-1$. 
    Hence $P_p^{[f,g]}$ and $[P_p^f,P_p^g]$ agree up to the $(k-1)^{\text{th}}$ term.
\end{proof}

Trivially, both the $\mathrm{Lip}(\gamma)$ Lie bracket and the $\mathrm{Lip}(\gamma)$ polynomial Lie bracket satisfy bilinearity and anti-symmetry.

\begin{theorem}\label{thm:Lipgamma_liebracket_Jacobi}
     Let $\gamma > 2$. 
     Both the $\mathrm{Lip}(\gamma)$ Lie bracket and the $\mathrm{Lip}(\gamma)$ polynomial Lie bracket satisfy the Jacobi identity in $\mathrm{Lip}(\gamma-2)$.
\end{theorem}

\begin{proof}
    Let $U$ be a Banach space, $E\subset U$, $\gamma>2$, $f,g,h\in\mathrm{Lip}(\gamma, E, U)$, $k$ be the non-negative integer such that $\gamma\in(k,k+1]$, and for $p\in E$, let $P^f_p, P^g_p, P^h_p \in \mathbf{P}^k(U, U)$ be the polynomial corresponding to $f$, $g$, and $h$ at point $p$, respectively. 
    As noted in Remark \ref{rem:poly_lie_algebra}, Polynomial vector fields form a Lie subalgebra. 
    Therefore for all $p\in E$,
    \begin{equation} \label{eq:Lipgamma_polyliebracket}
        \left[P^f_p, [P^g_p, P^h_p]\right](q) + \left[P^g_p, [P^h_p, P^f_p]\right](q) + \left[P^h_p, [P^f_p, P^g_p]\right](q) = 0
    \end{equation}
    and the polynomial Lie bracket satisfies the Jacobi identity. 
    Considering $[g,h]$, Lemma \ref{lem:Lipgamma_liebracket_agreement} means that $P^{[g,h]}_p$ and $[P^g_p, P^h_p]$ are identical in the first $k-1$ terms for all $p$. 
    Since the first $k-2$ terms of $\left[P^f_p, [P^g_p, P^h_p]\right]$ only depend on the first $k-1$ terms of $[P^g_p, P^h_p]$, Lemma \ref{lem:Lipgamma_liebracket_agreement} also implies that $P^{[f, [g,h]]}_p$ and $\left[P^f_p, [P^g_p, P^h_p]\right]$ are identical in the first $k-2$ terms for all $p$. 
    Therefore, \eqref{eq:Lipgamma_polyliebracket} implies that the first $k-2$ terms of 
    \begin{equation}
        P^{[f, [g,h]]}_p(q) + P^{[g, [h,f]]}_p(q) + P^{[h, [f,g]]}_p(q)
    \end{equation}
    are equal to $0$ for all $p$. 
    Furthermore, Lemma \ref{lem:lipgamma_liebracket_deriv} implies that $[f, [g, h]]\in\mathrm{Lip}(\gamma-2, E, U)$ and similarly for the other permutations. 
    Therefore, 
    \begin{equation}
        [f, [g, h]] + [g, [h, f]] + [h, [f, g]] = 0
    \end{equation}
    in $\mathrm{Lip}(\gamma-2, E, U)$ and the $\mathrm{Lip}(\gamma)$ Lie bracket satisfies the Jacobi identity. 
\end{proof}

\subsection{The Extension Theorem}

\label{sec:extension}

Whitney proved necessary and sufficient conditions for derivative data $(f^0,\ldots,f^k)$ on a closed set $E\subset\mathbb{R}^n$ to admit a $C^k$ extension to $\mathbb{R}^n$ \citep{Whitney1934analytic}. 
Stein later recast the problem using $\mathrm{Lip}(\gamma)$ functions and constructed a bounded linear extension operator with norm controlled independently of $E$ \citep{stein1970singular}. 
We refer to the general result as the Stein–Whitney extension theorem.

\begin{theorem} [Stein-Whitney Extension Theorem \parencite{Whitney1934analytic, stein1970singular}]
\label{thm:stein-whitney}
    Let $U$ and $V$ be Banach spaces, $U$ be finite dimensional, $E\subset U$ be closed, and $f\in\mathrm{Lip}(\gamma, E, V)$. 
    Then there exists a function $\hat{f}\in\mathrm{Lip}(\gamma, U, V)$ such that
    \begin{equation}
        \hat{f}(p) = f(p), \quad p\in E,
    \end{equation}
    and 
    \begin{equation}
        ||\hat{f}||_{\mathrm{Lip}(\gamma, U, V)} \leq C ||f||_{\mathrm{Lip}(\gamma, E, V)},
    \end{equation}
    for some constant $C$ independent of $f$ and $E$.
\end{theorem}
\begin{proof}
    Stein proves this theorem by construction for the case $U=\mathbb{R}^n$ and $V=\mathbb{R}$ \citep{stein1970singular}. 
    Given the equivalence of Banach norms on $\mathbb{R}^n$, $U$ can be replaced by any finite dimensional Banach space.
    Furthermore, as noted in \citep[Appendix B]{baldi2018time}, Stein's proof holds exactly for any Banach space $V$ and $\mathrm{Lip}(\gamma)$ functions as given in Definition \ref{def:lipgamma}.
\end{proof}

Despite $\mathrm{Lip}(\gamma)$ functions being well defined for infinite dimensional $U$, Theorem \ref{thm:stein-whitney} is in general not true, as shown by the counter example of \textcite{wells1973differentiable}. 
Fefferman has extended the work of Whitney by proving necessary and sufficient conditions for the existence of a $C^k$ extension to a function with only its value specified on $E$ \citep{fefferman2006whitney}.

\section{Composition of $\mathrm{Lip}(\gamma)$ Functions}

\subsection{Introduction}

In Chapter \ref{chap:ncde}, the Log-ODE method will be applied to a CDE where the vector field is parametrised by a neural network. 
A non-trivial application of the Log-ODE method requires the neural network vector field to be $\mathrm{Lip}(\gamma)$ for $\gamma>1$. 
A common strategy for establishing regularity of neural networks is to verify that each layer of the network satisfies an appropriate regularity condition and then invoke a composition theorem.
This section develops the mathematical foundations required to apply that argument for $\mathrm{Lip}(\gamma)$ functions.

\begin{definition}[Partition of a Set]
    Let $\mathcal{P}(n)$ be the set of all partitions $\pi = (\pi_1, \ldots, \pi_{|\pi|})$ of the set $\{1, \ldots, n\}$, where $|\pi|$ denotes the number of parts in the partition $\pi$, and $|\pi_i|$ the number of elements in the $i^{\text{th}}$ part of the partition.
\end{definition}

\begin{definition} [Composition of $\mathrm{Lip}(\gamma)$ Functions \citep{cass2012new}]
\label{def:comp}
    Let $U$, $V$, and $W$ be Banach spaces, $E\subset U$ and $F\subset V$. 
    The composition of $f\in\mathrm{Lip}(\gamma, E, F)$ and $g\in\mathrm{Lip}(\gamma,F,W)$, denoted $h = (h^0, h^1, \ldots, h^k) \in \mathrm{Lip}(\gamma, E, W)$, is defined by
    \begin{equation}
        \begin{aligned}
            h^0(p) &= g^0\big( f^0(p) \big), \\
            h^n(p)[u_1\otimes \cdots \otimes u_n] &= \sum_{\pi \in \mathcal{P}(n)} g^{|\pi|} \big( f^0(p) \big) \left[ f^{|\pi_1|}(p)[u_{\pi_1}] \otimes \cdots \otimes f^{|\pi_{|\pi|}|}(p)[u_{\pi_{|\pi|}}] \right],
    \end{aligned}
    \end{equation}
    where $u_{\pi_i}$ is the tensor product of the vectors indexed by $\pi_i$,
    \begin{equation}
        u_{\pi_i} = \bigotimes_{j\in\pi_i} u_j.
    \end{equation}
\end{definition}

When $\gamma<1$, the composition $h$ in Definition \ref{def:comp} is not necessarily a $\mathrm{Lip}(\gamma)$ function. For example, $f(x)=x^{\gamma}$ and $g(x)=x^{\gamma}$ are both $\gamma-$H\"older continuous on $[0,1]$, whereas $h(x)=(g\circ f)(x)=x^{\gamma^2}$ is not. When $\gamma=1$, then $h\in\mathrm{Lip}(1,E,W)$ and we have the standard Lipschitz norm bound
\begin{equation}
    ||h||_{\mathrm{Lip}(1,E,W)} \leq ||g||_{\mathrm{Lip}(1,F,W)}\max\left(||f||_{\mathrm{Lip}(1,E,F)},1\right).
\end{equation}

\begin{lemma}[Composed $\mathrm{Lip}(\gamma)-$norm \citep{cass2012new, boutaib2016lipschitz}]
\label{lem:comp}
    For $\gamma\geq1$, the composition $h$ defined in Definition \ref{def:comp} satisfies
    \begin{equation}
    \label{eq:normcomp}
        ||h||_{\mathrm{Lip}(\gamma,E,W)} \leq C_{\gamma}||g||_{\mathrm{Lip}(\gamma,F,W)}\max\left(||f||^{\gamma}_{\mathrm{Lip}(\gamma,E,F)},1\right),
    \end{equation}
    where $C_{\gamma}$ is a constant independent of $f$ and $g$.
\end{lemma}
\begin{proof}
    Explicit calculation can be used to verify that if $f$ and $g$ are $\mathrm{Lip}(\gamma)$, Definition \ref{def:lipgamma} implies $h$ is $\mathrm{Lip}(\gamma)$ with $||h||_{\mathrm{Lip}(\gamma)}$ obeying \eqref{eq:normcomp} \citep{cass2012new}.
\end{proof}
The original statement of Lemma \ref{lem:comp} in \citep{cass2012new} gives \eqref{eq:normcomp} as
\begin{equation}
    \label{eq:compboundincorrect}
        ||g\circ f||_{\mathrm{Lip}(\gamma)} \leq C_{\gamma}||g||_{\mathrm{Lip}(\gamma)}\max\left\{||f||^{\textcolor{red}{k}}_{\mathrm{Lip}(\gamma)},1\right\}.
\end{equation}
We believe this is a small erratum, with the intended power being $\gamma$, as for $g:[0,1]\rightarrow[0,1]$ defined by $g(x)=x$, \eqref{eq:compboundincorrect} implies there exists $C_1>0$ such that
\begin{equation}
||g\circ f||_{\mathrm{Lip}(1)} = ||f||_{\mathrm{Lip}(1)} \leq C_1||g||_{\mathrm{Lip}(1)}=C_1
\end{equation}
for all bounded and Lipschitz $f:[0,1]\rightarrow[0,1]$. As a counterexample, for any $C_1>0$, take $f(x)=x^{n}$ with $n>\max\{C_1, 1\}$. 
\citet{boutaib2016lipschitz} gives an alternative proof of Lemma \ref{lem:comp} with the bound $||f||^{\gamma}_{\mathrm{Lip}(\gamma,E,F)}$.
\\ \\
Accurately bounding the $\mathrm{Lip}(\gamma)-$norm of a neural network requires bounds on $C_{\gamma}$ in (\ref{eq:normcomp}). 
This can be obtained via the explicit calculations mentioned in the proof of Lemma \ref{lem:comp}, and these have been completed for the case $\gamma \in (1,2]$.

\subsection{The Case $1<\gamma\leq 2$}
\label{sec:1_2_case}

For $1<\gamma\leq 2$, a $\mathrm{Lip}(\gamma)$ function has only two components, which makes it feasible to track the relevant quantities explicitly and obtain a concrete bound for $C_\gamma$.

\begin{lemma} 
    \label{lem:normcomplip2}
    Let $U$, $V$, and $W$ be Banach spaces and $E\subset U$ and $F \subset V$. For $\gamma \in (1,2]$, let $f = ( f^{0} , f^{1} ) \in \mathrm{Lip}(\gamma,E,F)$ and $g = (g^{0} , g^{1} ) \in \mathrm{Lip}(\gamma,F,W)$. 
    Consider $h^{0} : E \to W$ and $h^{1} : E \to \mathbf{L}(U,W)$ defined for $p \in E$ and $u \in U$ by 
    \begin{equation}
    \label{eq:hdef}
    		h^{0}(p) := g^{0} \left( f^{0}(p) \right)
    		\qquad \text{and} \qquad 
    		h^{1}(p)[u] := g^{1}\left( f^{0}(p) \right) \left[ f^{1}(p)[u] \right].
    \end{equation}
    Then $h := \left( h^{0} , h^{1} \right) \in \mathrm{Lip}(\gamma,E,W)$ and 
    \begin{equation}
    	\label{eq:normcomplip2}
    		|| h ||_{\mathrm{Lip}(\gamma,E,W)} 
    		\leq 
    		\left( 1 + 2^{\gamma} \right) || g ||_{\mathrm{Lip}(\gamma,F,W)} \max \left\{ 1 , || f ||_{\mathrm{Lip}(\gamma,E,F)}^{\gamma} \right\}. 
    \end{equation}
\end{lemma}
\begin{proof}
To prove that $h=(h^0,h^1)\in \mathrm{Lip}(\gamma,E,W)$, we must bound the pointwise terms $h^0$ and $h^1$, together with the corresponding remainder terms from Definition \ref{def:lipgamma}. We proceed in four steps. First, we record the bounds for $f$ and $g$ that follow directly from Definition \ref{def:lipgamma}. Next, we prove the pointwise bounds for $h^0$ and $h^1$. We then estimate the remainder terms $R_0^h$ and $R_1^h$ separately in the cases $\|q-p\|_U>1$ and $\|q-p\|_U\leq 1$. Finally, we combine the estimates to obtain the norm bound claimed in \eqref{eq:normcomplip2}.
\\ \\
We begin by establishing the bounds on $f^{0} : E \to F$, $f^{1} : E \to \mathbf{L}(U,V)$, $g^{0} : F \to W$ and $g^{1} : F \to \mathbf{L}(V,W)$ that arise directly from Definition \ref{def:lipgamma}.  
Letting $M_f=||f||_{\mathrm{Lip}(\gamma,E,F)}$, for all $p\in E$
\begin{equation}
	\label{f_pointwise_bounds}
		(\bI) \quad \left|\left| f^{0}(p) \right|\right|_V \leq M_f
		\qquad \text{and} \quad 
		(\bII) \quad \left|\left| f^{1}(p) \right|\right|_{\mathbf{L}(U,V)} \leq M_f.
\end{equation}
Similarly, letting $M_g=||g||_{\mathrm{Lip}(\gamma,F,W)}$, for all $x \in F$ we have that
\begin{equation}
	\label{g_pointwise_bounds}
		(\bI) \quad \left|\left| g^{0}(x) \right|\right|_W \leq M_g
		\qquad \text{and} \quad 
		(\bII) \quad \left|\left| g^{1}(x) \right|\right|_{\mathbf{L}(V,W)} \leq M_g.
\end{equation}
Define $R^f_0 : E \times E \to V$ and $R^f_1 : E \times E \to \mathbf{L}(U,V)$ by
\begin{equation}
	\label{f_remainder_term_defs}
    \begin{aligned}
		R^f_0 (p,q) &:= f^{0}(q) - f^{0}(p) - f^{1}(p)[q-p], \\
		R^f_1(p,q)[u] &:= f^{1}(q)[u] - f^{1}(p)[u],
  \end{aligned}
\end{equation}
for any $p,q \in E$ and $u \in U$. 
Then
\begin{equation}
	\label{f_remainder_term_bounds}
		\begin{aligned}
			&	(\bI) \quad \left|\left| R^f_0(p,q) \right|\right|_V \leq M_f ||q-p||_U^{\gamma}, \\
			&	(\bII) \quad \left|\left| R^f_1(p,q) \right|\right|_{\mathbf{L}(U,V)} \leq M_f ||q-p||_U^{\gamma - 1}.
		\end{aligned}
\end{equation}
Similarly, define
$R^g_0 : F \times F \to W$ and $R^g_1 : F \times F \to \mathbf{L}(V,W)$ by
\begin{equation}
	\label{g_remainder_term_defs}
    \begin{aligned}
		R^g_0 (x,y) &:= g^{0}(y) - g^{0}(x) - g^{1}(x)[y-x], \\
		R^g_1(x,y)[v] &:= g^{1}(y)[v] - g^{1}(x)[v],
  \end{aligned}
\end{equation}
for $x,y \in F$ and $v \in V$. 
Then,
\begin{equation}
	\label{g_remainder_term_bounds}
		\begin{aligned}
			&	(\bI) \quad \left|\left| R^g_0(x,y) \right|\right|_W \leq M_g ||y-x||_V^{\gamma}, \\
			&	(\bII) \quad \left|\left| R^g_1(x,y) \right|\right|_{\mathbf{L}(V,W)} \leq M_g ||y-x||_V^{\gamma - 1}.
		\end{aligned}
\end{equation}
\\ \\
Now define $h^{0} : E \to W$ and $h^{1} :E \to \mathbf{L}(U,W)$ as in \eqref{eq:hdef},
\begin{equation}
	\label{lip_gamma_chain_rule_h_def_proof}
		h^{0}(p) := g^{0} \left( f^{0}(p) \right)
		\qquad \text{and} \qquad 
		h^{1}(p)[u] := g^{1}\left( f^{0}(p) \right) \left[ f^{1}(p)[u] \right],
\end{equation}
for $p \in E$ and $u \in U$. 
Then define the corresponding remainder terms $R^h_0 : E \times E \to W$ and $R^h_1 : E \times E \to \mathbf{L}(U,W)$ by 
\begin{equation}
	\label{h_remain_terms_def}
 \begin{aligned}
		R^h_0(p,q) &:= h^{0}(q) - h^{0}(p) - h^{1}(p)[q-p], \\
		R^h_1(p,q)[u] &:= h^{1}(q)[u] - h^{1}(p)[u],
  \end{aligned}
\end{equation}
for $p,q \in E$ and $u \in U$. 
We now establish that $h = ( h^{0} , h^{1} ) \in \mathrm{Lip}(\gamma,E,W)$ and that the norm estimate claimed in \eqref{eq:normcomplip2} is satisfied.
\\ \\
First we consider the bounds on $h^{0}$ and $h^{1}$. 
For any $p \in E$,   
(\bI) in \eqref{g_pointwise_bounds} implies that
\begin{equation}
	\label{h0_bd}
		\left|\left| h^{0}(p) \right|\right|_W 
		= 
		\left|\left| g^{0} \left( f^{0}(p) \right)  \right|\right|_W
		\leq
		M_g
\end{equation}
since $f^{0}(p) \in F$. 
Further, for any $p \in E$ and any $u \in U$, \eqref{g_pointwise_bounds} and (\bII) in \eqref{f_pointwise_bounds} imply that
\begin{align*}
	\left|\left| h^{1}(p)[u] \right|\right|_W &= 
		\left|\left| g^{1} \left( f^{0}(p) \right) \left[ f^{1}(p)[u] \right] \right|\right|_W \\
		&\leq \left|\left| g^{1} \left( f^{0}(p) \right) \right|\right|_{\mathbf{L}(V,W)} \left|\left| f^{1}(p) \right|\right|_{\mathbf{L}(U,V)} ||u||_U \\
		&\leq
		M_gM_f ||u||_U
\end{align*}
since $f^{0}(p) \in F$. 
Taking the supremum over $u \in U$ with unit $U$-norm, it follows that
\begin{equation}
	\label{h1_bd}
		\left| \left| h^{1}(p) \right|\right|_{\mathbf{L}(U,W)} \leq || g ||_{\mathrm{Lip}(\gamma,F,W)} || f ||_{\mathrm{Lip}(\gamma,E,F)}.
\end{equation}
\\ \\
Now we consider the bounds on $R^h_0$ and $R^h_1$. 
For this purpose we fix $p,q \in E$ and $u\in U$.
We first assume that $||q - p||_U > 1$. 
In this case we may use \eqref{h0_bd} and \eqref{h1_bd} to compute that
\begin{align*}
	\left|\left| R^h_0(p,q) \right|\right|_W &= \left| \left| h^{0}(q) - h^{0}(p) - h^{1}(p)[q-p] \right|\right|_W \\
		&\leq 
		2 || g ||_{\mathrm{Lip}(\gamma,F,W)} + || g ||_{\mathrm{Lip}(\gamma,F,W)} || f ||_{\mathrm{Lip}(\gamma,E,F)} || q - p ||_U.
\end{align*}
Since $\gamma > 1$ means that $1 < || q - p ||_U < || q  - p ||_U^{\gamma}$, we deduce that
\begin{equation}
	\label{Rh0_bd_a}
		\left| \left| R^h_0(p,q) \right|\right|_W \leq 
		|| g ||_{\mathrm{Lip}(\gamma,F,W)} \left( 2 + || f ||_{\mathrm{Lip}(\gamma,E,F)} \right) || q - p ||_U^{\gamma}.
\end{equation}
Similarly, we may use \eqref{h1_bd} and $ ||q-p||_U^{\gamma - 1} > 1$ to compute that
\begin{equation}
	\label{Rh1_bd_a_with_v}
	\left|\left| R^h_1(p,q)[u] \right|\right|_W = \left|\left| h^{1}(q)[u] - h^{1}(p)[u] \right|\right|_W 
		\leq
		2 || g ||_{\mathrm{Lip}(\gamma,F,W)} || f ||_{\mathrm{Lip}(\gamma,E,F)} || q - p ||_U^{\gamma - 1} ||u||_U.
\end{equation}
Taking the supremum over $u \in U$ with unit $U$-norm in \eqref{Rh1_bd_a_with_v} yields the estimate that
\begin{equation}
	\label{Rh1_bd_a}
		\left|\left| R^h_1(p,q) \right|\right|_{\mathbf{L}(U,W)} 
		\leq
		2 || g ||_{\mathrm{Lip}(\gamma,F,W)} || f ||_{\mathrm{Lip}(\gamma,E,F)} || q - p ||_U^{\gamma - 1}.
\end{equation}
Together, \eqref{Rh0_bd_a} and \eqref{Rh1_bd_a} establish the remainder term estimates required to conclude that $h = (h^{0} , h^{1} ) \in \mathrm{Lip}(\gamma,E,W)$
in the case that $|| q - p ||_U > 1$. 
\\ \\
We next establish similar remainder term estimates when $|| q - p ||_U < 1$. 
Thus we fix $p,q \in E$ and assume that $||q-p||_U < 1$.
Note that $\gamma > 1$ means that $||q - p||_U^{\gamma} < ||q-p||_U < 1$. 
Additionally,
\begin{equation}
\label{f0_diff_bd}
    \begin{aligned}
        \left|\left| f^{0}(q) - f^{0}(p) \right|\right|_V &\stackrel{(\ref{f_remainder_term_defs})}{=} \left|\left|f^{1}(p)[q-p] + R^f_0(p,q) \right|\right|_V, \\
        &\;\;\leq || f ||_{\mathrm{Lip}(\gamma,E,F)} \left( ||q-p||_U + ||q-p||_U^{\gamma} \right), \\
		&\;\;\leq 2 || f ||_{\mathrm{Lip}(\gamma,E,F)} ||q-p||_U,
    \end{aligned}
\end{equation}
where (\bII) in \eqref{f_pointwise_bounds} and (\bI) in \eqref{f_remainder_term_bounds} have been used. 
We now consider the term $R^h_0(p,q)$.
We start by observing that 
\begin{align*}
	R^h_0(p,q) &\stackrel{(\ref{h_remain_terms_def})}{=} h^{0}(q) - h^{0}(p) - h^{1}(p)[q-p] \\
		&\stackrel{(\ref{lip_gamma_chain_rule_h_def_proof})}{=} 
		g^{0}\left( f^{0}(q) \right) - g^{0} \left( f^{0}(p) \right) - g^{1}\left(f^{0}(p) \right) \left[ f^{1}(p)[q-p] \right]  \\
		&\stackrel{(\ref{g_remainder_term_defs})}{=}
		g^{1}\left( f^{0}(p) \right) \left[ f^{0}(q) - f^{0}(p) -  f^{1}(p)[q-p] \right] + R^g_0\left( f^{0}(p) , f^{0}(q) \right) \\
		&\stackrel{(\ref{f_remainder_term_defs})}{=}
		g^{1}\left( f^{0}(p) \right) \left[ R^f_0(p,q) \right] + R^g_0\left( f^{0}(p) , f^{0}(q) \right).
\end{align*}
Consequently, by using (\bII) in \eqref{g_pointwise_bounds} to estimate the term $g^{1}\left( f^{0}(p) \right)$, 
(\bI) in \eqref{f_remainder_term_bounds} to estimate the term $R^f_0(p,q)$, and 
(\bI) in \eqref{g_remainder_term_bounds} to estimate the term $R^g_0\left( f^{0}(p) , f^{0}(q) \right)$, 
we may deduce that
\begin{equation}
	\label{Rh0_bd_b_almost}
		\left|\left| R^h_0(p,q) \right|\right|_W \leq 
		|| g ||_{\mathrm{Lip}(\gamma,F,W)} \left( || f ||_{\mathrm{Lip}(\gamma,E,F)} || q - p ||_U^{\gamma} + 
		\left|\left| f^{0}(q) - f^{0}(p) \right|\right|_V^{\gamma} \right).
\end{equation}
The combination of \eqref{f0_diff_bd} and \eqref{Rh0_bd_b_almost} yields the estimate 
\begin{equation}
	\label{Rh0_bd_b}
		\left|\left| R^h_0(p,q) \right|\right|_W \leq 
		|| g ||_{\mathrm{Lip}(\gamma,F,W)} \left( || f ||_{\mathrm{Lip}(\gamma,E,F)}  + 
		2^{\gamma} ||f||_{\mathrm{Lip}(\gamma,E,F)}^{\gamma} \right) || q - p ||_U^{\gamma}.
\end{equation}
Turning our attention to $R^h_1$, we fix $u \in U$ and compute that
\begin{align*}
	R^h_1(p,q)[u] &\stackrel{(\ref{h_remain_terms_def})}{=} h^{1}(q)[u] - h^{1}(p)[u]  \\
		&\stackrel{(\ref{lip_gamma_chain_rule_h_def_proof})}{=}  
		g^{1}\left( f^{0}(q) \right) \left[ f^{1}(q)[u] \right] - g^{1}\left( f^{0}(p) \right) \left[ f^{1}(p)[u] \right] \\
		&\stackrel{(\ref{g_remainder_term_defs})}{=} 
		g^{1} \left( f^{0}(p) \right) \left[ f^{1}(q)[u] - f^{1}(p)[u] \right] + R^g_1 \left( f^{0}(p) , f^{0}(q) \right) \left[ f^{1}(q)[u] \right]  \\
		&\stackrel{(\ref{f_remainder_term_defs})}{=} 
		g^{1} \left( f^{0}(p) \right) \left[ R^f_1(p,q)[u] \right] + R^g_1 \left( f^{0}(p) , f^{0}(q) \right) \left[ f^{1}(q)[u] \right].
\end{align*}
Consequently, by using (\bII) in \eqref{g_pointwise_bounds} to estimate the term $g^{1} \left( f^{0}(p) \right)$,
(\bII) in \eqref{f_pointwise_bounds} to estimate the term $f^{1}(q)$,
(\bII) in \eqref{f_remainder_term_bounds} to estimate the term $R^f_1(p,q)$,  and 
(\bII) in \eqref{g_remainder_term_bounds} to estimate the term $R^g_1 \left( f^{0}(p) , f^{0}(q) \right)$, 
we may deduce that
\begin{equation}
	\label{Rh1_bd_b_almost}
		\left|\left| R^h_1(p,q)[u] \right|\right|_W \leq 
		|| g ||_{\mathrm{Lip}(\gamma,F,W)} || f ||_{\mathrm{Lip}(\gamma,E,F)} \left( || q - p ||_U^{\gamma - 1}
		+ 
		\left|\left| f^{0}(q)  - f^{0}(p) \right|\right|_{V}^{\gamma - 1}  \right) ||u||_U.
\end{equation}
The combination of \eqref{f0_diff_bd} and \eqref{Rh1_bd_b_almost} yields the estimate that
\begin{equation}
	\label{Rh1_bd_b_almost2}
		\left|\left| R^h_1(p,q)[u] \right|\right|_W \leq 
		|| g ||_{\mathrm{Lip}(\gamma,F,W)} \left( || f ||_{\mathrm{Lip}(\gamma,E,F)} 
		+ 2^{\gamma - 1}||f||_{\mathrm{Lip}(\gamma,E,F)}^{\gamma}   \right) ||q - p||_U^{\gamma - 1} ||u||_U.
\end{equation}
Taking the supremum over $u \in U$ with unit $U$-norm in \eqref{Rh1_bd_b_almost2} yields the estimate that
\begin{equation}
	\label{Rh1_bd_b}
		\left|\left| R^h_1(p,q) \right|\right|_{\mathbf{L}(U,W)} \leq 
		|| g ||_{\mathrm{Lip}(\gamma,F,W)} \left( || f ||_{\mathrm{Lip}(\gamma,E,F)}  + 
		2^{\gamma - 1} ||f||_{\mathrm{Lip}(\gamma,E,F)}^{\gamma} \right) || q - p ||_U^{\gamma - 1}.
\end{equation}
Finally, we complete the proof by combining the various estimates we have established for $h$ to obtain the $\mathrm{Lip}(\gamma,E,W)$-norm bound 
claimed in \eqref{eq:normcomplip2}.
\\ \\
We start this task by combining \eqref{Rh0_bd_a} and \eqref{Rh0_bd_b} to deduce that for every $p,q \in E$ we have 
\begin{equation}
	\label{Rh0_bd_c}
    \begin{aligned}
		\left|\left| R^h_0(p,q) \right|\right|_W \leq 
		\twopartdef{|| g ||_{\mathrm{Lip}(\gamma,F,W)} \left( 2 + || f ||_{\mathrm{Lip}(\gamma,E,F)} \right) || q - p ||_U^{\gamma}}{||q-p||_U > 1}
		{|| g ||_{\mathrm{Lip}(\gamma,F,W)} \left( || f ||_{\mathrm{Lip}(\gamma,E,F)}  + 
		2^{\gamma} ||f||_{\mathrm{Lip}(\gamma,E,F)}^{\gamma} \right) || q - p ||_U^{\gamma}}{||q - p ||_U \leq 1.}
  \end{aligned}
\end{equation}
Moreover, the combination of \eqref{Rh1_bd_a} and \eqref{Rh1_bd_b} yields the estimate that
\begin{equation}
\small
	\label{Rh1_bd_c}
		\left|\left| R^h_1(p,q) \right|\right|_{\mathbf{L}(U,W)} \leq 
		\twopartdef{2 || g ||_{\mathrm{Lip}(\gamma,F,W)} || f ||_{\mathrm{Lip}(\gamma,E,F)} || q - p ||_U^{\gamma - 1}}{||q-p||_U > 1}
		{|| g ||_{\mathrm{Lip}(\gamma,F,W)} \left( 
        || f ||_{\mathrm{Lip}(\gamma,E,F)}  + 
		2^{\gamma - 1} ||f||_{\mathrm{Lip}(\gamma,E,F)}^{\gamma} \right) || q - p ||_U^{\gamma - 1}}{||q - p ||_U \leq 1.}
\end{equation}
A consequence of \eqref{Rh0_bd_c} is that 
\begin{equation}
	\label{Rh0_bd_d}
		\left|\left| R^h_0(p,q) \right|\right|_W \leq \left( 1 + 2^{\gamma} \right)|| g ||_{\mathrm{Lip}(\gamma,F,W)} 
		\max \left\{||f||_{\mathrm{Lip}(\gamma,E,F)}^{\gamma} , 1 \right\} ||q - p||_U^{\gamma},
\end{equation}
whilst a consequence of \eqref{Rh1_bd_c} is that 
\begin{equation}
	\label{Rh1_bd_d}
		\left|\left| R^h_1(p,q) \right|\right|_{\mathbf{L}(U,W)}\leq \left( 1 + 2^{\gamma - 1} \right) || g ||_{\mathrm{Lip}(\gamma,F,W)} 
		\max \left\{||f||_{\mathrm{Lip}(\gamma,E,F)}^{\gamma}, 1 \right\} ||q - p||_U^{\gamma - 1}.
\end{equation}
Therefore, by combining \eqref{h0_bd}, \eqref{h1_bd}, \eqref{Rh0_bd_d}, and \eqref{Rh1_bd_d}, we conclude both that $h = (h^{0} , h^{1} ) \in \mathrm{Lip}(\gamma,E,W)$ 
and that
\begin{equation}
	\label{h_lip_gamma_norm}
		|| h ||_{\mathrm{Lip}(\gamma,E,W)} \leq 
        \left( 1 + 2^{\gamma} \right) 
        || g ||_{\mathrm{Lip}(\gamma,F,W)} 
		\max \left\{ 
        ||f||^{\gamma}_{\mathrm{Lip}(\gamma,E,F)}, 1 \right\}.
\end{equation}
\end{proof}

\subsection{Optimality}

\label{sec:opt}

Handling the terms individually in the proof of Lemma \ref{lem:normcomplip2} means it is unlikely that $C_{\gamma}=1+2^{\gamma}$ is optimal. However, $|| h ||_{\mathrm{Lip}(\gamma,E,G)}$ being of order $||f||^{\gamma}_{\mathrm{Lip}(\gamma,E,F)}$ is optimal, as shown by the following example.
\\ \\
Take $E=\{0,a\}\subset\mathbb{R}$ with $a<1$, $F=\{0,1\}\subset\mathbb{R}$, and $f\in\mathrm{Lip}(\gamma, E, F)$ with $\gamma\in(1,2]$ defined by $f^0(0)=0$, $f^0(a)=1$, and $f^1(0)=f^1(a)=c$. Then,
\begin{equation}
    \|f\|_{\mathrm{Lip}(\gamma, E, F)} = \max\left\{1, |c|, \frac{|1-ca|}{a^{\gamma}}\right\},
\end{equation}
and $\|f\|_{\mathrm{Lip}(\gamma, E, F)}$ is minimised when
\begin{equation}
    c = c^* = \begin{cases}
        1, \quad &a+a^{\gamma} > 1, \\
        \frac{1}{a+a^{\gamma}}, \quad &a+a^{\gamma} \leq 1,
    \end{cases}
\end{equation}
with $\|f\|_{\mathrm{Lip}(\gamma, E, F)}=c^*$. Take  $G=\{-1,1\}\subset\mathbb{R}$ and $g\in\mathrm{Lip}(\gamma, F, G)$ defined by $g^0(0)=-1$, $g^0(1)=1$, and $g^1(0)=g^1(1)=1$. Then $\|g\|_{\mathrm{Lip}(\gamma, F, G)}=1$. Further, $h=g\circ f$ is defined by $h^0(0)=-1$, $h^0(a)=1$, and $h^1(0)=h^1(a)=c$, with
\begin{equation}
    \|h\|_{\mathrm{Lip}(\gamma, E, G)} = \max\left\{1, |c|, \frac{|2-ca|}{a^{\gamma}}\right\}.
\end{equation}
When $a+a^{\gamma} \leq 1$,
\begin{equation}
    \|h\|_{\mathrm{Lip}(\gamma, E, G)} = \frac{2-ca}{a^{\gamma}}=\frac{1}{a+a^{\gamma}}+\frac{1}{a^{\gamma}}.
\end{equation}
Then
\begin{equation}
    \begin{aligned}
    \lim_{a \rightarrow 0}\frac{\|h\|_{\mathrm{Lip}(\gamma, E, G)}}{\|f\|_{\mathrm{Lip}(\gamma, E, F)}^\beta} &= \lim_{a \rightarrow 0} (a+a^{\gamma})^{\beta-1}+\frac{(a+a^{\gamma})^{\beta}}{a^{\gamma}}, \\
    &= \lim_{a\rightarrow 0}a^{\beta-1}(1+a^{\gamma-1})^{\beta-1} + a^{\beta-\gamma}(1+a^{\gamma-1})^{\beta}, \\
    &= \lim_{a\rightarrow 0}a^{\beta-\gamma}(1+a^{\gamma-1})^{\beta-1}(2a^{\gamma-1} + 1), \\
    &= \begin{cases} \infty \quad &\beta < \gamma, \\
    1 \quad &\beta=\gamma, \\ 0 \quad &\beta>\gamma. \end{cases}
    \end{aligned}
\end{equation}
Therefore, $|| h ||_{\mathrm{Lip}(\gamma,E,G)}$ being of order $||f||^{\gamma}_{\mathrm{Lip}(\gamma,E,F)}$ is optimal.
\\ \\
The same example can be used to obtain a lower bound on $C_{\gamma}$. Letting $a^{\gamma}+a=1$, then $\|f\|_{\mathrm{Lip}(\gamma, E, F)}=1$ and $\|g\|_{\mathrm{Lip}(\gamma, F, G)}=1$, but
\begin{equation}
    \|h\|_{\mathrm{Lip}(\gamma, E, G)} = \frac{2-a}{a^{\gamma}}=1+a^{-\gamma}.
\end{equation}
Figure \ref{fig:lipgamma_bound} shows the range of $C_{\gamma}$ when $\gamma\in(1,2]$ given by this example and Lemma \ref{lem:normcomplip2}. This example implies there is a jump in $C_{\gamma}$ at $\gamma=1$, as standard Lipschitz composition gives $C_{1}=1$, whereas $C_{\gamma}>3$ for all $\gamma \in(1,2]$.

\begin{figure}
    \centering
    \includegraphics[width=\linewidth]{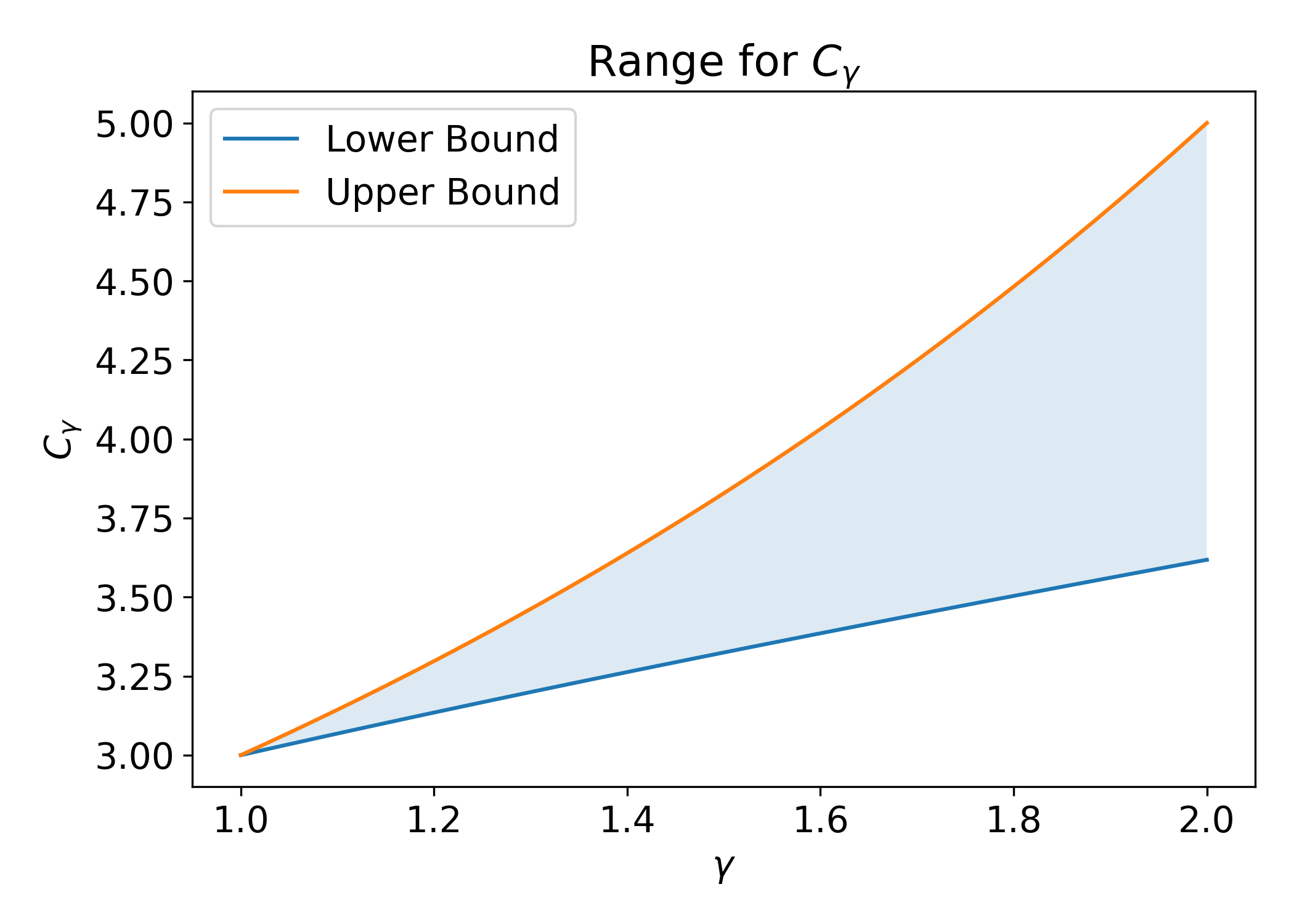}
    \caption{Let $U$, $V$, and $W$ be Banach spaces, with $E \subset U$ and $F \subset V$. If $f \in \mathrm{Lip}(\gamma,E,F)$ and $g \in \mathrm{Lip}(\gamma,F,W)$, then their composition $h=g\circ f$ satisfies $\| h \|_{\mathrm{Lip}(\gamma,E,W)} \leq C_{\gamma} \| g \|_{\mathrm{Lip}(\gamma,F,W)} \max \left\{ 1 , \| f \|_{\mathrm{Lip}(\gamma,E,F)}^{\gamma} \right\}$. This figure plots the bounds on $C_{\gamma}$ obtained by combining Lemma \ref{lem:normcomplip2} with the example of Section \ref{sec:opt}, namely $1+a(\gamma)^{-\gamma}\leq C_{\gamma}\leq 1+2^{\gamma}$, where $a(\gamma)$ is the unique solution to $a^{\gamma}+a=1$.}
    \label{fig:lipgamma_bound}
\end{figure}

\subsection{Future Work}

At present, extending Lemma \ref{lem:normcomplip2} by finding a bound on $C_{\gamma}$ for $\gamma>2$ remains open work. 
Proving such a bound would extend the explicit dimension-free control of composition to higher regularity classes. 
This would make it possible to give fully quantitative bounds for the $\mathrm{Lip}(\gamma)$-norms of neural architectures built by repeated composition, with constants depending explicitly on quantities such as depth and layer norms. 
Furthermore, such a result would support explicit error estimates for CDEs with neural network vector fields in higher-regularity regimes.
\\ \\
However, the proof of an explicit bound in the first non-trivial regime $\gamma\in(1,2]$ is already involved.
This motivates exploring alternative approaches for the case $\gamma>2$. 
In finite-dimensional $U$ and $V$ one could try to extend $f$ and $g$ to the whole spaces via the Stein--Whitney theorem (Theorem \ref{thm:stein-whitney}) and then invoke the equivalence of $\mathrm{Lip}(\gamma)$ and $C^{k,\alpha}_b$ on open convex sets (Lemma \ref{lem:CKalpha_Lipgamma}) to apply classical $C^{k,\alpha}_b$ compositional bounds. 
However, this approach produces a constant depending on the dimension of $U$ and $V$, as opposed to the dimension-free constant of Lemmas \ref{lem:comp} and \ref{lem:normcomplip2}. 
Another possible approach is to use the polynomial viewpoint of $\mathrm{Lip}(\gamma)$ functions from Definition \ref{def:lipgamma_poly} together with standard results on the composition of polynomial functions.

\section{Conclusion}

This chapter developed the regularity theory needed for the CDE framework used throughout the thesis. 
We related $\mathrm{Lip}(\gamma)$ regularity to $C^{k,\alpha}$ regularity, clarifying both where these notions agree and where $\mathrm{Lip}(\gamma)$ has stronger global properties. 
We then introduced the derivative and Lie bracket of $\mathrm{Lip}(\gamma)$ functions on arbitrary subsets of Banach spaces. 
Finally, we studied the composition of $\mathrm{Lip}(\gamma)$ functions and proved an explicit bound on the composition norm in the first non-trivial regime $1 < \gamma \leq 2$.
Together, these results provide the framework necessary for Chapter~\ref{chap:ncde} to apply the Log-ODE method to NCDEs, where the vector fields are parametrised by neural networks.
\\ \\
Chapters~\ref{chap:cde} and~\ref{chap:lipgamma} have provided the mathematical tools needed to develop the scalable and efficient continuous-time machine learning models that are the focus of the remainder of this thesis. 
This begins with Chapter~\ref{chap:ncde} examining how to solve NCDEs efficiently via the Log-ODE method. 
Chapter~\ref{chap:lin_ncde} then introduces Linear NCDEs, which enable parallel-in-time computation without sacrificing expressivity.
\chapter{Neural Controlled Differential Equations}

\label{chap:ncde}

\begin{quoting}
    Forty feet, down two and a half. Kicking up some dust. Thirty feet, two and a half down. Faint shadow.
\end{quoting}
\noindent\large{---Buzz Aldrin, \emph{Apollo 11 Air-to-Ground Voice Transcription} (1969)}
\normalsize

\section{Introduction}

Time series modelling is the development of mathematical, statistical, and computational methods for sequentially ordered data. 
It has played a central role across science and technology for nearly a century, from Yule’s 1927 autoregressive models for sunspot numbers, through Kalman filtering in the Apollo guidance system for estimating the Lunar Module’s altitude, to the autoregressive models that power today’s large-scale language models \citep{yule1927method, kalman1960, touvron2023llamaopenefficientfoundation}. 
\\ \\
Figure~\ref{fig:situation} illustrates the problem of interest for this chapter: given a sequence of irregularly spaced observations $\{X_{t_i}\}_{i=0}^{n} = \{(t_i, x_{t_i})\}_{i=0}^{n}$ from a multi-dimensional process, predict a corresponding output path $y_t$. 
This framework encompasses classification, where $y_t$ is a single label, regression, where $y_t$ varies over time, and autoregressive generation, where each new observation $X_{t_i}$ is produced from the previous output $y_{t_{i-1}}$.
Our focus is developing models that can generalise to unseen examples by leveraging large datasets of observation–output pairs. 
In particular, we explore Neural Controlled Differential Equations (NCDEs), which provide a continuous-time framework naturally suited to irregularly sampled data. 
The development of NCDEs builds on a rich history of time series modelling, from which we note a few of the important milestones.

\begin{figure}
\centering
\hspace{-1cm}\includegraphics[width=0.75\textwidth]{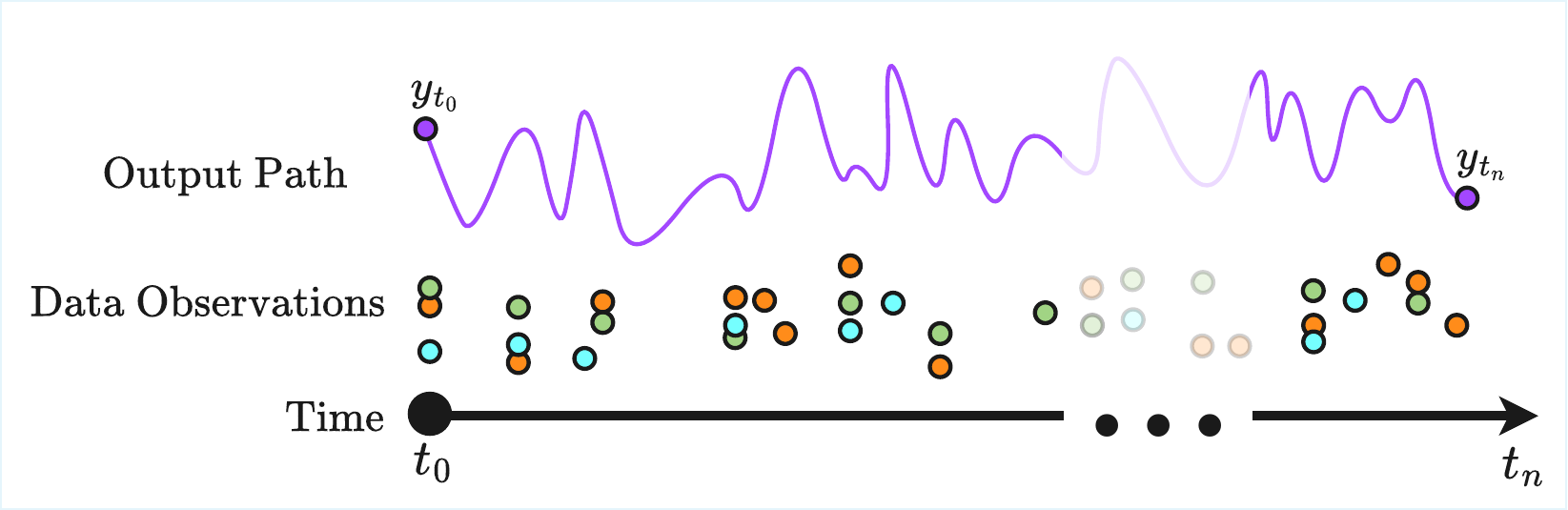}
\caption{A schematic diagram of observations from a three-dimensional, irregularly sampled time series and a corresponding output path one wishes to predict. The colour of each ball represents a different channel in the time series.}
\label{fig:situation}
\end{figure}

\section{Historical Milestones}

\subsection{Classical Approaches}

\label{sec:classical_approaches}

In 1927, Yule introduced autoregressive modelling, which approximates each term in a time series as a linear function of its previous values \citep{yule1927method}. 
The model is defined as
\begin{equation}
    \label{eq:ar_model}
    y_{t_i} = a_1 y_{t_{i-1}} + a_2 y_{t_{i-2}} + \dots + a_p y_{t_{i-p}} + \epsilon_{t_i},
\end{equation}
where $a_1, a_2, \dots, a_p$ are the autoregressive coefficients, $p$ is the model order, and $\epsilon_{t_i}$ are independent and identically distributed random variables with mean zero and constant variance. 
Building on this work, the Yule–Walker equations were derived as a method for estimating the parameters of the model, laying the groundwork for statistical time series analysis \citep{walker1931}. 
\\ \\
In parallel, Slutsky introduced moving average processes, which approximate each term in a time series as a linear function of previous noise values $\epsilon_{t_i}$ \citep{slutsky1927, slutsky1937}. 
The model is defined as
\begin{equation}
    y_{t_i} = \epsilon_{t_i} + b_1 \epsilon_{t_{i-1}} + b_2 \epsilon_{t_{i-2}} + \dots + b_q \epsilon_{t_{i-q}},
\end{equation}
where $b_1, b_2, \dots, b_q$ are the moving average coefficients. 
These two lines of development were unified in the autoregressive moving average model by Whittle in 1951, and later popularised as part of the Box--Jenkins framework \citep{whittle1951hypothesis,box1970time}. 
\\ \\
In 1965, building on Galtieri's 1964 work on estimation in discrete-time processes, Åström and Bohlin introduced a numerical method for identifying linear dynamical systems from observed input-output data \citep{galtieri1964,astrom1965numerical}. 
Their model extended the autoregressive moving average framework to include an observed input series $x_{t_i}$,
\begin{equation}
    \label{eq:armax}
    y_{t_i}
    =
    a_1 y_{t_{i-1}} + \dots + a_p y_{t_{i-p}}
    + \epsilon_{t_i} + b_1 \epsilon_{t_{i-1}} + \dots + b_q \epsilon_{t_{i-q}}
    + c_0 x_{t_i} + c_1 x_{t_{i-1}} + \dots + c_r x_{t_{i-r}},
\end{equation}
where $c_0, c_1, \dots, c_r$ are the input coefficients. 
The parameters are then estimated by maximum likelihood using a numerical optimisation procedure. 
This line of work was later consolidated under the name system identification \citep{astrom1971system}.
\\ \\
A separate line of development focused on estimating signals or hidden states from noisy observations once a model had been specified.
Early work in this direction includes Kolmogorov's 1941 treatment of interpolation and extrapolation for stationary random sequences \citep{kolmogorov1941, kolmogorov1962english}.
In 1949, Wiener developed the Wiener filter, a method for minimising the mean-squared error when estimating a signal from noisy observations \citep{wiener1949}. 
For a stationary process, the Wiener filter seeks to find an optimal linear filter $l(t)$ such that, when convolved with the observed signal $x(t)$, it produces an estimate $\hat{y}(t)$ of the desired signal $y(t)$,
\begin{equation}
\hat{y}(t) = \int_{-\infty}^{\infty} l(t - \tau) x(\tau) \, \mathrm{d}\tau.
\end{equation}
Wiener provided an explicit solution for the filter in the frequency domain. 
\\ \\
In 1960, Kalman formalised the state-space model framework for non-stationary processes and proposed the Kalman filter, an optimal recursive algorithm for estimating the hidden states of such processes \citep{kalman1960}. 
The state-space model is defined by
\begin{equation}
\begin{aligned}
\label{eq:sm}
h_{t_i} &= A h_{t_{i-1}} + B x_{t_i} + w_{t_i}, \\
y_{t_i} &= C h_{t_i} + v_{t_i},
\end{aligned}
\end{equation}
where $h_{t_i}$ is the hidden state, $A$ is the state transition matrix, $x_{t_i}$ are the inputs, $B$ represents the influence of the inputs, $y_{t_i}$ is the observed output, $C$ is the observation matrix, and $w_{t_i}$ and $v_{t_i}$ are zero-mean Gaussian noise processes with known covariance matrices. 
For linear Gaussian systems, the Kalman filter provides an optimal recursive solution for estimating the hidden states $h_{t_i}$.
Subsequent advancements extended the Kalman filter to accommodate non-linear dynamics and non-Gaussian noise, leading to algorithms such as the Extended Kalman Filter, Unscented Kalman Filter, and Particle Filters \citep{gelb1974applied, julier1997new, gordon1993novel}.
A further development was the extension of system identification methods from the autoregressive moving average models in \eqref{eq:armax} to the state-space setting of \eqref{eq:sm}.
In particular, Van Overschee and De Moor introduced N4SID in 1994, a subspace method that identifies a discrete-time state-space realisation from input-output data \citep{vanoverschee1994n4sid}.
\\ \\
In 1978, O'Hagan introduced a Bayesian non-parametric framework for curve fitting and prediction that laid the groundwork for what is now known as Gaussian process regression \citep{ohagan1978curve}. 
Given observations at times $T=(t_1,\dots,t_n)$, one assumes that
\begin{equation}
    y_{t_i} = \eta(t_i) + \epsilon_{t_i},
\end{equation}
where $\eta$ is an unknown regression function and $\epsilon_{t_i} \sim \mathcal{N}\!\bigl(0,\sigma^2(t_i)\bigr)$ independently. 
Rather than restricting $\eta$ to a finite-dimensional parametric family, the method assumes that for any finite collection of times $S=(s_1,\dots,s_m)$, the vector
\begin{equation}
    \bigl(\eta(s_1),\dots,\eta(s_m)\bigr)^\top
\end{equation}
is jointly Gaussian with mean
\begin{equation}
    m(S) = \bigl(m(s_1),\dots,m(s_m)\bigr)^\top
\end{equation}
and covariance
\begin{equation}
    K(S,S) =
    \begin{bmatrix}
        k(s_1,s_1) & \cdots & k(s_1,s_m) \\
        \vdots & \ddots & \vdots \\
        k(s_m,s_1) & \cdots & k(s_m,s_m)
    \end{bmatrix},
\end{equation}
where $m$ and $k$ specify the prior distribution of $\eta$. 
Let $T_*$ denote a collection of test times, and let $\Sigma(T)=\mathrm{diag}\!\bigl(\sigma^2(t_1),\dots,\sigma^2(t_n)\bigr)$.
Conditioning on the observed data yields the posterior distribution
\begin{equation}
\begin{aligned}
    \eta(T_*) \mid y \sim \mathcal{N}\!\Bigl(
    &m(T_*) + K(T_*,T)\bigl(K(T,T)+\Sigma(T)\bigr)^{-1}\bigl(y-m(T)\bigr), \\
    &K(T_*,T_*) - K(T_*,T)\bigl(K(T,T)+\Sigma(T)\bigr)^{-1}K(T,T_*)
    \Bigr).
\end{aligned}
\end{equation}
This yields both a prediction for $\eta$ at the test times and a corresponding posterior uncertainty estimate.
O'Hagan's original paper developed this framework in the context of Bayesian smoothing, curve fitting, and prediction, and it was later popularised in machine learning as a flexible kernel-based approach to non-linear regression \citep{williams1996gaussian,rasmussen2006gaussian}.

\subsection{Discrete Machine Learning Approaches}

In 1990, Elman introduced a recurrent neural network (RNN) architecture from which many subsequent RNNs would evolve \citep{elman1990finding}. 
The Elman network is defined as
\begin{align}
\label{eq:elman}
    h_{t_i} &= \sigma(W_h h_{t_{i-1}} + W_x x_{t_i} + b_h), \\
    y_{t_i} &= W_y h_{t_i} + b_y,
\end{align}
where $h_{t_i}$ is the hidden state at time $t_i$, $\sigma$ is an activation function, $W_h$, $W_x$, $W_y$ are learnable weight matrices, and $b_h$, $b_y$ are learnable bias vectors. 
The matrices $W_h$, $W_x$, and $W_y$ play analogous roles to those of $A$, $B$, and $C$ in \eqref{eq:sm}, respectively.
The generic architecture of an RNN can be expressed as
\begin{equation}
\label{eq:rnn}
\begin{aligned}
    h_{t_i} &= g_{\theta}(h_{t_{i-1}}, x_{t_i}), \\
    y_{t_i} &= l_{\psi}(h_{t_i}),
\end{aligned}
\end{equation}
where $g_{\theta}$ is a learnable non-linear function parametrised by $\theta$, and $l_{\psi}$ is a learnable affine transformation parametrised by $\psi$. 
Challenges such as vanishing and exploding gradients during training led to the development of more advanced architectures, including Long Short-Term Memory networks (LSTMs) in 1997 \citep{hochreiter1997long}, Gated Recurrent Units (GRUs) in 2014 \citep{GRU}, and Linear Recurrent Units (LRUs) in 2023 \citep{orvieto2023resurrecting}.
\\ \\
In 2015, Bahdanau et al.\ introduced the attention mechanism to RNN-based encoder-decoder models to enhance their ability to capture long-range dependencies \citep{bahdanau2015neural}. 
A widely used variant is the scaled dot-product attention of \citet{vaswani2017attention}. 
Given matrices $X \in \mathbb{R}^{n \times d_X}$ and $Y \in \mathbb{R}^{m \times d_Y}$ representing time series of dimension $d_X$ and $d_Y$ respectively, scaled dot-product attention computes
\begin{equation}
\text{Attention}(Y, X) = \text{softmax}\left(\frac{(Y W_Q)(X W_K)^\top}{\sqrt{d_k}}\right) (X W_V),
\end{equation}
where $W_Q \in \mathbb{R}^{d_Y \times d_k}$, $W_K \in \mathbb{R}^{d_X \times d_k}$, and $W_V \in \mathbb{R}^{d_X \times d_v}$ are learnable weight matrices, and the $\text{softmax}$ function is applied row-wise to normalise the attention scores. 
Here, $Y W_Q$, $X W_K$, and $X W_V$ are referred to as the queries, keys, and values respectively. 
This mechanism allows the model to focus on different parts of the input sequence when generating each part of the output. 
\\ \\
In 2017, Vaswani et al.\ built the Transformer architecture around scaled dot-product attention, entirely removing recurrent layers in favour of self-attention mechanisms \citep{vaswani2017attention}. 
In self-attention, the input sequence generates the queries, keys, and values ($X = Y$) within each layer. 
This design allows each position in the sequence to attend to all other positions, effectively capturing long-range dependencies. 
Combining the Transformer architecture with autoregressive modelling has led to significant breakthroughs in natural language processing and forms the foundation of modern large language models \citep{brown2020language, touvron2023llamaopenefficientfoundation}. 
However, for sequence length $n$, the computational complexity of the attention mechanism scales as $\mathcal{O}(n^2)$, compared to $\mathcal{O}(n)$ for RNNs.

\subsection{Neural Differential Equations}

In 1987, Pineda theoretically explored training a continuous-time recurrent neural network
\begin{equation}
    \label{eq:pineda_rnn}
    \frac{\mathrm{d}h_t}{\mathrm{d}t} = - h_t + \sigma(W_hh_t) + I_t,
\end{equation}
where
\begin{equation}
    h_t = \begin{bmatrix} h^i_t \\ h^h_t \\ h^o_t \end{bmatrix},
\end{equation}
with $h^i$, $h^h$, and $h^o$ being designated input, hidden, and output nodes, and
\begin{equation}
    I_t = \begin{bmatrix} X_t \\ 0 \\ 0 \end{bmatrix},
\end{equation}
with $X_t$ being a continuous input stream \citep{pineda1987generalization}. 
In 1989, building on earlier work which applied the adjoint sensitivity method to differential equations for optimal control problems \citep{Bryson1962steepest}, \citet{pearlmutter1989learning} introduced an approach for computing the gradients of \eqref{eq:pineda_rnn} by solving a backwards-in-time ODE.
Theoretically, this method allows gradient calculation without storing the intermediate hidden states of the network, although in practice \citet{pearlmutter1989learning} did store the intermediate states. 
To the best of our knowledge, \citet{pearlmutter1989learning} also trained the first neural differential equations to output desired trajectories, with Figure \ref{fig:first_node} showing their results on a figure eight dataset.
\\
\begin{figure}
\centering
\includegraphics[width=\textwidth]{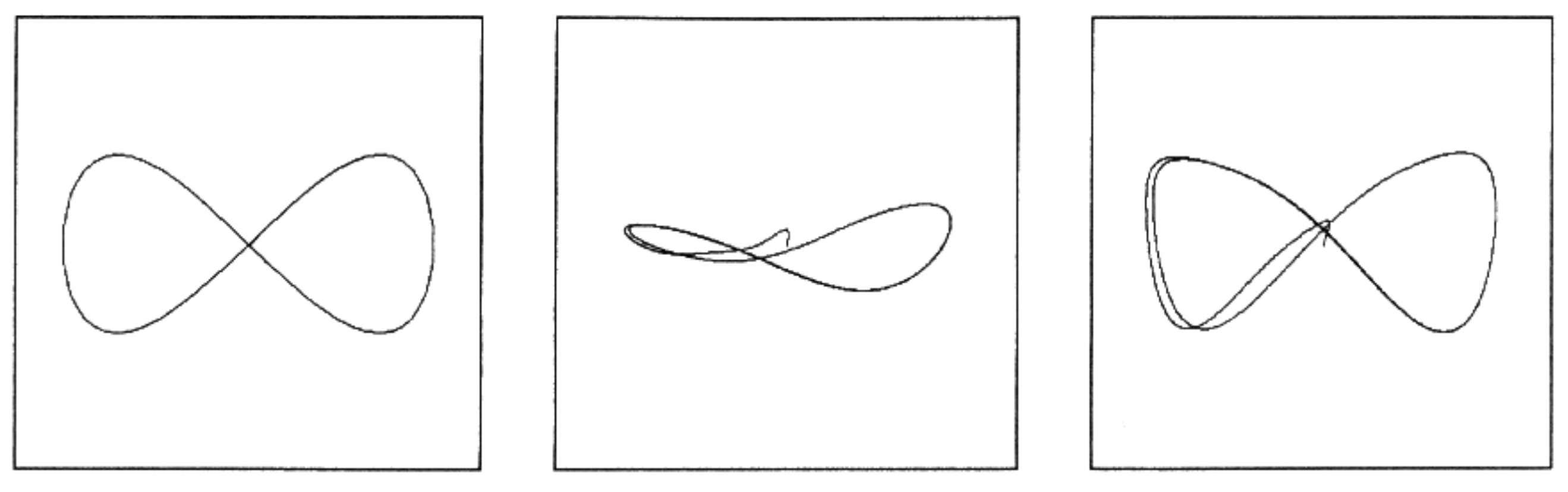}
\caption{``Desired states $d_1$ and $d_2$ plotted against each other (left); actual states $h_1$ and $h_2$ plotted against each other at epoch $3{,}182$ (centre) and $20{,}000$ (right)''. Reproduced with permission from \citet{pearlmutter1989learning}.}
\label{fig:first_node}
\end{figure}

In 1992, Rico-Mart\'{\i}nez et al.\ parametrised the vector field of a differential equation using a neural network,
\begin{equation} 
    \label{eq:node}
    \frac{\mathrm{d}h_t}{\mathrm{d}t} = f_{\theta}(t, h_t),
\end{equation}
and modelled the electrodissolution of copper in phosphoric acid solution \citep{martinez1992discrete}. 
In contrast to the adjoint sensitivity method developed by \citet{pearlmutter1989learning}, \citet{martinez1992discrete} backpropagated directly through the steps of their differential equation solve, which requires storing the intermediate hidden states.
\\ \\
In 2015, He et al.\ introduced Residual Neural Networks (ResNets), which allowed the training of very deep networks \citep{He2016DeepRL}. 
This was achieved by introducing skip connections, which allow the input to bypass one or more layers, mitigating issues such as vanishing gradients. 
The update rule for a ResNet layer is given by
\begin{equation} 
    \label{eq:resnet}
    h_{i+1} = h_i + f^i_{\theta_i}(h_i), 
\end{equation} 
where $h_i$ is the hidden state at layer $i$, and the $f^i_{\theta_i}$ are non-linear functions parametrised by $\theta_i$ respectively. 
In 2018, Chen et al. observed that \eqref{eq:resnet} resembles an Euler discretisation of \eqref{eq:node}, and proposed Neural ODEs as a continuous-depth analogue to ResNets \citep{chen2018neural}. 
Following a similar approach to \citet{pearlmutter1989learning}, \citet{chen2018neural} derived a method for calculating the gradients of solutions to \eqref{eq:node} using the adjoint sensitivity method. 
Furthermore, they utilised the ability to not store the intermediate hidden states to significantly reduce memory requirements during training.
\\ \\
These advancements in modelling continuous-time dynamics using neural networks set the stage for the development of NCDEs, the continuous-time analogue to recurrent neural networks.

\section{Neural Controlled Differential Equations}

\subsection{Definition}

\label{sec:ncde_intro}

\begin{definition}[Neural Controlled Differential Equation \citep{kidger2020neuralcde}]
\label{def:NCDE}
    Let $X:[t_0,t_{n}]\rightarrow \mathbb{R}^{d_X}$ be a continuous interpolation of $\{(t_i,x_{t_i})\}_{i=0}^{n}$, such that $X_{t_i}=(t_i, x_{t_i})$. 
    Let $\xi_{\phi}:\mathbb{R}^{d_X}\rightarrow\mathbb{R}^{d_h}$ and $f_{\theta}:\mathbb{R}^{d_h}\rightarrow\mathbb{R}^{d_h \times d_X}$ be neural networks and $l_{\psi}:\mathbb{R}^{d_h}\rightarrow\mathbb{R}^{d_y}$ be a linear map with learnable parameters $\phi$, $\theta$, and $\psi$, respectively. 
    An NCDE is defined by
    \begin{equation}
    \begin{aligned}
    \label{eq:ncde}
        h_{t_0} &= \xi_{\phi}(X_{t_0}), \\  
        h_t &= h_{t_0} + \int_{t_0}^t f_{\theta}(h_s)\mathrm{d} X_s, \\
        y_t &= l_\psi(h_t).
    \end{aligned}
    \end{equation}
\end{definition}
An NCDE consists of three learnable maps. First, $\xi_\phi$ maps the initial observation to the initial hidden state $h_{t_0}$. Second, $f_\theta$ maps each hidden state $h_s$ to a linear map that determines how increments of the control path $X$ update the hidden state. Finally, $l_\psi$ maps the hidden state $h_t$ to the output $y_t$.
\\ \\
The central innovation of NCDEs is that they interface with the data through a continuous control path $X$. 
Although Definition~\ref{def:NCDE} is stated for irregularly sampled, fully observed data, the same framework applies more broadly whenever the observations can be encoded as a continuous path $X$. 
For example, if only a subset of the components of $x_{t_i}$ is observed at time $t_i$, then one may either construct the control path channel-wise or first impute the missing values and then build $X$ from the resulting completed observations. 
As another example, when the observations $x_{t_i}$ take discrete values, one may first encode them in a Euclidean space and then choose $X$ to be the rectilinear interpolation \citep{morrill2022interpolation}, which between successive observations follows the path
\begin{equation}
    (t_i,x_{t_i}) \to (t_{i+1},x_{t_i}) \to (t_{i+1},x_{t_{i+1}}).
\end{equation}
The choice of $X$ requires some care, since different constructions expose different information to the model, affecting both the learned dynamics and the resulting performance.
In particular, some choices make $X_t$ depend on observations from times later than $t$, rendering them unsuitable for online settings.
\citet{morrill2022interpolation} theoretically and empirically studied a range of choices for $X$, and introduced rectilinear controls for online prediction tasks.
From this point onward, we assume that a suitable online control path $X$ has already been constructed from the observations.
\\ \\
To make use of the techniques developed for training Neural ODEs, NCDEs are typically rewritten as an ODE,
\begin{equation}
\label{eq:ncde_ode}
    \tilde{h}_t = \tilde{h}_{t_0} + \int_{t_0}^t g_{\theta, X}(\tilde{h}_s, s)\mathrm{d}s,
\end{equation}
where $\tilde{h}$ denotes the hidden state in this ODE representation, with $\tilde{h}_{t_0}=h_{t_0}$.
Originally, \citet{kidger2020neuralcde} proposed taking $X$ to be a differentiable interpolation and
\begin{equation}
\label{eq:cde_diff}
    g_{\theta, X}(\cdot) = f_{\theta}(\cdot)\frac{\mathrm{d} X}{\mathrm{d} s},
\end{equation}
which gives $\tilde{h}_t=h_t$ for all $t$. 
The Log-ODE method will allow us to retain the ODE form \eqref{eq:ncde_ode} for a wider class of driving paths, at the expense of replacing the exact dynamics by an approximation, such that $\tilde{h}_t\approx h_t$.

\subsection{Comparison with Alternative Approaches}

NCDEs are closely related to Neural ODEs \eqref{eq:node}. 
The key difference is that the trajectory of a Neural ODE's hidden state is determined entirely by its initial condition and the learned vector field, making them unsuitable for time series data.
Methods such as GRU-ODE and ODE-RNN address this by combining continuous-time evolution with discrete recurrent updates at observation times \citep{GRU-ODE, ODERNN}. 
In contrast, NCDEs incorporate the incoming signal directly into the dynamics, and so may be viewed as a continuous-time analogue of RNNs.
\\ \\
This connection can be made concrete by considering the residual RNN \citep{boxuan2018residual}
\begin{equation}
    \label{eq:resrnn}
    h_{t_{i+1}} = h_{t_i} + g_{\theta}(h_{t_i}, X_{t_{i+1}} - X_{t_i}).
\end{equation}
Just as a ResNet can be interpreted as a discretisation of a Neural ODE, \eqref{eq:resrnn} can be interpreted as the discretisation of a continuous process
\begin{equation}
    \label{eq:contresrnn}
    \mathrm{d}h_s = g_{\theta}(h_s, \mathrm{d}X_s).
\end{equation}
The key structural difference is that \eqref{eq:contresrnn} depends non-linearly on the increment $\mathrm{d}X_s$, whereas an NCDE depends linearly on the increment. 
However, this linear dependence does not reduce theoretical expressivity, as discussed further in Section~\ref{sec:ncde_expressivity}.
\\ \\
Alternative irregular-time methods include Gaussian processes and Neural Processes, which also naturally support predictive uncertainty quantification \citep{ohagan1978curve,williams1996gaussian,rasmussen2006gaussian,garnelo2018conditional}. 
In particular, causal Gaussian processes respect the same online temporal structure as the sequential prediction setting relevant for NCDEs \citep{cunningham2012gaussian}. 
Such approaches are appealing when quantifying predictive uncertainty is itself a central objective.
\\ \\
However, the focus of this thesis is on understanding parametric, causal, continuous-time models for sequential prediction. 
Therefore, to isolate the core questions of representation, architecture, and numerical approximation, we work in a simplified deterministic setting. 
Once the observations have been converted into a control path, the path $X$ is treated as fixed, the output path $y$ is treated as deterministic, and the maps $\xi_{\phi}$, $f_{\theta}$, and $l_\psi$ are taken to be deterministic parametrised functions. 
This excludes explicit modelling of aleatoric and epistemic uncertainty, but allows this thesis to focus on the fundamental mechanisms by which NCDEs process continuous paths and propagate information through time. 
There are natural extensions of the NCDE framework that do incorporate uncertainty, including Neural Stochastic Differential Equations and Bayesian Neural Controlled Differential Equations \citep{kidger2021neural,hess2024bncde}. 
Furthermore, the numerical and architectural advances developed in this thesis are compatible with these uncertainty-aware settings, although we do not pursue those extensions here.

\subsection{Expressivity}
\label{sec:ncde_expressivity}

\begin{definition}[Maximal expressivity \citep{walker2025structuredlinearcdesmaximally}]
\label{def:universal_approximation_general}
Let $\mathcal{X}$ be a topological space, and let 
$\mathcal{F} = \{ f_\theta : \mathcal{X} \to \mathbb{R} \mid \theta \in \Theta \}$
be a class of real-valued functions on $\mathcal{X}$, parametrised by some set $\Theta$. 
We say that $\mathcal{F}$ is maximally expressive (or universal) on $\mathcal{X}$ if, for every compact set $\mathcal{K} \subset \mathcal{X}$ and every real-valued continuous function $f : \mathcal{K} \to \mathbb{R}$, the following property holds:
\begin{equation}
\forall \epsilon > 0, \; \exists \theta \in \Theta \quad \text{s.t.} \quad 
\sup_{x \in \mathcal{K}} \big| f(x) - f_\theta(x) \big| < \epsilon.
\end{equation}
\end{definition}

Although maximal expressivity is not sufficient to ensure good performance, it is desirable, as it shows that at least theoretically the model class is rich enough to approximate any continuous target map on compact sets.
A classical result is the Universal Approximation Theorem for neural networks.

\begin{theorem}[Universal Approximation Theorem \citep{cybenko1989approximation, hornik1991approximation}]
\label{thm:uat}
Let $\sigma:\mathbb{R}\to\mathbb{R}$ be continuous, bounded, and nonconstant. Define
\begin{equation}
\mathcal{F} =
\left\{
x\mapsto \sum_{j=1}^{m} a_j \sigma(w_j^\top x+b_j)
\;\middle|\;
m\in\mathbb{N},\ a_j\in\mathbb{R},\ w_j\in\mathbb{R}^d,\ b_j\in\mathbb{R}
\right\}.
\end{equation}
Then $\mathcal{F}$ is maximally expressive on $\mathbb{R}^d$.
\end{theorem}

Theorem \ref{thm:uat} establishes that single hidden-layer neural networks are maximally expressive for continuous real-valued functions on $\mathbb{R}^d$. 
Theorem \ref{thm:universal_ncde} shows that NCDEs satisfy the same notion of maximal expressivity for continuous functions of entire paths.

\begin{theorem}[Maximally Expressive NCDEs \citep{kidger2020neuralcde}]
    \label{thm:universal_ncde}
    Let $\mathcal{X}$ be the space of bounded variation paths on the interval $[t_0,t_n]$ that start at a common point and include time as a channel, endowed with the $1$-variation topology. 
    Let $\mathcal{F}$ be the class of real-valued maps $X \mapsto y_{t_n}$ induced by NCDEs from Definition~\ref{def:NCDE} with $d_h \in \mathbb{N}$ and $d_y=1$. 
    Then $\mathcal{F}$ is maximally expressive on $\mathcal{X}$.
\end{theorem}

\begin{proof}
    The result follows from the signature being universal, Corollary \ref{cor:universality-on-paths}, the truncated signature solving a linear CDE, \eqref{eq:sig_cde}, and the Universal Approximation Theorem for neural networks \citep{cybenko1989approximation, hornik1991approximation}. 
    For more details, see \citep[Theorem C.25]{kidger2022neuraldifferentialequations}. 
\end{proof}

Theorem \ref{thm:universal_ncde} can be extended from path-to-point functions to path-to-path functions by replacing the linear readout $l_{\psi}$ with a neural network, as shown in \citep[Proposition D.2]{cirone2024deepSSM}.
In addition to their important theoretical properties, NCDEs achieve superior performance on a range of datasets when compared to similar methods, such as GRU-ODE and ODE-RNN, which handle irregularly sampled data by combining Neural ODEs with RNNs \citep{GRU-ODE, ODERNN}. 
\\ \\
Just as NCDEs can be viewed as the continuous-time analogue of RNNs, once an NCDE is discretised using a differential equation solver, it may be viewed as an RNN unrolled over the solver steps. 
If many solver steps are required, then training must backpropagate through a long sequence of hidden-state updates. 
Consequently, NCDEs can suffer from the same exploding and vanishing gradient issues as RNNs on long time series, leading to degraded performance \citep{morrill2021neuralrough}.

\subsection{Neural Rough Differential Equations}

A recurring issue in deep learning approaches to time series modelling is that large number of repeated forward passes through the neural network cause the gradient during training to either explode or vanish, as was first shown in \citep{Hochreiter1991UntersuchungenZD} and explored further in \citep{vanishing}. 
This issue was one of the motivations behind the development of LSTMs \citep{hochreiter1997long}. 
The Log-ODE method introduced in Section \ref{sec:logode} is an efficient and accurate method for approximating the solution to a CDE. 
Inspired by the Log-ODE method, \citeauthor{morrill2021neuralrough} introduced neural rough differential equations (NRDEs), which replaced \eqref{eq:cde_diff} with the piecewise in time
\begin{equation}
\label{eq:logode2}
    g_{\theta,X}(\cdot) = \bar{f}_{\theta}\left(\cdot\right)\frac{\log(S^{N}(X)_{[r_i,r_{i+1}]})}{r_{i+1}-r_i}, \quad s\in[r_i, r_{i+1}),
\end{equation}
where $\bar{f}_{\theta}:\mathbb{R}^{d_h}\rightarrow\mathbb{R}^{d_h\times \beta(d_X,N)}$ is a neural network and $\beta(d_X,N)=\mathcal{O}(d_X^N)$ is the dimension of $\mathfrak{L}^N(\mathbb{R}^{d_X})$, the space where the depth$-N$ truncated log-signature of a $d_X$ dimensional path lives \citep{morrill2021neuralrough}. 
\\ \\
Compared to NCDEs, NRDEs can reduce the number of forward passes through the network while evaluating the model, as the vector field is autonomous on each interval $[r_i,r_{i+1})$. 
This has been shown to lead to improved classification accuracy, alongside reduced time and memory-usage, on time series with up to 17,000 observations \citep{morrill2021neuralrough}. 
Furthermore, as it is no longer necessary to apply a differentiable interpolation to the time series data, NRDEs are applicable to a wider range of input signals. 
By neglecting the Lie bracket structure of $\bar{f}_{\theta}$, NRDEs reduce the computational burden of evaluating the vector field, at the cost of increasing the output dimension of the neural network. 
In contrast, Log-NCDEs retain the Lie bracket structure of $\bar{f}_{\theta}$.

\section{Log Neural Controlled Differential Equations}

\label{sec:log_ncde}

\subsection{Definition}

Log-NCDEs use the same underlying model as NRDEs,
\begin{equation}
\label{eq:Log-NCDE}
\begin{aligned}
    h_t &= h_{t_0} + \int_{t_0}^t g_{\theta, X}(h_s)\text{d}s, \\ g_{\theta, X}(\cdot) &= \bar{f}_{\theta}\left(\cdot\right)\frac{\log(S^{N}(X)_{[r_i,r_{i+1}]})}{r_{i+1}-r_i}, \quad s\in[r_i, r_{i+1}),
\end{aligned}
\end{equation}
but with two major changes. 
First, instead of parametrising $\bar{f}_{\theta}$ using a neural network, it is constructed using the iterated Lie brackets of an NCDE's neural network, $f_{\theta}$. 
Second, $f_{\theta}$ is ensured to be a $\mathrm{Lip}(\gamma)$ function for $\gamma\in(N-1,N]$. 
These changes have a major benefit. 
For $N>1$, Log-NCDEs are exploring a smaller output space during training than NRDEs, while maintaining the same expressivity, as NCDEs are maximally expressive. 
This is because the output dimension of $f_{\theta}$ is $d_h \times d_X$, whereas the output dimension of $\bar{f}_{\theta}$ is $d_h\times \beta(d_X,N)$, where $\beta(d_X,N)=\mathcal{O}(d_X^N)$.
Figure \ref{fig:logsig_dim} compares these values for paths of dimension $d_X$ from $1$ to $15$ and truncation depths of $N=1$ and $N=2$. 
The reduced output dimension comes at the cost of needing to calculate the iterated Lie brackets when evaluating Log-NCDEs, which is quantified in Section \ref{sec:cost} and explored empirically in Section \ref{sec:uea}. 
Figure \ref{fig:methods} is a schematic diagram comparing the approaches of an NCDE and a Log-NCDE.
\\
\begin{figure}
    \centering
    \includegraphics[width=\linewidth]{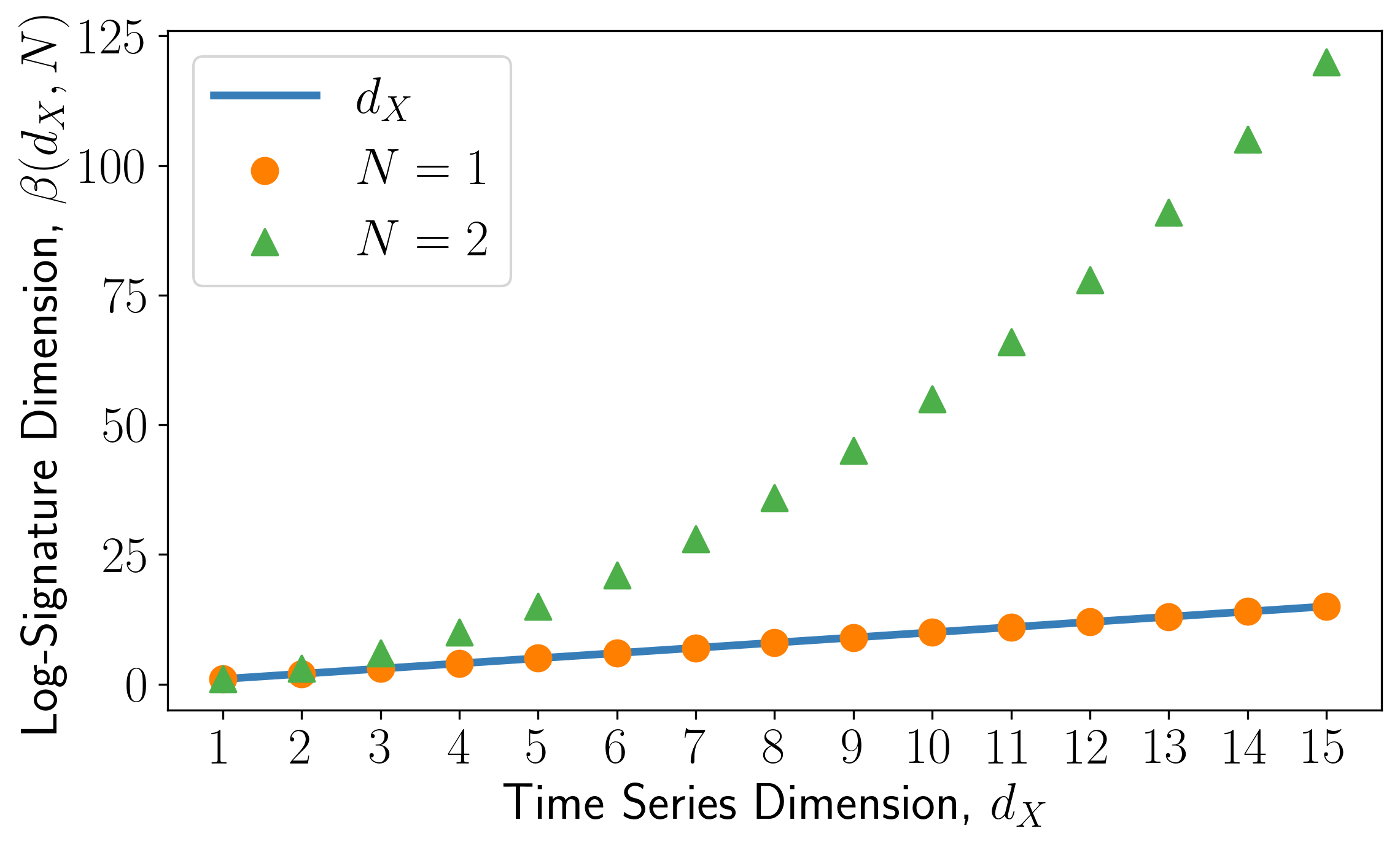}
    \caption{A plot of $\beta(d_X, N)$ against $d_X$ for $N=1,2$. The output dimension of an NRDE's neural network is $\mathbb{R}^{d_h\times \beta(d_X,N)}$, whereas for a Log-NCDE it is $\mathbb{R}^{d_h \times d_X}$.}
    \label{fig:logsig_dim}
\end{figure}

When $N=1$, \eqref{eq:Log-NCDE} simplifies to 
\begin{equation}
\label{eq:Log-NCDE_N1}
    g_{\theta, X}(\cdot) = f_{\theta}\left(\cdot\right)\frac{X_{r_{i+1}}-X_{r_i}}{r_{i+1}-r_i}, \quad s\in[r_i, r_{i+1}).
\end{equation}
Hence, in this case the only difference between Log-NCDEs and NRDEs is the regularisation of $f_{\theta}$. 
Furthermore, \eqref{eq:Log-NCDE_N1} and \eqref{eq:cde_diff} are equivalent when $X$ is a linear interpolation. 
Therefore, the approach of NCDEs, NRDEs, and Log-NCDEs coincide when using a depth$-1$ Log-ODE approximation \citep{morrill2021neuralrough}.
\begin{figure}
\centering

\subfloat[A schematic NCDE. A continuous path $X$ is constructed from the observations. Its time derivative $\frac{\mathrm{d}X}{\mathrm{d}s}$ is combined with the neural vector field $f_{\theta}$ to define the hidden-state dynamics, which are evolved with a differential equation solver. A linear map is then applied to $h_{t_n}$ to produce a prediction, and the loss $L$ against the target $y_{t_n}$ is used to update the learnable parameters.]{%
  \hspace{-0.5cm}\includegraphics[clip,width=0.9\textwidth]{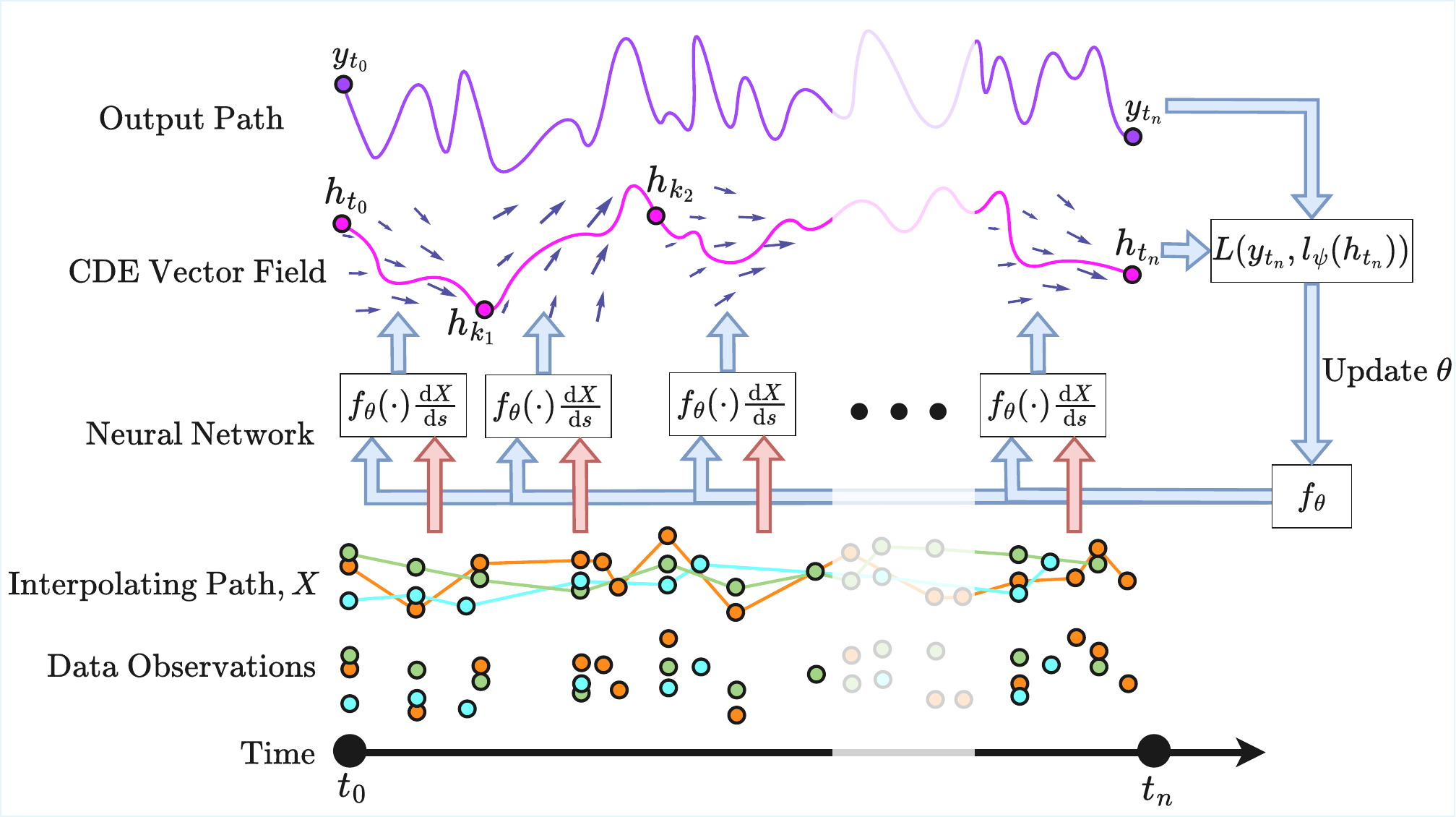}%
}
\vspace{5mm}
\subfloat[A schematic Log-NCDE. In contrast to the NCDE, the log-signature of the driving path $X$ is computed over intervals $\{r_i\}_{i=0}^m$ and combined with iterated Lie brackets of the neural vector field $f_{\theta}$ to define the hidden-state dynamics.]{%
  \hspace{-0.5cm}\includegraphics[clip,width=0.9\textwidth]{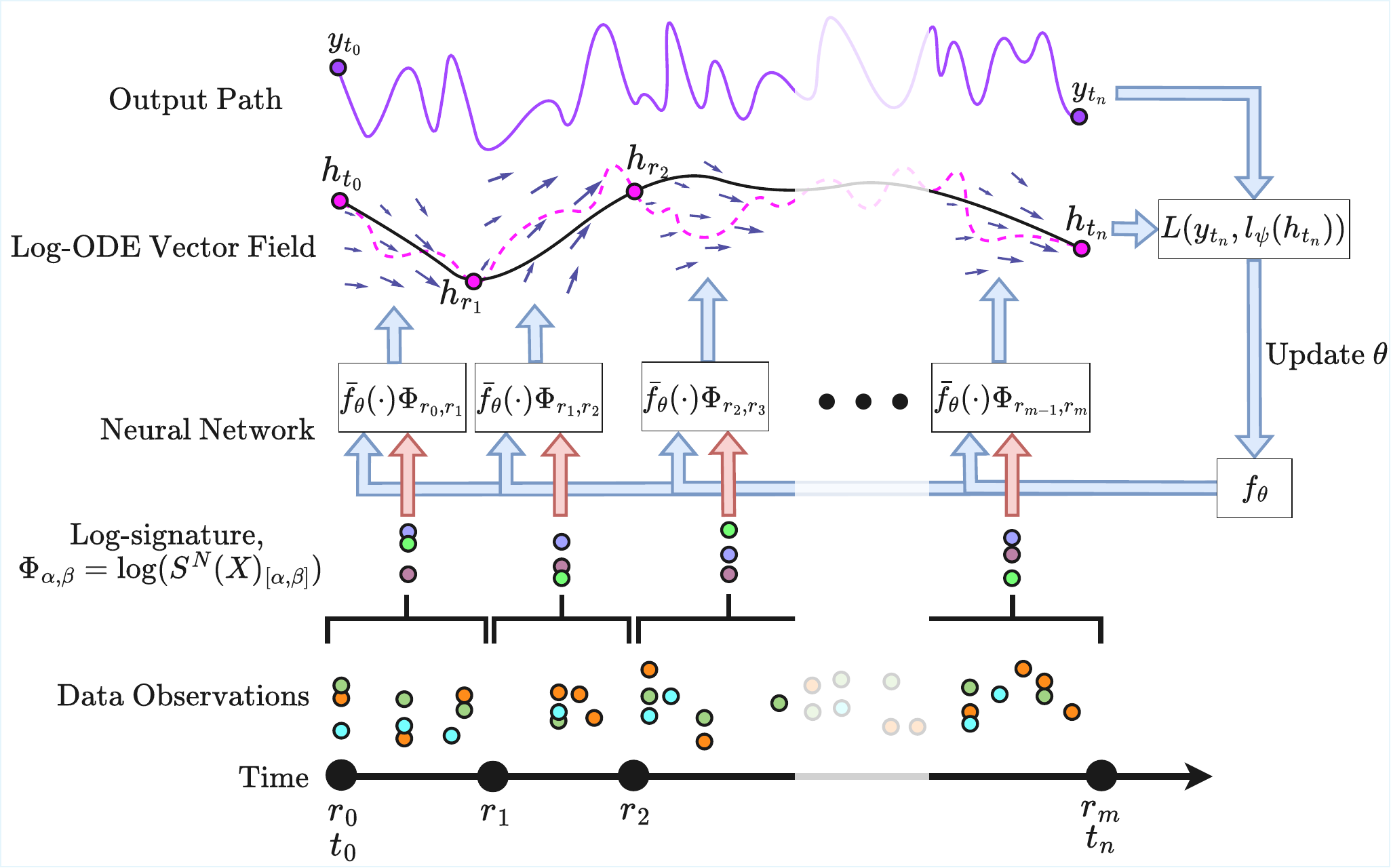}%
}
  \caption{A schematic diagram of an NCDE and a Log-NCDE.}
  \label{fig:methods}
\end{figure}

\subsection{$\mathrm{Lip}(\gamma)$ Neural Networks}
\label{sec:lipgammaNN}
As discussed in Section \ref{sec:logode}, applying a depth $N$ Log-ODE method requires the vector field $f_{\theta}$ to be $\mathrm{Lip}(\gamma)$ for $\gamma\in(N-1,N]$. 
There exist theoretical results linking the robustness of a learning algorithm to the algorithm's Lipschitz constant \citep{robustness}. 
Furthermore, there are results bounding the Lipschitz constant of a fully connected neural network (FCNN) \citep{NNLip}. 
Here, we extend these results to $\mathrm{Lip}(\gamma)$ for $1<\gamma\leq2$.

\begin{definition}[Fully Connected Neural Network]
\label{def:FCNN}
    Let $m,n_{in},n_{out},n_h\in\mathbb{N}$ and $f_{\theta}$ be a fully connected neural network (FCNN) with $m$ layers, input dimension $n_{in}$, output dimension $n_{out}$, hidden dimension $n_h$, and activation function $\sigma$. Given an input $x\in\mathbb{R}^{n_{in}}$,
    \begin{equation}
        f_{\theta}(x) = L^m(\cdots (L^1(x))\cdots),
    \end{equation}
    where $L^1:\mathbb{R}^{n_{in}}\rightarrow\mathbb{R}^{n_h}$, $L^i:\mathbb{R}^{n_h}\rightarrow\mathbb{R}^{n_h}$ for $i=2,\ldots,m-1$, and $L^m:\mathbb{R}^{n_h}\rightarrow\mathbb{R}^{n_{out}}$. Each layer is defined by
    \begin{equation}
        L^i(y) = \begin{bmatrix} L^i_1(y) \\ \vdots \\ L^i_\alpha(y) \end{bmatrix} = \sigma\left(\begin{bmatrix} l^i_1(y) \\ \vdots \\ l^i_\alpha(y)\end{bmatrix}\right) = \begin{bmatrix} \sigma(l^i_1(y)) \\ \vdots \\ \sigma(l^i_\alpha(y) )\end{bmatrix} = \begin{bmatrix} \sigma(W^i_1 \cdot y + b^i_1) \\ \vdots \\ \sigma(W^i_\alpha \cdot y + b^i_\alpha) \end{bmatrix},
    \end{equation}
    where $y\in\mathbb{R}^\beta$, $W^i=[W^i_1, \ldots, W^i_\alpha]^T\in\mathbb{R}^{\alpha\times \beta}$ and $b^i =[b^i_1, \ldots, b^i_\alpha]^T\in\mathbb{R}^\alpha$ are the learnable parameters and
    \begin{equation}
        (\alpha,\beta) = \begin{cases} (n_h,n_{\mathrm{in}}), \qquad i=1, \\
        (n_h, n_h), \qquad \; i=2,\ldots,m-1, \\ 
        (n_{\mathrm{out}}, n_h), \qquad i=m.
        \end{cases}
    \end{equation}
\end{definition}

\begin{assumption}
    \label{ass:activation}
    Let $1<\gamma\leq 2$. We will assume that the activation function $\sigma$ satisfies the following four conditions:
    \begin{enumerate}
        \item $\sigma$ is continuously differentiable with derivative $\sigma'$,
        \item $|\sigma(x)|\leq|x|$,
        \item $\sup_{x\in\mathbb{R}}|\sigma'(x)| \leq M_1$, and
        \item $[\sigma']_{\gamma - 1} = \sup_{\substack{y,x\in \mathbb{R}\\ y\neq x}} \frac{|\sigma'(y)-\sigma'(x)|}{|y-x|^{\gamma-1}} \leq M_2$,
    \end{enumerate}
    for constants $M_1,M_2>0$. 
\end{assumption}

Conditions $1$ and $4$ imply that $\sigma\in C^{1,\gamma-1}$. 
Additionally, by Lemma \ref{lem:Ck_taylor_bound}, $\sigma$ satisfies all the conditions to be $\mathrm{Lip}(\gamma)$ except $\sup_{x\in\mathbb{R}}|\sigma(x)| < \infty$, which we do not assume as it would exclude some standard activation functions which do satisfy Assumption \ref{ass:activation}, such as SiLU. 
It is worth noting that not all standard activation functions satisfy Assumption \ref{ass:activation}. 
For example, ReLU is not continuously differentiable. 

\begin{lemma}
\label{lem:lipgamma-layer}
    Let $1<\gamma \leq 2$ and $f_{\theta}$ be a FCNN with activation function satisfying Assumption \ref{ass:activation}.
    Take the Euclidean norm and assume the input $x\in A\subset \mathbb{R}^{n_{\mathrm{in}}}$, where $\sup_{x\in A}\|x\|_2=C$.
    Then each layer $L^i$ satisfies
    \begin{equation}
        \|L^i\|_{\mathrm{Lip}(\gamma)} 
        \leq \max\left\{\Gamma^i, M_1\|W^i\|_{\operatorname{op}}, M_2\|W^i\|^{\gamma}_{\operatorname{op}}\right\},
    \end{equation}
    where $\Gamma^i = \|W^i\|_{\operatorname{op}}\Gamma^{i-1} + \|b^i\|_2$ for $i\geq1$ and $\Gamma^0=C$.
\end{lemma}

\begin{proof}
    Let $Y_i$ be the input domain to the $i^{\text{th}}$ layer. The $\mathrm{Lip}(\gamma)$ norm of $L^i$ is a maximum over four terms:
    \begin{enumerate}
        \item First, 
        \begin{equation}
            \sup_{y\in Y_i}\|L^i(y)\|_2 \leq \sup_{y\in Y_i} \|W^i\|_{\operatorname{op}}\|y\|_2 + \|b^i\|_2,
        \end{equation}
        by condition $2$ in Assumption \ref{ass:activation}.
        \item Second,
        \begin{equation}
            \sup_{y\in Y_i}\|\nabla L^i(y)\|_{\operatorname{op}} = \sup_{y\in Y_i}\|\operatorname{diag}(\sigma'(W^iy + b^i))W^i\|_{\operatorname{op}} \leq M_1\|W^i\|_{\operatorname{op}},
        \end{equation}
        by condition $3$ in Assumption \ref{ass:activation}.
        \item Third,
        \begin{equation}
            \sup_{y\neq x} \frac{\|\nabla L^i(y)-\nabla L^i(x)\|_{\operatorname{op}}}{\|y-x\|_2^{\gamma-1}} \leq \sup_{y\neq x} \frac{\|\operatorname{diag}(\Delta)\|_{\operatorname{op}}}{\|y-x\|_2^{\gamma-1}}\|W^i\|_{\operatorname{op}}, 
        \end{equation}
        where
        \begin{equation}
            \Delta_j = \sigma'(W^i_j \cdot y + b^i_j) - \sigma'(W^i_j \cdot x + b^i_j).
        \end{equation}
        Using condition $4$ in Assumption \ref{ass:activation},
        \begin{equation}
            \|\operatorname{diag}(\Delta)\|_{\operatorname{op}} \leq M_2\|y-x\|_2^{\gamma-1}\max_j \|W^i_j\|^{\gamma-1}_2 \leq M_2\|y-x\|_2^{\gamma-1}\|W^i\|_{\operatorname{op}}^{\gamma-1}.
        \end{equation}
        Therefore, 
        \begin{equation}
            \sup_{y\neq x} \frac{\|\nabla L^i(y)-\nabla L^i(x)\|_{\operatorname{op}}}{\|y-x\|_2^{\gamma-1}} \leq M_2\|W^i\|^{\gamma}_{\operatorname{op}}.
        \end{equation}
        \item Fourth, since each $L^i$ belongs to $C^{1,\gamma-1}(\mathbb{R}^{\beta}, \mathbb{R}^\alpha)$, Lemma \ref{lem:Ck_taylor_bound} can be used to bound the final term,
        \begin{equation}
            \sup_{y\neq x} \frac{\|L^i(y)-L^i(x) - \nabla L^i(x)[y-x]\|_2}{\|y-x\|_2^{\gamma}} \leq \sup_{y\neq x} \frac{\|\nabla L^i(y)-\nabla L^i(x)\|_{\operatorname{op}}}{\|y-x\|_2^{\gamma-1}}.
        \end{equation}
    \end{enumerate}
    Therefore,
    \begin{equation}
        \|L^i\|_{\mathrm{Lip}(\gamma)} 
        \leq \max\left\{\sup_{y \in Y_i} \|W^i\|_{\operatorname{op}}\|y\|_2 + \|b^i\|_2, M_1\|W^i\|_{\operatorname{op}}, M_2\|W^i\|^{\gamma}_{\operatorname{op}}\right\}.
    \end{equation}
    Since $y\in Y_1$ satisfies $\|y\|_2\leq C$, 
    \begin{equation}
        \|L^i\|_{\mathrm{Lip}(\gamma)} 
        \leq \max\left\{\Gamma^i, M_1\|W^i\|_{\operatorname{op}}, M_2\|W^i\|^{\gamma}_{\operatorname{op}}\right\},
    \end{equation}
    where $\Gamma^i = \|W^i\|_{\operatorname{op}}\Gamma^{i-1} + \|b^i\|_2$ for $i\geq1$ and $\Gamma^0=C$.
\end{proof}

\begin{theorem}
\label{thm:lipnn}
    Let $1<\gamma \leq 2$ and $f_{\theta}$ be a FCNN with $m$ layers and activation function $\sigma$ satisfying Assumption \ref{ass:activation}. 
    Assume the input $x\in A\subset \mathbb{R}^{n_{\mathrm{in}}}$, where $\sup_{x\in A}\|x\|_2=C$. 
    Then $f_{\theta}\in\mathrm{Lip}(\gamma)$ and 
    \begin{equation}
    \label{eq:lipnn}
        ||f_{\theta}||_{\mathrm{Lip}(\gamma)} \leq (1+2^{\gamma})^{\frac{\gamma^{m-1}-1}{\gamma-1}} \prod_{i=1}^m\max\left(1, \|L^i\|_{\mathrm{Lip}(\gamma)}^{\gamma^{m-i}}\right)
    \end{equation}
    with 
    \begin{equation}
        \|L^i\|_{\mathrm{Lip}(\gamma)} 
        \leq \max\left\{\Gamma^i, M_1\|W^i\|_{\operatorname{op}}, M_2\|W^i\|^{\gamma}_{\operatorname{op}}\right\},
    \end{equation}
    where $\Gamma^i = \|W^i\|_{\operatorname{op}}\Gamma^{i-1} + \|b^i\|_2$ for $i\geq1$ and $\Gamma^0=C$.
\end{theorem}

\begin{proof}
    Lemma \ref{lem:normcomplip2} states that for $f,g\in\mathrm{Lip}(\gamma)$ with $1<\gamma\leq 2$,
    \begin{equation}
        \label{eq:normcomplip2_nn}
        \|g \circ f\|_{\mathrm{Lip}(\gamma)} \leq (1+2^{\gamma})\|g\|_{\mathrm{Lip}(\gamma)}\max\{1, \|f\|^{\gamma}_{\mathrm{Lip}(\gamma)}\}.
    \end{equation}
    Assume that
    \begin{equation}
    \label{eq:lipgamma_nn_assumption}
        \|L^n \circ \cdots \circ L^1 \|_{\mathrm{Lip}(\gamma)} \leq (1+2^{\gamma})^{\frac{\gamma^{n-1}-1}{\gamma-1}} \prod_{i=1}^n\max\left(1, \|L^i\|_{\mathrm{Lip}(\gamma)}^{\gamma^{n-i}}\right),
    \end{equation}
    which is true for $n=1$. Then by \eqref{eq:normcomplip2_nn},
    \begin{equation}
        \|L^{n+1} \circ \cdots \circ L^1 \|_{\mathrm{Lip}(\gamma)} \leq (1+2^{\gamma})\|L^{n+1}\|_{\mathrm{Lip}(\gamma)}\max\left(1, \left((1+2^{\gamma})^{\frac{\gamma^{n-1}-1}{\gamma-1}} \prod_{i=1}^n\max\left(1, \|L^i\|_{\mathrm{Lip}(\gamma)}^{\gamma^{n-i}}\right)\right)^{\gamma}\right).
    \end{equation}
    Note that for any $a,b>0$,
    \begin{equation}
        \max(1,ab) \leq \max(1,a)\max(1,b).
    \end{equation}
    Repeatedly applying this bound gives,
    \begin{equation}
        \begin{aligned}
        \|L^{n+1} \circ \cdots \circ L^1 \|_{\mathrm{Lip}(\gamma)} &\leq (1+2^{\gamma})^{\gamma\frac{\gamma^{n-1}-1}{\gamma-1} + 1}\|L^{n+1}\|_{\mathrm{Lip}(\gamma)}\prod_{i=1}^n\max\left(1, \|L^i\|_{\mathrm{Lip}(\gamma)}^{\gamma^{n+1-i}}\right), \\
        &\leq (1+2^{\gamma})^{\frac{\gamma^{n}-1}{\gamma-1}}\prod_{i=1}^{n+1}\max\left(1, \|L^i\|_{\mathrm{Lip}(\gamma)}^{\gamma^{n+1-i}}\right).
        \end{aligned}
    \end{equation}
    Therefore, \eqref{eq:lipgamma_nn_assumption} holds for all $1\leq n\leq m$, and
    \begin{equation}
        ||f_{\theta}||_{\mathrm{Lip}(\gamma)} = \|L^{m} \circ \cdots \circ L^1 \|_{\mathrm{Lip}(\gamma)} \leq (1+2^{\gamma})^{\frac{\gamma^{m-1}-1}{\gamma-1}} \prod_{i=1}^m\max\left(1, \|L^i\|_{\mathrm{Lip}(\gamma)}^{\gamma^{m-i}}\right).
    \end{equation}
    Lemma \ref{lem:lipgamma-layer} completes the proof by giving the stated bound on each layer's $\mathrm{Lip}(\gamma)$ norm.
\end{proof}

Although the bound in Theorem~\ref{thm:lipnn} is a worst case estimate, it reflects a genuine feature of composition in $\mathrm{Lip}(\gamma)$ spaces.
In particular, the explicit example in Section~\ref{sec:opt} shows that for $\gamma\in(1,2]$, the composition of two functions with $\mathrm{Lip}(\gamma)$ norm equal to $1$ can itself have $\mathrm{Lip}(\gamma)$ norm greater than $1$.
Thus, even when the individual layers are uniformly controlled, composition can rapidly grow the $\mathrm{Lip}(\gamma)$ norm.
Consistent with this, if each layer satisfies $\|L^i\|_{\mathrm{Lip}(\gamma)} \leq 1$, then \eqref{eq:lipnn} still gives
\begin{equation}
    \|f_{\theta}\|_{\mathrm{Lip}(\gamma)} \leq (1+2^{\gamma})^{\frac{\gamma^{m-1}-1}{\gamma-1}},
\end{equation}
which grows rapidly with the depth $m$.
\\ \\
To demonstrate this behaviour in practice, we train three neural networks to approximate $\operatorname{sign}(x)$ for $x\in[-1,1]$.
Each neural network has a hidden dimension of $8$,  SiLU activation functions, and a depth of $2$, $3$, or $4$, respectively.
As can be seen in Figure \ref{fig:lipnn}, the supremum over the second derivative grows rapidly as the depth increases.
This leads to the Lipschitz bound on the gradient dominating $\|f_{\theta}\|_{\mathrm{Lip}(2)}$ for depths $3$ and $4$, which have $\|f_{\theta}\|_{\mathrm{Lip}(2)}\approx 297$ and $\|f_{\theta}\|_{\mathrm{Lip}(2)}\approx 16910$, respectively.
\begin{figure}
    \centering
    \includegraphics[width=\linewidth]{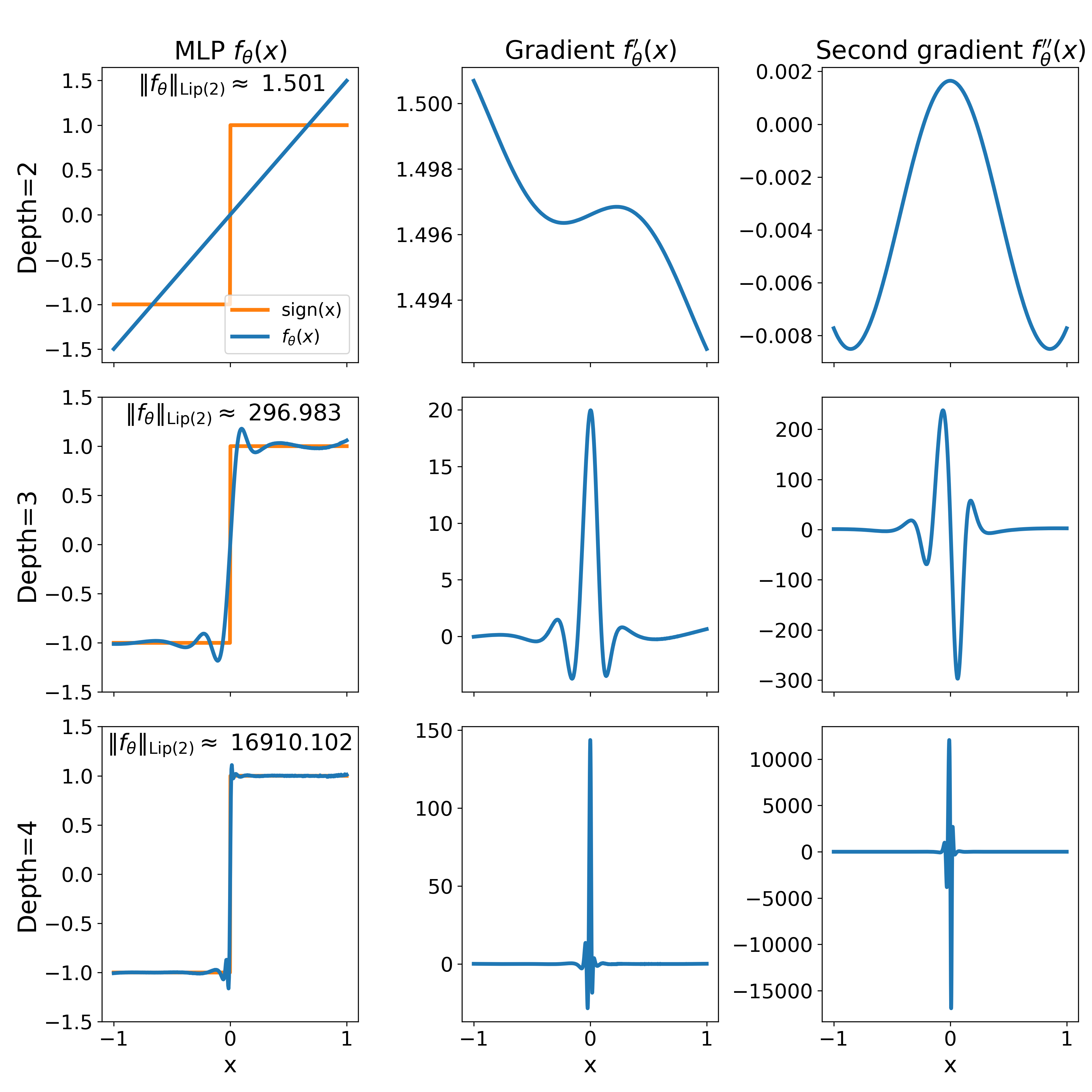}
    \caption{Comparing the $\mathrm{Lip}(2)$ norm of three fully connected neural networks trained to approximate $\operatorname{sign}(x)$. Each neural network has a hidden dimension of $8$,  SiLU activation functions, and a depth of $2$, $3$, or $4$, respectively.}
    \label{fig:lipnn}
\end{figure}
\\ \\
Although these empirical values are many orders of magnitude smaller than the worst case bound, they still illustrate the same qualitative phenomenon: $\|f_{\theta}\|_{\mathrm{Lip}(2)}$ can grow rapidly with depth in practice.
A simple way to encourage smaller parameter magnitudes during training is to introduce weight regularisation by modifying the loss function $L$ to
\begin{equation}
\label{eq:weightpenalty}
    L \mapsto L + \lambda\left(\sum_{i=1}^m\|W^i\|_2 + \|b^i\|_2\right),
\end{equation}
where $\lambda$ is a hyperparameter controlling the weight of the penalty \citep{weightreg, weightdecay}.
This introduces only a single additional hyperparameter and directly penalises the weights and biases appearing in the layer-wise bound of Theorem~\ref{thm:lipnn}.
Thus, it provides a simple proxy for encouraging smaller $\mathrm{Lip}(2)$ norms, even though it is not a sharp estimate of $\|f_{\theta}\|_{\mathrm{Lip}(2)}$.
\\ \\
In the experiments contained in Section~\ref{sec:log_ncde_experiments}, Log-NCDEs use a FCNN with $\mathrm{SiLU}$ activation functions as their vector field $f_{\theta}$.
The coefficient $\lambda$ is treated as one component of the hyperparameter grid search, with $\lambda=0$ included as a candidate value.
Empirically, this regularisation is not uniformly beneficial: some of the best runs select $\lambda=0$, while others select $\lambda>0$.
However, even with $\lambda=0$, we did not observe any training divergence or solver instability.

\subsection{Constructing the Log-ODE Vector Field}
\label{sec:liebracketNN}

As in Section \ref{sec:logode}, the linear map $\bar{f}_{\theta}$ in \eqref{eq:Log-NCDE} is defined recursively by
\begin{equation}
    \label{eq:logoderecurs1_2}
    \bar{f}_{\theta}(\cdot)a = f_{\theta}(\cdot)a,
\end{equation}
for $a\in \mathbb{R}^{d_X}$ and
\begin{equation}
    \label{eq:logoderecurs2_2}
    \bar{f}_{\theta}(\cdot)[a,b] = [\bar{f}_{\theta}(\cdot)a,\bar{f}_{\theta}(\cdot)b],
\end{equation}
where $f_{\theta}$ is the NCDE's vector field.
Assuming $f_{\theta}(\cdot)a$ is infinitely differentiable, then $f_{\theta}(\cdot)a$ is an element of the Lie algebra $C^{\infty}(\mathbb{R}^{d_h}, \mathbb{R}^{d_h})$ and from Definition \ref{def:smooth_liebracket},
\begin{equation}
    [f_{\theta}(\cdot)a,f_{\theta}(\cdot)b] = J_{f_{\theta}(\cdot)b}f_{\theta}(\cdot)a - J_{f_{\theta}(\cdot)a}f_{\theta}(\cdot)b.
\end{equation}
Calculating  \eqref{eq:logoderecurs2_2} requires a basis for $\mathfrak{L}^N(\mathbb{R}^{d_X})$, the space where the depth-$N$ truncated log-signature of the input path lives.
Let $\{e_j\}_{j=1}^{d_X}$ be the usual basis of $\mathbb{R}^{d_X}$. 
A choice of basis for $\mathfrak{L}^N(\mathbb{R}^{d_X})$ is a Hall basis, denoted $\{\hat{e}_k\}_{k=1}^{\beta(v, N)}$, which is a specific subset of up to the $(N-1)^{\text{th}}$ iterated Lie brackets of $\{e_j\}_{j=1}^{d_X}$ \citep{Hall1950ABF}. 
Rewriting \eqref{eq:Log-NCDE} using a Hall basis,
\begin{equation}
\label{eq:hall}
    \bar{f}_{\theta}\left(h_s\right)\frac{\log(S^{N}(X)_{[r_i,r_{i+1}]})}{r_{i+1}-r_i} =  \sum_{k=1}^{\beta(v, N)}\lambda_k \bar{f}_{\theta}(h_s)\hat{e}_k,
\end{equation}
where $\lambda_k$ is the term in the scaled log-signature corresponding to the basis element $\hat{e}_k$. Since each $\hat{e}_k$ can be written as iterated Lie brackets of $\{e_j\}_{j=1}^{d_X}$, it is possible to replace $\bar{f}_{\theta}(\cdot)\hat{e}_k$ with the iterated Lie brackets of $f_{\theta}(\cdot)e_i$ using \eqref{eq:logoderecurs1_2} and \eqref{eq:logoderecurs2_2}. Each $f_{\theta}(\cdot)e_i:\mathbb{R}^{d_h} \rightarrow \mathbb{R}^{d_h}$ is a vector field defined by the $i^{\text{th}}$ column of the neural network's output. Hence, $g_{\theta, X}$ can be evaluated at a point using iterated Jacobian-vector products (JVPs) of $f_{\theta}$.

\subsection{Computational Cost}
\label{sec:cost}
When the signature truncation depth $N$ is greater than $1$, NRDEs and Log-NCDEs incur an additional computational cost for each evaluation of the vector field, which we now quantify. 
Assume that an NCDE, NRDE, and Log-NCDE are all using an identical FCNN as their vector field, except for the dimension of the final layer in the NRDE. Let $m$ and $n_h$ be the depth and dimension of the FCNN's hidden layers, respectively, and $d_h$ and $d_X$ be the dimensions of $h_t$ and $X_t$ from \eqref{eq:ncde}. 
Let $\beta(d_X,N)=\dim\mathfrak{L}^N(\mathbb{R}^{d_X})=\mathcal{O}(d_X^N)$. 
Letting $F_{\text{x}}$ be the number of FLOPs required to evaluate model x's vector field,
\begin{equation}
    \begin{aligned}
        F_{\text{NCDE}}&= 2d_hn_h + 2(m-2)n_h^2 + 2d_hd_Xn_h,\\
        F_{\text{NRDE}}&= 2d_hn_h + 2(m-2)n_h^2 + 2d_h\beta(d_X,N)n_h.
    \end{aligned}
\end{equation}
The NRDE expression follows because the final layer outputs an element of $\mathbb{R}^{d_h\times \beta(d_X,N)}$ rather than $\mathbb{R}^{d_h\times d_X}$. 
\\ \\
For Log-NCDEs, an exact closed-form FLOP count depends on the implementation of the iterated Lie brackets. 
For fixed $N$, the computational cost of evaluating the Log-NCDE vector field is $\mathcal{O}(\beta(d_X,N))$ with respect to $d_X$, since the vector field is constructed from iterated Lie brackets corresponding to a Hall basis of $\mathfrak{L}^N(\mathbb{R}^{d_X})$. 
The constant hidden in this asymptotic notation grows rapidly with $N$, since higher-order Lie brackets require iterated JVPs of vector fields that are themselves defined recursively through lower-order JVPs. 
For our implementation of the $N=2$ case, the exact expression
\begin{equation}
    F_{\text{Log-NCDE}} = 3d_XF_{\text{NCDE}}
\end{equation}
is obtained, as the number of FLOPs required to calculate a JVP is $3$ times that of evaluating $f_{\theta}$ and $d_X$ JVPs of $f_{\theta}$ are needed to evaluate \eqref{eq:hall} \citep[Chapter~4]{Griewank2000}. 
\\ \\
Therefore, Log-NCDEs and NRDEs have the same asymptotic computational complexity with respect to $d_X$ for fixed $N$. 
Furthermore, each JVP is evaluated at the same point $h_s$. 
This allows the Log-NCDE vector field on high-dimensional time series to be evaluated in parallel. 
This computational advantage is demonstrated empirically in Section \ref{sec:uea}.

\subsection{Limitations}
\label{sec:lim}

In this thesis, we restrict attention to Log-NCDEs based on depth-$1$ and depth-$2$ Log-ODE approximations. 
This reflects two main limitations. 
First, there are currently no theoretical results explicitly bounding the $\mathrm{Lip}(\gamma)$ norm of a neural network for $\gamma>2$. 
Second, as discussed in Section~\ref{sec:cost}, the cost of evaluating $g_{\theta,X}$ grows rapidly with the truncation depth $N$. 
In particular, at $N=2$ this cost scales quadratically in the input dimension $d_X$, which in practice restricts depth-$2$ Log-NCDEs to moderate-dimensional inputs. 
The widest time series considered in our experiments has dimension $d_X=64$. 
Input dimensions in the low hundreds are likely still feasible, whereas dimensions on the order of $10^3$ are likely to be computationally prohibitive.
\\ \\
A more general limitation of NCDEs is that their hidden dynamics must be solved sequentially in time, which prevents parallelisation across time steps. 
This is in contrast to structured state-space models, whose linear dynamics admit explicit flows that can be composed in parallel across time \citep{gu2021efficiently}. 
This limitation serves as the primary motivation for the work of Chapter~\ref{chap:lin_ncde}.

\subsection{Experiments}

\label{sec:log_ncde_experiments}

\subsubsection{Baseline Methods}

Log-NCDEs are compared against six models, which represent the state-of-the-art for a range of deep learning approaches to time series modelling. 
Four of these models are stacked recurrent models, whose general architecture is based on the official implementation of S5 located at \url{https://github.com/lindermanlab/S5} \citep{S5}. 
A recurrent block consists of a batch or layer normalisation  \citep{ioffe2015batch, ba2016layer}, a recurrent layer, a gated linear unit (GLU) \citep{dauphin2017language}, dropout with rate $0.1$ \citep{srivastava2014dropout}, and a skip connection. 
A full model consists of a linear encoder, a number of stacked recurrent blocks, and a final linear layer. 
The four different recurrent layers considered are the LRU~\citep{orvieto2023resurrecting}, S5~\citep{S5}, S6, and Mamba, where S6 refers to the selective state-space recurrence introduced by \citet{gu2024mamba} and Mamba refers to the combination of a gated MLP, convolution, and S6 recurrence \citep{gu2024mamba}. 
S5 and LRU use batch normalisation \citep{ioffe2015batch}, whereas S6 and Mamba use layer normalisation \citep{ba2016layer}.
\\ \\
The other two baseline models are continuous models; an NCDE using a Hermite cubic spline with backward differences as the interpolation and an NRDE \citep{kidger2020neuralcde, morrill2021neuralrough}. 
NCDEs, NRDEs, and Log-NCDEs use a single linear layer as $\xi_{\phi}$. 
NCDEs and NRDEs use FCNNs as their vector fields configured in the same way as their original papers \citep{kidger2020neuralcde, morrill2021neuralrough}. 
NCDEs use $\text{ReLU}$ activation functions for the hidden layers and a final activation function of $\tanh$. 
NRDEs use the same, but move the $\tanh$ activation function to be before the final linear layer in the FCNN. 
Log-NCDEs use a FCNN with SiLU activation functions for the hidden layers and a final activation function of $\tanh$. NRDEs and Log-NCDEs take their intervals $r_{i+1} - r_i$ to be a fixed number of observations, referred to as the Log-ODE step. 

\subsubsection{Toy Dataset}

We construct a toy dataset of $100{,}000$ time series with $6$ channels and $100$ regularly spaced samples each. 
For every time step, the change in each channel is sampled independently from the discrete probability distribution with density
\begin{equation}
    p(n) = \int_{n-0.5}^{n+0.5}\frac{1}{\sqrt{2\pi}}e^{-\frac{1}{2}x^2}\text{d}x,
\end{equation}
where $n\in\mathbb{Z}$. In other words, the change in a channel at each time step is a sample from a standard normal distribution rounded to the nearest integer. 
Figure \ref{fig:toy_example} is a plot of a sample path from the toy dataset. 
\\
\begin{figure}
\centering
\includegraphics[width=1.0\linewidth]{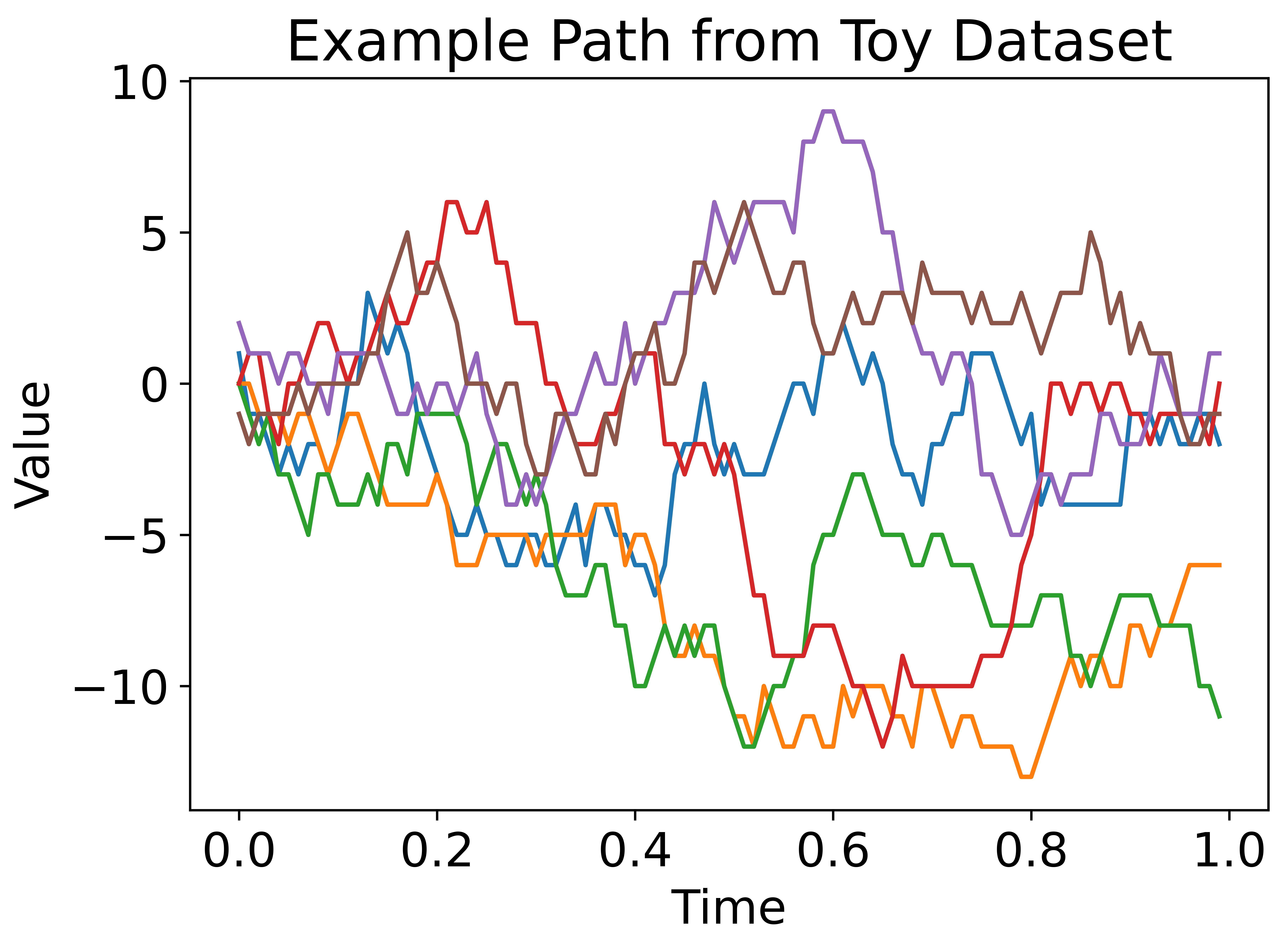}
\caption{An example path from the toy dataset, where each colour represents a channel in the path.}
\label{fig:toy_example}
\end{figure}

We consider four different binary classifications on the toy dataset. 
Each classification is a specific term in the signature of the path which depends on a different number of channels.
\begin{enumerate}
    \item Was the change in the third channel, $\int_0^1\text{d}X^3_s$, greater than zero? 
    \item Was the area integral of the third and sixth channels, $\int_0^1\int_0^u\text{d}X^3_s\text{d}X^6_u$, greater than zero?
    \item Was the volume integral of the third, sixth, and first channels, $\int_0^1\int_0^v\int_0^u\text{d}X^3_s\text{d}X^6_u\text{d}X^1_v$, greater than zero?
    \item Was the $4$D volume integral of the third, sixth, first, and fourth channels, $\int_0^1\int_0^w\int_0^v\int_0^u\text{d}X^3_s\text{d}X^6_u\text{d}X^1_v\text{d}X^4_w$, greater than zero?
\end{enumerate} 
Each task is asking the model to check the sign of a specific term in the signature of the input path.
\\ \\
On the toy dataset, all models use a hidden dimension of $64$ and Adam with a learning rate of $0.0003$ \citep{kingma2017adam}. 
LRU, S5, S6, and Mamba use $6$ blocks and S5, S6, and Mamba use a state dimension of $64$. 
S5 uses $2$ initialisation blocks and Mamba uses a convolution dimension of $4$ and an expansion factor of $2$. 
NCDEs, NRDEs, and Log-NCDEs use a FCNN with width $128$ and depth $3$ as their vector field. 
Furthermore, all NCDEs use Heun as their differential equation solver with a fixed stepsize of $0.01$ \citep{Heun1900, AtkinsonHanStewart2009}. 
NRDEs and Log-NCDEs use a Log-ODE step of $4$ and a signature truncation depth of $2$. 
Log-NCDEs do not use any $\mathrm{Lip}(\gamma)$ regularisation, i.e. $\lambda=0$ in \eqref{eq:weightpenalty}.

\subsubsection{UEA Multivariate Time Series Classification Archive}

The models are evaluated on six datasets from the UEA multivariate time series classification archive (UEA-MTSCA)\footnote{As of November 10$^{\text{th}}$ 2025, the EigenWorms dataset available for download at \url{https://timeseriesclassification.com} has $23$ duplicated time series, which were removed for the experiments in this thesis.} \citep{bagnall2018ueamultivariatetimeseries}. 
These six datasets were chosen via the following two criteria. 
First, only datasets with more than $200$ total time series were considered. 
Second, the six datasets with the most observations were chosen, as datasets with many observations have previously proved challenging for deep learning approaches to time series modelling. 
Table \ref{tab:UEA_summary} provides details on the dimension, number of observations, and number of classes for the datasets chosen from the UEA-MTSCA for the experiments conducted in this thesis. 
Following \citep{morrill2021neuralrough}, the original train and test cases are combined and resplit into new random train, validation, and test cases using a $70:15:15$ split. 
\\
\begin{table}
\centering
\caption{A summary of the subset of the UEA-MTSCA datasets used in this thesis.}
\vspace{0.2cm}
\begin{tabular}{l|c c c c}
Dataset & Dimension & Number of Observations & Classes \\ \hline
EigenWorms & $6$ & $17984$ & $5$ \\
EthanolConcentration & $3$ & $1751$ & $4$ \\
Heartbeat & $61$ & $405$ & $2$ \\
MotorImagery & $64$ & $3000$ & $2$ \\
SelfRegulationSCP1 & $6$ & $896$ & $2$ \\
SelfRegulationSCP2 & $7$ & $1152$ & $2$ 
\end{tabular}
\label{tab:UEA_summary}
\end{table}  

Hyperparameters for all models are found using a grid search over the validation accuracy on a fixed random split of the data. 
Having fixed their hyperparameters, models are compared on their average test set accuracy over five different random splits of the data. 
In order to compare models on their average GPU memory and runtime, $1000$ steps of training are run on an NVIDIA RTX 4090. 
Each training step consists of a forward pass, loss computation, backward pass, and parameter update.
The average runtime is estimated by combining the time for $1000$ training steps with the average total number of training steps from the five runs over the random data splits. 

\subsubsection{PPG-DaLiA}

PPG-DaLiA is a multivariate time series regression dataset, where the aim is to predict a person's heart rate using data collected from a wrist-worn device \citep{Reiss2019DeepPPG}. 
The dataset consists of fifteen individuals with around $150$ minutes of recording each at a maximum sampling rate of $128\,\mathrm{Hz}$. 
There are six channels: blood volume pulse, electrodermal activity, body temperature, and three-axis acceleration. 
For each individual, the data is split into training, validation, and test sets using a $70{:}15{:}15$ split. 
After splitting, a sliding window is applied to each subset to convert the long continuous recordings into shorter overlapping time series. 
Specifically, each window contains $49920$ consecutive samples, and successive windows are offset by $4992$ samples. 
\\ \\
Hyperparameters are found using the same method as for the UEA-MTSCA, but with validation mean squared error and slightly different hyperparameter choices given the high number of observations. 
Having fixed their hyperparameters, models are compared on their average mean squared error over five different runs on the same fixed data split.

\subsubsection{Hyperparameter Optimisation}
\label{sec:hypopt}

The seven models considered in the experiments fall into two groups. 
The first group consists of the stacked recurrent models LRU, S5, S6, and Mamba. 
The second group consists of the continuous-time models NCDE, NRDE, and Log-NCDE. 
Tables \ref{tab:UEA_hypopt_recurrent} and \ref{tab:UEA_hypopt_ncde} list the hyperparameters optimised over for these two groups on the UEA-MTSCA and PPG-DaLiA experiments.
\\ \\
For each dataset and model, hyperparameters were selected by grid search on a fixed training-validation split. 
For the UEA-MTSCA datasets, each configuration was trained using cross-entropy loss and compared using validation accuracy. 
For PPG-DaLiA, each configuration was trained using mean squared error loss and compared using validation mean squared error. 
Training used early stopping based on the corresponding validation metric, and the checkpoint with the best validation performance was used for the final test evaluation. 
After hyperparameter selection, the chosen configuration was used in the evaluation protocols described above, namely five random data splits for the UEA-MTSCA datasets and five random seeds on a fixed split for PPG-DaLiA.
\\ \\
All models and experiments used Adam as their optimiser \citep{kingma2017adam}. 
A batch size of $32$ was used throughout, except for the stacked recurrent models on PPG-DaLiA, where a batch size of $4$ was required due to memory constraints. 
NCDEs, NRDEs, and Log-NCDEs used Heun as their differential equation solver with fixed stepsize
\begin{equation}
    \frac{1}{\max\{500, 1 + (\text{time series length} / \text{Log-ODE step})\}},
\end{equation}
with Log-ODE step equal to $1$ for NCDEs. 
Additionally, Log-NCDEs scaled down their initial FCNN parameters by a factor of $1000$ to reduce the initial $\mathrm{Lip}(2)$ norm of the vector field.

\begin{table}
\caption{Hyperparameters selected by the optimisation for LRU, S5, S6, and Mamba on the UEA-MTSCA datasets and PPG-DaLiA dataset. 
The following abbreviations are used: EigenWorms (EW), EthanolConcentration (EC), Heartbeat (HB), MotorImagery (MI), SelfRegulationSCP1 (SCP1), SelfRegulationSCP2 (SCP2), and PPG-DaLiA (PPG). 
A \ding{55} denotes that the hyperparameter is not applicable to that model.}
\label{tab:UEA_hypopt_recurrent}
\vspace{0.2cm}
\centering
\resizebox{\textwidth}{!}{
\begin{tabular}{l|l|c|c|c|c}
\multirow{2}{*}{Hyperparameters} & \multirow{2}{*}{Options} & \multicolumn{4}{c}{Method} \\ \cline{3-6}
 &  & LRU & S5 & S6 & Mamba \\ \Xhline{2\arrayrulewidth}
Learning Rate & $10^{-3}$ & \makecell{EW, MI, SCP1, \\ SCP2, PPG} & \makecell{HB, MI, SCP1, \\ PPG} & \makecell{EW, HB, MI, \\ SCP2} & EW, EC, PPG \\ \cline{2-6}
 & $10^{-4}$ & HB & EW, SCP2 & SCP1, PPG & HB, SCP2 \\ \cline{2-6}
 & $10^{-5}$ & EC & EC & EC & MI, SCP1 \\ \cline{2-6} \Xhline{2\arrayrulewidth}
Include Time & True & EC, HB, SCP2 & \makecell{EW, EC, SCP2, \\ PPG} & \makecell{EC, HB, MI, \\ PPG} & EW, EC, SCP2 \\ \cline{2-6}
 & False & \makecell{EW, MI, SCP1, \\ PPG} & HB, MI, SCP1 & EW, SCP1, SCP2 & \makecell{HB, MI, SCP1, \\ PPG} \\ \cline{2-6} \Xhline{2\arrayrulewidth}
Hidden Dimension & 16 & MI & MI, SCP2, PPG & \makecell{EW, EC, HB, \\ MI, SCP2} & EW \\ \cline{2-6}
 & 64 & \makecell{EW, EC, SCP1, \\ SCP2} & EW & SCP1, PPG & EC, HB, SCP2 \\ \cline{2-6}
 & 128 & HB, PPG & EC, HB, SCP1 & & MI, SCP1, PPG \\ \cline{2-6} \Xhline{2\arrayrulewidth}
Number of Layers & 2 & HB, SCP1, SCP2 & EW, EC, SCP2 & SCP1, SCP2, PPG & MI, SCP1 \\ \cline{2-6}
 & 4 & EW & HB & EW, EC, HB, MI & EC, HB, PPG \\ \cline{2-6}
 & 6 & EC, MI, PPG & MI, SCP1, PPG & & EW, SCP2 \\ \cline{2-6} \Xhline{2\arrayrulewidth}
State Dimension & 16 & EC, SCP1, SCP2 & \makecell{EW, EC, HB, \\ SCP1} & EC, HB, SCP1 & SCP1 \\ \cline{2-6}
 & 64 & EW & MI, SCP2, PPG & EW, PPG & \makecell{EW, MI, SCP2, \\ PPG} \\ \cline{2-6}
 & 256 & HB, MI, PPG & & MI, SCP2 & EC, HB \\ \cline{2-6} \Xhline{2\arrayrulewidth}
S5 Initialisation Blocks & 2 & \ding{55} & SCP2, PPG & \ding{55} & \ding{55} \\ \cline{2-6}
 & 4 & \ding{55} & HB, MI & \ding{55} & \ding{55} \\ \cline{2-6}
 & 8 & \ding{55} & EW, EC, SCP1 & \ding{55} & \ding{55} \\ \cline{2-6} \Xhline{2\arrayrulewidth}
Convolution Dimension & 2 & \ding{55} & \ding{55} & \ding{55} & EW, HB, SCP2 \\ \cline{2-6}
 & 3 & \ding{55} & \ding{55} & \ding{55} & MI, PPG \\ \cline{2-6}
 & 4 & \ding{55} & \ding{55} & \ding{55} & EC, SCP1 \\ \cline{2-6} \Xhline{2\arrayrulewidth}
Expansion Factor & 1 & \ding{55} & \ding{55} & \ding{55} & EW, MI, SCP1 \\ \cline{2-6}
 & 2 & \ding{55} & \ding{55} & \ding{55} & SCP2, PPG \\ \cline{2-6}
 & 4 & \ding{55} & \ding{55} & \ding{55} & EC, HB \\ \cline{2-6} \Xhline{2\arrayrulewidth}
\end{tabular}
}
\end{table}

\begin{table}
\caption{Hyperparameters selected by the optimisation for NCDE, NRDE, and Log-NCDE on the UEA-MTSCA datasets and PPG-DaLiA dataset. 
Given the length of each time series in the PPG-DaLiA dataset, different choices were considered for the Log-ODE depth and step, which are shown here in red. 
The following abbreviations are used: EigenWorms (EW), EthanolConcentration (EC), Heartbeat (HB), MotorImagery (MI), SelfRegulationSCP1 (SCP1), SelfRegulationSCP2 (SCP2), and PPG-DaLiA (PPG). A \ding{55} denotes that the hyperparameter is not applicable to that model.}
\label{tab:UEA_hypopt_ncde}
\vspace{0.2cm}
\centering
\resizebox{\textwidth}{!}{
\begin{tabular}{l|l|c|c|c}
\multirow{2}{*}{Hyperparameters} & \multirow{2}{*}{Options} & \multicolumn{3}{c}{Method} \\ \cline{3-5}
 &  & NCDE & NRDE & Log-NCDE \\ \Xhline{2\arrayrulewidth}
Learning Rate & $10^{-3}$ & \makecell{EW, EC, HB, \\ MI, SCP2, PPG} & \makecell{EW, EC, HB, \\ SCP1, PPG} & \makecell{EW, HB, MI, \\ PPG} \\ \cline{2-5}
 & $10^{-4}$ & SCP1 & MI, SCP2 & EC, SCP1, SCP2 \\ \cline{2-5}
 & $10^{-5}$ & & & \\ \cline{2-5} \Xhline{2\arrayrulewidth}
Include Time & True & \makecell{EW, EC, HB, \\ MI, PPG} & \makecell{EC, SCP1, SCP2, \\ PPG} & \makecell{EW, EC, HB, \\ PPG} \\ \cline{2-5}
 & False & SCP1, SCP2 & EW, HB, MI & MI, SCP1, SCP2 \\ \cline{2-5} \Xhline{2\arrayrulewidth}
Hidden Dimension & 16 & MI, PPG & & HB, MI \\ \cline{2-5}
 & 64 & & \makecell{EW, EC, HB, \\ SCP1} & EC, SCP1 \\ \cline{2-5}
 & 128 & \makecell{EW, EC, HB, \\ SCP1, SCP2} & MI, SCP2, PPG & EW, SCP2, PPG \\ \cline{2-5} \Xhline{2\arrayrulewidth}
Vector Field (Depth, Width) & (2, 32) & & & EW, SCP2 \\ \cline{2-5}
 & (3, 64) & & EW & EC, MI, PPG \\ \cline{2-5}
 & (3, 128) & EW & HB & HB, SCP1 \\ \cline{2-5}
 & (4, 128) & \makecell{EC, HB, MI, \\ SCP1, SCP2, PPG} & \makecell{EC, MI, SCP1, \\ SCP2, PPG} & \\ \cline{2-5} \Xhline{2\arrayrulewidth}
Log-ODE (Depth, Step) & (1, 1) & \ding{55} & EC, MI, SCP1, SCP2 & EC \\ \cline{2-5}
 & (2, 2) & \ding{55} & HB & HB \\ \cline{2-5}
 & (2, 4) & \ding{55} & EW & SCP2 \\ \cline{2-5}
 & (2, 8) & \ding{55} & & \\ \cline{2-5}
 & (2, 12) & \ding{55} & & EW \\ \cline{2-5}
 & (2, 16) & \ding{55} & & MI, SCP1 \\ \cline{2-5}
 & \textcolor{red}{(1, 10)} & \ding{55} & PPG & PPG \\ \cline{2-5}
 & \textcolor{red}{(2, 10)} & \ding{55} & & \\ \cline{2-5}
 & \textcolor{red}{(2, 100)} & \ding{55} & & \\ \cline{2-5}
 & \textcolor{red}{(2, 1000)} & \ding{55} & & \\ \cline{2-5} \Xhline{2\arrayrulewidth}
Regularisation $\lambda$ & $10^{-3}$ & \ding{55} & \ding{55} & EW, MI, SCP2 \\ \cline{2-5}
 & $10^{-6}$ & \ding{55} & \ding{55} & EC, HB \\ \cline{2-5}
 & $0.0$ & \ding{55} & \ding{55} & SCP1, PPG \\ \cline{2-5} \Xhline{2\arrayrulewidth}
\end{tabular}
}
\end{table}

\newpage 
\subsection{Results}
\label{sec:results}

\subsubsection{Toy Dataset}

Figure \ref{fig:toy_results} compares the performance of the models on the four different toy dataset classifications. 
As expected, given that the classifications considered are solutions to CDEs, NCDEs are the best performing model. 
Since NRDEs and Log-NCDEs are fixed to $r_{i+1}-r_i$ being $4$ observations and $N=2$, they are both approximations of a CDE. 
Notably, Log-NCDEs consistently outperform NRDEs, providing empirical evidence that NRDEs do not always accurately learn the Lie bracket structure of $\bar{f}_{\theta}$. 
All of the stacked recurrent models perform well when the label depends on one or two channels. 
However, their performance begins to decrease for three channels, and only Mamba performs well when the label depends on four channels.

\begin{figure}
\centering
    \includegraphics[width=\linewidth]{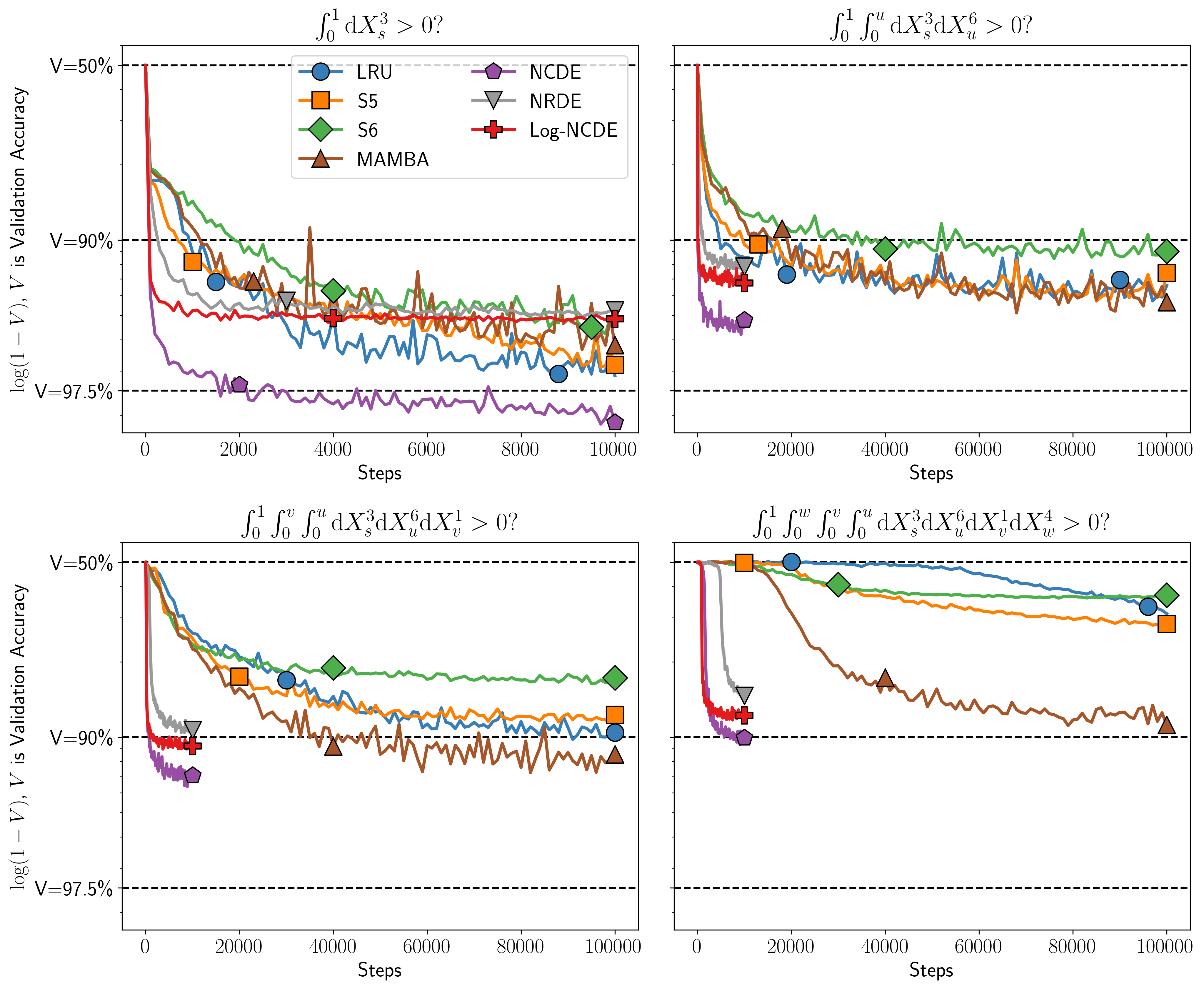}
    \caption{Validation accuracy against number of steps for LRU, S5, S6, Mamba, NCDE, NRDE, and Log-NCDE on the four different classifications considered for the toy dataset.}
    \label{fig:toy_results}
\end{figure}

\subsubsection{UEA-MTSCA}
\label{sec:uea}

Table~\ref{tab:UEA_results_hypopt} reports the mean and standard deviation of each model's test set accuracy over five data splits. 
Among the stacked recurrent models, LRU, S5, and S6 achieve similar average accuracies overall.
However, Mamba attains the lowest average accuracy of all seven methods. 
Since S6 achieves performance in line with S5, Mamba's weaker results do not appear to be caused by the selective recurrence itself. 
Instead, they may reflect the effect of Mamba's additional architectural components, particularly its short convolution, which places greater weight on local information and may therefore be less well suited to classification tasks that reward preserving information across the full time series.
\\ \\
NCDEs and NRDEs have similar average accuracies overall, although NRDEs perform notably better on EigenWorms, the dataset with the most observations. 
However, NRDEs are still outperformed in average accuracy by the stacked recurrent models LRU, S5, and S6. 
In contrast, Log-NCDEs achieve the best average accuracy and the best average rank across the six datasets. 
Compared to NRDEs, they attain an equal or higher average accuracy on all six datasets and a lower standard deviation on four datasets.
\\ \\
A Friedman test across the six datasets and seven methods did not detect a statistically significant difference in performance among the models at the $5\%$ significance level, with $\chi^2 = 9.69$, $df = 6$, and $p = 0.138$ \citep{demsar2006statistical}. 
Since Log-NCDEs are a direct modification of NCDEs and NRDEs, we additionally performed one-sided Wilcoxon signed-rank tests comparing Log-NCDEs against these two baselines, with Holm correction across the two comparisons \citep{wilcoxon1945individual}. 
At the $5\%$ significance level, these tests indicated that Log-NCDEs significantly outperformed both NCDEs and NRDEs, with adjusted $p = 0.0313$ in each case. 
These results suggest that incorporating Lie bracket information can improve predictive performance.

\begin{table}
\small
\caption{Mean and standard deviation of test set accuracy over five data splits on a subset of the UEA-MTSCA.
The best performing model is highlighted in bold and the second best is underlined. The average accuracy and average rank are also reported.}
\label{tab:UEA_results_hypopt}
\centering
\resizebox{\textwidth}{!}{
\begin{tabular}{l|c|c|c|c|c|c|c}
\hline
\multirow{2}{*}{Dataset} & \multicolumn{7}{c}{Method} \\ \cline{2-8}
& LRU & S5 & S6 & Mamba & NCDE & NRDE & Log-NCDE \\ \hline
EigenWorms & $\mathbf{87.8 \pm 2.8}$ & $81.1 \pm 3.7$ & $85.0 \pm 16.1$ & $70.9 \pm 15.8$ & $75.0 \pm 3.9$ & $83.9 \pm 7.3$ & $\underline{85.6 \pm 5.1}$ \\
EthanolConcentration & $21.5 \pm 2.1$ & $24.1 \pm 4.3$ & $26.4 \pm 6.4$ & $27.9 \pm 4.5$ & $\underline{29.9 \pm 6.5}$ & $25.3 \pm 1.8$ & $\mathbf{34.4 \pm 6.4}$ \\
Heartbeat & $\mathbf{78.4 \pm 6.7}$ & $\underline{77.7 \pm 5.5}$ & $76.5 \pm 8.3$ & $76.2 \pm 3.8$ & $73.9 \pm 2.6$ & $72.9 \pm 4.8$ & $75.2 \pm 4.6$ \\
MotorImagery & $48.4 \pm 5.0$ & $47.7 \pm 5.5$ & $\underline{51.3 \pm 4.7}$ & $47.7 \pm 4.5$ & $49.5 \pm 2.8$ & $47.0 \pm 5.7$ & $\mathbf{53.7 \pm 5.3}$ \\
SelfRegulationSCP1 & $82.6 \pm 3.4$ & $\mathbf{89.9 \pm 4.6}$ & $82.8 \pm 2.7$ & $80.7 \pm 1.4$ & $79.8 \pm 5.6$ & $80.9 \pm 2.5$ & $\underline{83.1 \pm 2.8}$ \\
SelfRegulationSCP2 & $51.2 \pm 3.6$ & $50.5 \pm 2.6$ & $49.9 \pm 9.5$ & $48.2 \pm 3.9$ & $53.0 \pm 2.8$ & $\underline{\mathbf{53.7 \pm 6.9}}$ & $\underline{\mathbf{53.7 \pm 4.1}}$ \\ \hline
Av. & $61.7$ & $61.8$ & $\underline{62.0}$ & $58.6$ & $60.2$ & $60.6$ & $\mathbf{64.3}$ \\
Av. Rank & $\underline{3.5}$ & $4.0$ & $\underline{3.5}$ & $5.5$ & $4.5$ & $4.9$ & $\mathbf{2.1}$ \\
\end{tabular}
}
\end{table}

\subsubsection{PPG-DaLiA}

\begin{table}
\small
\caption{Mean and standard deviation of test set mean squared error over five runs with different random seeds on the PPG-DaLiA dataset.}
\label{tab:PPG_results_hypopt}
\centering
\begin{tabular}{c|c}
Model & MSE $(\times 10^{-2})$ \\ \hline
LRU & $12.17 \pm 0.49$ \\
S5 & $12.63 \pm 1.25$ \\
S6 & $12.88 \pm 2.05$ \\
Mamba & $10.65 \pm 2.20$ \\
NCDE & $13.54 \pm 0.69$ \\
NRDE & $\underline{9.90\pm 0.97}$ \\
Log-NCDE & $\mathbf{9.56 \pm 0.59}$ 
\end{tabular}
\end{table}

Table \ref{tab:PPG_results_hypopt} contains the average and standard deviation of each model's test set mean squared error on the PPG-DaLiA dataset. 
In contrast to the UEA-MTSCA experiments, Mamba is the best performing stacked recurrent model on PPG-DaLiA and clearly outperforms S6. 
This suggests that Mamba's additional architectural components, such as its short convolution, are beneficial for this regression task. 
This aligns with the intuition that heart-rate prediction depends heavily on recent observations.
Among the neural differential equation models, both NRDEs and Log-NCDEs substantially outperform the NCDE, indicating that the Log-ODE based models are better suited to handling very long sequences. 
Log-NCDEs still achieve the best overall performance, obtaining the lowest average test set mean squared error and the second lowest standard deviation. 

\subsubsection{Memory and Time}

\begin{figure}
\centering
    \begin{subfigure}{0.48\textwidth}
        \includegraphics[width=\linewidth]{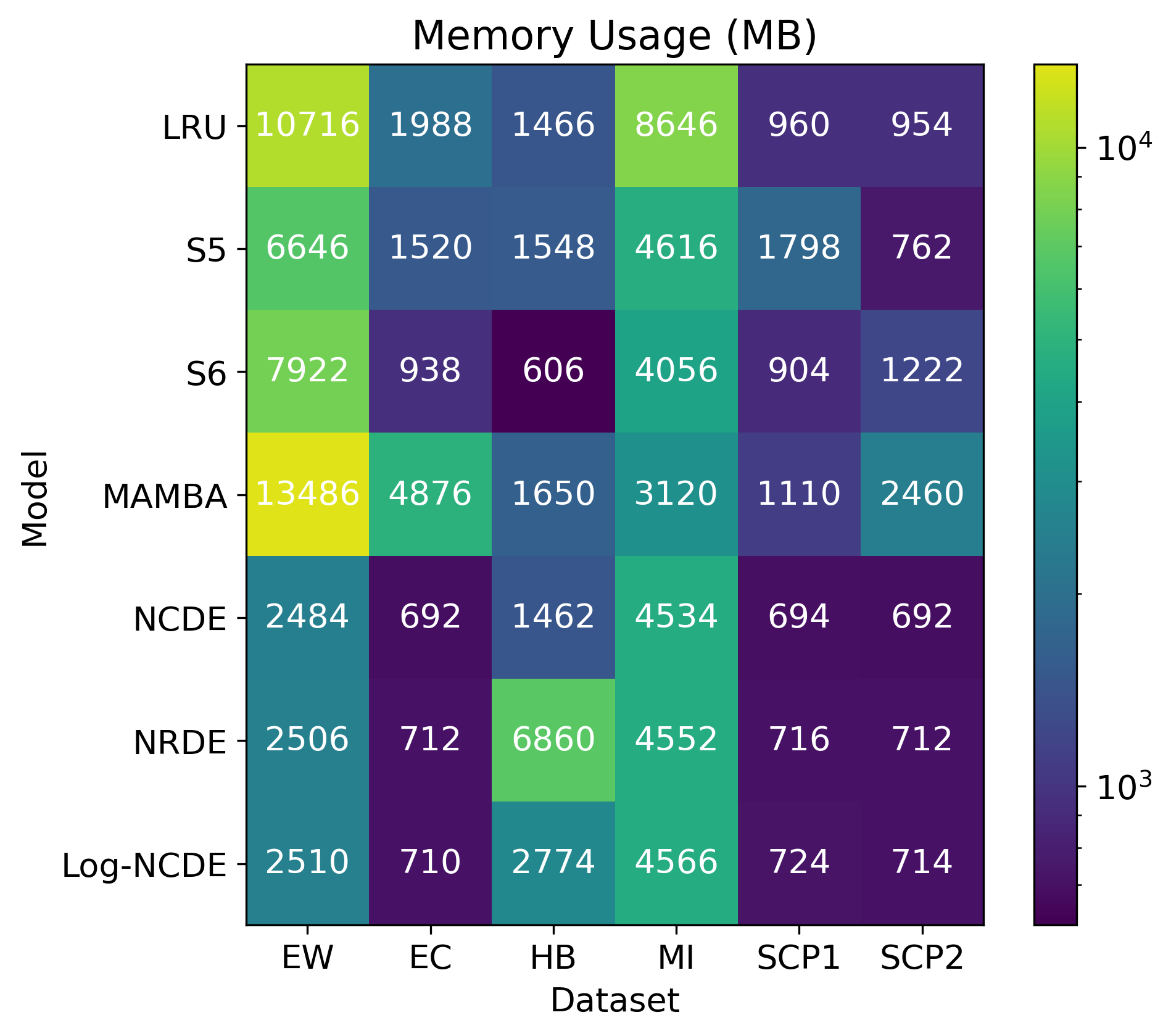}
        \caption{Memory}
        \label{fig:mem}
    \end{subfigure}
    \begin{subfigure}{0.48\textwidth}
        \includegraphics[width=\linewidth]{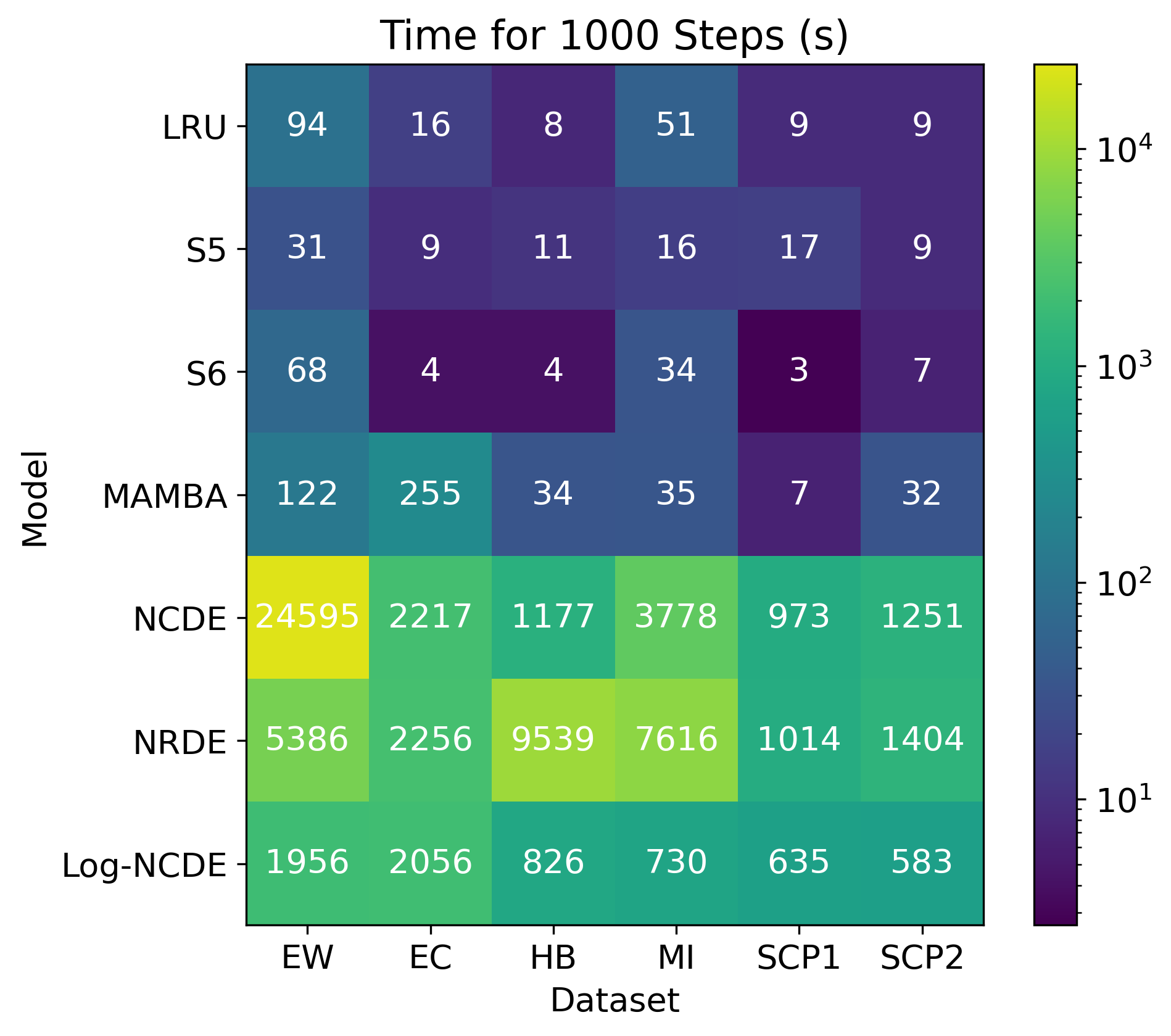}
        \caption{Time}
        \label{fig:time}
    \end{subfigure}

    \medskip
    \begin{subfigure}{0.48\textwidth}
        \includegraphics[width=\linewidth]{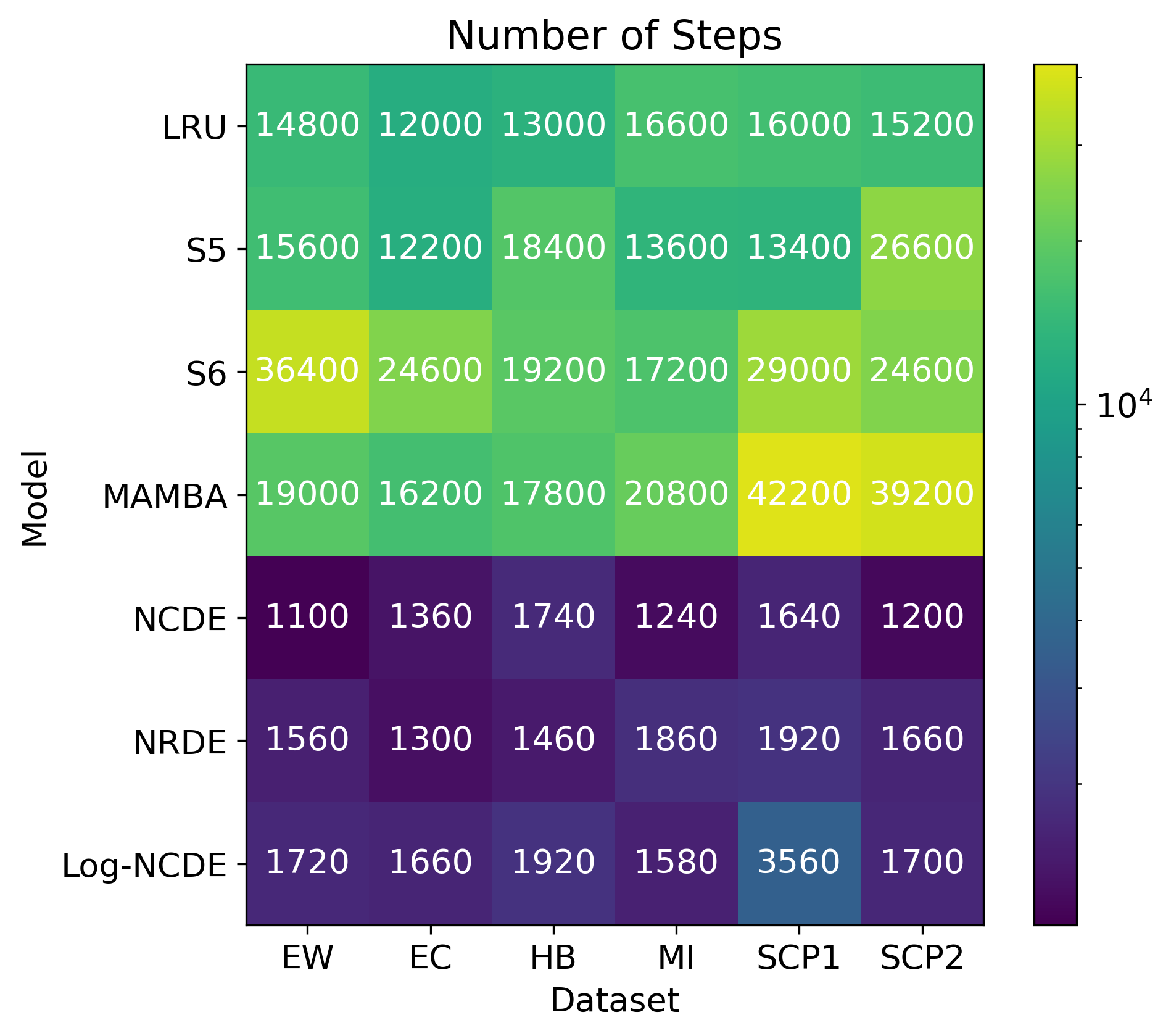}
        \caption{Number of steps}
        \label{fig:num_steps}
    \end{subfigure}
    \begin{subfigure}{0.48\textwidth}
        \includegraphics[width=\linewidth]{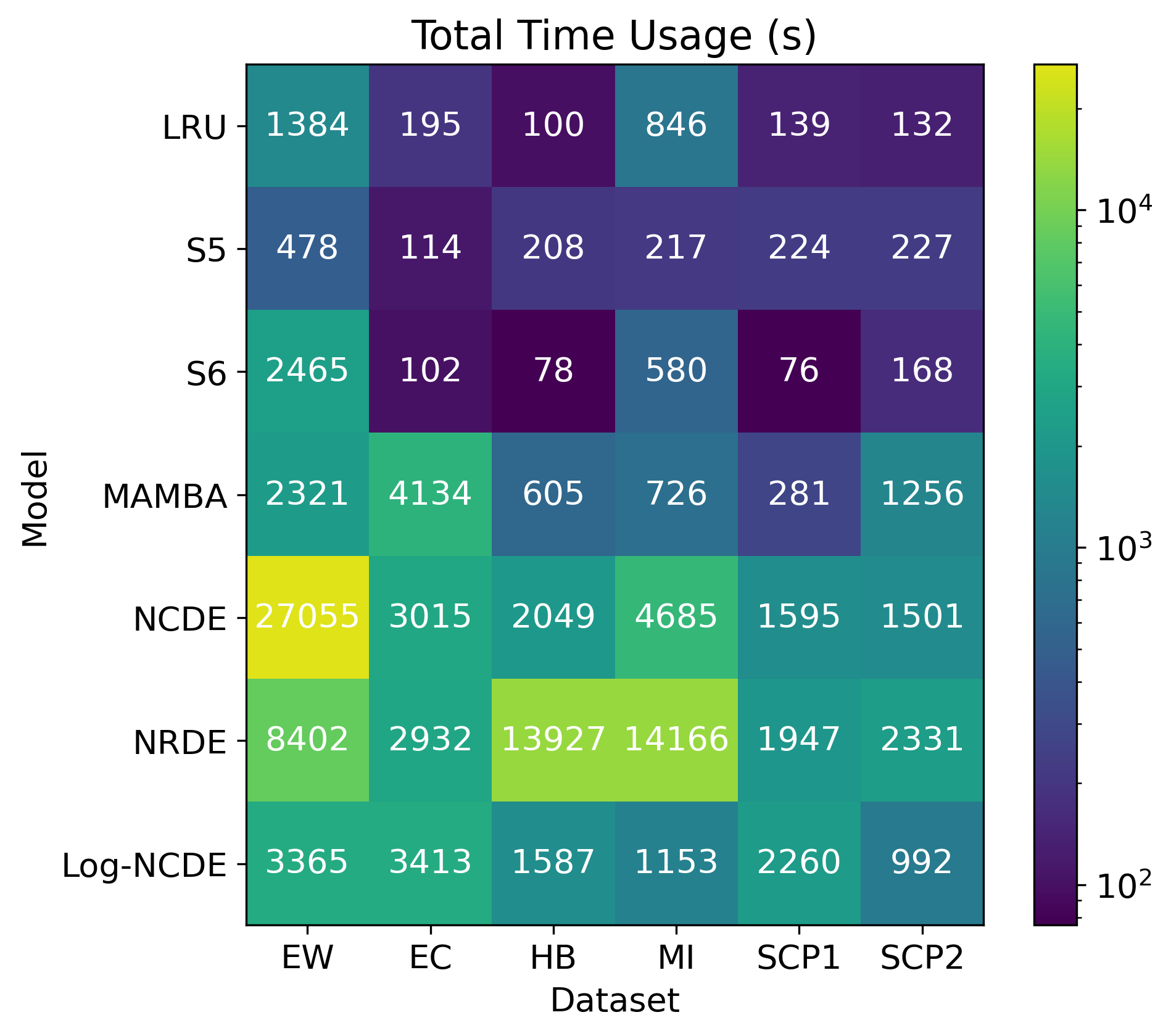}
        \caption{Total time}
        \label{fig:total_time}
    \end{subfigure}
    \caption{Memory, time per $1000$ training steps, number of steps, and approximate total time for each model and dataset from the UEA-MTSCA on an NVIDIA RTX 4090. 
    The following abbreviations are used: EigenWorms (EW), EthanolConcentration (EC), Heartbeat (HB), MotorImagery (MI), SelfRegulationSCP1 (SCP1), and SelfRegulationSCP2 (SCP2).}
    \label{fig:mem_time}
\end{figure}

Models are compared on their average GPU memory usage and runtime for the UEA-MTSCA datasets. 
In order to compare the models, $1000$ steps of training were run on an NVIDIA RTX 4090 with each model using the hyperparameters obtained from the hyperparameter optimisation, as detailed in Tables \ref{tab:UEA_hypopt_recurrent} and \ref{tab:UEA_hypopt_ncde}. 
In addition to the time for $1000$ steps and GPU memory usage, shown in Figures \ref{fig:mem} and \ref{fig:time}, the average number of total training steps taken to produce the results in Table \ref{tab:UEA_results_hypopt} is recorded in Figure \ref{fig:num_steps}. 
Combining the results for time per $1000$ training steps and the total number of training steps gives an approximation of the total runtime on the same hardware, and these results are shown in Figure \ref{fig:total_time}. 
The average GPU memory and runtime across the six datasets is given in Table \ref{tab:UEA_results_mem}.
\\ \\
Although the time per training step is lower for stacked recurrent models than NCDEs, NRDEs, or Log-NCDEs, they also require more training steps to converge. 
Additionally, NCDEs, NRDEs, and Log-NCDEs require less GPU memory. 
The largest contributors to the average runtime of NCDEs are the datasets with the most observations, EigenWorms and MotorImagery. 
The positive impact of the Log-ODE method on computational burden is demonstrated empirically by the decrease in runtime achieved by NRDEs and Log-NCDEs on EigenWorms when using a depth$-2$ Log-ODE method. 
When a depth$-1$ Log-ODE method is used, such as NRDEs on MotorImagery, the same decrease is not observed. 
\\ \\
Section \ref{sec:cost} demonstrated that Log-NCDEs and NRDEs have the same asymptotic computational complexity. 
However, when using a depth$-2$ Log-ODE approximation and the same stepsize, NRDEs and Log-NCDEs exhibit different runtimes on Heartbeat, a high-dimensional dataset. 
This difference is partly explained by the model's having different optimal hyperparameter choices, but even when using identical hyperparameters to the NRDE, Log-NCDE's time per $1000$ training steps is $1673$ seconds, whereas NRDE's is $9539$ seconds. 
The remaining difference is due to being able to calculate the JVPs of $f_{\theta}$ in parallel, as discussed in Section \ref{sec:cost}. 
If instead the JVPs are calculated recurrently, then Log-NCDEs time per $1000$ training steps increases to $17045$ seconds.
\\ \\
From a practical standpoint, the Log-ODE method improves the computational viability of NCDE-style models, but it does not eliminate the underlying scalability challenge. Log-NCDEs are markedly faster than NCDEs and NRDEs on the longest datasets, use relatively little GPU memory, and still achieve strong empirical performance. They also naturally accommodate irregularly sampled and over-sampled data, unlike the stacked recurrent baselines considered here. However, they remain much slower than those recurrent models, which limits their suitability for rapid experimentation, large-scale hyperparameter tuning, and settings requiring very long training runs.

\begin{table}
\small
\caption{Average GPU memory and runtime for each model over the six datasets from the UEA-MTSCA experiments.}
\label{tab:UEA_results_mem}
\centering
\begin{tabular}{l|c|c}
Model & Av. GPU Mem. (MB) & Av. runtime (s) \\ \hline
LRU &  4121.67 & 466.09 \\
S5 &  2815.00 & 244.78 \\
S6 &  2608.00 & 578.15 \\
Mamba &  4450.33 & 1553.83 \\
NCDE &  1759.67 & 6649.91 \\
NRDE &  2676.33 & 7284.20 \\
Log-NCDE &  1999.67 & 2128.32 \\
\end{tabular}
\end{table}

\section{Conclusion}

This chapter showed that the Log-ODE method provides a principled way to approximate NCDE dynamics during training. 
It built on the groundwork laid by NRDEs and used the $\mathrm{Lip}(\gamma)$ theory developed in Chapter \ref{chap:lipgamma} to ensure that the resulting models are well defined. 
Empirically, Log-NCDEs matched or exceeded the performance of NCDEs and NRDEs on all six real-world multivariate time series classification datasets considered, while reducing the average runtime by more than a factor of $3$.
\\ \\
Log-NCDEs also achieved the best average test accuracy and the best average rank among all seven models in Table \ref{tab:UEA_results_hypopt}. 
However, their computational cost remains substantial. 
The average runtime of Log-NCDEs is still nearly an order of magnitude larger than that of the fastest baseline, S5. 
The gap is even more striking in the time per $1000$ training steps, which is $1131$ seconds for Log-NCDEs and only $16$ seconds for S5. 
Thus, although Log-NCDEs are more practical than NCDEs and NRDEs, they remain too slow for genuinely large-scale applications.
\\ \\
The reason for this remaining gap is the nature of the hidden-state dynamics. 
Log-NCDEs still require the numerical solution of a non-linear controlled differential equation during each forward pass. 
By contrast, S5 is based on linear dynamics whose flow can be written in closed form on each interval. 
This removes the need for an expensive non-linear solve and allows the hidden state over a sequence of length $n$ to be computed in only $\mathcal{O}(\log n)$ parallel steps via an associative scan. 
\\ \\ 
Motivated by this observation, Chapter \ref{chap:lin_ncde} introduces Linear NCDEs, where the vector fields are constrained to be linear in the hidden state. 
This restriction yields closed-form solutions for the dynamics on each interval, making parallel-in-time computation possible while retaining the theoretical expressivity and continuous-time structure of NCDEs. 
Furthermore, Log-NCDEs naturally translate to this linear setting, where the Log-ODE method continues to provide substantial empirical runtime benefits.

\chapter{Linear Neural Controlled Differential Equations}

\label{chap:lin_ncde}

\begin{quoting}
    Aren't linear CDEs expressive enough?
\end{quoting}
\noindent\large{---Massimiliano Gubinelli (2023)}
\normalsize

\section{Introduction}

Chapter~\ref{chap:cde} developed the continuous-time mathematical framework underlying this thesis. 
Paths are the fundamental objects, signatures and log-signatures provide principled summaries of path segments, CDEs describe how paths influence the state of a system, and the Log-ODE method provides an efficient and accurate approximation to the solution of a CDE. 
Building on the regularity theory developed in Chapter~\ref{chap:lipgamma}, Chapter~\ref{chap:ncde} translated this framework into machine learning through Log-NCDEs, showing that continuous-time models can achieve performance comparable to strong discrete baselines on regularly sampled datasets, whilst naturally being able to handle irregularly sampled and over-sampled time series.
However, Chapter~\ref{chap:ncde} also showed that the non-linear hidden-state dynamics remain a major computational bottleneck. 
Even with the Log-ODE method, Log-NCDEs are still substantially slower than the strongest stacked recurrent baselines. 
The aim of this chapter is to retain the empirical performance and continuous-time advantages of NCDEs while replacing their non-linear dynamics with a more scalable model class.
\\ \\
From the perspective of CDE theory, the linear case is the natural place to look. 
Linear CDEs admit explicit solutions, as discussed in Section~\ref{sec:lin_cde}, and the truncated signature is itself the solution of a linear CDE, as shown in Section~\ref{sec:cde_def}. 
Moreover, when equipped with a linear readout, linear CDEs are maximally expressive on the space of time-augmented driving paths with the $1$-variation topology, by Corollary~\ref{cor:universality-on-paths} and the proof of Theorem~\ref{thm:universal_ncde}. 
These facts suggest that vector fields which are linear in the hidden state may already be sufficient for time series modelling. 
Given this, it may seem surprising that Linear NCDEs were developed after their non-linear counterparts. 
The epigraph, a question posed by Prof.\ Massimiliano Gubinelli after a presentation on Log-NCDEs, captures this sentiment.
Historically, non-linear NCDEs were explored first in order to maximise modelling capacity and to align more closely with non-linear recurrent neural networks. 
Chapter~\ref{chap:ncde} shows that this choice comes at a significant computational cost.
\\ \\
In this chapter, we show that constraining an NCDE's vector field to be linear in the hidden state leads to the Log-ODE method producing closed-form flows on each interval. 
This removes the need for a differential equation solver, enables parallel-in-time computation, and directly addresses the scalability limitations identified in Chapter~\ref{chap:ncde}. 
We also show that Linear NCDEs retain the theoretical expressivity of NCDEs. 
Finally, we demonstrate empirically that, when combined with the Log-ODE method, Linear NCDEs match the performance of Log-NCDEs while reducing the time per training step by up to two orders of magnitude. 
This makes continuous-time machine learning practical at scales that were previously infeasible.

\section{Linear NCDEs}

\label{sec:lncde}

\subsection{Introduction}

\begin{definition}[Linear NCDE \citep{cirone2024deepSSM, walker2025structuredlinearcdesmaximally}]
\label{def:lin_NCDE}
    Let $\mathcal{X}(d)$ denote the space of bounded-variation $d$-dimensional paths on the interval $[t_0,t_n]$ which all begin at the same point and contain time as a channel. 
    Let $\omega:\mathcal{X}(d_X) \to \mathcal{X}(d_{\omega})$ be a continuous function on paths, with the shorthand $\omega^X=\omega(X)$.
    Let $L^1_{\theta}\in\mathbb{R}^{d_h \times d_X}$, $A_{\theta} \in \mathbb{R}^{d_h\times d_{\omega} \times d_h}$, and $L^2_{\theta}\in\mathbb{R}^{d_y \times d_h}$ be trainable parameters. 
    Then a Linear NCDE is defined by
    \begin{equation}\label{eq:lin_ncde}
    \begin{aligned}
        h_{t_0} &= L^1_{\theta}X_{t_0}, \\  
        h_t &= h_{t_0} + \int_{t_0}^t A_{\theta} h_s \,\mathrm{d}\omega^{X}_s, \\
        y_t &= L^2_{\theta}h_t. 
    \end{aligned}
    \end{equation}
\end{definition}

There are two core differences between NCDEs and Linear NCDEs. 
First, the driving path is a function of the interpolated data, $\omega^X=\omega(X)$, rather than the interpolation $X$ itself. 
Second, the vector field is linear in the hidden state,
\begin{equation}
   \label{eq:lncde} 
    h_t = h_{t_0} + \int_{t_0}^t A_{\theta}h_s\,\mathrm{d} \omega^{X}_s 
    = h_{t_0} + \int_{t_0}^t \sum_{i=1}^{d_{\omega}} A^i_{\theta}h_s\,\mathrm{d} \omega^{X,i}_s,
\end{equation}
where $A^i_{\theta}\in\mathbb{R}^{d_h \times d_h}$ are the parameters corresponding to the $i^{\text{th}}$ channel of the driving path $\omega^{X,i}_s$. 
As discussed in Section~\ref{sec:lin_cde}, \eqref{eq:lncde} admits an explicit solution in terms of the signature of the driving path $\omega^X$. 
Although exact, this representation is not directly practical for computation. 
However, when $\omega^X$ is piecewise linear, the flow of \eqref{eq:lncde} admits an explicit solution. 
More generally, the approximate flows generated by the Log-ODE method also admit explicit solutions. 
This is the key structural property that makes Linear NCDEs more scalable than NCDEs, since an explicit expression for the flow determines the evolution of any initial hidden state, rather than requiring each trajectory to be obtained by numerically solving a non-linear differential equation.

\subsection{Computing the Flow}
\label{sec:flow}

Let $\omega^X$ be piecewise-linear on the grid $t_0<\cdots<t_n$.
On each subinterval $[t_j,t_{j+1}]$, the dynamics are constant, so the update can be computed exactly,
\begin{equation}
\label{eq:lin_cde_logode_1}
\tilde{h}_{t_{j+1}}=\exp\left(\sum_{i=1}^{d_\omega}(\omega^{X,i}_{t_{j+1}}-\omega^{X,i}_{t_j})A^i\right)\tilde{h}_{t_j}.
\end{equation}
Equation \eqref{eq:lin_cde_logode_1} does more than simply remove the general-purpose ODE solver from NCDE inference, it represents a fundamental change in approach. 
Rather than tracking a trajectory $t\mapsto h_t$, we compute the state-transition map from $t_j$ to $t_{j+1}$,
\begin{equation}
    \label{eq:lin_ncde_flow}
    \exp\left(\sum_{i=1}^{d_\omega}(\omega^{X,i}_{t_{j+1}}-\omega^{X,i}_{t_j})A^i\right),
\end{equation}
which is also known as the flow. 
This shift in perspective allows for the flow over any interval to be solved for using only $\mathcal{O}(\log(n))$ parallel steps by using a parallel associative scan.
\\ \\
Given  a sequence $(x_1,\ldots,x_{n})$ and a binary operation $\otimes$, a prefix scan returns the running prefix products
\begin{equation}
    y_k=x_1\otimes \cdots \otimes x_k.
\end{equation}
Given the natural recurrence $y_{k} = y_{k-1} \otimes x_k$, a scan can be computed in $\mathcal{O}(n)$ work and $\mathcal{O}(n)$ recurrent steps. If the operation is associative,
\begin{equation}
    (a \otimes b) \otimes c = a \otimes (b \otimes c),
\end{equation}
then it is possible to compute the prefix scan in only $\mathcal{O}(\log(n))$ parallel steps using a tree-based algorithm \citep{blelloch1993prefix}. We demonstrate the approach for $(x_1,\ldots, x_8)$. First you compute a parallel up-sweep,
\begin{equation}
\begin{aligned}
&\textbf{Level 1}: s_1=x_1 \otimes x_2, \quad s_2=x_3 \otimes x_4, \quad
s_3=x_5 \otimes x_6, \quad s_4=x_7 \otimes x_8, \\[2pt]
&\textbf{Level 2}: t_1 = s_1 \otimes s_2, \quad t_2 = s_3 \otimes s_4, \\[2pt]
&\textbf{Level 3}: u=t_1 \otimes t_2.
\end{aligned}
\end{equation}
This calculation created a tree, with $u$ as the root, $t_1$ and $t_2$ as $u$'s left and right children, and so on. 
Now, you perform a parallel down-sweep of this tree, with the rule that given incoming carry $P$, the left child receives $P$ and the right child receives $P \otimes L$, where $L$ is the subtree with the left child as the root.
Initialising the root's carry as the identity, $P(u)=e$, and applying the carry rule to our example,
\begin{equation}
\begin{aligned}
&\textbf{Level 3}: &&P(t_1)=e,\quad &&P(t_2)=e\otimes t_1=t_1,\\[2pt]
&\textbf{Level 2}: &&P(s_1)=e,\quad &&P(s_2)=e\otimes s_1=s_1.\\
&&&P(s_3)=t_1,\quad &&P(s_4)=t_1\otimes s_3,\\[2pt]
&\textbf{Level 1}: &&P(x_1)=e,\quad &&P(x_2)=e\otimes x_1 =x_1,\\
&&&P(x_3)=s_1,\quad &&P(x_4)=s_1\otimes x_3,\\
&&&P(x_5)=t_1,\quad &&P(x_6)=t_1\otimes x_5,\\
&&&P(x_7)=t_1\otimes s_3,\quad &&P(x_8)=t_1\otimes s_3\otimes x_7.
\end{aligned}
\end{equation}
These carries are exactly the exclusive prefixes and the inclusive outputs are $y_k=P(x_k)\otimes x_k$.
\\ \\
For \eqref{eq:lin_cde_logode_1}, the scan elements are the flows \eqref{eq:lin_ncde_flow} indexed by $j$, and the associative binary operation is matrix multiplication. 
Once the flows over $[t_0,t_j]$ have been obtained for $j=1,\ldots,n$, a batched matrix-vector multiplication with $h_{t_0}$ yields the hidden states $h_{t_j}$ for $j=1,\ldots,n$, so inference can be completed in $\mathcal{O}(\log(n))$ parallel steps. 
However, a scan over all $n$ observation intervals requires materialising $n$ matrices in $\mathbb{R}^{d_h\times d_h}$.
Therefore, when $d_h$ is large, the practical cost can be dominated by GPU memory traffic \citep{yang2024parallelizing}. 
The Log-ODE method mitigates this by replacing the fine observation grid with a coarser partition $t_0=r_0<\cdots<r_m=t_n$, so that the scan is performed over only $m<n$ interval flows.
\\ \\
An equivalent viewpoint of the approximation \eqref{eq:lin_cde_logode_1} is applying a depth$-1$ Log-ODE method to \eqref{eq:lin_ncde} on the grid $\{t_j\}_{j=0}^n$. Extending this approach to a depth$-N$ Log-ODE method on generic intervals $t_0=r_0<\cdots<r_m=t_n$ produces flows
\begin{equation}
    \exp\left(\sum_{k=1}^{\beta(d_{\omega}, N)} \bar{A}^k_{\theta}\lambda^l_k\right),
\end{equation}
where 
\begin{equation}
    \log(S^N(\omega^X)_{[r_l,r_{l+1}]}) = \sum_{k=1}^{\beta(d_{\omega}, N)} \lambda^l_k\hat{e}_k,
\end{equation}
$\hat{e}_k$ is a Hall basis for the space where the depth-$N$ truncated log-signature lives,
\begin{equation}
    \bar{A}^k_{\theta} = A^k_{\theta}
\end{equation}
for $1\leq k\leq d_{\omega}$, and 
\begin{equation}
    \bar{A}^k_{\theta} = [\bar{A}^i_{\theta}, \bar{A}^j_{\theta}] = \bar{A}^i_{\theta}\bar{A}^j_{\theta} - \bar{A}^j_{\theta}\bar{A}^i_{\theta}
\end{equation}
when the basis element $\hat{e}_k$ corresponds to the Lie bracket of $\hat{e}_i$ and $\hat{e}_j$ \citep{reutenauer1993free}. 
Further details are given in Sections \ref{sec:logode} and \ref{sec:log_ncde}.
The key point is that the Log-ODE method again produces a linear flow on each interval, so the parallel associative scan remains applicable. 
Moreover, the Lie brackets $\bar{A}^k_{\theta}$ are computed from products and commutators of the matrices $A^i_{\theta}$, rather than from the forward-mode auto-differentiated Jacobian-vector products required in Log-NCDEs, which substantially reduces the computational cost. 
It also reduces the memory and I/O cost of the scan, since only $m<n$ state-transition matrices must be materialised in GPU memory. 
We refer to the combination of Linear NCDEs and the Log-ODE method as Log-Linear NCDEs.

\subsection{Expressivity}

The previous subsection showed that Linear NCDEs offer substantial computational advantages. 
This raises the question of whether these gains come at the cost of expressivity. 
For a fixed hidden dimension, Linear NCDEs are less expressive than NCDEs, since the former are a strict subclass of the latter. 
However, once the hidden dimension is allowed to vary, both model classes have the same theoretical expressivity.

\begin{theorem}[Maximally Expressive Linear NCDEs \citep{cirone2024deepSSM}]
    \label{thm:universal_lin_ncde}
    Let $\mathcal{X}$ be the space of bounded variation paths on the interval $[t_0,t_n]$ that start at a common point and include time as a channel, endowed with the $1-$variation topology. 
    Let $\mathcal{F}$ be the class of real-valued maps $X \mapsto y_{t_n}$ induced by Linear NCDEs from Definition \ref{def:lin_NCDE} with $\omega^X_s=X_s$ for $X\in\mathcal{X}$, $d_h\in\mathbb{N}$, and $d_y=1$. 
    Then $\mathcal{F}$ is maximally expressive on $\mathcal{X}$.
\end{theorem}

\begin{proof}
    The result follows from the universality of the signature, Corollary~\ref{cor:universality-on-paths}, together with the fact that the truncated signature solves a linear CDE, as shown in \eqref{eq:sig_cde}. For more details, see \citep[Theorem B.13]{cirone2024deepSSM}.
\end{proof}

Although a Linear CDE is linear in the hidden state, its evolution is driven by the multiplicative interaction between the hidden state and the increments of the driving path. 
This is precisely the mechanism that generates iterated integrals, since each tensor level of the truncated signature is obtained by integrating the previous level against the path. 
Therefore, just as repeated application of the map $(x,y)\mapsto xy$ generates monomials, repeated multiplication of the hidden state by path increments generates the tensor levels of the truncated signature.
Hence, this single multiplicative interaction is sufficient for maximal expressivity.
\\ \\
Similarly to NCDEs, Theorem \ref{thm:universal_lin_ncde} can be extended from path-to-point functions to path-to-path functions by replacing the linear readout $L^2_{\theta}$ with a FCNN acting on $h_t$, as shown in \citep[Proposition D.2]{cirone2024deepSSM}.
Furthermore, the linear setting allows us to give a statement about expressivity even when the entries of $A_{\theta}$ are sampled randomly from a prescribed distribution.

\begin{definition}[Maximal Probabilistic Expressivity]
    Let $\mathcal{X}$ be a topological space, $M\in\mathbb{N}$, and $\mathcal{F}^M = \{ f^M_\theta : \mathcal{X} \to \mathbb{R} \mid \theta \in \Theta^M\}$ be a class of real-valued functions on $\mathcal{X}$ defined by
    \begin{equation}
        f^M_{\theta}(\omega) = l_{\theta_2}(\tilde{f}^M_{\theta_1}(\omega)),
    \end{equation}
    for $\omega\in\mathcal{X}$, where $\tilde{f}^M_{\theta_1}:\mathcal{X}\rightarrow\mathbb{R}^M$, $l_{\theta_2}\in\mathbb{R}^M$ is a linear readout, $\theta_1\in\Theta_1^M$, $\theta_2\in\Theta_2^M$, and $\Theta^M = \Theta_1^M\cup \Theta_2^M$. Given a sequence of probability measures $\mathbb{P}_{M}$ on $\Theta_1^M$ with $\theta_1 \sim \mathbb{P}_{M}$, $\mathcal{F}$ has maximal probabilistic expressivity if, for every compact set $\mathcal{K} \subset \mathcal{X}$ and every real-valued continuous function $f : \mathcal{K} \to \mathbb{R}$, the following property holds:
    \begin{equation}
        \forall \epsilon > 0, \lim\limits_{M \to \infty}\mathbb{P}_{M}\Bigg\{
    \exists l_{\theta_2} ~
    \text{s.t.} 
    \sup_{\omega \in \mathcal{K}}\Big|f(\omega) - f_{\theta}^M(\omega)\Big| < \epsilon
    \Bigg\} = 1.
    \end{equation}
\end{definition}

\begin{theorem}[Maximal Probabilistic Expressivity for Linear NCDEs \citep{cirone2024deepSSM}]
    \label{thm:max_prob_lin_ncde}
    Let $\mathcal{X}$ be the space of bounded variation paths on the interval $[t_0,t_n]$ that start at a common point and include time as a channel, endowed with the $1-$variation topology.
    Let $\mathcal{F}$ be the class of real-valued maps $X \mapsto y_{t_n}$ induced by Linear NCDEs from Definition \ref{def:lin_NCDE} with $\omega^X_s=X_s$ for $X\in\mathcal{X}$, $d_h\in\mathbb{N}$, $d_y=1$, and $A^i_{\theta}$ with i.i.d entries from $\mathcal{N}\left(0,\frac{1}{d_h}\right)$. 
    Then $\mathcal{F}$ has maximal probabilistic expressivity on $\mathcal{X}$.
\end{theorem}

\begin{proof}
For a detailed proof see \citep[Theorem B.13]{cirone2024deepSSM}. 
Here we present the core idea.
In the proof of Theorem \ref{thm:universal_lin_ncde} the depth$-N$ truncated tensor algebra of $\mathbb{R}^{d_X}$ is recreated as a Euclidean space, requiring $\mathcal{O}(d_X^N)$ orthogonal vectors. 
At the cost of losing exactness, we can leverage results of the Johnson-Lindenstrauss type to find $O(e^{\epsilon^2 N})$ vectors in $\mathbb{R}^N$ orthogonal up to an $\epsilon$ error, using random projections \citep{Dasgupta2003AnEP}. 
This relaxation allows us to use random matrices $A^i_{\theta}$ and still achieve maximal expressivity.
\end{proof}

In the context of machine learning, maximal probabilistic expressivity may be viewed as a more promising property than maximal expressivity. 
For sufficiently large hidden dimension $d_h$, it implies there exists a significant abundance of parameters $\theta_1$ that are capable of achieving uniformly bounded and arbitrarily low error rates with a linear readout layer. 
This in turn suggests that, when $d_h$ is large, random initialisations of the matrices $A^i_{\theta}$ can already encode informative features of the input path, and may therefore provide a useful starting point for optimisation.
\\ \\
Linear NCDEs with random $A^i_{\theta}$ have also been studied under the name randomised signatures.
There is a growing body of work establishing their expressivity and empirical effectiveness \citep{cuchiero2021expressive, Cuchiero2022Discrete, Compagnoni2023On}.
In particular, the concept behind the proof of Theorem \ref{thm:max_prob_lin_ncde} goes back to \cite{cuchiero2021expressive}, who used a Johnson-Lindenstrauss type result to show that randomised signatures can match the expressivity of classical signatures with substantially fewer features.

\subsection{Related Approaches}

Linear NCDEs are the continuous-time analogue of a tensor RNN,
\begin{equation}
\label{eq:2rnn}
h_{t_{i+1}}=\sigma\left(\left(\sum_{j=1}^{d_x} A^j_{\theta}\, x^j_{t_i}\right) h_{t_i}\right),
\end{equation}
with $\sigma(x)=x$. Tensor RNNs were introduced by \citet{Giles1990HigherOrderRNN}, who conjectured they could learn any regular language. 
Architectures of the form \eqref{eq:2rnn} are also known as $2$-RNNs \citep{pmlr-v235-lizaire24a}. 
Building on tensor RNNs, \cite{sutskever2011generating} introduced multiplicative RNNs, which take the generic form
\begin{equation}
\label{eq:mrnn}
h_{t_{i+1}}=\sigma\left(A_{\theta}(x_{t_i})h_{t_i}+B_{\theta}(x_{t_i})x_{t_i}\right).
\end{equation}
When $\sigma(x)=x$, \eqref{eq:mrnn} is an RNN whose update is linear in the hidden state,
\begin{equation}
\label{eq:lin_rnn}
h_{t_{i+1}}=A_{\theta}(x_{t_i})h_{t_i}+B_{\theta}(x_{t_i})x_{t_i}.
\end{equation}
Structured state-space models (SSMs), so named for the similarity between \eqref{eq:lin_rnn} and classical state-space models \eqref{eq:sm}, are a subclass of \eqref{eq:lin_rnn} that utilise a structured state-transition matrix $A_{\theta}$ \citep{gu2021efficiently}.
\\ \\
Section \ref{sec:ssms_are_lcdes} will show that two prominent SSMs, S4 and Mamba, can be recast as affine Linear NCDEs,
\begin{equation}
    \label{eq:affine_lin_cde}
    h_t = h_{t_0} + \int_{t_0}^t A_{\theta}h_s \mathrm{d}\omega^{X}_s + \int_{t_0}^t B_{\theta}\mathrm{d}\xi^{X}_s,
\end{equation}
where $B_{\theta}\in\mathbb{R}^{d_h \times d_\xi}$ is a trainable matrix and $\xi:\mathcal{X}(d_X) \to \mathcal{X}(d_{\xi})$ is another function of the interpolated data \citep{gu2021efficiently, gu2024mamba}.
Concurrently with Linear NCDEs, \citet{merrill2024illusion} introduced IDS4, an SSM corresponding to a discretised \eqref{eq:affine_lin_cde} with 
\begin{equation}
    \frac{\omega^{X}_{t_{j+1}}-\omega^{X}_{t_{j}}}{t_{j+1}-t_j} = X_{t_j},
    \qquad \xi^X_s=\omega^X_s.
\end{equation}
Motivated by the inability of earlier SSM architectures to learn regular languages, \citet{merrill2024illusion} proved that IDS4 can learn any regular language, thereby answering the conjecture of \citet{Giles1990HigherOrderRNN} in the case of an affine tensor RNN. 
In Section~\ref{sec:expressiveness}, we use the Linear NCDE framework to fully characterise the expressivity of IDS4, S4, and Mamba.

\subsection{Experiments}
\label{sec:lncde_exp}

We compare an NCDE, Log-NCDE, Linear NCDE, and Log-Linear NCDE on EigenWorms, the time series classification dataset with the most observations from those considered in Section \ref{sec:log_ncde_experiments}. 
This makes it a natural stress test for practical scalability, since both the cost of repeatedly solving a non-linear differential equation and the memory cost of materialising linear flows are amplified on long sequences. 
Taking 
\begin{equation}
    \frac{\omega^{X}_{t_{j+1}}-\omega^{X}_{t_{j}}}{t_{j+1}-t_j} = (1, X_{t_j}),
\end{equation} 
and setting $\xi^X_t=0$ gives an input dimension of $d_{\omega}=7$. 
All four models use a hidden dimension of $d_h=128$, and each time series contains $17984$ observations. 
We keep the hyperparameters for the NCDE and its linear variant, and for the Log-NCDE and its linear variant, identical to those selected by the hyperparameter optimisation in Section \ref{sec:hypopt}, except for replacing the non-linear vector field $f_{\theta}$ with a linear vector field $A_{\theta}$. 
See Table \ref{tab:UEA_hypopt_ncde} for full details. 
The models are compared on average test set accuracy, as well as time per training step and GPU memory on an NVIDIA RTX 4090 with a batch size of $1$. 
The parallel associative scan is applied to chunks of 128 steps, with each chunk processed recurrently.
\begin{table*}
\caption{Mean and standard deviation of test accuracy over five different splits, together with the time per $1000$ training steps and GPU memory usage, for NCDE, Log-NCDE, Linear NCDE (LNCDE), and Log-Linear NCDE (Log-LNCDE) on EigenWorms.}
\label{tab:lncde_EW}
\centering
\begin{tabular}{c|cccc}
\hline
& NCDE & Log-NCDE & LNCDE & Log-LNCDE \\ \hline
Test Accuracy & $75.0\pm 3.9$ & $85.6\pm 5.1$ & $87.2 \pm 5.2$ & $87.8 \pm 5.7$\\
\makecell{Recurrent time per $1000$ \\ training steps (s)} & $26020$ & $2263$ & $299.4$  & $29.9$ \\
Recurrent GPU Memory (MB) & $3484$ & $3494$ & $9624$ & $3486$ \\
\makecell{Parallel time per $1000$ \\ training steps (s)} & $-$ & $-$ & $100.4$ & $17.5$ \\
Parallel GPU Memory (MB) & $-$ & $-$ & $13730$ & $3492$ \\
\end{tabular}
\end{table*}
\\ \\
As shown in Table~\ref{tab:lncde_EW}, replacing the non-linear vector field of an NCDE with a linear vector field increases the average test accuracy from $75.0\%$ to $87.2\%$, bringing performance in line with Log-NCDEs, and therefore with the other state-of-the-art time series models considered in Section~\ref{sec:results}. 
A possible cause is that removing the differential equation solver improves training stability, since the Linear NCDE is solved exactly over each interval. 
This has the added benefit of reducing the recurrent time per training step by a factor of over $80$. 
Equivalently, $10000$ training steps would take around $3$ days for the NCDE, but only around $50$ minutes for the Linear NCDE. 
When a parallel associative scan is applied, the reduction in time per training step increases to a factor of over $250$, so that $10000$ training steps now take under $17$ minutes. 
However, this change also significantly increases GPU memory usage.
\\ \\
Applying the Log-ODE method brings GPU memory usage back to a level comparable to that of the models with non-linear vector fields, while further reducing the time per training step and maintaining a high average test set accuracy.
Overall, a Log-Linear NCDE combined with a parallel associative scan reduces the time per training step by a factor of almost $1500$ relative to the NCDE, so that $10000$ training steps take under $3$ minutes rather than around $3$ days. 
This is achieved while increasing the average test accuracy by $12.8$ percentage points and using only $8$MB more GPU memory than the NCDE.
\\ \\
Before broadening our empirical study of Linear NCDEs to additional tasks and baselines in Sections~\ref{sec:empirical} and \ref{sec:slice_exp}, we establish their theoretical relationship to SSMs. 
This connection provides the foundation for Section~\ref{sec:slice}, which introduces Structured Linear NCDEs, where $A^i_{\theta}$ is replaced by a structured variant to further improve model efficiency.

\section{Structured State-Space Models}

\label{sec:ssms}

\subsection{Definition}

In 2021, drawing inspiration from traditional state-space models such as \eqref{eq:sm}, \citet{gu2021efficiently} introduced S4, the first SSM. 
The model is based on a continuous differential equation,
\begin{equation}
\label{eq:s4_ct}
\begin{aligned}
    \mathrm{d}h^j_s &= C^j_{\theta} h^j_s + D_{\theta} X_{s}^j \mathrm{d}s, \\
    y^j_s &= E_{\theta} \cdot h^j_s
\end{aligned}
\end{equation}
where each channel of the input path $X^{j}_s$ produces a complex-valued hidden state $h^j_s \in \mathbb{C}^{d_h}$ and the trainable parameters are $D_{\theta}\in\mathbb{C}^{d_h}$, $E_{\theta}\in\mathbb{C}^{d_h}$, and the channel specific $C_{\theta}^j\in\mathbb{C}^{d_h \times d_h}$. 
S4 is a stable discretisation of \eqref{eq:s4_ct},
\begin{equation}
    \label{eq:s4}
    \begin{aligned}
        h^j_{t_{i+1}} = \bar{C}_{\theta}^jh^j_{t_i} + \bar{D}_{\theta}^jx^j_{t_i}
    \end{aligned}
\end{equation}
where $\bar C_{\theta}^j$ and $\bar D_{\theta}^j$ are determined by the method of discretisation and the channel-dependent step size $\Delta^j$. 
A common choice is the zero-order hold discretisation, 
\begin{equation}
\begin{aligned}
    \bar C_{\theta}^j &= \exp(\Delta^j C_{\theta}^j), \\
    \bar D_{\theta}^j &= (\Delta^j C_{\theta}^j)^{-1} (\exp(\Delta^j C_{\theta}^j)-I)\Delta^j D_{\theta} \approx \Delta^j D_{\theta}.
\end{aligned}
\end{equation}
During inference, S4 is equivalent to a linear RNN. 
However, the training via gradient descent is performed on the continuous-time variables, which helps manage vanishing and exploding gradients \citep{orvieto2023resurrecting, zucchet2024recurrent}. 
The ``structured'' aspect refers to specific parametrisations and initialisations of the state-to-state transition matrices $C^j_{\theta}$ to ensure stability and efficiency, particularly for processing long sequences. 
For example, S4D uses diagonal $C^j_{\theta}$, which has become the dominant choice \citep{gu2022s4d}.
\\ \\
SSMs have achieved state-of-the-art results on long-range reasoning benchmarks~\citep{tay2020long} and demonstrated strong performance in various domains including vision~\citep{nguyen2022s4nd}, audio~\citep{goel2022sashimi}, biological signals~\citep{gu2021efficiently}, and reinforcement learning~\citep{lu2023structured}. 
SSMs have garnered significant interest, as their computational complexity scales linearly in sequence length, while attention scales quadratically. 
Moreover, unlike non-linear RNNs such as LSTMs \citep{hochreiter1997long} and GRUs \citep{GRU}, they can be efficiently parallelised on GPUs during training using the approach outlined in Section \ref{sec:flow} \citep{S5}.
While standard SSMs perform well on signal processing tasks, their computational power is limited: the core sequential mechanism of S4 is equivalent to a convolution \citep{li2022makes}. 
This represents a drawback in challenging domains such as text and genetics, where the ability to select data efficiently in an input-dependent manner is crucial \citep{wang2023pretraining, fu2022hungry, arora2023zoology}. 
\\ \\
In 2023, Gu et al.\ proposed Mamba, which uses a recurrent layer based on a real-valued discretised model,
\begin{equation}
    \label{eq:s6}
    h^j_{t_{i+1}} = \bar{C}_{\theta}^j(x_{t_i})h^j_{t_i} + \bar{D}_{\theta}^j(x_{t_i})x^j_{t_i},
\end{equation}
where 
\begin{equation}
\begin{aligned}
    \bar C_{\theta}^j(x_{t_i}) &= \exp(\Delta^j(x_{t_i}) C_{\theta}), \\
    \bar D_{\theta}^j(x_{t_i}) &= (\Delta^j(x_{t_i}) C_{\theta})^{-1} (\exp(\Delta^j(x_{t_i}) C_{\theta})-I)\Delta^j(x_{t_i}) D_{\theta}x_{t_i}, \\
    &\approx \Delta^j(x_{t_i}) D_{\theta}x_{t_i},
\end{aligned}
\end{equation}
and the trainable parameters are $D_{\theta}\in\mathbb{R}^{d_h \times d_x}$, a shared diagonal state-transition matrix $C_{\theta}\in\mathbb{R}^{d_h \times d_h}$, and
\begin{equation}
    \Delta^{j}(x_{t_i}) = \text{softplus}(\alpha_{\theta}^j \cdot x_{t_i} + \beta_{\theta}^j),
\end{equation}
with trainable parameters $\alpha_{\theta}^j\in\mathbb{R}^{d_x}$ and $\beta_{\theta}^j\in\mathbb{R}$ \citep{gu2024mamba}.
This recurrent layer is known as S6.
Compared to S4, the evolution of each channel's hidden state is now dependent on all channels of the current input $x_{t_i}$ through $\Delta^{j}(x_{t_i})$. 
This dependence is intended to gate the flow of information by controlling the balance between the previous hidden state and the new update. 
When $\Delta^{j}(x_{t_i}) \ll 1$, the state-transition matrix is close to the identity and $\bar{D}_{\theta}^j(x_{t_i})x^j_{t_i}$ is small, so the hidden state is largely preserved. 
For larger values of $\Delta^{j}(x_{t_i})$, the new update has greater influence, leading to stronger state changes and greater forgetting. 
This modification allows Mamba to achieve state-of-the-art performance on a range of language modelling tasks.
Similar ideas appear in recent attention-inspired architectures such as RWKV, Gated Linear Attention, and HGRN2 \citep{peng2023rwkv, yang2024gated, qin2024hgrn2}.
\\ \\
The expressiveness of non-linear RNNs, such as \eqref{eq:rnn}, has been extensively studied since the seminal work of \citet{siegelmann1992computational}. 
In particular, \citet{hanson2020universal} proved that wide enough non-linear RNNs can approximate non-linear time-homogeneous systems of differential equations driven by input paths to arbitrary precision. 
However, SSMs have state-to-state transitions which are linear in the hidden state. 
Although this allows for parallel-in-time computation, it also reduces the recurrence's expressivity.
In 2022, \citeauthor{li2022approximation} showed that linear RNNs, a generic term for S4 like recurrences, can approximate arbitrary convolution filters in the width limit \citep{li2022approximation}. 
It has also been shown that single layer linear recurrences are universal approximators, when equipped with a fixed point-wise FCNN acting across the recurrence output \citep{orvieto2023universality, wang2023state}.
\\ \\
Mamba's recurrence falls neither in the linear RNN nor the non-linear RNN setting: it is linear in the hidden state, but unlike S4 it is not linear time-invariant, since the input controls the recurrence's eigenvalues. 
This input dependence increases Mamba's expressivity relative to S4 and improves language modelling performance.
In the remainder of this section, we investigate the approximation capabilities of Mamba by recasting the model as an affine Linear NCDE. 
Existing work on Mamba's expressiveness has focused on specific toy tasks \citep{jelassi2024repeat} or the framework of formal language theory \citep{merrill2024illusion}. 
Here, we seek a generic result.

\subsection{SSMs are Linear NCDEs}

\label{sec:ssms_are_lcdes}

First, we show that the real-valued continuous version of S4, \eqref{eq:s4_ct}, can be rewritten as an affine Linear NCDE, \eqref{eq:affine_lin_cde}. 
Let $h_t\in\mathbb{R}^{d_hd_X}$,
\begin{equation}
    \begin{aligned}
    \omega^{X,k}_t &= t, \\ 
    \xi^{X}_{t} &= \int_{t_0}^t X_s ds,
    \end{aligned}
\end{equation}
and
\begin{equation}
    \begin{aligned}
    A^k_{\theta} &=\text{diag}(0,\ldots,0,C_{\theta}^k,0,\ldots,0)\in\mathbb{R}^{d_hd_X \times d_hd_X}, \\
    B_{\theta}&=\text{diag}(D_{\theta},\ldots,D_{\theta})\in\mathbb{R}^{d_hd_X\times d_X},
    \end{aligned}
\end{equation}
where the non-zero diagonal element of $A^k$ is in the $k^{\text{th}}$ position. 
Then the affine Linear NCDE \eqref{eq:affine_lin_cde} corresponds to a stacked version of \eqref{eq:s4_ct} with real-valued parameters.
\\ \\
The discrete version of Mamba's recurrence, \eqref{eq:s6}, can be considered a zero-order hold discretisation of
\begin{equation}
    \mathrm{d}h^j_s = C_{\theta}\Delta^j(X_{s})h^j_s + D_{\theta}X_s\Delta^j(X_{s})X^j_s \mathrm{d}s
\end{equation}
with a step size of $1$, where $C_{\theta}\in\mathbb{R}^{d_h \times d_h}$ and $D_{\theta}\in\mathbb{R}^{d_h \times d_X}$. 
These equations can be stacked and rewritten as an affine Linear NCDE by taking
\begin{equation}
\begin{aligned}
    \omega^{X,k}_t &= \int_{t_0}^t \text{softplus}(\alpha^k \cdot X_{s} + \beta^k)  \mathrm{d}s, \\
    \xi^{X}_t &= \int_{t_0}^t \begin{bmatrix}X_t\text{softplus}(\alpha^1 \cdot X_{s} + \beta^1)X^1_s \\ \vdots \\  X_t\text{softplus}(\alpha^{d_X} \cdot X_{s} + \beta^{d_X})X^{d_X}_s \end{bmatrix} ds,
\end{aligned}
\end{equation}
and
\begin{equation}
    \begin{aligned}
    A^k_{\theta} &=\text{diag}(0,\ldots,0,C_{\theta},0,\ldots,0)\in\mathbb{R}^{d_hd_X \times d_hd_X}, \\
    B_{\theta}&=\text{diag}(D_{\theta},\ldots,D_{\theta})\in\mathbb{R}^{d_hd_X\times d_X^2},
    \end{aligned}
\end{equation}
where the non-zero diagonal element of $A^k_{\theta}$ is in the $k^{\text{th}}$ position.
\\ \\
In this framework, the major difference between S4 and Mamba's recurrence is the choice of $\omega$ and $\xi$. 
\citet{gu2024mamba} argue that this difference allows Mamba to gate the hidden state based on the input stream, and therefore perform in-context learning. 
For this reason, we refer to $\omega$ and $\xi$ as the gating functions. 
A notable difference between SSMs and the general Linear NCDE is that SSMs process the hidden state for each channel individually, whereas a Linear NCDE mixes the hidden state and the channels of the transformed input path $\omega^X$ in the recurrent step. 
As shown in the next section, this has a significant impact on the expressivity of SSMs.

\subsection{Expressivity of SSMs}

\label{sec:expressiveness}

For certain choices of $\omega^{X}_t$ and $\xi^{X}_t$, the affine Linear NCDE is maximally expressive. 
For example, $\omega^{X,k}_t = X^k_t$ and $\xi^{X}_t = 0$ puts you in the setting of Theorem \ref{thm:universal_lin_ncde}. 
However, both S4 and Mamba use alternative choices for $\omega^{X}_t$ and $\xi^{X}_t$, so we now characterise the expressiveness of generic affine Linear NCDEs.

\begin{theorem}\label{thm:linear_CDEs_closure}
        Let $\mathcal{X}(d)$ denote the space of bounded-variation $d$-dimensional paths on the interval $[t_0,t_n]$ which all begin at the same point and contain time as a channel, endowed with the $1-$variation topology.
        For continuous gates $\omega^X:\mathcal{X}(d_X)\rightarrow \mathcal{X}(d_\omega)$ and $\xi^X:\mathcal{X}(d_X)\rightarrow \mathcal{X}(d_\xi)$, let 
        \begin{equation}
        \label{eq:lin_cde_F}
            \mathcal{F} = \left\{
                (X,t) \mapsto \Psi(\omega^{X}_{[t_0,t]}) \cdot X_0 + \int_{t_0}^t \Phi(\omega^{X}_{[s,t]}) \cdot \mathrm{d}\xi^{X}_{s} 
            \right\},
        \end{equation}
        where $\Psi: \mathcal{X}(d_\omega) \rightarrow \mathbb{R}^{d_X}$ and $\Phi: \mathcal{X}(d_\omega) \rightarrow \mathbb{R}^{d_{\xi}}$ are continuous functions and $\omega^X_{[s,t]}$ is $\omega^X$ restricted to the interval $[s,t]$. Then 
        for any compact set $\mathcal{K} \subseteq \mathcal{X}(d_X)$, any continuous paths $\omega^X$ and $\xi^X$ with $\omega^{X,1}_t \equiv t$ and $\omega^{X,2}_t \equiv t^2$, any $\epsilon>0$, and any $F\in\mathcal{F}$, there exists a choice of hidden dimension $d_h \geq 1$ and parameters for the affine Linear NCDE such that
        \begin{equation}\label{eqn:dense_RKHS_main}
            \sup\limits_{(X,t) \in \mathcal{K} \times [t_0,t_n]} |F(X,t) - y_t| \leq \epsilon.
        \end{equation}
\end{theorem}

A complete proof of Theorem \ref{thm:linear_CDEs_closure} can be found in \citep[Appendix B]{cirone2024deepSSM}. 
Here, we give an overview of the argument. 
We begin by deriving an explicit formula for the solution of an affine Linear NCDE, following the same Picard iteration argument used to prove Theorem~\ref{thm:linear_cde_solution}. 
This formula shows that the solution is represented in terms of the signature of the transformed path $\omega^X$, with $\xi^X$ entering through weighted integrals against those features. 
In particular, it makes clear that the choice of $\omega^X$ is the main factor governing the model's expressivity.

\begin{lemma} \label{lem:lin_cde_sol}
    Let $\omega_t$ and $\xi_t$ be bounded variation paths of dimension $d_{\omega}$ and $d_{\xi}$, respectively. 
    For any choice of $A^i \in \mathbb{R}^{d_h \times d_h}$, and $B \in \mathbb{R}^{d_h \times d_{\xi}}$, the unique solution to 
    \begin{equation} \label{eqn:Z_CDE} 
            dh_t = \sum_{i=1}^{d_{\omega}} A^i h_t \mathrm{d}\omega^{i}_t  + B \mathrm{d}\xi_t,
    \end{equation}
    is 
    \begin{equation} \label{eqn:linear_on_sig}
        h_t = \sum_{I \in \mathcal{I}}S^I_{[t_0,t]}(\omega) A^I h_{t_0}  +  \sum_{I \in \mathcal{I}}A^I B  \int_{t_0}^t  S^I_{[s, t]}(\omega)\mathrm{d}\xi_s,
    \end{equation}
    where $\mathcal{I}$ is the set of multi-indices 
    \begin{equation}
        \mathcal{I} = \{\emptyset\} \cup \{I | I=(i_1, \ldots, i_k), k\in\mathbb{N},1\leq i_j \leq d_{\omega} \}
    \end{equation}
    with $|I|=|(i_1,\ldots,i_k)|=k$, $A^I=A^{i_k}\cdots A^{i_1}$ with $A^{\emptyset}$ being the identity, and $S^I_{[s, t]}(\omega)$ is the term in the signature of $\omega$ over $[s,t]$ corresponding to the multi-index $I$,
    \begin{equation}
        S^I_{[s,t]}(\omega) = \underbrace{\int\cdots\int}_{\substack{s\leq u_1<\cdots<u_k\leq t}} \mathrm{d}\omega^{i_1}_{u_1}\cdots \mathrm{d}\omega^{i_k}_{u_k},
    \end{equation}
    with $S^{\emptyset}=1$.
\end{lemma}

\begin{proof}
The proof uses Picard iteration, following the same approach as the proof of Theorem \ref{thm:linear_cde_solution}.
Let $h^{(0)}_t = h_{t_0}$ and for $n \geq 0$ define
\begin{equation}
h^{(n+1)}_t = h_{t_0} + \sum_{i=1}^{d_\omega} \int_{t_0}^t A^i h^{(n)}_s \, \mathrm{d}\omega^{i}_s + B \int_{t_0}^t \mathrm{d}\xi_s.
\end{equation}
Assume for some $n \geq 0$ that,
\begin{equation}
\label{eq:picard_induction_assumption}
h^{(n)}_t = \sum_{|I| \leq n} S^I_{[t_0, t]}(\omega) A^Ih_{t_0} + \sum_{|I| \leq n-1} A^I B \int_{t_0}^t S^I_{[s, t]}(\omega) \, \mathrm{d}\xi_s.
\end{equation}
Then, 
\begin{equation}
\label{eq:picard_induction2}
    \begin{aligned}
        h^{(n+1)}_t = h_{t_0} &+ \sum_{i=1}^{d_\omega}  A^i \int_{t_0}^t \left(\sum_{|I| \leq n} S^I_{[t_0, s]}(\omega) A^Ih_{t_0}\right)\mathrm{d}\omega^{i}_s \\
        &+ \sum_{i=1}^{d_\omega}A^i \int_{t_0}^t\left(\sum_{|I| \leq n-1} A^I B \int_{t_0}^s S^I_{[u, s]}(\omega) \, \mathrm{d}\xi_u\right) \, \mathrm{d}\omega^{i}_s + B \int_{t_0}^t \mathrm{d}\xi_s.
    \end{aligned}
\end{equation}
For the first term,
\begin{equation}
\sum_{i=1}^{d_\omega}  A^i \int_{t_0}^t \left(\sum_{|I| \leq n} S^I_{[t_0, s]}(\omega) A^I h_{t_0}\right)\mathrm{d}\omega^{i}_s
=
\sum_{i=1}^{d_\omega}\sum_{|I| \leq n} A^i A^I h_{t_0} \int_{t_0}^t S^I_{[t_0, s]}(\omega)\mathrm{d}\omega^{i}_s.
\end{equation}
By the recursive definitions of $A^I$ and $S^I$,
\begin{equation}
\sum_{i=1}^{d_\omega}\sum_{|I| \leq n} A^i A^I h_{t_0} \int_{t_0}^t S^I_{[t_0, s]}(\omega)\mathrm{d}\omega^{i}_s
=
\sum_{1 \leq |I| \leq n+1} S^I_{[t_0, t]}(\omega) A^I h_{t_0}.
\end{equation}

For the second term,
\begin{equation}
\sum_{i=1}^{d_\omega} A^i \int_{t_0}^t \left(\sum_{|I| \leq n-1} A^I B \int_{t_0}^s S^I_{[u, s]}(\omega) \, \mathrm{d}\xi_u\right)\mathrm{d}\omega^{i}_s
=
\sum_{i=1}^{d_\omega}\sum_{|I| \leq n-1} A^i A^I B \int_{t_0}^t \int_{t_0}^s S^I_{[u, s]}(\omega)\,\mathrm{d}\xi_u\,\mathrm{d}\omega^i_s.
\end{equation}
By Fubini's theorem,
\begin{equation}
\sum_{i=1}^{d_\omega}\sum_{|I| \leq n-1} A^i A^I B \int_{t_0}^t \int_{t_0}^s S^I_{[u, s]}(\omega)\,\mathrm{d}\xi_u\,\mathrm{d}\omega^i_s
=
\sum_{i=1}^{d_\omega}\sum_{|I| \leq n-1} A^i A^I B \int_{t_0}^t \int_s^t S^I_{[s, u]}(\omega)\,\mathrm{d}\omega^i_u\,\mathrm{d}\xi_s.
\end{equation}
Again using the recursive definitions of $A^I$ and $S^I$,
\begin{equation}
\sum_{i=1}^{d_\omega}\sum_{|I| \leq n-1} A^i A^I B \int_{t_0}^t \int_s^t S^I_{[s, u]}(\omega)\,\mathrm{d}\omega^i_u\,\mathrm{d}\xi_s
=
\sum_{1 \leq |I| \leq n} A^I B \int_{t_0}^t S^I_{[s, t]}(\omega)\,\mathrm{d}\xi_s.
\end{equation}

Therefore,
\begin{equation}
h^{(n+1)}_t = \sum_{|I| \leq n+1} S^I_{[t_0, t]}(\omega) A^Ih_{t_0} + \sum_{|I| \leq n} A^I B \int_{t_0}^t S^I_{[s, t]}(\omega) \, \mathrm{d}\xi_s.
\end{equation}
Since \eqref{eq:picard_induction_assumption} is true for $n=0$, it holds for all $n\geq0$. 
Let $M=\max_i\{\|A^i\|_{\operatorname{op}}\}$.
By \citep[Theorem 2.2.1]{Lyons1994DIFFERENTIALED}, there exists finite $C>1$ such that \begin{equation}\label{eq:young_factorial2} 
\big\|S^I_{[t_0, t]}(\omega)\big\| \le C\frac{\|\omega\|^{|I|}_{1\text{-var};[t_0,t]}}{|I|!}.
\end{equation}
Hence,
\begin{equation}
    \|S^I_{[t_0, t]}(\omega) A^Ih_{t_0}\| \leq CM^{|I|}\frac{\|\omega\|^{|I|}_{1\text{-var};[t_0,t]}}{|I|!}\|h_{t_0}\|
\end{equation}
and
\begin{equation}
    \left\|A^I B \int_{t_0}^t S^I_{[s, t]}(\omega) \, \mathrm{d}\xi_s\right\| \leq CM^{|I|}\|B\|_{\operatorname{op}}\frac{\|\omega\|^{|I|}_{1\text{-var};[t_0,t]}}{|I|!}\|\xi\|_{1\text{-var};[t_0,t]},
\end{equation}
The factorial decay means $\lim_{n\to\infty}h^{(n)}_t$ converges uniformly to $h_t$ where
\begin{equation}
    h_t = \sum_{I \in \mathcal{I}}S^I_{[t_0,t]}(\omega) A^I h_{t_0}  +  \sum_{I \in \mathcal{I}}A^I B  \int_{t_0}^t  S^I_{[s, t]}(\omega)\mathrm{d}\xi_s.
\end{equation}
Continuity of the Riemann-Stieltjes integral for bounded variation paths allows the limit to be passed through the integral, so $h_t$ is a solution to \eqref{eqn:Z_CDE}. For uniqueness, suppose there are two solutions $h_t$ and $\tilde{h}_t$ to the CDE. Then by linearity, $g_t = h_t - \tilde{h}_t$ satisfies
\begin{equation}
g_t = \int_{t_0}^t \sum_{i=1}^{d_{\omega}} A^i g_s \mathrm{d}\omega^{i}_s
\end{equation}
with $g_{t_0}=0$. Therefore, $g_t=0$ and the solution is unique.
\end{proof}

Lemma~\ref{lem:lin_cde_sol} shows that an affine Linear NCDE is not a simple linear recurrence. 
Rather, its solution is a linear map on features built from iterated integrals of the transformed path $\omega^X$, together with additional features built from integrating those terms against $d\xi^X$. 
This helps explain why Linear NCDEs are maximally expressive. 
Although the dynamics are linear in the hidden state, the hidden state itself is built from highly non-linear path features generated recursively through repeated interactions with the driving path. 
Compared to classical signature methods, a Linear NCDE learns to construct the weighted-signature features most relevant to the task, rather than relying on a fixed truncated collection of signature terms.
We now package these weighted-signature features into a single feature map.

\begin{definition}
    Let $\mathbb{W}_{d_X, d_{\omega}, d_{\xi}}$ be the set of words in the alphabet 
    \begin{equation}
        \mathcal{A}_{d_X,d_{\omega},d_{\xi}} = \{  \boldsymbol{e}_i \}_{i=1}^{d_X} \cup 
        \{  \boldsymbol{\epsilon}^{\xi}_j \}_{j=1}^{d_{\xi}} \cup
        \{  \boldsymbol{\epsilon}^{\omega}_k \}_{k=1}^{d_{\omega}}
    \end{equation}
    For fixed $\omega$ and $\xi$, define $T(X): [t_0,t_n]^2 \rightarrow l^2(\mathbb{W}_{d_X, d_{\omega}, d_{\xi}}) \subseteq T((\mathcal{A}_{d_X,d_{\omega},d_{\xi}}))$
    as the unique solution to:
    \begin{equation} \label{app:eqn:T_CDE} \begin{aligned}
        T_{[s,t]}(X) = & ~ \sum_{i=1}^{d_X} X_{s}^i \boldsymbol{e}_i
        + \sum_{j=1}^{d_{\xi}} \xi^{X, j}_t \boldsymbol{\epsilon}^{\xi}_j
        + \sum_{k=1}^{d_{\omega}}  \int_{s}^t T_{[s,u]}(X) ~ \mathrm{d}\omega^{X, k}_u
        \otimes\boldsymbol{\epsilon}^{\omega}_k
    \end{aligned} \end{equation}
\end{definition}

This is similar to the tensor-valued CDE representation of the signature seen in Section \ref{sec:cde_def},
\begin{equation}
    \mathrm{d}S_{[t_0,s]}(\omega^X) = S_{[t_0,s]}(\omega^X) \otimes \mathrm{d}\omega^X_s,
\end{equation}
with the addition of two terms to track $X_s$ and $\xi^{X}$ \citep{salvi2021signature}. 
We could also understand $T(X)_{[s,t]}$ as a sub-tensor of 
\begin{equation}
    X_{s}\otimes S_{[s,t]}([\omega^X,\xi^X]),
\end{equation}
but in doing this we would have to explicitly ignore most of the terms in this tensor.
The CDE \eqref{app:eqn:T_CDE} does exactly this, but implicitly. 
In any case, the subtensor view shows that $T: \mathcal{X}(d_X) \times [t_0,t_n]^2 \to l^2(\mathbb{W}_{v, v_{\omega}, v_{\xi}})$ is well defined and continuous.
\\ \\
Having defined a feature map $T(\cdot)_{[s,t]}$ with values in the Hilbert space $l^2(\mathbb{W}_{d_X, d_{\omega}, d_{\xi}})$, it is possible to associate to it a Reproducing Kernel Hilbert Space (RKHS) \citep{BerlinetThomasAgnan2004}, where the kernel is induced by the $l^2$ product. We denote the RKHS by $\mathcal{H}^{\omega,\eta}_t$. Proposition B.10 in \cite{cirone2024deepSSM} demonstrates that linear maps of $h_t$ are in the uniform closure of $\mathcal{H}^{\omega,\eta}_t$, and Proposition B.11 allows us to characterise the closure as
\begin{equation}
    F \in \left\{
            (X,t) \mapsto \Psi(\omega^{X}_{[t_0,t]}) \cdot X_0 + \int_{t_0}^t \Phi(\omega^{X}_{[s,t]}) \cdot \mathrm{d}\xi^{X}_{s} 
        \right\}.
\end{equation}
The proof of Proposition B.11 relies on the signature being able to separate points in the image of the map $(X,s,t) \mapsto \omega^{X}_{[s,t]}$. 
By Lemma \ref{lem:separate-signatures}, the signature separates the points $\omega^X_{[s,t]}$ from $\tilde\omega^X_{[s,t]}$ as $\omega^X$ is augmented to include time. In order to separate points of the form $\omega^X_{[s,t]}$ from $\tilde\omega^X_{[s',t']}$, we include $t^2$ as a channel in $\omega^X$, as then $S_{[s,t]}(\omega^X)=S_{[s,t]}(\tilde\omega^X)$ implies that
\begin{equation}
    \begin{aligned}
        \int_s^t d(r^2) &= t^2 - s^2 = (t')^2 - (s')^2 = \int_{s'}^{t'} d(r^2), \\
        \int_s^t d(r) &= t - s = t' - s' =  \int_{s'}^{t'} d(r).
    \end{aligned}
\end{equation}
Therefore, $s'=s$ and $t'=t$.
\\ \\
The proof of Theorem \ref{thm:linear_CDEs_closure} concludes by showing that linear maps on $h_t$ are dense in the uniform closure of $\mathcal{H}^{\omega,\eta}_{[t_0,t_n]}$, using the same strategy as the proof of Theorem \ref{thm:universal_ncde}, that NCDEs are maximally expressive \citep{kidger2022neuraldifferentialequations}. 
\\ \\
Theorem \ref{thm:linear_CDEs_closure} can be seen as a generalisation to generic functions $\omega$ and $\xi$ of \citep[Theorem 7]{li2022approximation}. 
That result considered the case $\omega^X_t = t$ and $\xi^X_t = \int_{t_0}^t X_s \mathrm{d}s$, which is the setting of S4, S5, and the LRU \citep{gu2021efficiently, S5, orvieto2023resurrecting}. 
In this setting, the only information contained in $\omega_{[s,t]}$ is the increment $t-s$. 
Therefore, \eqref{eq:lin_cde_F} reduces to 
\begin{equation}
    \left\{
                (X,t) \mapsto \psi(t-t_0) + \int_{t_0}^t \phi(t - s) \cdot X_s \mathrm{d}s
    \right\},
\end{equation}
which is the set of linear filters on the input. 
\\ \\
Now consider $\omega^X_t = X_t$. The first term in the function class 
\begin{equation}
    \left\{\Psi(X_{[t_0,t]}) + \int_{t_0}^t \Phi(X_{[s,t]}) \cdot \mathrm{d}\xi^{X}_{s}\right\}
\end{equation}
is already enough to establish that the output $L^2_{\theta}h_t$ is a non-linear function of all previously seen inputs $X_{[t_0,t]}$. However, the term $\Psi(X_{[t_0,t]})$ is only non-trivial when $h_{t_0}\ne 0$. A case similar to Mamba is $h_{t_0}=0$ and $\xi^X_t = \int_{t_0}^t X_s \mathrm{d}s$. Here, we can approximate arbitrarily well outputs of the form
\begin{equation}
 \left\{
            (X,t) \mapsto  \int_{t_0}^t \Phi(X_{[s,t]}) \cdot X_s \mathrm{d}s
\right\}
\label{eq:lcde_integral}
\end{equation}
where $\Phi$ is any continuous function of the input path, restricted to the portion $[s,t]$. This clearly shows that dense Linear NCDEs are capable of context-dependent filtering: the output is again a linear combination of previously seen inputs, but weights are not predetermined as in linear RNNs like S4. 
However, in S4D and Mamba, the $A^i$ are constrained to be diagonal, and this severely restricts the expressivity.
\begin{theorem}
\label{thm:diagonal_expr}
       If the $A^i$ are diagonal, then the requirements $\omega^{X,1}_t \equiv t$, $\omega^{X,2}_t \equiv t^2$ can be dropped and the existence result only holds with
        \begin{equation}\label{app:eqn:main_poly_family}
        F \in \left\{
                (X,t) \mapsto \psi(\omega^{X}_t) \cdot X_{t_0} +  \int_{t_0}^t \phi(\omega^{X}_t - \omega^{X}_s) \cdot \mathrm{d}\xi^{X}_{s}
        \right\}
        \end{equation}
        for continuous functions $\psi:\mathbb{R}^{d_{\omega}} \rightarrow \mathbb{R}^{d}$ and $\phi: \mathbb{R}^{d_{\omega}} \rightarrow \mathbb{R}^{d_{\xi}}$.
\end{theorem}
\begin{proof}
    See \citep[Appendix B]{cirone2024deepSSM}.
\end{proof}
Taking $\omega^X_t=X_t$ and $\xi^X_t = \int_{t_0}^t X_s ~ \mathrm{d}s$, dense matrices filter based on the entire trajectory $X_{[s,t]}$,
\begin{equation}
    \int_{t_0}^t \Phi(X_{[s,t]}) \cdot X_s ~\mathrm{d}s
\end{equation}
where diagonal matrices restrict you to comparing two elements of the input sequence,
\begin{equation}
    \int_{t_0}^t \phi(X_t - X_s) \cdot X_{s} ~ \mathrm{d}s
\end{equation}
The key difference is that dense matrices allow hidden coordinates to interact, which lets the recurrence recursively build higher-order features of the path. 
When the $A^i$ are diagonal, each hidden coordinate evolves in isolation, so these recursive interactions are lost. 
As a result, diagonal models cannot generate the same class of higher-order path features as dense Linear NCDEs. 
A smart choice of gating functions $\omega$ and $\xi$ can still improve the resulting non-linear filtering strategy, but it cannot remove this fundamental processing discrepancy relative to the dense setting. 
\\ \\
To give a simplistic example of the difference in expressivity between diagonal and dense state-transition matrices, consider a stream of bits
\begin{equation}
x_1,x_2,\dots ,\; x_n\in\{0,1\},
\end{equation}
where we want to predict the parity label defined by 
\begin{equation}
p_n=S_n \bmod 2 \in\{0,1\}, \quad S_n = \sum_{k=1}^nx_k.
\end{equation}
Whenever a new bit is $1$ the label flips; if the bit is $0$ the label stays the same. 
Taking a diagonal Linear NCDE with a hidden dimension of $2$ and $\omega^x_{k+1}-\omega^x_k=x_{k+1}$, then 
\begin{equation}
h_{n+1}=\exp\left(\begin{bmatrix}a_1&0\\0&a_2\end{bmatrix} x_{n+1}\right)\,h_n,
\end{equation}
and
\begin{equation}
h_n^{i}=h_0^{i}\exp(a_i S_n),\qquad i=1,2.
\end{equation}
With a linear read‑out $r=(r_1,r_2)^\top$ followed by a monotone activation $\phi$ (such as tanh, ReLU, sigmoid):
\begin{equation}
\hat p_n=\phi\Bigl(r^\top h_n\Bigr)
        =\phi\Bigl(r_1h^1_0 e^{a_1 S_n}+r_2h^2_0 e^{a_2 S_n}\Bigr)
        =\phi\Bigl(f(S_n)\Bigr).
\label{eq:sum-of-exp}
\end{equation}
Since $f(S)$ has at most one turning point, and $\phi$ is monotone, $\hat{p}_n$ can cross any chosen threshold at most twice. 
However, the true label $p_n$ flips every time $S_n\mapsto S_n+1$. 
Hence, no diagonal $2\times2$ Linear NCDE can realise parity on arbitrarily long input.
Similarly, for a hidden dimension of $n$, $f(S)$ can have at most $n-1$ turning points, so no diagonal Linear NCDE with a fixed hidden dimension can realise parity on arbitrarily long input.
If you replace $A$ with
\begin{equation}
A=\begin{pmatrix}0&\pi\\-\pi&0\end{pmatrix},
\end{equation}
then
\begin{equation}
\exp(A x_{n+1})=
\begin{pmatrix}
1&0\\0&1
\end{pmatrix},
\end{equation} 
when $x_{n+1}=0$ and
\begin{equation}
\exp(A x_{n+1})=
\begin{pmatrix}
-1&0\\0&-1
\end{pmatrix},
\end{equation}
when $x_{n+1} = 1$. Thus
\begin{equation}
h_{n+1}=(-1)^{x_{n+1}}h_n
\end{equation}
and
\begin{equation}
h_n=(-1)^{S_n}h_0.
\end{equation}
Taking $r=(1,0)^\top$, $h_0^{(1)}=1$, and $\phi(s)=\bigl(1-\operatorname{sign}(s)\bigr)/2$, 
\begin{equation}
\hat p_n=\frac{1-\operatorname{sign}\bigl((-1)^{S_n}\bigr)}{2}=S_n\bmod 2=p_n.
\end{equation}
Therefore, a dense Linear NCDE can solve parity exactly with a hidden dimension of 2.
\\ \\
In practice, it is possible to regain expressivity without sacrificing the computational advantages of diagonal matrices through stacking, where a new Linear NCDE is driven by the solution of a previous Linear NCDE. 
A formal statement and proof of the recovery of expressivity is given in Appendix C of \cite{cirone2024deepSSM}. 
An important corollary of the results on stacking is that linear RNNs, such as S4, require non-linearities in-between layers to recover expressivity, whereas selective SSMs like Mamba only need linear mixing layers. 
Intuitively, stacking recovers the mixing between the hidden dimensions which is crucial for the expressiveness of dense Linear NCDEs. 

\subsection{Experiments}
\label{sec:empirical}

The first task considered is based on the toy dataset from Section \ref{sec:log_ncde_experiments}, where the aim is to predict terms in the input path’s signature. 
The dataset's objective aligns with the proofs of Theorems \ref{thm:max_prob_lin_ncde} and \ref{thm:linear_CDEs_closure}, which characterise the expressivity using the path's signature.
We use two datasets with dimensions $2$ and $3$, respectively. 
The increment in each channel at each step is an integer-rounded sample from a standard Normal distribution,
\begin{equation}
    p(n) = \int_{n-0.5}^{n+0.5}\frac{1}{\sqrt{2\pi}}e^{-\frac{1}{2}x^2}\mathrm{d}x,
\end{equation}
where $n\in\mathbb{Z}$.
The 2D dataset's target is the area integral 
\begin{equation}
    \label{eq:sig_pred_2}
    \int_0^1\int_0^vd X^1_u dX^2_v,
\end{equation}
and the 3D dataset's target is a volume integral 
\begin{equation}
    \label{eq:sig_pred_3}
    \int_0^1\int_0^w\int_0^v dX^1_u dX^2_v dX^3_w.
\end{equation}
We consider seven models on this dataset:
\begin{itemize}
    \item $(1, 2)$: A single S4D or S6 recurrence with a linear readout,
    \item $(3, 4)$: Two stacked S4D or S6 recurrences with a linear mixing layer in-between and a linear readout,
    \item $(5, 6)$: Two stacked S4D or S6 recurrences with a linear mixing layer and ReLU in-between and a linear readout,
    \item $7$: A Linear NCDE with gates $\omega^{X}_t = \xi^{X}_t = (t,X_t)$ and a linear readout.
\end{itemize}
All of the SSMs have trainable matrices in their recurrences, whereas the Linear NCDE is using fixed random matrices. 
All models use a hidden dimension of $256$, with the SSMs using a state dimension of $256$. 
The SSMs are trained using gradient descent with a batch size of $32$ and Adam with a learning rate of $10^{-4}$ \citep{kingma2017adam}. 
The output from the Linear NCDE's recurrence is obtained using the Tsit5 adaptive ODE solver, with an absolute and relative tolerance of $10^{-2}$ \citep{tsitouras2011runge}. 
The Linear NCDE's linear readout is obtained via ordinary least squares.
\\ \\
The second task we consider is the $A_5$ benchmark from \citet{merrill2024illusion}. 
It tests models on state-tracking, a crucial ability for tasks involving permutation composition, such as chess. 
The dataset comprises sequences from the group of even permutations on five elements, $A_5$, where the target is the cumulative composition of all preceding permutations. 
Datasets vary by sequence length, ranging from $3$ to $20$.
The models are evaluated on the number of stacked layers required to achieve greater than $90\%$ validation accuracy.
\\ \\
We consider six models on the $A_5$ benchmark: a Linear NCDE, a diagonal Linear NCDE, an LSTM, a Transformer, S4D, and Mamba. 
The LSTM, Transformer, S4D, and Mamba use a hidden dimension of $1024$. 
LSTM uses direct stacking whereas the other baseline models use stacked blocks consisting of a sequence model, a GLU layer \citep{dauphin2017language}, and layer normalisation \citep{ba2016layer}.
The Linear NCDEs take 
\begin{equation}
    \frac{\omega^{X}_{t_{j+1}}-\omega^{X}_{t_{j}}}{t_{j+1}-t_j} = (1, X_{t_j}) 
\end{equation}
and $\xi^X_t=0$.
Furthermore, they use $1024$ trainable parameters per recurrence, corresponding to a hidden dimension of $1024$ for the diagonal Linear NCDE and $32$ for the Linear NCDE.
Models are trained using a token-tagging loss for $1{,}000{,}000$ steps with a batch size of $256$. 
For all sequence lengths, a small batch of sequences of length $2$ are included at each training step to aid convergence. 
All models have trainable matrices in their recurrences and use Adam with weight decay as the optimiser \citep{kingma2017adam}, alongside a linear warm-up followed by cosine annealing with a minimum learning rate of $10^{-5}$ and a maximum learning rate of $10^{-3}$. 
Additionally, all models use dropout with a rate of $0.1$ and a trainable embedding layer \citep{srivastava2014dropout}.

\subsection{Results}
\label{sec:ssm_results}

\begin{figure}
    \includegraphics[width=\linewidth]{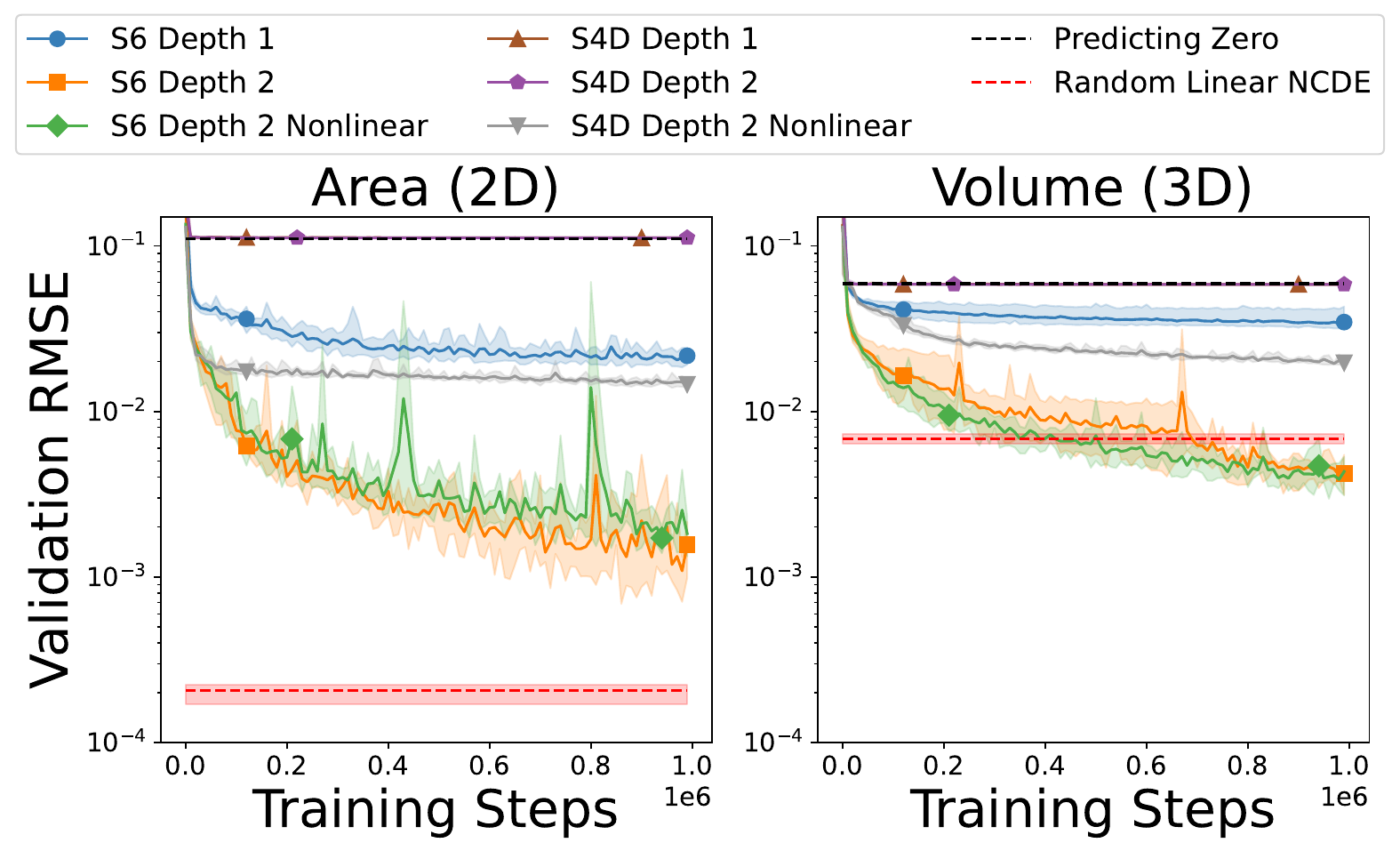}
    \caption{Comparison of the Linear NCDE, S4D, and S6 on the two signature prediction tasks, \eqref{eq:sig_pred_2} and \eqref{eq:sig_pred_3}. For each model, we plotted the mean and range of the validation root mean square error over 5 independent runs.
    }
    \label{fig:exp_toy}
\end{figure}



Figure \ref{fig:exp_toy} presents the results on the signature prediction task. 
These results empirically demonstrate a number of the theoretical results presented in the previous sections. 
Firstly, as discussed in Section \ref{sec:expressiveness}, recurrences which are linear in the input, such as S4D, require a non-linearity in-between the recurrent layers to perform well. 
Furthermore, as stated in Theorem \ref{thm:diagonal_expr}, even if the recurrence is non-linear in the input, such as Mamba's recurrence S6, the expressivity of models with diagonal matrices is improved by stacking. 
Additionally, the inclusion of the non-linearity in-between the S6 layers does not improve performance, as the recurrences themselves are expressive enough. 
Finally, as stated in Theorem \ref{thm:max_prob_lin_ncde}, dense matrices can achieve strong expressivity with random initialisation, no stacking, and only a trainable linear readout.
\\
\begin{figure}
    \includegraphics[width=\linewidth]{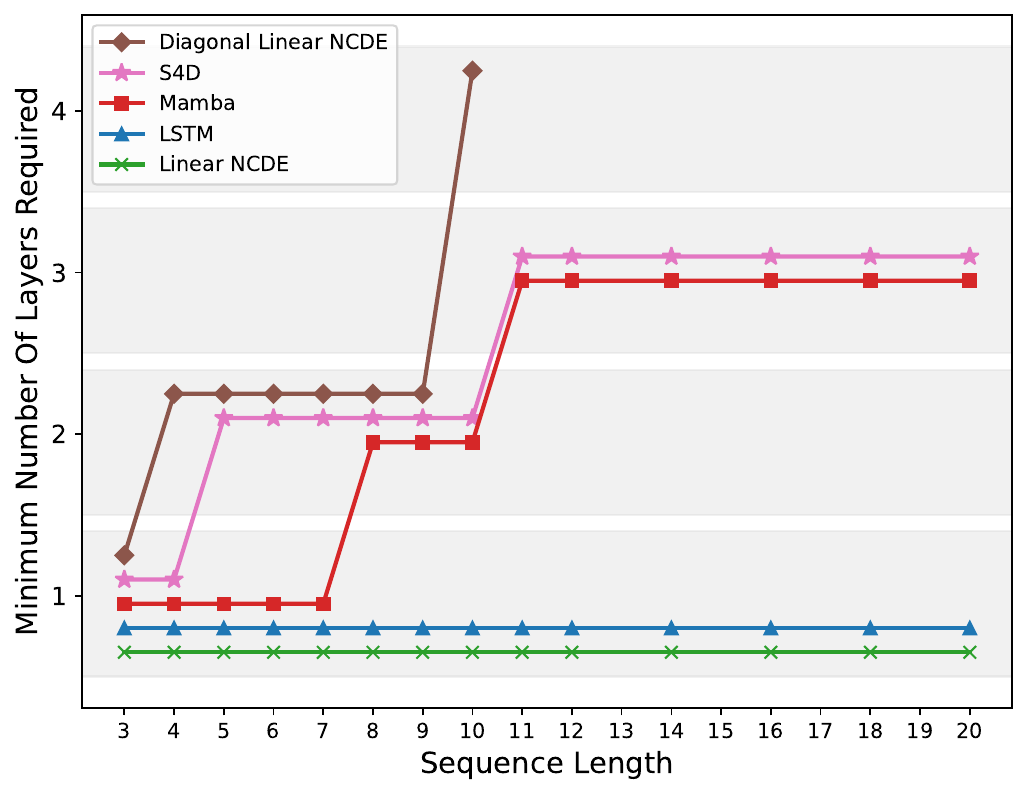}
    \caption{Comparison of a diagonal Linear NCDE, S4D, Mamba, LSTM, and Linear NCDE on the $A_5$ benchmark. For each sequence length, the plot shows the minimum number of blocks required to achieve at least $90\%$ validation accuracy, with each shaded band corresponding to a number of blocks. Missing points mean the model did not achieve at least $90\%$ validation accuracy with $4$ blocks or less.
    }
    \label{fig:exp_a5}
\end{figure}


Figure \ref{fig:exp_a5} is a plot of the results on the $A_5$ benchmark. The figure shows that the number of blocks S4D, Mamba, and the Linear NCDE with diagonal state-transition matrices require to achieve greater than $90\%$ validation accuracy grows with the sequence length. 
On the other hand, the LSTM and Linear NCDE are able to achieve greater than $90\%$ validation accuracy for all lengths considered using only one block. 
These empirical results further validate Theorem \ref{thm:diagonal_expr}: there exists a gap in expressivity between diagonal and dense state-transition matrices, with stacking required to recover expressivity. 
Furthermore, they provide empirical evidence that even for simple state-tracking problems, the number of stacked blocks required to recover expressivity can grow quickly with sequence length.

\subsection{Limitations}

As shown in Section \ref{sec:lncde_exp}, replacing the non-linear vector field of an NCDE with a linear vector field leads to substantial improvements in time per training step. Furthermore, as shown theoretically in Section \ref{sec:expressiveness} and empirically in Section \ref{sec:ssm_results}, using dense state-transition matrices $A^i_{\theta}$ provides more expressivity than the diagonal state-transition matrices of S4D and Mamba. However, there remains a core limitation of Linear NCDEs. The number of parameters and computational cost scales as $\mathcal{O}(d_h^3)$, where $d_h$ is the hidden dimension. This is in contrast to diagonal state-transition matrices, where they scale as $\mathcal{O}(d_h^2)$. This makes Linear NCDEs infeasible in large models, where hidden dimensions can reach $8192$ \citep{touvron2023llamaopenefficientfoundation}. 
\\ \\
This limitation motivates the following question: Do there exist structured matrices $A^i_{\theta}$ which are computationally more efficient than dense matrices, whilst having the same theoretical expressivity and comparable empirical performance? It is this question we aim to answer in Section \ref{sec:slice}.

\section{Structured Linear NCDEs}
\label{sec:slice}

\subsection{Introduction}
\label{sec:slice_intro}

Structured Linear Neural Controlled Differential Equations (SLiCEs) are Linear NCDEs
\begin{equation}
\label{eq:slice}
h_t = h_{t_0} + \int_{t_0}^t \sum_{i=1}^{d_\omega} A^i_{\theta} h_s\mathrm{d}\omega^X_s,
\end{equation}
where each $A^i_{\theta}$ is constrained to have a particular structure. 
While using diagonal matrices is computationally efficient, Section \ref{sec:ssms} showed that this choice theoretically limits expressivity and empirically hurts performance on state-tracking tasks.  
This section explores four more powerful alternatives, two inspired by prior work and two that are novel.

\begin{itemize}
    \item Block-diagonal: The matrix is composed of smaller, dense blocks along the diagonal,
    \begin{equation}
    A^i_{\theta} = \mathrm{BlockDiag}\big(B^i_{\theta,1}, B^i_{\theta,2}, \dots, B^i_{\theta,k}\big),
    \end{equation}
    where each $B^i_{\theta,j} \in \mathbb{R}^{b_j \times b_j}$ is a trainable dense block, $k$ is the number of blocks, and $d_h=\sum_{j=1}^k b_j$. This structure is inspired by the input-dependent block-diagonal linear RNN \citep{fan-etal-2024-advancing}
    
    \item Diagonal-plus-low-rank (DPLR): The matrix is the sum of a diagonal matrix and a low-rank matrix,
    \begin{equation}
        A^i_{\theta} = D^i_{\theta} + u^{i}_{\theta}\big(v^{i}_{\theta}\big)^\top,
    \end{equation}
    where $D^i_{\theta}$ is diagonal and $u^i_{\theta}, v^i_{\theta} \in \mathbb{R}^{d_h \times r}$ with rank $r < d_h$. DPLR structures are used by DeltaNet, DeltaProduct, and Gated DeltaNet \citep{schlag2021linear, yang2024parallelizing, siems2025deltaproductimprovingstatetrackinglinear, yang2025gateddeltanetworksimproving}.
    
    \item Sparse: Each $A^i_{\theta}$ is a sparse matrix with $\mathcal{O}(d_h^{1+\epsilon})$ non-zero entries for some $0<\epsilon<1$, sampled at random via a Bernoulli mask.
    
    \item Walsh--Hadamard: The matrix is the product of a fixed matrix $H$ and a trainable diagonal matrix $D^i_{\theta}$,
    \begin{equation}
        A^i_{\theta} = H D^i_{\theta}.
    \end{equation}
    where $H$ is a Hadamard matrix of order $d_h$ (entries $\pm 1$ with mutually orthogonal columns and rows).
\end{itemize}
Figure \ref{fig:slice} is a visual comparison of the structures.
\begin{figure}
    \centering
    \includegraphics[width=\linewidth]{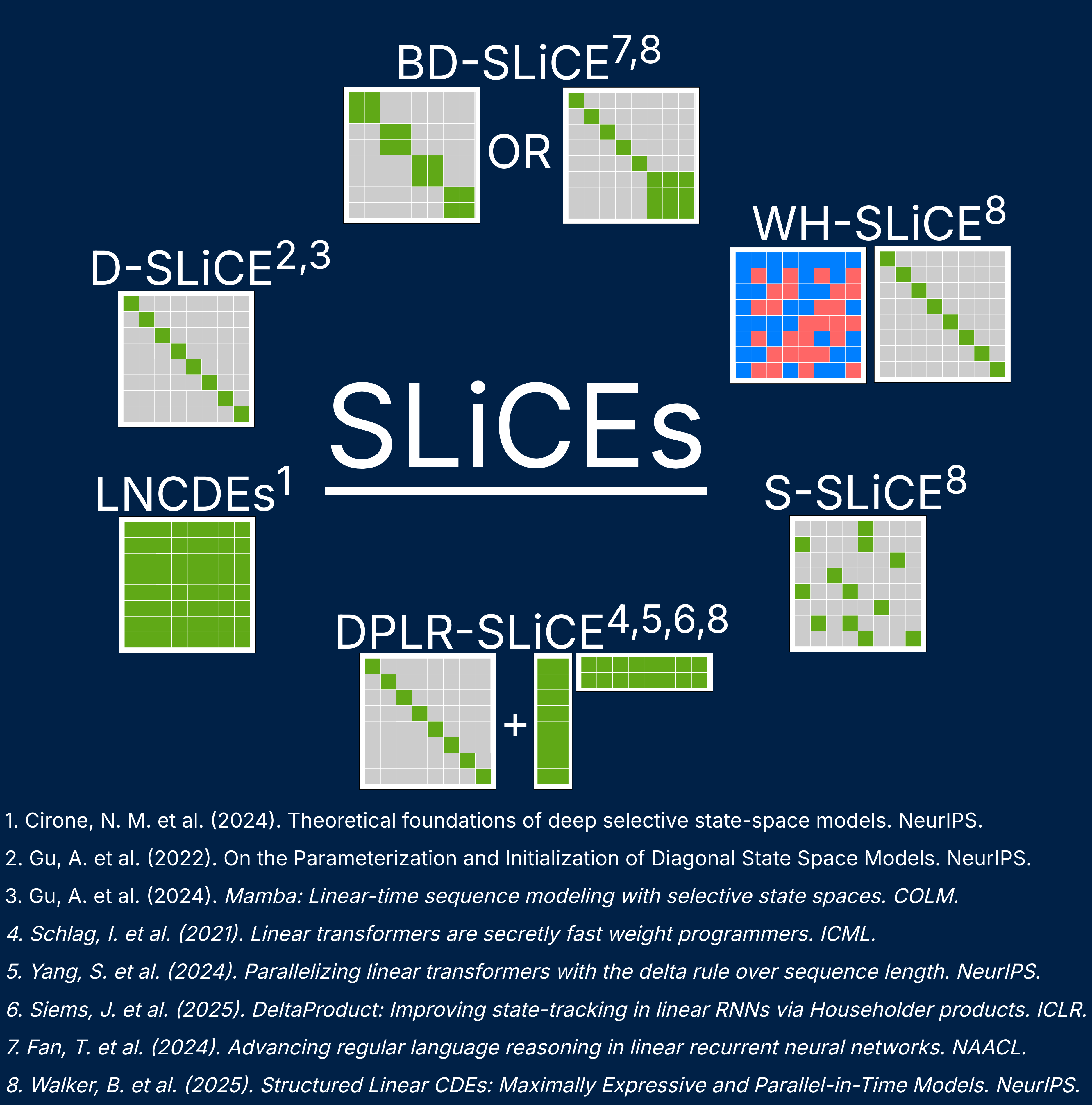}
    \caption{A visualisation of the structure for one $A^i_{\theta}$ when $d_h=8$ in a diagonal (D), block-diagonal (BD), Walsh--Hadamard (WH), sparse (S), and diagonal-plus-low-rank (DPLR) SLiCE. Green entries represent trainable parameters, grey represent parameters fixed at zero, blue represent parameters fixed at $1$, and red represent parameters fixed at $-1$.}
    \label{fig:slice}
\end{figure}
Compared to dense matrices, all four proposed structures significantly reduce both the parameter count and the computational cost of a recurrent update. 
Additionally, the Linear NCDE framework can be used to prove that unlike diagonal matrices, all four structures possess maximal probabilistic expressivity, as will be discussed further in Section \ref{sec:slice_expressiveness}. 
Empirical results in Section \ref{sec:slice_exp} demonstrate that all four structures successfully perform state-tracking on the $A_5$ benchmark used in Section \ref{sec:empirical}. 
Furthermore, the block-diagonal and DPLR SLiCE establish a new state-of-the-art among parallel-in-time models for length generalisation on regular language tasks \citep{deletang2023neural}.

\subsection{Related Work}

\label{sec:slice_related}

The block-diagonal and DPLR SLiCEs take inspiration from block-diagonal input-dependent linear RNN and DeltaNet, respectively \citep{fan-etal-2024-advancing, schlag2021linear, yang2024parallelizing}. 
Both models are linear RNNs, in the sense that their recurrent update is linear in the previous hidden state. 
For example, block-diagonal input dependent linear RNN is defined by
\begin{equation}
    \label{eq:bd_lin_rnn}
    h_{t_{i+1}}=A_{\theta}(x_{t_i})h_{t_i}+B_{\theta}x_{t_i},
\end{equation}
where $A_{\theta}$ is a block-diagonal matrix that depends non-linearly on $x_{t_i}$. 
In contrast, the state-transition matrix of block-diagonal SLiCE depends linearly on $x_{t_i}$.
As shown in Section \ref{sec:slice_expressiveness} and \ref{sec:slice_exp}, a linear dependence is theoretically sufficient for expressivity and empirically effective for state tracking.
\\ \\
DeltaNet belongs to a subset of linear RNNs whose hidden states are matrix-valued. This perspective originated as an alternative viewpoint on linear Transformers, where softmax is replaced with a kernel admitting a finite-dimensional feature map,
$\kappa(q,k)=\phi(q)^\top\phi(k)$ for $\phi:\mathbb{R}^{d_q}\to\mathbb{R}^{d_\phi}$. Letting $W^q_{\theta}\in\mathbb{R}^{d_q\times d_x}$, $W^k_{\theta}\in\mathbb{R}^{d_q\times d_x}$, and
$W^v_{\theta}\in\mathbb{R}^{d_v\times d_x}$ be learnable weights, then causal linear attention is defined by
\begin{equation}
o_{t_i} = \frac{\sum_{j=1}^i \phi\big(W^q_{\theta}x_{t_i}\big)^{\top}\phi\big(W^k_{\theta}x_{t_j}\big)\, W^v_{\theta}x_{t_j}}
             {\sum_{j=1}^i \phi\big(W^q_{\theta}x_{t_i}\big)^{\top}\phi\big(W^k_{\theta}x_{t_j}\big)}.
\end{equation}
As shown by \citet{pmlr-v119-katharopoulos20a}, this admits an RNN formulation,
\begin{equation}
\begin{aligned}
S_{t_i} &= S_{t_{i-1}} + \phi\big(W^k_{\theta} x_{t_i}\big)\big(W^v_{\theta} x_{t_i}\big)^{\top}, \\
y_{t_i} &= y_{t_{i-1}} + \phi\big(W^k_{\theta} x_{t_i}\big), \\
o_{t_i} &= \frac{\phi\big(W^q_{\theta} x_{t_i}\big)^{\top} S_{t_i}}{\phi\big(W^q_{\theta} x_{t_i}\big)^{\top} y_{t_i}},
\end{aligned}
\end{equation}
where $S_i \in \mathbb{R}^{d_\phi \times d_v}$ with $S_0=0$, $y_i\in\mathbb{R}^{d_\phi}$ with $y_0=0$, and $o_i\in\mathbb{R}^{d_v}$. Ignoring the normalisation, the core recurrence is a matrix-valued linear RNN.
\\ \\
Matrix-valued linear RNNs are a useful framework for understanding several recent sequence models, as highlighted by \citet{yang2024parallelizing}. Many of these models share the general form:
\begin{equation}
    \label{eq:mv_lrnn}
    S_{t_{i+1}} = S_{t_i} \bullet M_{t_i} + v(x_{t_i})k(x_{t_i})^\top,
\end{equation}
where $\bullet$ denotes an associative operator and $v$ and $k$ are arbitrary functions of the current input with compatible output dimensions. 
In practice, they are often linear projections, with possible additions such as feature maps, element-wise non-linearities, or normalisation.
In Section \ref{sec:ssms_are_lcdes}, the hidden state for each channel of Mamba's input was stacked vertically to view the entire model as a Linear NCDE.
If instead you view the hidden state for each channel of the input as columns of a matrix, then Mamba's recurrence \eqref{eq:s6} can be rewritten as
\begin{equation}
    \label{eq:mv_hrnn}
    S_{t_{i+1}} = S_{t_i} \odot M_{t_i} + v(x_{t_i})k(x_{t_i})^\top,
\end{equation}
where $\odot$ refers to the Hadamard (element-wise) product and $S_{t_i} \in \mathbb{R}^{d_h \times d_x}$. 
The Hadamard product is a direct consequence of Mamba's diagonal state-transition matrix, which inherently prevents interaction between individual elements of the hidden state.
This framework gives a clear interpretation for one of the key modifications introduced by Mamba-2,
\begin{equation}
    S_{t_{i+1}} = \gamma(x_{t_i})S_{t_i}+ v(x_{t_i})k(x_{t_i})^\top,
\end{equation}
where $\gamma$ is a real-valued function \citep{Dao2024Transformers}. 
Additionally, it highlights the structural similarity between Mamba and linear Transformers.
\\ \\
Replacing the Hadamard product by a matrix product allows richer interactions between the hidden state,
\begin{equation}
    \label{eq:mv_mlrnn}
    S_{t_{i+1}} = S_{t_i} M_{t_i} + v(x_{t_i})k(x_{t_i})^\top.
\end{equation}
However, similarly to using dense matrices in a Linear NCDE, the cost makes layers of this type intractable in larger models. 
DeltaNet uses a diagonal-plus-rank-one structure,
\begin{equation}
    S_{t_{i+1}} = S_{t_i} (I - \beta(x_{t_i}) k(x_{t_i})k(x_{t_i})^\top)  + \beta(x_{t_i}) v(x_{t_i})k(x_{t_i})^\top,
\end{equation}
where $\beta$ is a real-valued function \citep{schlag2021linear, yang2024parallelizing}. 
DeltaProduct later generalised DeltaNet to DPLR matrices \citep{siems2025deltaproductimprovingstatetrackinglinear}. 
Beyond reducing parameter count and recurrent computational cost, this structure also facilitates an efficient chunk-wise algorithm that can outperform parallel associative scans for large hidden dimensions, as outlined in \citep[Section 3.2]{yang2024parallelizing}. 
As will be seen in Section~\ref{sec:empirical}, it also significantly improves state-tracking performance.
\\ \\
Many other recent sequence models can also be viewed as matrix-valued linear RNNs. These include, but are not limited to, Gated DeltaNet \citep{yang2025gateddeltanetworksimproving}, RWKV-7 \citep{peng2025rwkv7gooseexpressivedynamic}, HGRN-2 \citep{qin2024hgrn2}, mLSTM \citep{beck2024xlstm}, Gated Linear Attention \citep{yang2024gated}, Gated Random Feature Attention \citep{peng2021random}, Gated Slot Attention \citep{zhang2024gated}, TTT-Linear \citep{sun2025learninglearntesttime}, and Titans \citep{behrouz2025titanslearningmemorizetest}. A detailed comparison of the specific form of \eqref{eq:mv_lrnn} for many of these models is provided in Table 2 of \citep{yang2024parallelizing}. \citet{beck2024xlstm} introduced mLSTM alongside a non-linear recurrent model sLSTM, which together form their sequence model xLSTM. These components are designed to play different roles, with mLSTM acting as the memory and sLSTM performing the reasoning. Section \ref{sec:slice_exp} will use mLSTM, sLSTM, and xLSTM as baseline methods to highlight the different roles the components play.
\\ \\
Matrix-valued linear RNNs can be converted into vector-valued linear RNNs by returning to the stacking approach of Section \ref{sec:ssms_are_lcdes}.
Let $\operatorname{vec}:\mathbb{R}^{m \times n} \to \mathbb{R}^{mn}$ be the column-major vectorisation operator, which transforms a matrix into a vector by stacking its columns. 
Letting $\otimes$ denote the Kronecker product, the Hadamard product linear RNN \eqref{eq:mv_hrnn} can be rewritten as
\begin{equation}
    \operatorname{vec}(S_{t_{i+1}}) = \operatorname{diag}(\operatorname{vec}(M_{t_i}))\operatorname{vec}(S_{t_i}) +  k(x_{t_i}) \otimes v(x_{t_i}),
\end{equation}
where the state-transition matrix is diagonal, consistent with Mamba's design. 
Noting that
\begin{equation}
    \operatorname{vec}(AXB) = (B^\top \otimes A)\operatorname{vec}(X),
\end{equation}
then the linear RNN with a matrix product \eqref{eq:mv_mlrnn} can be written as
\begin{equation}
    \operatorname{vec}(S_{t_{i+1}}) = (M_{t_i}^\top \otimes I_{d_h}) \operatorname{vec}(S_{t_i}) +  k(x_{t_i}) \otimes v(x_{t_i}).
\end{equation}
In both cases, a matrix-valued recurrence is equivalent to a vector-valued recurrence on a flattened hidden state, where the state-transition matrix is constrained to have a specific structure. 
For the Hadamard product, this gives a diagonal matrix, while the matrix product leads to a Kronecker product structure.
\\ \\
For the remainder of this work, we will focus our analysis on the vector-valued case with the structured matrices introduced in Section \ref{sec:slice_intro} (block-diagonal, DPLR, sparse, and Walsh-Hadamard). 
The insights gained from this analysis naturally extend to the matrix-valued setting. 
For instance, the limitations in expressivity of diagonal matrices demonstrated in Section \ref{sec:ssms} directly translate to the Hadamard product formulation. 
Similarly, the Kronecker product structure arising from vectorising a matrix-valued recurrence clearly illustrates how the properties of the constituent matrices determine the properties of the overall state transition. 
\\ \\
SLiCEs can be generalised to matrix-valued hidden states by selecting a suitable structure and reversing the above vectorisation. Alternatively, $m$ copies of the same SLiCE can be run via
\begin{equation}
    H_t = H_{t_0} + \int_{t_0}^t \left(\sum_{i=1}^{d_{\omega}} A^i_{\theta}\,\mathrm{d}\omega^{X,i}_s\right)H_s,
\end{equation}
where $H_s\in\mathbb{R}^{d_h\times m}$. Clearly, $m$ copies of a vector-valued SLiCE match the expressivity of a single copy, so all of our theoretical results naturally carry over to this setting. However, the columns only differ due to their initial conditions. To introduce meaningful column-specific dynamics, you can include a column-specific bias term in the vector field. A natural way to incorporate such biases into a matrix-valued SLiCE is to consider an affine SLiCE,
\begin{equation}
    \label{eq:matrix_slice}
    H_t = H_{t_0} + \int_{t_0}^t \left(\sum_{i=1}^{d_{\omega}} A^i_{\theta}\,\mathrm{d}\omega^{X,i}_s\right)H_s + B_{\theta}\,\mathrm{diag}(\mathrm{d}\xi^{X}_s),
\end{equation}
where $B_{\theta}\in\mathbb{R}^{d_h \times m}$. Letting $h^k_t$ for $1\leq k\leq m$ denote the columns of $H_t$, then
\begin{equation}
    h^k_t = h^k_{t_0} + \int_{t_0}^t \left(\sum_{i=1}^{d_{\omega}} A^i_{\theta}\,\mathrm{d}\omega^{X, i}_s\right)h^k_s + B^k_{\theta}\mathrm{d}\xi^{X,k}_s.
\end{equation}
Approximating $\omega_s$ and $\xi_s$ with linear interpolation on the grid $t_0<\cdots<t_n$ yields
\begin{equation}
    \tilde{h}^k_{t_{j+1}} \approx \exp\left(\sum_{i=1}^{d_{\omega}} A^i_{\theta}\,\left(\omega^{X, i}_{t_{j+1}} - \omega^{X, i}_{t_{j}}\right)\right)\tilde{h}^k_{t_{j}} + B^k_{\theta}(\xi^{X,k}_{t_{j+1}}-\xi^{X,k}_{t_{j}}),
\end{equation}
where we have used the approximation
\begin{equation}
    \int_{0}^{t_{j+1}-t_j} \exp\left(\tau\sum_{i=1}^{d_{\omega}} A^i_{\theta}\frac{\omega^{X, i}_{t_{j+1}} -\omega^{X, i}_{t_{j}}}{t_{j+1}-t_j}\right)\mathrm{d}\tau \approx (t_{j+1}-t_j)I,
\end{equation}
for small increments. The outputs $\tilde{h}^k_{t_j}$ remain computable in $\mathcal{O}(\log(n))$ parallel steps using a parallel associative scan \citep{blelloch1993prefix}. 
\\ \\
This construction is closely related to the path development layer, whose hidden state evolves on a matrix Lie group via the recurrence
\begin{equation}
    H_{t_{j+1}}=\exp\!\left(M_\theta(x_{t_{j+1}}-x_{t_j})\right)H_{t_j},
\end{equation}
where $M_\theta$ is a linear map into a matrix Lie algebra $\mathfrak{g}$ \citep{lou2024pathdevelopment, jiang2024gcndevlstm}.
Any such linear map can be written coordinate-wise as
\begin{equation}
    M_\theta(v)=\sum_{i=1}^{d_\omega} A_\theta^i v^i,
\end{equation}
where the choice of $\mathfrak{g}$ determines the structure of the matrices $A_\theta^i\in\mathfrak{g}$.
Hence, this is a special case of \eqref{eq:matrix_slice} with $\xi^X_s=0$, $H_{t_0}=I$, and $\omega^X_s=X_s$, where $X_s$ is the linear interpolation of $\{x_{t_i}\}_{i=0}^n$.
Therefore, the scan-based methods developed here can also be applied to parallelise the computation of the path development layer.

\subsection{Expressiveness}

\label{sec:slice_expressiveness}

As demonstrated in Section \ref{sec:expressiveness}, a core difference between diagonal and dense state-transition matrices $A^i_{\theta}$ is their theoretical expressivity. 
Here, we show that all four of the proposed SLiCE architectures retain maximal probabilistic expressivity.
\\ \\
Applying Lemma \ref{lem:lin_cde_sol} with $B=0$, the solution to \eqref{eq:slice} is
\begin{equation}
    h_t = \sum_{I \in \mathcal{I}}S^I_{[t_0,t]}(\omega) A^I h_0,
\end{equation}
where $\mathcal{I}$ is the set of multi-indices
\begin{equation}
    \mathcal{I} = \{\emptyset\} \cup \{I | I=(i_1, \ldots, i_k), k\in\mathbb{N},1\leq i_j \leq d_{\omega} \}
\end{equation}
$A^I=A^{i_k}\cdots A^{i_1}$ with $A^{\emptyset}$ being the identity, and $S^I_{[s, t]}(\omega)$ is the term in the signature of $\omega$ over $[s,t]$ corresponding to the multi-index $I$,
\begin{equation}
    S^I_{[s,t]}(\omega) = \underbrace{\int\cdots\int}_{\substack{s\leq u_1<\cdots<u_k\leq t}} \mathrm{d}\omega^{i_1}_{u_1}\cdots \mathrm{d}\omega^{i_k}_{u_k},
\end{equation}
with $S^{\emptyset}=1$.
The proof of maximal probabilistic expressivity relies on showing that the feature vectors $A^Ih_0$ and $A^Jh_0$ become approximately orthogonal as the hidden dimension increases.
This property is guaranteed if the following bound holds:
\begin{equation}
    \label{eq:key_bound}
    \left\|\frac{1}{d_h} \langle A^I h_0, A^Jh_0\rangle_{\mathbb{R}^{d_h}} - \delta_{I,J}\right\|_{L^2(\mathbb{P}_{d_h})} \leq ( \kappa (|I| + |J|) )!! ~ \mathcal{O}\left(\frac{1}{\sqrt{d_h}}\right),
\end{equation}
where $\kappa$ is a constant independent of $d_h$.
\citet{Cirone2023neural} showed that \eqref{eq:key_bound} holds when $A^i_{\theta}$ have i.i.d entries from $\mathcal{N}\left(0, \frac{1}{d_h}\right)$ and this fact was used by \citet{cirone2024deepSSM} to prove Theorem \ref{thm:max_prob_lin_ncde}.
The proof that each of the four SLiCE structures considered possesses maximal probabilistic expressivity proceeds by demonstrating that \eqref{eq:key_bound} holds for a distribution appropriate to that structure.
\\ \\
The detailed proofs of these results were primarily developed by a collaborator, Nicola Mu\c{c}a Cirone.
The proofs build on their work developing a novel graphical framework for neural networks that linearises the effect of activation functions, allowing for the application of Wick's principle and the genus expansion technique when proving convergence results \citep{cirone2025genusexpansionnonlinearrandom}.
This framework facilitates simple proofs that \eqref{eq:key_bound} holds for all four SLiCE structures, but introducing the required machinery falls outside the scope of this thesis.
Here, we present the expressivity results for the four structures with references to the locations of detailed proofs.
\begin{theorem}
\label{thm:slice_max_prob}
    Let $\mathcal{X}$ be the space of bounded variation paths on the interval $[t_0,t_n]$ that start at a common point and include time as a channel, endowed with the $1-$variation topology. Let $\mathcal{F}$ be the space of SLiCEs with $d_h\in\mathbb{N}$. Then $\mathcal{F}$ has maximal probabilistic expressivity when $h_0\sim\mathcal{N}(0, 1)$ and $A^i_{\theta}$ satisfies one of the following conditions:
    \begin{enumerate}
        \item $A^i_{\theta} = \mathrm{BlockDiag}\big(B^i_{\theta,1}, B^i_{\theta,2}, \dots, B^i_{\theta,k}\big)$, where $B^i_{\theta,j}\in\mathbb{R}^{b \times b}$, $b =  \lceil \log(d_h)\rceil$, and $B^i_{\theta, j}$ has i.i.d entries from $\mathcal{N}\left(0, \frac{1}{b}\right)$.
        \item $A^i_{\theta} = u^{i}_{\theta}\big(v^{i}_{\theta}\big)^\top$, where $u^i_{\theta}, v^i_{\theta} \in \mathbb{R}^{d_h \times r}$ with $r=\lceil\log(d_h)\rceil$ and $u^i_{\theta}$ and $v^i_{\theta}$ have i.i.d entries from $\mathcal{N}(0,1)$.
        \item $A^i_{\theta}$ has entries obtained by pointwise multiplying $W$ with i.i.d entries from $\mathcal{N}(0,\frac{1}{d_hp(d_h)})$ and $B$ with i.i.d entries from a Bernoulli distribution with probability $p(d_h)$ of being $1$, where $p(d_h)$ satisfies $d_hp(d_h) \to \infty$ as $d_h\to\infty$.
        \item $A^i_{\theta} = H D^i_{\theta}$, where $H \in \mathbb{R}^{d_h \times d_h}$ is a fixed matrix with entries bounded uniformly in $d_h$ by a constant $C$ satisfying $HH^\top = d_h I_{d_h}$ and $D^i_{\theta}$ is a diagonal matrix with i.i.d entries from $\mathcal{N}\left(0, 1\right)$.
    \end{enumerate}
\end{theorem}
\begin{proof}
    The proof of (1), (3), and (4) can be found as Proposition B.6, B.2, B.4 in \citep{walker2025structuredlinearcdesmaximally}, respectively. The proof of (2) can be found as Proposition F.2 in \citep{movahedi2025fixedpointrnnsdiagonaldense}. 
\end{proof}
The expressivity result for the low-rank structure (2) naturally extends to a DPLR structure, by simply taking $D^i_{\theta}=0$. 
Additionally, condition (3) shows that maximal probabilistic expressivity in the sparse case is retained only when the matrices do not become too sparse, in the sense that the expected number of non-zero entries per row  must diverge as $d_h\to\infty$. 
For instance, sparse matrices with $\mathcal{O}(d_h^{1+\epsilon})$ non-zero entries for some $\epsilon>0$ satisfy this condition. 
Furthermore, a Hadamard matrix matches the conditions of (4). 
Therefore, Theorem \ref{thm:slice_max_prob} confirms that, unlike diagonal matrices, SLiCEs with block-diagonal, DPLR, sparse, and Walsh--Hadamard matrices have maximal probabilistic expressivity.

\subsection{Comparison}

\label{sec:slice_comp}

Table \ref{tab:comparison} summarises the differences in parameter count, computational cost, existence of an efficient implementation, and expressivity of the proposed SLiCE structures, where for simplicity we have taken $d_{\omega}=d_h$. 
As can be seen, the four proposed structures lead to a reduction in parameter count and recurrent cost of inference, whilst still maintaining maximal expressivity. 
However, current deep-learning frameworks, such as JAX \citep{jax2018github} and PyTorch \citep{paszke2019pytorchimperativestylehighperformance}, are not optimised for unstructured sparsity, so the sparse structure does not lead to practical speed-ups in our implementations.
Furthermore, not all the structures reduce the theoretical computational cost of applying a parallel associative scan.
This is because the parallel associative scan is repeatedly composing
\begin{equation}
    \label{eq:composee}
    \exp\left(\sum_{i=1}^{d_\omega}(\omega^i_{t_{j+1}}-\omega^i_{t_j})A^i_{\theta}\right).
\end{equation}
When the $A^i_{\theta}$ are diagonal or block-diagonal, the composition of \eqref{eq:composee} preserves the structure, as these classes of matrices are closed under multiplication. 
Therefore, using a parallel associative scan reduces the scan depth from $n$ to $\log(n)$, whilst having a computational cost per composition of $\mathcal{O}(d_h^2)$ or $\mathcal{O}(d_h\sum_{j}b_j^2)$, respectively. 
However, for DPLR, sparse, and Walsh--Hadamard SLiCEs, the structured matrices are not closed under multiplication, which means that the limiting computational cost per composition is the same as a dense Linear NCDE, $\mathcal{O}(d_h^3)$.
\begin{table}
    \caption{Comparison of SLiCEs with dense, diagonal, diagonal-plus-low-rank (DPLR), sparse, Walsh--Hadamard (WH), and block-diagonal (BD) structure on parameter count, computational cost, the existence of an efficient implementation, and expressivity. Here, $d_{h}$ is the hidden dimension, $n$ is the sequence length, $b_j$ are BD-SLiCE's block-sizes, $r$ is DPLR-SLiCE's rank, $\epsilon$ is S-SLiCE's sparsity, and for simplicity we have taken $d_{\omega}=d_h$, as is common in practice when stacking layers. Parallel cost is measured as $\mathcal{O}($scan depth$,$ cost per composition$)$ when applying a parallel associative scan.}
    \footnotesize
    \centering
    \begin{tabular}{lccccc}
    \toprule
    \textbf{Structure} & \textbf{Parameters} & \textbf{Recurrent Cost} & \textbf{Parallel Cost} & \makecell{\textbf{Efficient} \\ \textbf{Impl.}} & \makecell{\textbf{Maximally} \\ \textbf{Expressive}} \\
    \midrule
    Dense
      & $\mathcal{O}(d_{h}^{3})$
      & $\mathcal{O}(nd_{h}^{3})$
      & $\mathcal{O}(\log(n), d_{h}^3)$
      & Yes
      & Yes \\[6pt]
    Diagonal
      & $\mathcal{O}(d_{h}^2)$
      & $\mathcal{O}(nd_{h}^2)$
      & $\mathcal{O}(\log(n), d_{h}^2)$
      & Yes
      & No \\[6pt]
    DPLR
      & $\mathcal{O}(rd_{h}^2)$
      & $\mathcal{O}(nrd_{h}^2)$
      & $\mathcal{O}(\log(n), d_{h}^3)$
      & Yes
      & Yes \\[6pt]
    Sparse
      & $\mathcal{O}(d_{h}^{2+\epsilon})$
      & $\mathcal{O}(nd_{h}^{2+\epsilon})$
      & $\mathcal{O}(\log(n), d_{h}^3)$
      & No
      & Yes \\[6pt]
    WH
      & $\mathcal{O}(d^2_{h})$
      & $\mathcal{O}(nd^2_{h})$
      & $\mathcal{O}(\log(n), d_{h}^3)$
      & Yes
      & Yes \\[6pt]
    BD
      & $\mathcal{O}\left(d_h\sum_jb_j^2\right)$
      & $\mathcal{O}\left(nd_h\sum_jb_j^2\right)$
      & $\mathcal{O}\left(\log(n), d_h\sum_jb_j^2\right)$
      & Yes
      & Yes \\
    \bottomrule
    \end{tabular}
    \label{tab:comparison}
\end{table}
\\ \\
As discussed in Section \ref{sec:flow}, parallel associative scans result in high I/O costs for large models, as each state-transition matrix must be materialised in GPU memory \citep{yang2024parallelizing}. 
\citet{yang2024parallelizing} introduced an alternative approach for DeltaNet, where a chunk-wise algorithm specifically tailored for diagonal-plus-rank-one state-transition matrices is used to bypass the need to materialise every intermediate matrix, significantly cutting down I/O costs \citep{yang2024parallelizing}. 
These approaches can also be applied to diagonal state-transition matrices. 
Therefore, a block-diagonal SLiCE with a large diagonal portion ($b_i=1$ for $i=1,\dots,k-1$) followed by a small dense block emerges as an attractive solution. 
The large diagonal section can efficiently utilise the chunk-wise algorithm and the smaller dense section can be processed using parallel associative scans without incurring significant I/O costs. 
We refer to this structure as diagonal-dense SLiCE (D-DE-SLiCE).

\subsection{Implementation Details}

Algorithm~\ref{alg:slice} provides a pseudo-code implementation for the forward pass of a SLiCE. The approach is demonstrated for dense state-transition matrices $A^i_{\theta}$, with comments highlighting where a SLiCE's structure can be used to speed up computation or reduce memory footprint. In this thesis, each SLiCE recurrence is embedded in a simple block structure, combining a linear layer to mix the channels, tanh activation function, layer normalisation \citep{ba2016layer}, and a skip connection. Inspired by the $\mathrm{Lip}(\gamma)$ regularisation introduced in Section \ref{sec:lipgammaNN}, all SLiCEs use weight regularisation on their state-transition matrices $A^i_{\theta}$. 
Additionally, due to instability during training, Walsh--Hadamard's diagonal matrix $D^i_{\theta}$ is parametrised to take values between $-1$ and $1$.
Finally, to improve training speed the following approximation is made,
\begin{equation}
    \label{eq:composee2}
    \exp\left(\sum_{i=1}^{d_\omega}(\omega^i_{t_{j+1}}-\omega^i_{t_j})A^i_{\theta}\right) \approx I+\sum_{i=1}^{d_\omega}(\omega^i_{t_{j+1}}-\omega^i_{t_j})A^i_{\theta}.
\end{equation}
This is equivalent to applying an Euler discretisation with step-size $1$, and hence aligns with the approach of recurrent models.
There is ongoing work to develop an efficient GPU kernel to reduce the computational burden of repeated matrix exponentials, as will be discussed further in Section \ref{sec:lncde_conclusion}.

\begin{algorithm}
\caption{\textbf{Structured Linear NCDE}: The algorithm is presented for a dense state-transition matrix $A_{\theta}$. Comments indicate where the structure of $A_{\theta}$ can be used to reduce memory footprint and speed-up computation.}%
\label{alg:slice}
\textbf{Input:}  $\boldsymbol{\omega} : (B,\,L,\,d_{\omega})$\\
\textbf{Output:} $\mathbf{h} : (B,\,L,\,d_h)$
\begin{algorithmic}[1]
    \State $\mathbf{h}_0 : (B,\,d_h) \gets \xi_{\phi}(\boldsymbol{\omega}_0)$
    \State $\boldsymbol{\omega}^{\text{inc}} : (B,\,L-1,\,d_{\omega}) \gets \text{diff}(\boldsymbol{\omega})$
    \State $A_{\theta} : (d_\omega,\,d_h,\,d_h) \gets \textit{Parameter}$ \Comment{Exploit structure of $A_{\theta}$}
    \State $\mathbf{I} : (d_h,\,d_h) \gets d_h\times d_h \text{ identity matrix.}$
    \If{$\textit{mode}=\text{parallel}$}
        \State $\mathbf{F} : (B,\,L-1,\,d_h,\,d_h) \gets \mathbf{I} + \text{einsum}(bli,ijk\rightarrow bljk, \boldsymbol{\omega}^{\text{inc}}, A_{\theta}) $ \Comment{Broadcast $\mathbf{I}$, exploit structure of $A_{\theta}$}
        \State $\mathbf{F}^{\text{comp}} \gets \text{pscan}(\mathbf{F})$ \Comment{Parallel associative scan, exploit structure of $\mathbf{F}$}
        \State $\mathbf{h}_{1:L-1} \gets \text{einsum}(bljk,bk\!\rightarrow\!blj,\,\mathbf{F}^{\text{comp}},\,\mathbf{h}_0)$
        \State $\mathbf{h} \gets \bigl[\mathbf{h}_0,\mathbf{h}_{1:L-1}\bigr]$
    \Else\Comment{Recurrent pass}
        \For{$t \gets 0$ \textbf{to} $L-2$}
            \State $\mathbf{F}_t : (B, \,d_h, \,d_h) \gets \mathbf{I} + \text{einsum}(bi,ijk\rightarrow bjk, \boldsymbol{\omega}^{\text{inc}}_t, A_{\theta})$ \Comment{Exploit structure of $A_{\theta}$}
            \State $\mathbf{h}_{t+1} \gets \text{einsum}(bjk,bk\!\rightarrow\!bj,\, \mathbf{F}_t,\, \mathbf{h}_{t})$ \Comment{Exploit structure of $\mathbf{F}$}
        \EndFor
        \State $\mathbf{h} \gets \bigl[\mathbf{h}_{0},\ldots,\mathbf{h}_{L-1}\bigr]$  \Comment{Stack along length axis}
    \EndIf
    \State \Return $\mathbf{h}$
\end{algorithmic}
\end{algorithm}

\subsection{Experiments}

\label{sec:slice_exp}

First, all the SLiCE variants and a wide variety of sequence model baselines are evaluated on the $A_5$ benchmark used in Section \ref{sec:empirical}. 
Our experiments follow the approach of \citet{merrill2024illusion}. 
The models are trained on sequences ranging from length $3$ to $20$ and compared on the number of stacked layers required to achieve greater than $90\%$ validation accuracy.
Models are trained using a token-tagging loss for $1{,}000{,}000$ steps with a batch size of $256$. 
For all sequence lengths, a small batch of sequences of length $2$ are included at each training step to aid convergence. 
All models use Adam with weight decay as the optimiser, and linear warm-up followed by cosine annealing with a minimum learning rate of $10^{-5}$ and a maximum learning rate of $10^{-3}$ \citep{kingma2017adam}. 
Additionally, all models use dropout with a rate of $0.1$ and a trainable embedding layer \citep{srivastava2014dropout}.
\\ \\
The five baseline models, Mamba, LSTM, mLSTM, sLSTM, and DeltaProduct, all use a hidden dimension of $1024$. 
LSTM uses direct stacking whereas the other baseline models use stacked blocks consisting of a sequence model, a GLU layer \citep{dauphin2017language}, and layer normalisation \citep{ba2016layer}. 
DeltaProduct uses both gating and negative eigenvalues, maximising the potential expressivity \citep{yang2025gateddeltanetworksimproving, grazzi2024unlockingstatetrackinglinearrnns}. 
All of the SLiCEs use 
\begin{equation}
    \frac{\omega^{X}_{t_{j+1}}-\omega^{X}_{t_{j}}}{t_{j+1}-t_j} = (1, X_{t_j}) 
\end{equation}
and $1024$ non-zero parameters for each $A^i_{\theta}$. For the diagonal and Walsh--Hadamard SLiCE this corresponds to a hidden dimension of $1024$ and for the dense SLiCE this corresponds to a hidden dimension of $32$. 
The DPLR SLiCE uses a rank of $2$, giving a hidden dimension of $205$, the block-diagonal SLiCE uses $b_i=4$, giving a hidden dimension of $256$, and the diagonal-dense SLiCE uses a dense block size of $23$, giving a hidden dimension of $518$. 
The sparse SLiCE uses a hidden dimension of $128$ and a sparsity of $\epsilon=\frac{3}{7}$.
All SLiCEs use stacked blocks consisting of a SLiCE, a linear layer followed by a tanh activation function, layer normalisation \citep{ba2016layer}, and weight regularisation.
\\ \\
The second experiment evaluates length generalisation on sequences from the $A_5$ benchmark. 
Models which achieve at least $90\%$ validation accuracy on sequences of length $20$ in the previous experiment are retrained on sequences from length $3$ to $40$ with early stopping on a validation set of sequences from length $40$ to $128$. 
The models are then evaluated on test sequences from length $20$ to $5120$. The mLSTM is excluded as it operates on fixed-length inputs.
\\ \\
The third experiment consists of the four regular tasks from the formal language benchmark, a collection of language style tasks split into categories using the Chomsky Hierarchy \citep{deletang2023neural}. They are: 
\begin{enumerate}
    \item Cycle navigation: Infer the final position of a walk on a cycle starting at the origin. Actions are randomly sampled from ``go forward one step'', ``stay in the same place'', and ``go backward one step''. We use a cycle of length $5$. Therefore, a random guesser will achieve an accuracy of $20\%$.
    \item Even pairs: There are two states in the system and the goal is to determine if there is an equal number of transitions between the two states. A random guesser will achieve an accuracy of $50\%$.
    \item Modular arithmetic no brackets: Performs modular arithmetic consisting of only addition and multiplication. We use mod $5$ and hence a random guesser will achieve $20\%$.
    \item Parity: There are two elements in the system. To determine the parity, the number of the second element is counted to determine if it is even or odd. This can be viewed as modular summation with mod $2$. A random guesser will achieve an accuracy of $50\%$.
\end{enumerate}
All four of these tasks can be solved by processing inputs sequentially with a fixed set of internal states and no external memory, i.e. state-tracking. 
The models are challenged to generalise to longer sequences, by training on sequences from length $3$ to $40$ and evaluating on sequences from length $40$ to $256$. 
Following \citet{beck2024xlstm}, all models use two stacked layers and a trainable embedding layer. 
In addition to the hidden dimension of $512$ used by \citet{beck2024xlstm}, we also train the baseline models with a hidden dimension of $128$, selecting the value that yields the highest average validation accuracy for each model on each task. 
For Mamba, which does not support a hidden dimension of $128$, we instead choose between $256$ and $512$ based on validation performance. 
All models are trained using a token-tagging loss for $100{,}000$ steps with a batch size of $256$.
Additionally, all models use dropout with a rate of $0.01$ \citep{srivastava2014dropout}, Adam with weight decay as the optimiser  \citep{kingma2017adam}, and a linear warm-up followed by cosine annealing with a minimum learning rate of $10^{-5}$ and a maximum learning rate of $2\times10^{-3}$.
\\ \\
Similarly to the $A_5$ benchmark, LSTM uses direct stacking whereas the other baseline models use stacked blocks consisting of a sequence model, a GLU layer \citep{dauphin2017language}, and layer normalisation \citep{ba2016layer}. The baseline models considered are vanilla DeltaNet \citep{schlag2021linear, yang2024gated}, DeltaNet with negative eigenvalues (DeltaNet[-1,1]) \citep{grazzi2024unlockingstatetrackinglinearrnns}, Gated DeltaNet \citep{yang2025gateddeltanetworksimproving}, Gated DeltaProduct with negative eigenvalues and a rank of $2$ \citep{siems2025deltaproductimprovingstatetrackinglinear}, sLSTM \citep{beck2024xlstm}, mLSTM \citep{beck2024xlstm}, xLSTM \citep{beck2024xlstm}, RWKV-7 \citep{peng2025rwkv7gooseexpressivedynamic}, a Transformer \citep{vaswani2017attention}, S4D \citep{gu2022s4d}, and Mamba \citep{gu2024mamba}.
\\ \\
We consider all SLiCEs on this benchmark except sparse due to the lack of an efficient implementation. 
All SLiCEs use 
\begin{equation}
    \frac{\omega^{X}_{t_{j+1}}-\omega^{X}_{t_{j}}}{t_{j+1}-t_j} = (1, X_{t_j}), 
\end{equation}
and two stacked blocks consisting of the sequence layer, a linear layer followed by a tanh activation function, and layer normalisation \citep{ba2016layer}. 
For the diagonal and Walsh--Hadamard SLiCE variants, we consider hidden dimensions of $128$ and $512$, corresponding to $128$ and $512$ non-zero parameters per state-transition matrix, respectively. 
For all other SLiCE variants, the number of non-zero parameters in the state-transition matrix is fixed at $512$. 
For DPLR we consider ranks of $r=1,2,4,8$, and for block-diagonal we consider two variants, $b_i=b$ for all $i$ with $b=2,4,8,16$, and $b_i=1$ for $i=1,\ldots,k-1$, and then a final dense block $b_k=b$ for $b=2,4,8,16$, referred to as diagonal-dense SLiCE. 
\\ \\
The final experiment expands the comparison of Section \ref{sec:lncde_exp} to consider all SLiCE structures and all six datasets from the UEA-MTSCA used in Section \ref{sec:log_ncde_experiments}. 
In particular, the vector field of a Log-NCDE is replaced by a structured $A^i_{\theta}$, with all other hyperparameters kept identical to those in Table \ref{tab:UEA_hypopt_ncde}.
The Log-SLiCEs are then retrained on each of the six UEA-MTSCA datasets from \ref{sec:log_ncde_experiments} and compared on average test set accuracy, time per training step, and GPU memory.
All timing and GPU memory results were performed on an NVIDIA H100 GPU.

\subsection{Results}

Figure \ref{fig:slice_a5} presents the results on the $A_5$ benchmark. 
As expected from Theorem \ref{thm:slice_max_prob}, the DPLR, sparse (S), Walsh--Hadamard (WH), and block-diagonal (BD) structures allow a SLiCE to achieve greater than $90\%$ validation accuracy on all sequence lengths considered with only one layer, matching the performance of the dense Linear NCDE (DE-LNCDE).
The only other models to achieve this are the non-linear recurrent models, sLSTM and LSTM. 
All other parallelisable models require a growing number of layers. This includes the diagonal-dense (D-DE) SLiCE and Gated DeltaProduct with negative eigenvalues, despite their structured state-transition matrices. 
However, DPLR and BD SLiCE both need one layer for all sequence lengths, suggesting this is not an inherent limitation of these structures.
\begin{figure}
    \includegraphics[width=\linewidth]{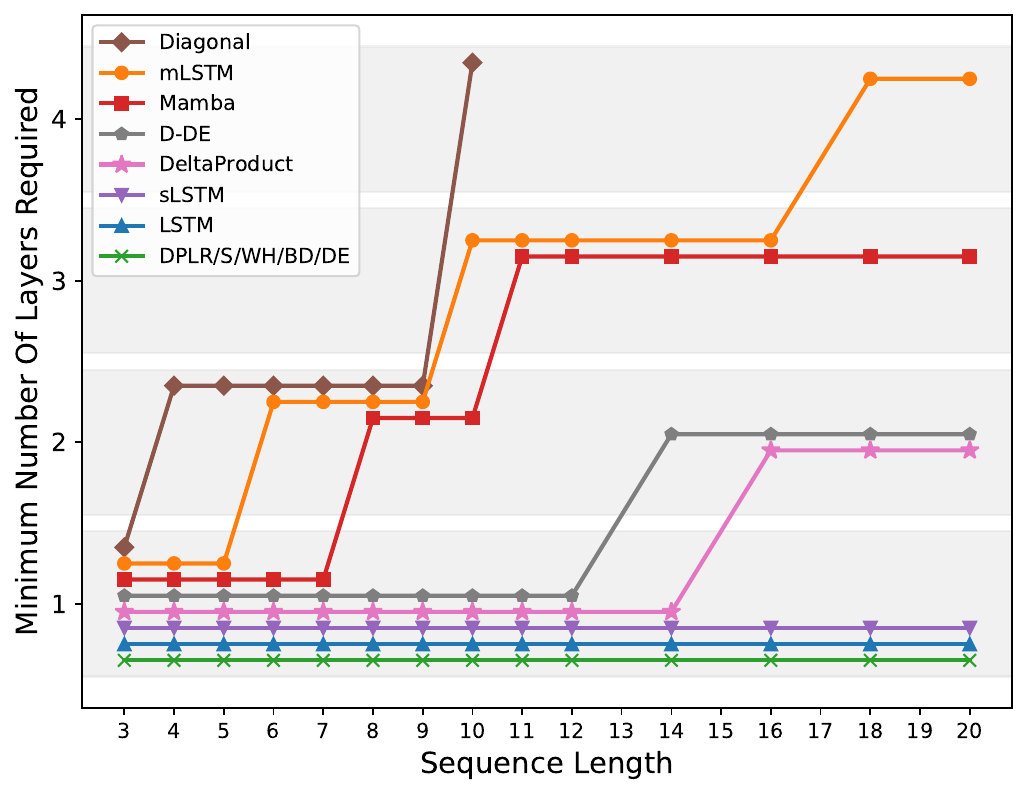}
    \caption{Comparison of an mLSTM, Mamba, DeltaProduct, sLSTM, LSTM, nd SLiCEs with diagonal, dense (DE), sparse (S), Walsh--Hadamard (WH), block-diagonal (BD), and diagonal-dense (D-DE) matrices on the $A_5$ benchmark. For each sequence length, the plot shows the minimum number of blocks required to achieve at least $90\%$ validation accuracy, with each shaded band corresponding to a number of blocks. Missing points mean the model did not achieve at least $90\%$ validation accuracy with $4$ blocks or less.
    }
    \label{fig:slice_a5}
\end{figure}
\\ \\
Figure \ref{fig:a5_length_gen} presents the results on the $A_5$ length generalisation experiment. The non-linear recurrent LSTM and sLSTM generalise well, maintaining high test accuracy beyond both the training and validation ranges. Among the parallel-in-time models, three patterns emerge: (i) WH-SLiCE and Mamba do not attain high accuracy even at training lengths; (ii) DeltaProduct and D-DE-SLiCE generalise to approximately twice the training length but not beyond the validation range; and (iii) DE-LNCDE, DPLR-SLiCE, S-SLiCE, and BD-SLiCE sustain high accuracy on sequences at least $8$ times the training length, exceeding the maximum validation length.
\begin{figure}
    \includegraphics[width=\linewidth]{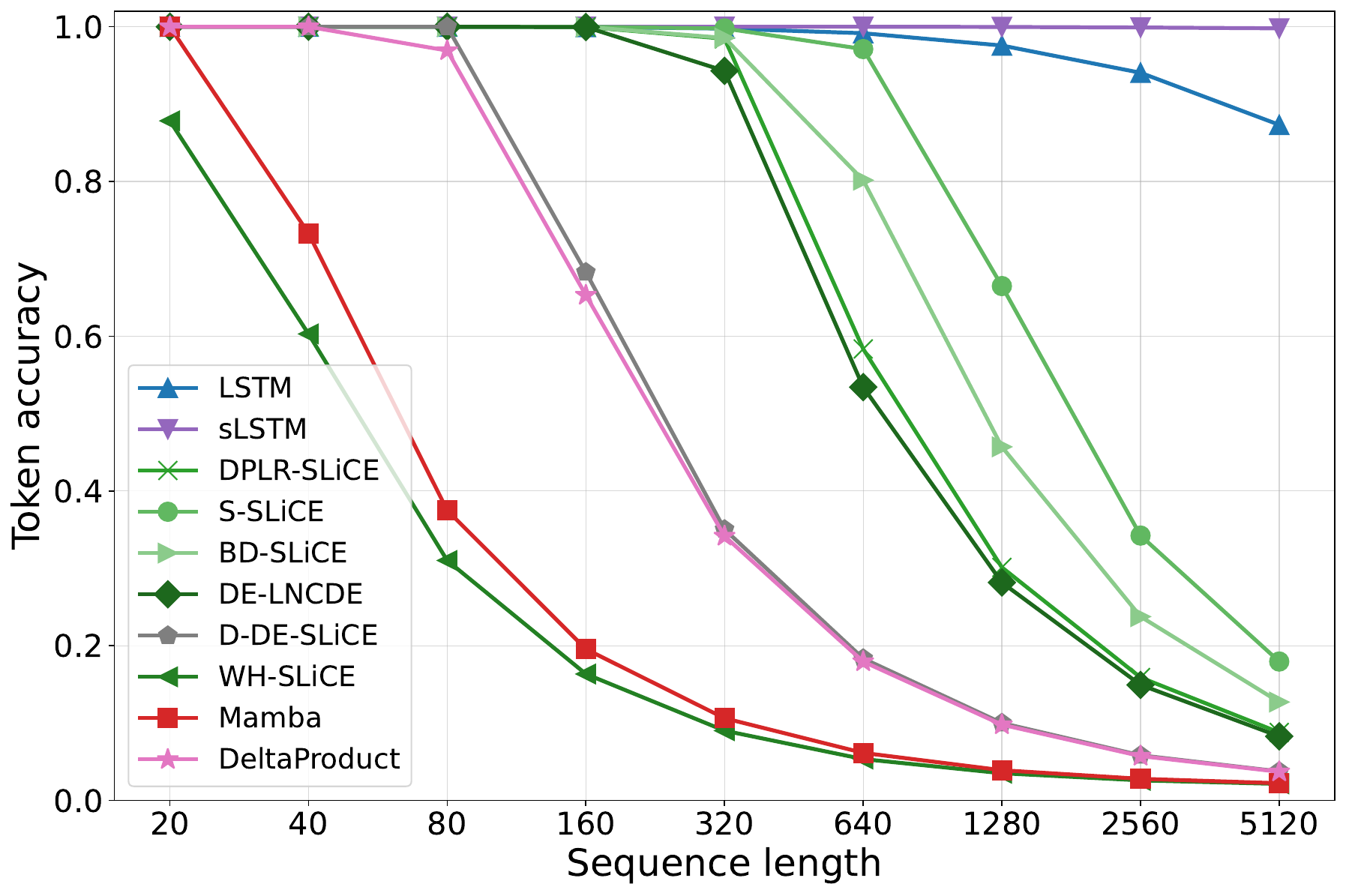}
    \caption{Comparison of an Mamba, DeltaProduct, sLSTM, LSTM, and SLiCEs with diagonal, dense (DE), sparse (S), Walsh--Hadamard (WH), block-diagonal (BD), and diagonal-dense (D-DE) matrices on $A_5$ length generalisation. Models are trained on sequences from length $3$ to $40$ with early stopping on a validation set of sequences from length $40$ to $128$. The plot shows test set token accuracy against sequence length. 
    }
    \label{fig:a5_length_gen}
\end{figure}
\\ \\
The results presented in Table \ref{tab:slice_fl_regular_only} show that expressive structures like BD, D-DE, and DPLR significantly outperform the diagonal baseline on the regular language tasks. The top-performing D-DE and DPLR model, both with an average accuracy of $85.1\%$, demonstrate a substantial improvement over the best diagonal model's $70.2\%$. 
Conversely, the WH structure performs worse than the diagonal model, despite its greater theoretical expressivity. 
This provides an empirical demonstration that, although maximal probabilistic expressivity is a desirable property, it is an asymptotic notion and does not guarantee favourable optimisation, sample efficiency, or inductive bias at finite width. 
We hypothesise that the Walsh--Hadamard structure's fixed global mixing is poorly matched to the algorithmic state-tracking required by regular language tasks, as also suggested by its poor performance on the $A_5$ length generalisation task.
\\ \\
An analysis of the successful hyperparameter choices reveals a fundamental trade-off between model expressivity and hidden state dimension. For a fixed parameter budget, increasing one of these factors necessitates a decrease in the other. For instance, BD-SLiCE achieves its best performance not at an extreme, but with a block size of $b=4$ and hidden dimension $d_h=128$, yielding an average accuracy of $83.8\%$. Similarly, DPLR-SLiCE achieves its best performance with a rank of $4$, yielding the joint highest average accuracy of $85.1\%$. This suggests that the strongest results emerge from maintaining a sufficient degree of structure without overly compressing the hidden state. The D-DE structure operates in a notably different regime, with even its largest dense block configuration retaining a larger hidden dimension than any BD or DPLR variant. As a result, D-DE-SLiCE achieves its best performance at the largest block size, matching the top average accuracy of DPLR-SLiCE.
\begin{table}
\centering
\small
\caption{Average and standard deviation of validation accuracy over five runs on the regular language tasks for SLiCEs with diagonal, Walsh--Hadamard (WH), block-diagonal (BD), diagonal-dense (D-DE), and DPLR structures.}
\begin{tabular}{lccccc}
\toprule
\multicolumn{1}{c|}{\textbf{Model}} & \textbf{Cycle Nav.} & \textbf{Even Pairs} & \makecell{\textbf{Mod Arith.} \\ \textbf{No Brack.}} & \multicolumn{1}{c}{\textbf{Parity}} & \multicolumn{1}{|c}{\textbf{Average}} \\
\midrule
Diagonal$_{d_h=128}$ & $69.5 \pm 6.3$ & $100.0 \pm 0.0$ & $20.8 \pm 0.2$ & $89.12 \pm 18.6$ & $69.9$ \\
Diagonal$_{d_h=512}$ & $59.7 \pm 5.1$ & $100.0 \pm 0.0$ & $20.9 \pm 0.1$ & $100.0 \pm0.0$ & $70.2$ \\
WH$_{d_h=128}$ & $69.7 \pm 8.8$ & $93.1 \pm 13.9$ & $20.5 \pm 0.3$ & $50.7 \pm 0.3$ & $58.5$ \\
WH$_{d_h=512}$ & $35.5 \pm 1.8$ & $58.5 \pm 2.5$ & $23.8 \pm 1.1$ & $71.4 \pm 12.9$ & $47.3$ \\
BD$_{d_h=256,\;b=2}$ & $92.5 \pm 14.0$ & $72.7 \pm 3.0$ & $37.7 \pm 2.6$ & $99.6 \pm 0.8$ & $75.6$ \\
BD$_{d_h=128,\;b=4}$ & $99.8 \pm 0.2$ & $85.9 \pm 11.3$ & $54.0 \pm 12.5$ & $95.3 \pm 3.9$ & $83.8$ \\
BD$_{d_h=64,\;b=8}$ & $99.9 \pm 0.1$ & $91.3 \pm 6.3$ & $70.6 \pm 21.4$ & $54.1 \pm 4.9$ & $79.0$ \\
BD$_{d_h=32,\;b=16}$ & $97.6 \pm 3.5$ & $94.7 \pm 6.5$ & $76.6 \pm 22.1$ & $50.8 \pm 0.3$ & $79.9$ \\
D\textendash DE$_{d_h=510,\;b=2}$ & $61.9 \pm 20.4$ & $91.3 \pm 11.5$ & $20.8 \pm 0.2$ & $97.8 \pm 2.9$ & $67.9$ \\
D\textendash DE$_{d_h=500,\;b=4}$ & $81.6 \pm 15.0$ & $85.3 \pm 18.0$ & $29.4 \pm 15.4$ & $83.7 \pm 9.4$ & $70.0$ \\
D\textendash DE$_{d_h=456,\;b=8}$ & $90.6 \pm 9.4$ & $90.7 \pm 3.0$ & $31.0 \pm 4.2$ & $79.9 \pm 3.5$ & $73.1$ \\
D\textendash DE$_{d_h=272,\;b=16}$ & $73.3 \pm 29.4$ & $84.8 \pm 8.5$ & $98.4 \pm 0.7$ & $83.8 \pm 11.3$ & $85.1$ \\
DPLR$_{d_h=171,r=1}$ & $46.5 \pm 26.3$ & $91.1 \pm 4.4$ & $25.8 \pm 10.2$ & $87.8 \pm 9.5$ & $62.8$ \\
DPLR$_{d_h=102,r=2}$ & $53.1 \pm 14.7$ & $96.8 \pm 5.1$ & $43.9 \pm 9.0$ & $79.7 \pm 14.4$ & $68.4$ \\
DPLR$_{d_h=57,r=4}$ & $81.1 \pm 16.6$ & $100.0 \pm 0.0$ & $68.3 \pm 19.3$ & $91.0 \pm 18.0$ & $85.1$ \\
DPLR$_{d_h=30,r=8}$ & $90.1 \pm 10.2$ & $100.0 \pm 0.0$ & $60.3 \pm 19.7$ & $50.7 \pm 0.3$ & $75.3$ \\
\midrule
\textbf{Random} & $20.0$ & $50.0$ & $20.0$ & $50.0$ & $35.0$ \\
\bottomrule
\end{tabular}
\label{tab:slice_fl_regular_only}
\end{table}
\\ \\
Table~\ref{tab:fl_regular_only} compares SLiCEs to a range of state-of-the-art sequence model baselines. As expected, the recurrent LSTM generalises near-perfectly on all four tasks, establishing a clear performance benchmark. This highlights the inherent strength of non-linear recurrent architectures for problems that require precise, step-by-step state tracking. 
Interestingly, sLSTM and xLSTM do not replicate this success. 
Among the parallel models, the strongest baselines are DeltaNet with negative eigenvalues and Gated DeltaProduct with negative eigenvalues, achieving $77.7\%$ and $80.7\%$, respectively. 
This aligns with the expectation that increased complexity in the state-transition matrix improves state-tracking performance.
\\ \\
This principle is further reinforced by the observation that a diagonal SLiCE outperforms Mamba. Mamba is restricted to a state-transition matrix with eigenvalues in the range $[0,1]$, whereas diagonal SLiCE is unrestricted. 
As shown by \citet{grazzi2024unlockingstatetrackinglinearrnns}, expanding the eigenvalue range of the state-transition matrix to $[-1,1]$ is crucial for solving regular language tasks. 
On average validation accuracy, the diagonal-dense and DPLR SLiCE are the two strongest performing parallelisable models. 
They outperform not only widespread modern architectures like Mamba and the Transformer, but also the strongest DeltaNet baselines specifically designed for increased expressivity.
\\ \\
A Friedman test detected a statistically significant difference in performance among the parallel-in-time models across the four tasks at the $5\%$ significance level, with $\chi^2 = 25.41$, $df = 13$, and $p = 0.0204$.
To assess our architectural hypothesis more directly, we partitioned the parallel models into those with expressive recurrences and those without.
The expressive class consists of DPLR-SLiCE, D-DE-SLiCE, BD-SLiCE, WH-SLiCE, DeltaNet$[-1,1]$, Gated DeltaProduct$[-1,1]$, DeltaNet, and Gated DeltaNet, while the non-expressive class consists of RWKV-7, mLSTM, Transformer, Mamba, S4D, and D-SLiCE. 
We then compared the mean rank of each class on each task, with lower rank indicating better performance. 
The expressive class attained mean ranks of $5.62$, $6.75$, $5.25$, and $6.12$ on Cycle Navigation, Even Pairs, Modular Arithmetic without Brackets, and Parity, respectively, compared to $10.00$, $8.50$, $10.50$, and $9.33$ for the non-expressive class. 
Overall, this corresponds to average ranks of $5.94$ for the expressive class and $9.58$ for the non-expressive class. 
We do not extend this descriptive summary to a statistical one, since four paired datasets are not sufficient for a Wilcoxon signed-rank test to assign significance \citep{wilcoxon1945individual}.
\begin{table}
\centering
\footnotesize
\caption{Average and standard deviation of validation accuracy over five runs for a range of recurrent and parallel models on the formal language tasks.}
\begin{tabular}{lccccc}
\toprule
\multicolumn{1}{c|}{\textbf{Model}} & \textbf{Cycle Nav.} & \textbf{Even Pairs} & \makecell{\textbf{Mod Arith.} \\ \textbf{No Brack.}} & 
\multicolumn{1}{c}{\textbf{Parity}} & \multicolumn{1}{|c}{\textbf{Average}} \\
\midrule
\multicolumn{5}{l}{\textbf{Recurrent}} \\
\midrule
LSTM & $100.0 \pm 0.0$  & $100.0 \pm 0.0$ & $99.9 \pm 0.1$ & $100.0 \pm 0.0$ & $100$ \\
sLSTM & $32.5 \pm 0.4$  & $100.0 \pm 0.0$ & $27.7 \pm 0.6$ & $100.0 \pm 0.0$ & $65.1$\\
xLSTM[1:1] & $53.5 \pm 5.6$  & $99.0 \pm 1.9$ & $29.3 \pm 1.6$ & $100.0 \pm 0.0$ & $70.5$ \\
\midrule
\multicolumn{5}{l}{\textbf{Parallel}} \\
\midrule
DeltaNet & $49.8 \pm 4.7$  & $100.0 \pm 0.0$ & $42.2 \pm 4.8$ & $57.8 \pm 0.8$ & $62.5$ \\
DeltaNet$[-1,1]$ & $46.7 \pm 6.1$  & $100.0 \pm 0.0$ & $66.4 \pm 8.8$ & $97.7 \pm 2.0$  & $77.7$ \\
Gated DeltaNet & $53.8 \pm 8.8$  & $100.0 \pm 0.0$ & $42.8 \pm 8.2$ & $56.5 \pm 1.9$ & $63.3$ \\
Gated DeltaProduct[-1,1] & $46.3 \pm 6.6$ & $100.0 \pm 0.0$ & $78.4 \pm 10.9$ & $98.0 \pm 1.4$ & $80.7$ \\
RWKV-7 & $37.8 \pm 5.0$  & $88.1 \pm 14.2$ & $39.5 \pm 6.1$ & $51.1 \pm 0.3$ & $54.1$\\
mLSTM  &$52.4 \pm 10.5$  & $99.9\pm 0.1$ & $28.8 \pm 3.1$ & $53.0 \pm 2.1$ & $58.5$ \\
Transformer & $24.4\pm 0.5$  & $90.4 \pm 10.4$ & $23.6 \pm 0.7$ & $52.2 \pm 0.4$ & $47.7$ \\
Mamba & $48.4 \pm 2.2$  & $100.0\pm 0.0$ & $33.1\pm 6.6$ & $54.2 \pm 2.1$ & $58.9$ \\
S4D & $23.7 \pm 1.1$  & $68.7 \pm 4.7$ & $21.7 \pm 0.4$ & $51.2 \pm 1.0$ & $41.3$ \\
Diagonal$_{d_h=512}$ & $59.7 \pm 5.1$ & $100.0 \pm 0.0$ & $20.9 \pm 0.1$ & $100.0 \pm0.0$ & $70.2$ \\
WH$_{d_h=128}$ & $69.7 \pm 8.8$ & $93.1 \pm 13.9$ & $20.5 \pm 0.3$ & $50.7 \pm 0.3$ & $58.5$ \\
BD$_{d_h=128,\;b=4}$ & $99.8 \pm 0.2$ & $85.9 \pm 11.3$ & $54.0 \pm 12.5$ & $95.3 \pm 3.9$ & $83.8$ \\
D\textendash DE$_{d_h=272\;b=16}$ & $73.3 \pm 29.4$ & $84.8 \pm 8.5$ & $98.4 \pm 0.7$ & $83.8 \pm 11.3$ & $85.1$ \\
DPLR$_{d_h=57,r=4}$ & $81.1 \pm 16.6$ & $100.0 \pm 0.0$ & $68.3 \pm 19.3$ & $91.0 \pm 18.0$ & $85.1$ \\
\midrule
\bottomrule
\textbf{Random} & $20.0$  & $50.0$  & $20.0$ & $50.0$ & $35.0$ \\
\bottomrule
\end{tabular}
\vspace{-1em}
\label{tab:fl_regular_only}
\end{table}
\\ \\
Table~\ref{tab:uea_full_results_rounded} and \ref{tab:uea_av_results} present the impact of replacing the Log-NCDE's non-linear vector field with a dense, BD, WH, diagonal, D-DE, sparse, or DPLR structured linear vector field. 
The original baselines from Table~\ref{tab:UEA_results_hypopt} are also included. 
A tie-corrected Friedman test across the 14 models and 6 datasets did not detect a statistically significant difference in performance at the $5\%$ significance level, with $\chi^2 = 19.89$, $df = 13$, and $p = 0.098$.
Nevertheless, the results show that the block-diagonal structure achieves very similar performance to the non-linear model, whilst reducing the average time per training step by a factor of nearly $20$. 
Notably, BD-SLiCE also outperforms a dense linear vector field of the same hidden dimension, which suggests that the block-diagonal structure may provide benefits beyond simply reducing computational cost. 
A possible explanation for this is the structure's conceptual link to the multi-head attention mechanism found in Transformers. 
Each block in BD-SLiCE can be viewed as an independent head, processing the input path in its own distinct subspace. 
This architectural choice encourages the model to learn multiple specialised representations of the time-series dynamics in parallel, with each block potentially focusing on different features or patterns. 
This inherent modularity could prevent the overfitting that may affect a single dense model, ultimately contributing to the strong performance of BD-SLiCE.
\begin{table}
\caption{Mean and standard deviation of test set accuracy over five data splits on a subset of the UEA-MTSCA. The best-performing model in each column is highlighted in bold, and the second-best is underlined. Models are sorted by average rank, with the top performing model first.}
\label{tab:uea_full_results_rounded}
\centering
\footnotesize
\begin{tabular}{lcccccc}
\toprule
\textbf{Model} & \textbf{EW} & \textbf{EC} & \textbf{HB} & \textbf{MI} & \textbf{SCP1} & \textbf{SCP2} \\
\midrule
BD-SLiCE & $86.1 \pm 3.6$ & $28.6 \pm 6.4$ & $77.4 \pm 5.6$ & $53.0 \pm 2.1$ & $\underline{84.9 \pm 1.9}$ & $\mathbf{54.0 \pm 7.4}$ \\
Log-NCDE & $85.6 \pm 5.1$ & $\mathbf{34.4 \pm 6.4}$ & $75.2 \pm 4.7$ & $53.7 \pm 5.3$ & $83.1 \pm 2.9$ & $53.7 \pm 4.1$ \\
D-DE-SLiCE & $85.6 \pm 6.0$ & $27.3 \pm 6.9$ & $73.9 \pm 3.8$ & $\mathbf{54.7 \pm 3.5}$ & $84.7 \pm 3.7$ & $51.6 \pm 5.0$ \\
WH-SLiCE & $85.0 \pm 6.0$ & $30.1 \pm 4.8$ & $76.1 \pm 5.9$ & $49.5 \pm 6.8$ & $82.4 \pm 2.2$ & $51.9 \pm 5.4$ \\
DPLR-SLiCE & $84.4 \pm 5.2$ & $27.6 \pm 4.7$ & $74.2 \pm 5.4$ & $51.6 \pm 5.8$ & $83.5 \pm 2.2$ & $50.9 \pm 5.8$ \\
D-SLiCE & $79.4 \pm 5.8$ & $27.1 \pm 4.6$ & $72.9 \pm 5.1$ & $\underline{54.4 \pm 6.3}$ & $83.5 \pm 2.2$ & $53.0 \pm 5.9$ \\
LRU & $\underline{87.8 \pm 2.9}$ & $21.5 \pm 2.2$ & $\mathbf{78.4 \pm 6.7}$ & $48.4 \pm 5.1$ & $82.6 \pm 3.5$ & $51.2 \pm 3.6$ \\
S6 & $85.0 \pm 16.1$ & $26.4 \pm 6.5$ & $76.5 \pm 8.3$ & $51.3 \pm 4.8$ & $82.8 \pm 2.8$ & $49.9 \pm 9.5$ \\
S5 & $81.1 \pm 3.7$ & $24.1 \pm 4.4$ & $\underline{77.7 \pm 5.6}$ & $47.7 \pm 5.5$ & $\mathbf{89.9 \pm 4.7}$ & $50.5 \pm 2.6$ \\
S-SLiCE & $\underline{87.8 \pm 4.2}$ & $\underline{30.4 \pm 6.7}$ & $72.6 \pm 5.5$ & $47.7 \pm 2.4$ & $82.8 \pm 1.6$ & $49.5 \pm 4.0$ \\
DE-LNCDE & $\mathbf{88.3 \pm 3.7}$ & $25.8 \pm 4.1$ & $74.2 \pm 4.5$ & $49.8 \pm 5.3$ & $81.9 \pm 4.0$ & $49.8 \pm 2.4$ \\
NCDE & $75.0 \pm 4.0$ & $29.9 \pm 6.6$ & $73.9 \pm 2.6$ & $49.5 \pm 2.9$ & $79.8 \pm 5.7$ & $53.0 \pm 2.9$ \\
NRDE & $83.9 \pm 7.3$ & $25.3 \pm 1.8$ & $72.9 \pm 4.9$ & $47.0 \pm 5.8$ & $80.9 \pm 2.6$ & $\underline{53.7 \pm 6.9}$ \\
Mamba & $70.9 \pm 15.9$ & $27.9 \pm 4.6$ & $76.2 \pm 3.9$ & $47.7 \pm 4.5$ & $80.7 \pm 1.4$ & $48.2 \pm 4.0$ \\
\bottomrule
\end{tabular}
\end{table}
\begin{table}[H]
\caption{Average test accuracy, rank, training time per 1,000 steps, and GPU memory usage across six datasets from the UEA-MTSCA. All SLiCE variants use a parallel associative scan during training. Therefore, the Walsh–Hadamard, DPLR, and sparse SLiCEs are treated as dense LNCDEs (see Section \ref{sec:slice_comp}), and their timing and GPU memory results are omitted. All timing and GPU memory results were performed on an NVIDIA H100 GPU.}
\vspace{0.5mm}
\label{tab:uea_av_results}
\centering
\footnotesize
\begin{tabular}{lcccc}
\toprule
\textbf{Model} & \textbf{Av. Acc} & \textbf{Av. Rank} & \textbf{Av. Time / 1000 Steps (s)} & \textbf{Av. GPU mem (MB)} \\
\midrule
BD-SLiCE & $64.0$ & $3.2$ & $68.1$ & $2344$\\
Log-NCDE & $64.3$ & $4.0$ & $1321.7$ & $2177$ \\
D-DE-SLiCE & $63.0$ & $5.7$ & $66.7$ & $2302$ \\
WH-SLiCE & $62.5$ & $6.7$ & $-$ & $-$ \\
DPLR-SLiCE & $62.0$ & $7.0$ & $-$ & $-$ \\
D-SLiCE & $61.7$ & $7.2$ & $11.0$ & $1875$ \\
LRU & $61.7$ & $7.3$ & $26.9$ & $4308$ \\
S6 & $62.0$ & $7.7$ & $20.1$ & $2938$ \\
S5 & $61.8$ & $8.0$ & $21.9$ & $3327$ \\
S-SLiCE & $61.8$ & $8.2$ & $-$ & $-$ \\
DE-LNCDE & $61.6$ & $8.3$ & $77.2$ & $12756$ \\
NCDE & $60.2$ & $8.8$ & $6923$ & $1962$ \\
NRDE & $60.6$ & $10.3$ & $3431$ & $2858$ \\
Mamba & $58.6$ & $10.8$ & $60.0$ & $4535$ \\
\bottomrule
\end{tabular}
\end{table}
Figure~\ref{fig:time-vs-acc} provides a visual summary of the accuracy--speed trade-off on the UEA-MTSCA benchmark, with the Pareto frontier highlighting the non-dominated models. 
The frontier consists of diagonal SLiCE, S6, block-diagonal SLiCE, and Log-NCDE, each representing a different balance between training time and predictive accuracy. 
Log-NCDE attains the strongest test accuracy, but at a substantially higher average time per training step. 
Moving along the frontier, block-diagonal SLiCE reduces the average time per training step by a factor of $20$ relative to Log-NCDE while maintaining very similar test accuracy. 
Diagonal SLiCE pushes training time lower still, below that of the SSM baselines, though with some further loss in accuracy consistent with its reduced expressivity. 
S6 occupies an intermediate point on the frontier, offering a modest improvement in accuracy over diagonal SLiCE, although the gain is smaller than that obtained by moving from S6 to block-diagonal SLiCE. 
By contrast, dense LNCDE slightly increases run-time and substantially raises GPU memory usage relative to the block-diagonal variant, while also reducing test accuracy, which may indicate over-fitting due to the large number of parameters per state transition. 
Finally, Figure~\ref{fig:time-vs-acc} shows that Mamba, NCDE, and NRDE do not offer competitive trade-offs on either axis on this benchmark.
\begin{figure}
    \centering
    \includegraphics[width=\linewidth]{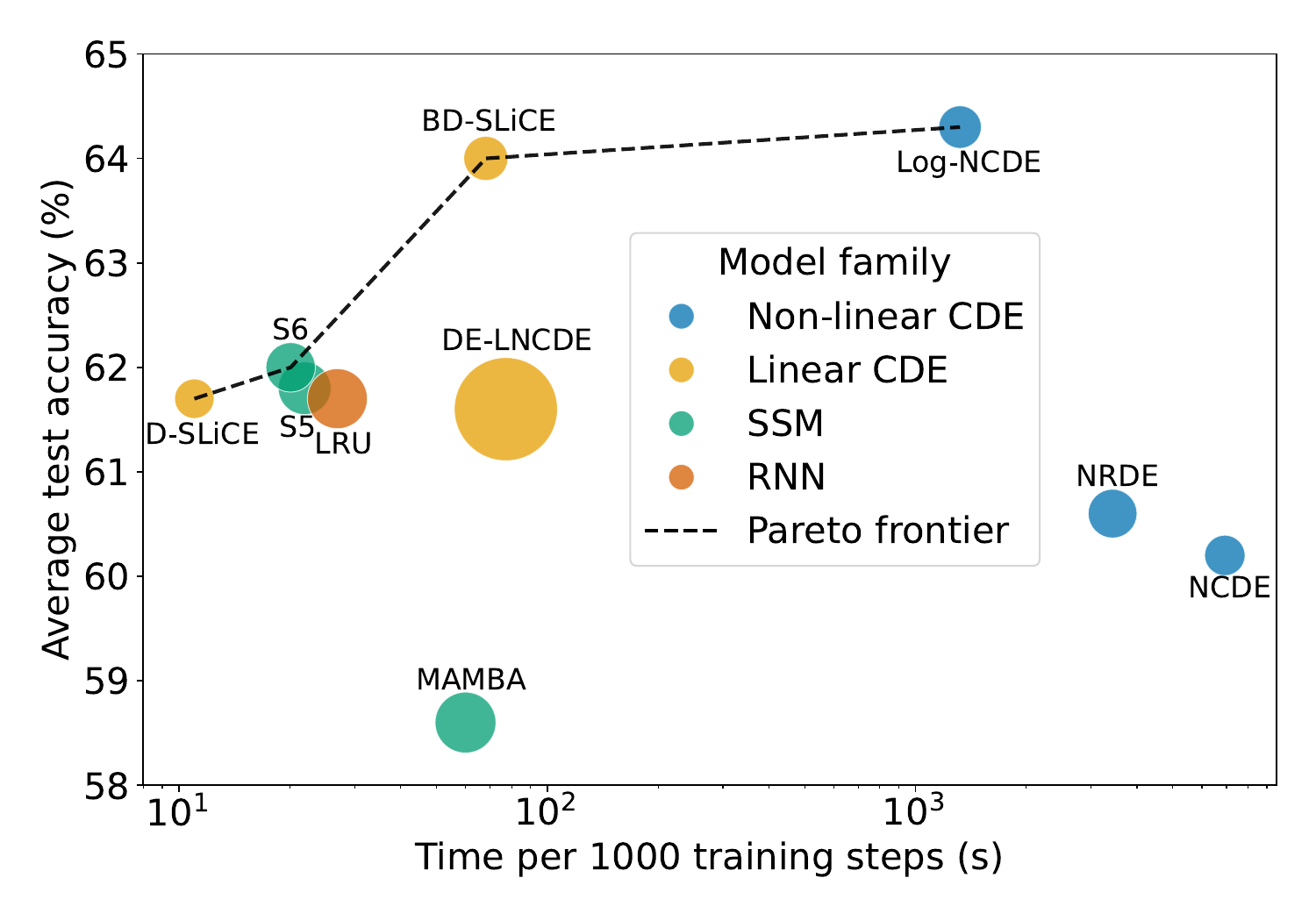}
    \caption{Average per-step training time versus average validation accuracy across six multivariate time series classification datasets from the UEA-MTSCA. Each point represents a model, with circle area proportional to average GPU memory usage. We compare four families of models: a recurrent neural network (LRU), SSMs (S5, S6, and Mamba), non-linear NCDEs (NCDE, NRDE, and Log-NCDE), and Linear NCDEs (Diagonal SLiCE, Block-Diagonal SLiCE, and Dense LNCDE). All timing and GPU memory results were performed on an NVIDIA H100 GPU.}
    \label{fig:time-vs-acc}
\end{figure}
\\ \\
To understand their relative impact, we evaluate how the Log-ODE method and parallel associative scan influence time per training step and GPU memory for the diagonal SLiCE, block-diagonal SLiCE, and dense LNCDE. The EigenWorms dataset is chosen for this comparison, as it contains approximately $18{,}000$ observations per time series. Table~\ref{tab:time_par_assoc_scan} summarises the effect of applying a parallel associative scan with varying chunk sizes on time per $1{,}000$ training steps without the Log-ODE method. For all three models, increasing the number of parallel steps yields strong reductions in time per training step. The impact of the high I/O costs associated with an associative scan is evident from the diminishing benefit of a small number of parallel steps as you move from diagonal, through block-diagonal, to dense matrices.
\begin{table}
\caption{Time per $1{,}000$ training steps (s) for diagonal SLiCE, block-diagonal SLiCE, and dense LNCDE on EigenWorms when a parallel associative scan is applied with various chunk sizes. Experiments were performed on an NVIDIA H100, and the batch size is $1$.}
\label{tab:time_par_assoc_scan}
\centering
\begin{tabular}{c|ccc}
\toprule
Parallel Steps & D-SLiCE & BD-SLiCE & DE-LNCDE \\
\midrule
None & $311.0$ & $374.89$ & $444.68$ \\
$4$ & $134.2$ & $326.75$ & $439.51$ \\
$16$ & $57.50$ & $161.74$ & $257.56$ \\
$64$ & $31.01$ & $68.59$ & $126.01$ \\
$256$ & $21.10$ & $30.53$ & $71.54$ \\
\bottomrule
\end{tabular}
\end{table}
\\ \\
Table~\ref{tab:time_assoc_scan_log_ode} compares the time per $1{,}000$ training steps and GPU memory for diagonal SLiCE, block-diagonal SLiCE, and dense LNCDE on EigenWorms when using a parallel associative scan and the Log-ODE method. As expected, both GPU memory and run-time increase monotonically for every combination of associative scan and Log-ODE method as the model structure transitions from diagonal, through block-diagonal, to dense matrices. The consistent $2.69$ GB floor across several configurations likely reflects peak memory usage from fixed operations outside the recurrence. Without the Log-ODE method, the dense LNCDE exhibits high GPU memory consumption, which constrained experiments to a batch size of $4$. When both the Log-ODE method and a parallel associative scan are applied, the diagonal and block-diagonal SLiCE achieve comparable times per training step, indicating that the recurrence contributes less to overall computation under these parameter settings. Overall, combining the Log-ODE method with a parallel associative scan reduces the time per training step by a factor of approximately $30$ for the diagonal and block-diagonal SLiCE without affecting GPU memory, and by a factor of $16$ for the dense LNCDE while lowering GPU memory usage by over a factor of $5$.
\begin{table}
\caption{Comparison of training time and GPU memory for EigenWorms using a parallel associative scan and applying the Log-ODE method. Experiments were performed on an NVIDIA H100 with a batch size of $4$.}
\label{tab:time_assoc_scan_log_ode}
\centering
\begin{tabular}{cccrr}
\toprule
\multirow{2}{*}{Model} &
\multirow{2}{*}{Metric} &
\multirow{2}{*}{\makecell{Log-ODE \\ Interval}} &
\multicolumn{2}{c}{Parallel Steps} \\ \cmidrule(l){4-5}
 & & & None & $128$ \\
\midrule
\multirow{4}{*}{D-SLiCE}
& \multirow{2}{*}{Time / $1$k steps [s]} & $1$ & $317.49$ & $23.51$ \\
& & $12$ & $33.03$ & $10.96$ \\[4pt]
& \multirow{2}{*}{GPU Memory [GB]} & $1$ & $2.69$ & $2.69$ \\
& & $12$ & $2.69$ & $2.69$ \\
\midrule
\multirow{4}{*}{BD-SLiCE}
& \multirow{2}{*}{Time / $1$k steps [s]} & $1$ & $378.20$ & $48.74$ \\
& & $12$ & $46.98$ & $13.24$ \\[4pt]
& \multirow{2}{*}{GPU Memory [GB]} & $1$ & $2.69$ & $4.73$ \\
& & $12$ & $2.69$ & $2.69$ \\
\midrule
\multirow{4}{*}{DE-LNCDE}
& \multirow{2}{*}{Time / $1$k steps [s]} & $1$ & $465.94$ & $167.86$ \\
& & $12$ & $47.95$ & $29.37$ \\[4pt]
& \multirow{2}{*}{GPU Memory [GB]} & $1$ & $35.45$ & $51.84$ \\
& & $12$ & $2.69$ & $6.79$ \\
\bottomrule
\end{tabular}
\end{table}

\subsection{Limitations}

First, we do not impose stability constraints on the matrices $A_\theta^i$ in the models considered here. 
Consequently, although we did not observe instability on the long time series considered in this chapter, one should not expect the current parameterisations to remain uniformly stable as the sequence length grows arbitrarily large. 
Therefore, investigating stable matrix structures for the $A_\theta^i$ is a natural direction for extending SLiCEs to more extreme long-range settings.
\\ \\
Second, SLiCEs are sequence-to-sequence models. 
They produce a state update for each new sample in the input, and are therefore susceptible to over-sampled data in the same manner as other discrete sequence models. 
Combining SLiCEs with the Log-ODE method allows them to consume paths as input, by abstracting the flow on each interval away from the individual samples. 
However, the output of a Log-SLiCE is itself a sequence of samples, corresponding to the value of the output path at the end of each interval to which the Log-ODE method is applied. 
This means that one cannot stack two Log-SLiCEs directly, as the first level produces a sequence whereas the second consumes a path. 
A natural next step is therefore to develop a path-to-path model, as discussed in more detail in Section~\ref{sec:conc_future}.

\section{Conclusion}

\label{sec:lncde_conclusion}

This chapter introduced Linear NCDEs, where the non-linear vector field $f_{\theta}(h_s)$ is replaced by a linear one, $A_{\theta}h_s$. Linear NCDEs retain the theoretical expressivity of their non-linear counterparts while enabling parallel-in-time computation. The core innovation is removing the recurrent differential equation solve and instead solving for the flow independently on each interval. These flows are then composed using a parallel associative scan. This yields substantial performance gains, particularly for long sequences. On the EigenWorms dataset, which has nearly $18{,}000$ observations per time series, this method reduces the time per training step by a factor of over $250$. When combined with the Log-ODE method from Chapter~\ref{chap:ncde}, this performance gain increases to a factor of nearly $1500$, reducing the time for $10{,}000$ training steps from $3$ days to $3$ minutes.
\\ \\
Section \ref{sec:ssms} used the Linear NCDE framework to theoretically analyse the expressivity of SSMs. 
This analysis demonstrated that prominent models like S4D and Mamba correspond to affine Linear NCDEs whose state-transition matrices are constrained to be diagonal. 
The diagonal structure was shown to severely limit theoretical expressivity. 
This limitation was confirmed empirically on state-tracking benchmarks, where diagonal models required stacking multiple layers to solve tasks that dense Linear NCDEs could solve with one.
\\ \\
While theoretically powerful, the number of parameters and computational cost of dense state-transition matrices scales with the cube of the hidden dimension, making them impractical for large models. 
To address this, Section \ref{sec:slice} introduced SLiCEs, which bridge the efficiency-expressivity gap. 
By employing structured matrices, such as block-diagonal, diagonal-plus-low-rank, and Walsh-Hadamard, SLiCEs reduce computational complexity while retaining the maximal probabilistic expressivity of their dense counterparts.
Extensive experiments validated the SLiCE approach. 
On formal language benchmarks that test state-tracking, SLiCEs outperformed specialised SSM baselines. 
In particular, a block-diagonal and DPLR structure established the new state-of-the-art for parallel models on these tasks. 
This success extends to the real-world UEA datasets, where replacing the non-linear vector field of a Log-NCDE with a block-diagonal linear vector field does not decrease performance, while bringing the training time in line with other modern sequence models.
\\ \\
SLiCEs provide a strong foundation for future work in several key directions. 
An immediate technical goal is the development of efficient GPU kernels for the matrix exponential to allow for exact flow computations. 
Further research can address open questions around path-to-path training, input symmetries, and the inclusion of bidirectional context in the vector fields, all of which are discussed in more detail in Section~\ref{sec:conc_future}. 
Addressing these challenges will be crucial for the ultimate goal of scaling these expressive, parallelisable architectures to new and more complex domains.
\chapter{Conclusion}

\begin{quoting}
    We can only see a short distance ahead, but we can see plenty there that needs to be done.
\end{quoting}
\noindent\large---Alan Turing, \emph{Computing Machinery and Intelligence} (1950)\normalsize

\section{Principal Contributions}

Chapter~\ref{chap:cde} established a mathematical framework for understanding data that evolves continuously in time. 
Paths are the central objects, signatures provide principled summaries of path segments, CDEs describe how paths drive the evolution of hidden states, and the Log-ODE method provides an efficient way to approximate the resulting dynamics. 
NCDEs translate this framework into machine learning, but their practical scalability is limited by the need to solve a non-linear differential equation sequentially during training. 
The principal contributions of this thesis addressed this limitation in three complementary stages. 
First, Chapter~\ref{chap:ncde} applied the Log-ODE method, an efficient and accurate approach to approximating the solution of a CDE, to NCDEs. 
This improved both their computational efficiency and empirical performance. 
Second, Section~\ref{sec:lncde} introduced Linear NCDEs, which replaced the non-linear vector field of an NCDE with a linear one. 
This yielded explicit solutions for the flow, removing the need for a non-linear differential equation solver and enabling parallel-in-time computation. 
Third, Section~\ref{sec:slice} developed SLiCEs, which improved the balance between expressivity and computational efficiency through the use of structured matrices. 
Together, these contributions reduced the time per training step by up to three orders of magnitude while improving performance on real-world benchmarks, making continuous-time models feasible at scales that were previously impractical.
\\ \\
Applying the Log-ODE method to NCDEs required showing that the neural network vector field satisfied the necessary regularity assumptions.
Chapter~\ref{chap:lipgamma} strengthened the regularity theory needed for this step by proving an explicit bound for the composition of $\mathrm{Lip}(\gamma)$ functions when $1 < \gamma \leq 2$. 
Chapter~\ref{chap:ncde} then applied this theory to a class of fully connected neural networks parametrising NCDE vector fields, proving explicit $\mathrm{Lip}(\gamma)$ bounds for these architectures. 
This justified the use of the Log-ODE method during training. 
Empirically, Log-NCDEs improved upon standard NCDEs across real-world multivariate time series benchmarks, increasing predictive performance while reducing training cost. 
However, because the hidden state was still evolved by a non-linear differential equation, the forward pass remained inherently sequential.
\\ \\
Linear NCDEs addressed this bottleneck by replacing the non-linear vector field with one that is linear in the hidden state. 
This change yielded closed-form flows on each interval, removing the need for a recurrent differential equation solver and enabling parallel-in-time computation via an associative scan. 
Section~\ref{sec:lncde} further showed that this linearisation preserves the theoretical expressivity of NCDEs while maintaining strong empirical performance on the benchmarks considered. 
Section~\ref{sec:ssms_are_lcdes} provided a unified interpretation of modern structured state-space models, such as S4D and Mamba, as affine Linear NCDEs with diagonal structure. 
Both the theory and the state-tracking experiments showed that, although efficient, diagonal structure discards hidden-state interactions that are important for expressivity.
\\ \\
SLiCEs refined this picture by improving the trade-off between expressivity and efficiency within the Linear NCDE framework. 
Rather than using dense or diagonal matrices, Section~\ref{sec:slice} introduced structured matrices, such as block-diagonal and diagonal-plus-low-rank, which reduce the computational cost while retaining maximal probabilistic expressivity. 
This showed that poor state-tracking is not an unavoidable consequence of efficient parallel training, but instead a consequence of overly restrictive recurrence structure. 
Empirically, SLiCEs outperformed specialised parallel baselines on formal language and state-tracking benchmarks. 
On real-world multivariate time series tasks, they also delivered a substantially stronger accuracy-speed trade-off than either non-linear continuous-time models or diagonal structured state-space models.
\\ \\
Overall, this thesis showed that the main obstacle to scalable continuous-time learning was not the continuous-time viewpoint itself, but the combination of non-linear vector fields and inefficient numerical methods. 
By combining the Log-ODE method, closed-form linear flows, and expressive structured recurrences, it showed that continuous-time models can be made scalable without giving up the advantages of the continuous-time framework.

\section{Future Work}

\label{sec:conc_future}

There are a number of interesting directions for future work, from which we highlight five we view as particularly promising or important.
\\ \\
First, parallel associative scans provide significant benefit for training models on the million parameter scale. However, the high I/O costs prevent scaling to billions of parameters, as discussed in Section \ref{sec:slice_comp}. 
We are currently exploring the possibility of a GPU kernel that will allow fast matrix exponentials and parallel associative scans when we have many small independent systems, i.e. a block-diagonal SLiCE.
Another approach to overcoming this barrier is building on the work of \cite{yang2024parallelizing}, and developing fast chunk-wise methods for a wider range of structured matrices, as explored by \citet{cirone2025parallelflowparallelizinglineartransformers}. 
\\ \\
Second, there is an inherent over-parametrisation in Linear NCDEs. To see why, consider a trained Linear NCDE
\begin{equation}
    \label{eq:change_of_basis1}
    \mathrm{d} h_s = \sum_{i=1}^{d_\omega}A^i_{\theta}h_s\mathrm{d}\omega^{X,i}_s,
\end{equation}
with a linear readout $L_{\theta}h_s$. 
For any invertible matrix $P \in \mathbb{R}^{d_h\times d_h}$, we can apply a change of basis $h'_s=Ph_s$. The dynamics of the transformed state are then given by
\begin{equation}
    \label{eq:change_of_basis2}
    \mathrm{d} h'_s = \sum_{i=1}^{d_\omega}PA^i_{\theta}P^{-1}h'_s\mathrm{d}\omega^{X,i}_s = \sum_{i=1}^{d_\omega}B^i_{\theta}h'_s\mathrm{d}\omega^{X,i}_s.
\end{equation}
When equipped with the transformed linear readout $L_{\theta}P^{-1}h'_s$, the system described by \eqref{eq:change_of_basis2} yields an identical output to \eqref{eq:change_of_basis1}, despite having different learnt parameters. 
This establishes an entire equivalence class of models for any single solution found during training, with each member of the class corresponding to a different choice of basis for the hidden state space. 
Similarly to the identifiability issues arising in the classical system identification setting discussed in Section~\ref{sec:classical_approaches}, the parameters of a Linear NCDE are therefore only identifiable up to similarity transformations.
\\ \\
This equivalence has significant potential implications for optimisation and interpretation. 
The existence of a continuous family of equivalent solutions creates manifolds in the loss landscape along which the gradient is zero, corresponding to moves in parameter space that amount only to a change of basis \citep{Li2019Symmetry}. 
In principle, this can lead to ill-conditioned optimisation and may hinder the convergence of gradient-based methods \citep{saarinen1993ill}. 
However, in the SLiCE experiments, we did not observe clear optimisation failures that could be specifically attributed to this symmetry. 
Even so, it implies that different training runs may converge to parametrisations that realise the same input-output map in different hidden bases, which makes the learned matrices difficult to compare or interpret directly. 
Therefore, a natural direction for future work is to optimise over similarity classes rather than individual parametrisations.
Similar issues have motivated SGD variants for ReLU neural networks that operate on a quotient space of the weights \citep{meng2019gsgd}. 
Alternatively, one could reduce the redundancy by imposing additional structure on the hidden-state representation, for example by constraining one of the $A^i_\theta$ to be upper triangular.
\\ \\
Third, \citet{cirone2024deepSSM} established that replacing the Linear NCDE's linear readout with a non-linear function is sufficient to extend their path-to-point universality to path-to-path universality.
However, the optimal approach to training a path-to-path Linear NCDE remains an open question.
Standard objectives, such as pointwise mean squared error, treat paths as collections of independent samples. 
This neglects global temporal structure and leads to over-fitting local fluctuations, while failing to capture essential features such as ordering, variability, or long-range dependence \citep{ramsay2005functional, cuturi2017softdtw, chen2025patchwise}.
Alternative choices include integral norms \citep{ferraty2007nonparametric}, distances in signature space \citep{Lyons1998, chevyrev2016characteristic, kiraly2019kernels, salvi2021signature, cass2024weighted}, or Kernel/MMD-based metrics tailored to streams \citep{Cuturi2007AKF, mikalsen2018time, wynne2022kernel}.
These options can handle irregular sampling, are reparametrisation-invariant when desired, and are suitable for models whose outputs are paths.
\\ \\
Fourth, Log-NCDEs are path-to-sequence models, as they produce the value of the output path only at the endpoints of the intervals where the Log-ODE method has been applied.
Linear NCDEs provide a tractable approach to path-to-path models, by extending their flow from describing the hidden state trajectory to the signature of those trajectories. In this way, it is possible to query the output path and obtain the signature over an interval in the same manner we would query the input path.
This elevates continuous-time learning from path-to-sequence architectures to genuine path-to-path models.
Furthermore, this would enable path-to-path stacking, where the output of one Log-Linear NCDE becomes the input to another.
\\ \\
Finally, as Linear CDEs allow us to contextualise the state-transition matrix based on the input, we could seek to contextualise the CDE based on the past, and potentially future, of the input path.
A possible implementation would be allowing the Linear NCDE's vector field to depend linearly on the signature of the input path.
However, this significantly increases the parameter count and computational complexity of the model.
An alternative approach is conditioning the vector field to depend on the output of another Linear NCDE.
This is an alternative form of model stacking, where instead of using the output of the Linear NCDE from one layer as the input to the next, we instead condition the vector field of the next layer based on the output of the previous.
\\ \\
Each of these directions offers an opportunity to advance both the theoretical underpinnings and the practical scalability of continuous-time models.
Pursuing them will mature the path-based methodologies advocated for throughout this thesis, enhancing their power and broadening their applicability.
\newpage
\printbibliography

@book{roman2007advanced14,
  title={Advanced Linear Algebra},
  author={Roman, S.},
  chapter={14},
  series={Graduate Texts in Mathematics},
  year={2007},
  publisher={Springer New York}
}

@book{Luenberger1969,
  author    = {David G. Luenberger},
  title     = {Optimization by Vector Space Methods},
  publisher = {John Wiley \& Sons},
  year      = {1969},
  address   = {New York},
}

@article{hambly2010uniqueness,
  title={Uniqueness for the signature of a path of bounded variation and the reduced path group},
  author={Hambly, B. and Lyons, T.},
  journal={Annals of Mathematics},
  volume={171},
  pages={109-167},
  year={2010}
}

@article{Morrill2019,
  author={Morrill, James and Kormilitzin, Andrey and Nevado-Holgado, Alejo and Swaminathan, Sumanth and Howison, Sam and Lyons, Terry},
  journal={Computing in Cardiology (CinC)}, 
  title={The Signature-Based Model for Early Detection of Sepsis From Electronic Health Records in the Intensive Care Unit}, 
  year={2019},
}

@article{Moore2019-tp,
  title    = "Using path signatures to predict a diagnosis of Alzheimer's
              disease",
  author   = "Moore, P J and Lyons, T J and Gallacher, J and {Alzheimer's
              Disease Neuroimaging Initiative}",
  journal  = "PLoS One",
  volume   =  14,
  number   =  9,
  year     =  2019,
}

@inproceedings{Compagnoni2023On,
author = {Compagnoni, Enea and Scampicchio, Anna and Biggio, Luca and Orvieto, Antonio and Hofmann, Thomas and Teichmann, Josef},
booktitle = {International Joint Conference on Neural Networks (IJCNN)},
year = {2023},
title = {On the Effectiveness of Randomized Signatures as Reservoir for Learning Rough Dynamics},
}

@article{Cuchiero2022Discrete,
  author  = {Christa Cuchiero and Lukas Gonon and Lyudmila Grigoryeva and Juan-Pablo Ortega and Josef Teichmann},
  title   = {Discrete-time Signatures and Randomness in Reservoir Computing},
  journal = {IEEE Transactions on Neural Networks and Learning Systems},
  year    = {2022},
  volume  = {33},
  number  = {11},
  pages   = {6321--6330},
}

@inproceedings{kidger2020neuralcde,
    title={{N}eural {C}ontrolled {D}ifferential {E}quations for {I}rregular {T}ime {S}eries},
    author={Kidger, Patrick and Morrill, James and Foster, James and Lyons, Terry},
    booktitle={Proceedings of the 34th Conference on Neural Information Processing System (NeurIPS)},
    year={2020}
}

@inproceedings{morrill2021neuralrough,
  title={Neural Rough Differential Equations for Long Time Series},
  author={Morrill, James and Salvi, Cristopher and Kidger, Patrick and Foster, James and Lyons, Terry},
  booktitle={Proceedings of the 38th International Conference on Machine Learning (ICML)},
  year={2021}
}

@inproceedings{Walker2024LogNCDE,
  title={Log Neural Controlled Differential Equations: The Lie Brackets Make a Difference},
  author={Walker, Benjamin and McLeod, Andrew D. and Qin, Tiexin and Cheng, Yichuan and Li, Haoliang and Lyons, Terry},
  booktitle={Proceedings of the 41st International Conference on Machine Learning (ICML)},
  year={2024}
}

@article{Young1936AnIO,
  title={An inequality of the H{\"o}lder type, connected with Stieltjes integration},
  author={L. C. Young},
  journal={Acta Mathematica},
  year={1936},
  volume={67},
  pages={251-282}
}

@article{Chen1957IntegrationOP,
  title={Integration of Paths, Geometric Invariants and a Generalized Baker-Hausdorff Formula},
  author={Kuo Tsai Chen},
  journal={Annals of Mathematics},
  year={1957},
  volume={65},
  number={1},
}

@article{Chen1954Iterated,
author = {Kuo Tsai Chen},
title = {Iterated Integrals and Exponential Homomorphisms},
journal = {Proceedings of the London Mathematical Society},
volume = {s3-4},
number = {1},
pages = {502-512},
year = {1954}
}

@article{Weierstrass1885Uber,
author = {Weierstrass, K.}, 
title = {Uber die analytische Darstellbarkeit sogenannter willkurlicher Functionen einer reellen Veranderlichen},
journal = {Sitzungsberichte der Akademie zu Berlin},
pages = {633-639,789-805},
year = {1885}
}

@book{lyons2002,
    author = {Lyons, Terry and Qian, Zhongmin},
    title = {System Control and Rough Paths},
    publisher = {Oxford University Press},
    series = {Oxford Mathematical Monographs},
    year = {2002},
}

@article{BOEDIHARDJO2016720,
title = {The signature of a rough path: Uniqueness},
journal = {Advances in Mathematics},
volume = {293},
pages = {720-737},
year = {2016},
author = {Horatio Boedihardjo and Xi Geng and Terry Lyons and Danyu Yang},
}

@book{lyons2007differential,
  title={Differential Equations Driven by Rough Paths: {\'E}cole D'{\'e}t{\'e} de Probabilit{\'e}s de Saint-Flour XXXIV-2004},
  author={Lyons, T. and Caruana, M. and L{\'e}vy, T.},
  number={no. 1908},
  year={2007},
  publisher={Springer}
}

@book{Rudin1991,
  author    = {Walter Rudin},
  title     = {Functional Analysis},
  edition   = {2nd},
  publisher = {McGraw-Hill},
  year      = {1991},
  address   = {New York},
}

@article{Whitney1934analytic,
 author = {Hassler Whitney},
 journal = {Transactions of the American Mathematical Society},
 number = {1},
 pages = {63--89},
 publisher = {American Mathematical Society},
 title = {Analytic Extensions of Differentiable Functions Defined in Closed Sets},
 volume = {36},
 year = {1934}
}

@book{reutenauer1993free,
  title={Free Lie Algebras},
  author={Reutenauer, C.},
  series={London Mathematical Society Monographs},
  year={1993},
  publisher={Clarendon Press}
}

@book{cass2012new,
  title={New Trends in Stochastic Analysis and Related Topics: A Volume in Honour of Professor K. D. Elworthy},
  chapter={2},
  author={Cass, Thomas and Litterer, Christian and Lyons, Terry},
  series={Interdisciplinary mathematical sciences},
  year={2012},
  publisher={World Scientific}
}

@book{stein1970singular,
 author = {Elias M. Stein},
 publisher = {Princeton University Press},
 title = {Singular Integrals and Differentiability Properties of Functions},
 year = {1970}
}

@book{Lee2013,
  author    = {John M. Lee},
  title     = {Introduction to Smooth Manifolds},
  edition   = {2nd},
  series    = {Graduate Texts in Mathematics},
  volume    = {218},
  publisher = {Springer},
  address   = {New York},
  year      = {2013},
}

@article{baldi2018time,
author={Baldi, Pietro
and Berti, Massimiliano
and Haus, Emanuele
and Montalto, Riccardo},
title={Time quasi-periodic gravity water waves in finite depth},
journal={Inventiones mathematicae},
year={2018},
day={01},
volume={214},
number={2},
pages={739-911},
}

@article{hochreiter1997long,
author = {Hochreiter, Sepp and Schmidhuber, Jürgen},
year = {1997},
pages = {1735-80},
title = {Long Short-term Memory},
volume = {9},
number = {8},
journal = {Neural computation},
}

@inproceedings{GRU,
  title={On the Properties of Neural Machine Translation: Encoder--Decoder Approaches},
  author={Cho, Kyunghyun and van Merri{\"e}nboer, Bart and Bahdanau, Dzmitry and Bengio, Yoshua},
  booktitle={Proceedings of SSST-8, Eighth Workshop on Syntax, Semantics and Structure in Statistical Translation},
  pages={103--111},
  year={2014},
}

@inproceedings{orvieto2023resurrecting,
  title        = {Resurrecting Recurrent Neural Networks for Long Sequences},
  author       = {Orvieto, Antonio and Smith, Samuel L. and Gu, Albert and Fernando, Anushan and G{\"u}l{\c c}ehre, {\c C}aglar and Pascanu, R{\'a}zvan and De, Soham},
  booktitle    = {Proceedings of the 40th International Conference on Machine Learning (ICML)},
  year         = {2023},
}

@article{walker1931,
  author    = {G. T. Walker},
  title     = {On Periodicity in Series of Related Terms},
  journal   = {Proceedings of The Royal Society A: Mathematical, Physical and Engineering Sciences},
  volume    = {131},
  number    = {818},
  pages     = {518--532},
  year      = {1931},
  publisher = {The Royal Society},
}

@book{wiener1949,
  author    = {Norbert Wiener},
  title     = {Extrapolation, Interpolation, and Smoothing of Stationary Time Series},
  publisher = {The MIT Press},
  year      = {1949},
  address   = {Cambridge, MA}
}

@article{kalman1960,
  author    = {Rudolf E. Kalman},
  title     = {A New Approach to Linear Filtering and Prediction Problems},
  journal   = {ASME Journal of Basic Engineering},
  volume    = {82},
  pages     = {35--45},
  year      = {1960},
}

@book{gelb1974applied,
  author    = {Arthur Gelb},
  title     = {Applied Optimal Estimation},
  chapter   = {6},
  publisher = {MIT Press},
  year      = {1974},
  address   = {Cambridge, MA}
}

@inproceedings{julier1997new,
author = {Simon J. Julier and Jeffrey K. Uhlmann},
title = {New extension of the Kalman filter to nonlinear systems},
volume = {3068},
booktitle = {Signal Processing, Sensor Fusion, and Target Recognition VI},
organization = {International Society for Optics and Photonics},
publisher = {SPIE},
pages = {182 -- 193},
year = {1997},
}

@article{gordon1993novel,
author = {N.J. Gordon  and D.J. Salmond  and A.F.M. Smith },
title = {Novel approach to nonlinear/non-Gaussian Bayesian state estimation},
journal = {IEE Proceedings F (Radar and Signal Processing)},
volume = {140},
issue = {2},
pages = {107-113},
year = {1993},
}

@article{elman1990finding,
  author    = {Jeffrey L. Elman},
  title     = {Finding Structure in Time},
  journal   = {Cognitive Science},
  volume    = {14},
  number    = {2},
  pages     = {179--211},
  year      = {1990},
}

@inproceedings{bahdanau2015neural,
  author    = {Dzmitry Bahdanau and Kyunghyun Cho and Yoshua Bengio},
  title     = {Neural Machine Translation by Jointly Learning to Align and Translate},
  booktitle = {Proceedings of the 3rd International Conference on Learning Representations (ICLR)},
  year      = {2015},
}

@inproceedings{vaswani2017attention,
  author    = {Ashish Vaswani and Noam Shazeer and Niki Parmar and Jakob Uszkoreit and Llion Jones and Aidan N. Gomez and Łukasz Kaiser and Illia Polosukhin},
  title     = {Attention Is All You Need},
  booktitle = {Proceedings of the 31st Conference on Neural Information Processing Systems (NeurIPS)},
  year      = {2017},
}

@inproceedings{gu2024mamba,
  title     = {Mamba: Linear-Time Sequence Modeling with Selective State Spaces},
  author    = {Albert Gu and Tri Dao},
  booktitle = {Proceedings of the First Conference on Language Modeling},
  year      = {2024},
}

@inproceedings{gu2021efficiently,
  title={Efficiently Modeling Long Sequences with Structured State Spaces},
  author={Gu, Albert and Goel, Karan and R\'e, Christopher},
  booktitle={Proceedings of The 10th International Conference on Learning Representations ({ICLR})},
  year={2022}
}

@inproceedings{brown2020language,
  title={Language models are few-shot learners},
  author={Brown, Tom B and Mann, Benjamin and Ryder, Nick and Subbiah, Melanie and Kaplan, Jared and Dhariwal, Prafulla and Neelakantan, Arvind and Shyam, Pranav and Sastry, Girish and Askell, Amanda and others},
  booktitle={Proceedings of the 34th Conference on Neural Information Processing Systems (NeurIPS)},
  year={2020}
}

@inproceedings{He2016DeepRL,
  title     = {Deep Residual Learning for Image Recognition},
  author    = {He, Kaiming and Zhang, Xiangyu and Ren, Shaoqing and Sun, Jian},
  booktitle = {Proceedings of the 2016 IEEE Conference on Computer Vision and Pattern Recognition (CVPR)},
  pages     = {770--778},
  year      = {2016}
}

@article{martinez1992discrete,
	author = {R. Rico-Mart{\'i}nez and K. Krischer and I. G. Kevrekidis and M. C. Kube and  J. L. Hudson},
	title = {Discrete-vs. continuous-time nonlinear signal processing of Cu electrodissolution data},
	journal = {Chemical Engineering Communications},
	volume = {118},
	number = {1},
	pages = {25-48},
	year  = {1992},
	publisher = {Taylor \& Francis},
}

@inproceedings{chen2018neural,
  title={Neural Ordinary Differential Equations},
  author={Chen, Ricky TQ and Rubanova, Yulia and Bettencourt, Jesse and Duvenaud, David K},
  booktitle={Proceedings of the 32nd Conference on Neural Information Processing Systems (NeurIPS)},
  year={2018}
}

@article{bagnall2018ueamultivariatetimeseries,
      title={The UEA multivariate time series classification archive, 2018}, 
      author={Anthony Bagnall and Hoang Anh Dau and Jason Lines and Michael Flynn and James Large and Aaron Bostrom and Paul Southam and Eamonn Keogh},
      year={2018},
      journal={arXiv preprint, arXiv:1811.00075},
}

@article{pineda1987generalization,
 author = {Pineda, Fernando J.},
 journal = {Physical Review Letters},
 volume = {59},
 number = {19},
 title = {Generalization of Back-Propagation to Recurrent Neural Networks},
 year = {1987},
 pages = {2229-2232}
}

@article{Bryson1962steepest,
    author = {Bryson, A. E. and Denham, W. F.},
    title = "{A Steepest-Ascent Method for Solving Optimum Programming Problems}",
    journal = {Journal of Applied Mechanics},
    volume = {29},
    number = {2},
    pages = {247-257},
    year = {1962},
}

@article{pearlmutter1989learning,
  author={Pearlmutter, Barak A.},
  journal={Neural Computation}, 
  title={Learning State Space Trajectories in Recurrent Neural Networks}, 
  year={1989},
  volume={1},
  number={2},
  pages={263-269},
}

@article{boxuan2018residual,
author = {Yue, Boxuan and Fu, Junwei and Liang, Jun},
year = {2018},
title = {Residual Recurrent Neural Networks for Learning Sequential Representations},
volume = {9},
number = {3},
journal = {Information},
}

@phdthesis{kidger2022neuraldifferentialequations,
  title       = {On Neural Differential Equations},
  author      = {Kidger, Patrick},
  school      = {University of Oxford},
  year        = {2022},
}

@inproceedings{GRU-ODE,
author = {Brouwer, Edward De and Simm, Jaak and Arany, Adam and Moreau, Yves},
title = {GRU-ODE-Bayes: Continuous Modeling of Sporadically-Observed Time Series},
year = {2019},
booktitle = {Proceedings of the 33rd International Conference on Neural Information Processing Systems (NeurIPS)},
}

@inproceedings{ODERNN,
  title={Latent Ordinary Differential Equations for Irregularly-Sampled Time Series},
  author={Yulia Rubanova and Tian Qi Chen and David Kristjanson Duvenaud},
  booktitle={Proceedings of the 33rd International Conference on Neural Information Processing Systems (NeurIPS)},
  year={2019}
}

@phdthesis{Hochreiter1991UntersuchungenZD,
  title={Untersuchungen zu dynamischen neuronalen Netzen},
  author={Sepp Hochreiter},
  year={1991},
  school={Technische Universität München}
}

@incollection{vanishing,
  author       = {Sepp Hochreiter and Yoshua Bengio and Paolo Frasconi and J{\"u}rgen Schmidhuber},
  title        = {Gradient Flow in Recurrent Nets: The Difficulty of Learning Long–Term Dependencies},
  booktitle    = {A Field Guide to Dynamical Recurrent Neural Networks},
  editor       = {S.~C. Kremer and J.~F. Kolen},
  publisher    = {Wiley-IEEE Press},
  year         = {2001},
  pages        = {237--244},
}

@article{robustness,
author = {Xu, Huan and Mannor, Shie},
year = {2012},
title = {Robustness and Generalization},
volume = {86},
journal = {Machine Learning},
pages = {391--423}
}

@inproceedings{NNLip,
  author    = {Szegedy, Christian and Zaremba, Wojciech and Sutskever, Ilya and Bruna, Joan and Erhan, Dumitru and Goodfellow, Ian J. and Fergus, Rob},
  title     = {Intriguing properties of neural networks},
  booktitle = {Proceedings of the 2nd International Conference on Learning Representations (ICLR)},
  year      = {2014},
}

@article{ELFWING20183,
title = {Sigmoid-weighted linear units for neural network function approximation in reinforcement learning},
journal = {Neural Networks},
volume = {107},
pages = {3-11},
year = {2018},
note = {Special issue on deep reinforcement learning},
author = {Stefan Elfwing and Eiji Uchibe and Kenji Doya},
}

@inproceedings{weightreg,
author = {Hinton, Geoffrey E.},
title = {Learning Translation Invariant Recognition in Massively Parallel Networks},
year = {1987},
publisher = {Springer-Verlag},
address = {Berlin, Heidelberg},
booktitle = {Proceedings of the Parallel Architectures and Languages Europe, Volume I: Parallel Architectures},
pages = {1–13},
}

@inproceedings{weightdecay,
author = {Krogh, Anders and Hertz, John A.},
title = {A Simple Weight Decay Can Improve Generalization},
year = {1991},
booktitle = {Proceedings of the 5th International Conference on Neural Information Processing Systems (NeurIPS)},
}

@book{Griewank2000,
title = {Evaluating Derivatives: Principles and Techniques of Algorithmic Differentiation},
author = {Andreas Griewank and Andrea Walther},
year = {2008},
edition = {2nd},
publisher = {Society for Industrial and Applied Mathematics},
}

@inproceedings{Hall1950ABF,
  title={A basis for free Lie rings and higher commutators in free groups},
  author={Marshall Hall},
  year={1950},
  booktitle={Proceedings of the American Mathematical Society},
  volume={1},
  pages={575-581}
}

@software{jax2018github,
  author = {James Bradbury and Roy Frostig and Peter Hawkins and Matthew James Johnson and Chris Leary and Dougal Maclaurin and George Necula and Adam Paszke and Jake Vander{P}las and Skye Wanderman-{M}ilne and Qiao Zhang},
  title = {{JAX}: composable transformations of {P}ython+{N}um{P}y programs},
  url = {http://github.com/google/jax},
  year = {2018},
}

@article{Ree1958LieEA,
  title={Lie Elements and an Algebra Associated With Shuffles},
  author={Rimhak Ree},
  journal={Annals of Mathematics},
  year={1958},
  volume={68},
  number={2},
}

@article{StoneWeierstrass,
 author = {M. H. Stone},
 journal = {Mathematics Magazine},
 number = {4},
 pages = {167--184},
 publisher = {Mathematical Association of America},
 title = {The Generalized Weierstrass Approximation Theorem},
 volume = {21},
 year = {1948}
}

@article{Arribas2018DerivativesPU,
  author = {Perez Arribas, Imanol},
  title = {Derivatives pricing using signature payoffs},
  year = {2018},
  journal= {arXiv preprint arXiv:1809.09466}
}

@inproceedings{Lyons2014,
  author       = {Terry Lyons},
  title        = {Rough Paths, Signatures and the Modelling of Functions on Streams},
  booktitle    = {Proceedings of the International Congress of Mathematicians (ICM)},
  volume = {4},
  year         = {2014},
}

@book{ryan2002introduction,
  author    = {Raymond A. Ryan},
  title     = {Introduction to Tensor Products of Banach Spaces},
  volume    = {73},
  year      = {2002},
  publisher = {Springer},
  series    = {Springer Monographs in Mathematics}
}

@article{ChangLyonsNi2018,
  author       = {Jiawei Chang and Terry Lyons and Hao Ni},
  title        = {Super-multiplicativity and a lower bound for the decay of the signature of a path of finite length},
  journal      = {Comptes Rendus Mathématique},
  volume       = {356},
  number       = {7},
  pages        = {720--724},
  year         = {2018},
}

@inproceedings{Bonnier2019DeepST,
  title={Deep Signature Transforms},
  author={Patric Bonnier and Patrick Kidger and Imanol Perez Arribas and Cristopher Salvi and Terry Lyons},
  booktitle={Neural Information Processing Systems (NeurIPS)},
  year={2019}
}

@article{SigPrimer,
  author = {Chevyrev, Ilya and Kormilitzin, Andrey},
  title = {A Primer on the Signature Method in Machine Learning},
  publisher = {arXiv},
  year = {2016},
  journal= {arXiv preprint arXiv:1603.03788}
}

@article{SigPrimer2,
  author       = {Andrew McLeod and Terry Lyons},
  title        = {Signature methods in machine learning},
  journal      = {EMS Surveys in Mathematical Sciences},
  year         = {2025},
}

@article{Lyons1994DIFFERENTIALED,
  title={Differential Equations Driven by Rough Signals (I): An Extension of an Inequality of L. C. Young},
  author={Terry Lyons},
  journal={Mathematical Research Letters},
  year={1994},
  volume={1},
  pages={451-464}
}

@article{Boutaib2013,
author = {Boutaib, Youness and Gyurkó, Lajos and Lyons, Terry and Yang, Danyu},
year = {2013},
pages = {25-53},
title = {Dimension-free Euler estimates of rough differential equations},
volume = {59},
number = {1},
journal = {Revue Roumaine des Mathematiques Pures et Appliquees}
}

@article{CASTELL199513,
title = {An efficient approximation method for stochastic differential equations by means of the exponential Lie series},
journal = {Mathematics and Computers in Simulation},
volume = {38},
number = {1},
pages = {13-19},
year = {1995},
author = {Fabienne Castell and Jessica Gaines},
}

@inproceedings{S5,
      title={Simplified State Space Layers for Sequence Modeling}, 
      author={Jimmy T. H. Smith and Andrew Warrington and Scott W. Linderman},
      year={2023},
      booktitle={Proceedings of The 11th International Conference on Learning Representations (ICLR)}
}

@inproceedings{dauphin2017language,
  title        = {Language Modeling with Gated Convolutional Networks},
  author       = {Dauphin, Yann N. and Fan, Angela and Auli, Michael and Grangier, David},
  booktitle    = {Proceedings of the 34th International Conference on Machine Learning (ICML)},
  year         = {2017},
}

@inproceedings{kingma2017adam,
      title={Adam: A Method for Stochastic Optimization}, 
      author={Diederik P. Kingma and Jimmy Ba},
      year={2015},
      booktitle={Proceedings of the 3rd International Conference on Learning Representations (ICLR)}
}

@inproceedings{tay2020long,
  title={Long Range Arena: A Benchmark for Efficient Transformers},
  author={Yi Tay and Mostafa Dehghani and Samira Abnar and Yikang Shen and Dara Bahri and Philip Pham and Jinfeng Rao and Liu Yang and Sebastian Ruder and Donald Metzler},
  booktitle={Proceedings of the 9th International Conference on Learning Representations (ICLR)},
  year={2021}
}

@inproceedings{ioffe2015batch,
  title={Batch Normalization: Accelerating Deep Network Training by Reducing Internal Covariate Shift},
  author={Ioffe, Sergey and Szegedy, Christian},
  booktitle={Proceedings of the 32nd International Conference on Machine Learning (ICML)},
  year={2015},
}

@article{nguyen2022s4nd,
      title={S4ND: Modeling Images and Videos as Multidimensional Signals Using State Spaces}, 
      author={Eric Nguyen and Karan Goel and Albert Gu and Gordon W. Downs and Preey Shah and Tri Dao and Stephen A. Baccus and Christopher Ré},
      year={2022},
      journal={Proceedings of the 36th Conference on Neural Information Processing Systems (NeurIPS)}
}

@inproceedings{goel2022sashimi,
  title={It's Raw! Audio Generation with State-Space Models},
  author={Goel, Karan and Gu, Albert and Donahue, Chris and R{\'e}, Christopher},
  booktitle={Proceedings of the 39th International Conference on Machine Learning (ICML)},
  year={2022}
}

@inproceedings{lu2023structured,
      title={Structured State Space Models for In-Context Reinforcement Learning}, 
      author={Chris Lu and Yannick Schroecker and Albert Gu and Emilio Parisotto and Jakob Foerster and Satinder Singh and Feryal Behbahani},
      year={2023},
      booktitle={Proceedings of the 37th Conference on Neural Information Processing Systems (NeurIPS)}
}

@inproceedings{peng2023rwkv,
  title={RWKV: Reinventing RNNs for the Transformer Era},
  author={Peng, Bo  and
      Alcaide, Eric  and
      Anthony, Quentin  and
      Albalak, Alon  and
      Arcadinho, Samuel  and
      Biderman, Stella  and
      Cao, Huanqi  and
      Cheng, Xin  and
      Chung, Michael  and
      Derczynski, Leon  and
      Du, Xingjian  and
      Grella, Matteo  and
      Gv, Kranthi  and
      He, Xuzheng  and
      Hou, Haowen  and
      Kazienko, Przemyslaw  and
      Kocon, Jan  and
      Kong, Jiaming  and
      Koptyra, Bart{\l}omiej  and
      Lau, Hayden  and
      Lin, Jiaju  and
      Mantri, Krishna Sri Ipsit  and
      Mom, Ferdinand  and
      Saito, Atsushi  and
      Song, Guangyu  and
      Tang, Xiangru  and
      Wind, Johan  and
      Wo{\'z}niak, Stanis{\l}aw  and
      Zhang, Zhenyuan  and
      Zhou, Qinghua  and
      Zhu, Jian  and
      Zhu, Rui-Jie},
  booktitle={Findings of the Association for Computational Linguistics: EMNLP 2023},
  year      = {2023},
}

@inproceedings{yang2024gated,
  title={Gated Linear Attention Transformers with Hardware-Efficient Training},
  author={Yang, Songlin and Wang, Bailin and Shen, Yikang and Panda, Rameswar and Kim, Yoon},
  booktitle={Proceedings of the 41st International Conference on Machine Learning (ICML)},
  year={2024}
}

@inproceedings{li2022makes,
  title={What Makes Convolutional Models Great on Long Sequence Modeling?},
  author={Li, Yuhong and Cai, Tianle and Zhang, Yi and Chen, Deming and Dey, Debadeepta},
  booktitle={Proceedings of the 11th International Conference on Learning Representations (ICLR)},
  year={2023}
}

@inproceedings{qin2024hgrn2,
  title={HGRN2: Gated linear RNNs with state expansion},
  author={Qin, Zhen and Yang, Songlin and Sun, Weixuan and Shen, Xuyang and Li, Dong and Sun, Weigao and Zhong, Yiran},
  booktitle={Proceedings of the 1st Conference on Language Modeling (COLM)},
  year={2024}
}

@inproceedings{fu2022hungry,
  title={Hungry Hungry Hippos: Towards Language Modeling with State Space Models},
  author={Fu, Daniel Y and Dao, Tri and Saab, Khaled Kamal and Thomas, Armin W and Rudra, Atri and Re, Christopher},
  booktitle={Proceedings of the 11th International Conference on Learning Representations (ICLR)},
  year={2023}
}

@inproceedings{wang2023pretraining,
  title     = {Pretraining Without Attention},
  author    = {Wang, Junxiong and Yan, Jing Nathan and Gu, Albert and Rush, Alexander M.},
  booktitle = {Findings of the Association for Computational Linguistics: EMNLP 2023},
  year      = {2023},
}

@inproceedings{arora2023zoology,
  title={Zoology: Measuring and improving recall in efficient language models},
  author={Arora, Simran and Eyuboglu, Sabri and Timalsina, Aman and Johnson, Isys and Poli, Michael and Zou, James and Rudra, Atri and R{\'e}, Christopher},
  booktitle={Proceedings of the 12th International Conference on Learning Representations (ICLR)},
  year={2024}
}

@article{Reiss2019DeepPPG,
AUTHOR = {Reiss, Attila and Indlekofer, Ina and Schmidt, Philip and Van Laerhoven, Kristof},
TITLE = {Deep PPG: Large-Scale Heart Rate Estimation with Convolutional Neural Networks},
JOURNAL = {Sensors},
VOLUME = {19},
YEAR = {2019},
NUMBER = {14},
ARTICLE-NUMBER = {3079},
}

@inproceedings{zucchet2024recurrent,
  title={Recurrent neural networks: vanishing and exploding gradients are not the end of the story},
  author={Zucchet, Nicolas and Orvieto, Antonio},
  booktitle={Proceedings of the 38th Conference on Neural Information Processing Systems (NeurIPS)},
  year={2024}
}

@inproceedings{siegelmann1992computational,
  title={On the computational power of neural nets},
  author={Siegelmann, Hava T and Sontag, Eduardo D},
  booktitle={Proceedings of the 5th Annual Workshop on Computational Learning Theory (COLT)},
  pages={440--449},
  year={1992}
}

@InProceedings{hanson2020universal,
  title = 	 {Universal Simulation of Stable Dynamical Systems by Recurrent Neural Nets},
  author =       {Hanson, Joshua and Raginsky, Maxim},
  booktitle = 	 {Proceedings of the 2nd Conference on Learning for Dynamics and Control},
  pages = 	 {384--392},
  year = 	 {2020},
}

@article{li2022approximation,
  author  = {Zhong Li and Jiequn Han and Weinan E and Qianxiao Li},
  title   = {Approximation and Optimization Theory for Linear Continuous-Time Recurrent Neural Networks},
  journal = {Journal of Machine Learning Research},
  year    = {2022},
  volume  = {23},
  number  = {42},
}

@inproceedings{orvieto2023universality,
  title={Universality of Linear Recurrences Followed by Non-linear Projections: Finite-Width Guarantees and Benefits of Complex Eigenvalues},
  author={Orvieto, Antonio and De, Soham and Gulcehre, Caglar and Pascanu, Razvan and Smith, Samuel L},
  booktitle={Proceedings of the 41st International Conference on Machine Learning (ICML)},
  year={2024}
}

@inproceedings{wang2023state,
  title={State-space Models with Layer-wise Nonlinearity are Universal Approximators with Exponential Decaying Memory},
  author={Wang, Shida and Xue, Beichen},
  booktitle={Proceedings of the 37th Conference on Neural Information Processing Systems (NeurIPS)},
  year={2023}
}

@inproceedings{jelassi2024repeat,
  title={Repeat after me: Transformers are better than state space models at copying},
  author={Jelassi, Samy and Brandfonbrener, David and Kakade, Sham M and Malach, Eran},
  booktitle={Proceedings of the 41st International Conference on Machine Learning (ICML)},
  year={2024}
}

@inproceedings{merrill2024illusion,
  title={The illusion of state in state-space models},
  author={Merrill, William and Petty, Jackson and Sabharwal, Ashish},
  booktitle={Proceedings of the 41st International Conference on Machine Learning (ICML)},
  year={2024}
}

@book{friz_victoir_2010, 
series={Cambridge Studies in Advanced Mathematics},
title={Multidimensional Stochastic Processes as Rough Paths: Theory and Applications},
publisher={Cambridge University Press},
author={Friz, Peter K. and Victoir, Nicolas B.},
year={2010},
}

@book{BerlinetThomasAgnan2004,
  title     = {Reproducing Kernel Hilbert Spaces in Probability and Statistics},
  author    = {Berlinet, Alain and Thomas-Agnan, Christine},
  year      = {2004},
  publisher = {Springer},
}

@book{KreyszigFunctionalAnalysis,
  author    = {Kreyszig, Erwin},
  title     = {Introductory Functional Analysis with Applications},
  publisher = {John Wiley \& Sons},
  year      = {1978},
}

@book{Apostol1974MathematicalAnalysis,
  author    = {Apostol, Tom M.},
  title     = {Mathematical Analysis},
  edition   = {2nd},
  publisher = {Addison-Wesley},
  year      = {1974},
}

@inproceedings{cirone2024deepSSM,
  title     = {Theoretical Foundations of Deep Selective State-Space Models},
  author    = {Nicola Muca Cirone and Antonio Orvieto and Benjamin Walker and Cristopher Salvi and Terry Lyons},
  booktitle = {Proceedings of the 38th Conference on Neural Information Processing Systems (NeurIPS)},
  year      = {2024},
}

@article{Dasgupta2003AnEP,
  title={An elementary proof of a theorem of Johnson and Lindenstrauss},
  author={Sanjoy Dasgupta and Anupam Gupta},
  journal={Random Structures \& Algorithms},
  year={2003},
  volume={22},
  number={1},
  pages={60-65}
}

@inproceedings{cuchiero2021expressive,
  title     = {Expressive Power of Randomized Signature},
  author    = {Cuchiero, Christa and Gonon, Lukas and Grigoryeva, Lyudmila and Ortega, Juan-Pablo and Teichmann, Josef},
  booktitle = {NeurIPS 2021 Workshop on the Symbiosis of Deep Learning and Differential Equations (DLDE)},
  year      = {2021}
}

@article{bayer2023adaptive,
      title={An Adaptive Algorithm for Rough Differential Equations}, 
      author={Christian Bayer and Simon Breneis and Terry Lyons},
      year={2023},
      journal={arXiv preprint arXiv:2307.12590},
}

@inproceedings{walker2025structuredlinearcdesmaximally,
    title={Structured Linear CDEs: Maximally Expressive and Parallel-in-Time Sequence Models}, 
    author={Benjamin Walker and Lingyi Yang and Nicola Muca Cirone and Cristopher Salvi and Terry Lyons},
    year={2025},
    booktitle = {Proceedings of the 39th Conference on Neural Information Processing Systems (NeurIPS)}
}

@book{Lang1999,
  author    = {Serge Lang},
  title     = {Fundamentals of Differential Geometry},
  publisher = {Springer},
  year      = {1999},
  series    = {Graduate Texts in Mathematics},
  volume    = {191},
}

@article{FaaDiBruno1855,
  author  = {Faà di Bruno, Francesco},
  title   = {Sullo sviluppo delle funzioni},
  journal = {Annali di Scienze Matematiche e Fisiche},
  volume  = {6},
  pages   = {479--480},
  year    = {1855},
}

@article{albrecht1971,
 author = {Felix Albrecht and Harold G. Diamond and Maurice Heins},
 journal = {Indiana University Mathematics Journal},
 number = {4},
 pages = {347--350},
 publisher = {Indiana University Mathematics Department},
 title = {A Converse of Taylor's Theorem},
 volume = {21},
 year = {1971}
}

@article{wells1973differentiable,
  author    = {John C. Wells},
  title     = {Differentiable functions on Banach spaces with Lipschitz derivatives},
  journal   = {Journal of Differential Geometry},
  volume    = {8},
  number    = {1},
  pages     = {135--152},
  year      = {1973},
  publisher = {International Press},
}

@phdthesis{boutaib2016lipschitz,
  author       = {Youness Boutaib},
  title        = {Lipschitz Geometry and Rough Paths},
  school       = {University of Oxford},
  year         = {2016},
  type         = {PhD thesis},
}

@article{lyons2025,
author = {Lyons, Terry and McLeod, Andrew},
year = {2025},
title = {Higher order Lipschitz Sandwich theorems},
volume = {111},
number = {3},
journal = {Journal of the London Mathematical Society},
}

@article{fefferman2006whitney,
  title={Whitney's extension problem for $C^m$},
  author={Fefferman, Charles},
  journal={Annals of Mathematics},
  volume={164},
  number={1},
  pages={313--359},
  year={2006},
  publisher={JSTOR}
}

@article{yule1927method,
  author    = {Yule, G. Udny},
  title     = {On a Method of Investigating Periodicities in Disturbed Series, with Special Reference to Wolfer’s Sunspot Numbers},
  journal   = {Philosophical Transactions of the Royal Society of London. Series A, Containing Papers of a Mathematical or Physical Character},
  volume    = {226},
  pages     = {267--298},
  year      = {1927}
}

@article{touvron2023llamaopenefficientfoundation,
      title={LLaMA: Open and Efficient Foundation Language Models}, 
      author={Hugo Touvron and Thibaut Lavril and Gautier Izacard and Xavier Martinet and Marie-Anne Lachaux and Timothée Lacroix and Baptiste Rozière and Naman Goyal and Eric Hambro and Faisal Azhar and Aurelien Rodriguez and Armand Joulin and Edouard Grave and Guillaume Lample},
      year={2023},
      journal={arXiv pre-print arXiv:2302.13971},
}

@article{slutsky1927,
  author    = {Eugen E. Slutsky},
  title     = {The Summation of Random Causes as the Source of Cyclical Processes},
  journal   = {Voprosy Koniunktury},
  year      = {1927},
  volume    = {3},
  number    = {1},
  pages     = {34--64},
}

@article{slutsky1937,
  author    = {Eugen E. Slutsky},
  title     = {The Summation of Random Causes as the Source of Cyclical Processes},
  journal   = {Econometrica},
  year      = {1937},
  volume    = {5},
  number    = {2},
  pages     = {105--146},
  note      = {English translation of Slutsky (1927)}
}

@book{box1970time,
  author    = {George E. P. Box and Gwilym M. Jenkins},
  title     = {Time Series Analysis: Forecasting and Control},
  publisher = {Holden-Day},
  year      = {1970},
  address   = {San Francisco}
}

@article{cybenko1989approximation,
  title={Approximation by superpositions of a sigmoidal function},
  author={Cybenko, George},
  journal={Mathematics of Control, Signals and Systems},
  volume={2},
  pages={303--314},
  year={1989},
}

@article{hornik1991approximation,
  title={Approximation capabilities of multilayer feedforward networks},
  author={Hornik, Kurt},
  journal={Neural Networks},
  volume={4},
  number={2},
  pages={251--257},
  year={1991},
  publisher={Elsevier}
}

@inproceedings{gu2022s4d,
  title     = {On the Parameterization and Initialization of Diagonal State Space Models},
  author    = {Gu, Albert and Gupta, Ankit and Goel, Karan and R{\'e}, Christopher},
  booktitle = {Proceedings of the 36th Conference on Neural Information Processing Systems (NeurIPS)},
  year      = {2022},
}

@article{srivastava2014dropout,
  title={Dropout: A simple way to prevent neural networks from overfitting},
  author={Srivastava, Nitish and Hinton, Geoffrey and Krizhevsky, Alex and Sutskever, Ilya and Salakhutdinov, Ruslan},
  journal={Journal of Machine Learning Research},
  volume={15},
  number={56},
  pages={1929--1958},
  year={2014},
  publisher={JMLR}
}

@article{ba2016layer,
  title={Layer normalization},
  author={Ba, Jimmy Lei and Kiros, Jamie Ryan and Hinton, Geoffrey E},
  journal={arXiv preprint arXiv:1607.06450},
  year={2016}
}

@incollection{blelloch1993prefix,
  author    = {Blelloch, Guy E.},
  title     = {Prefix Sums and Their Applications},
  booktitle = {Synthesis of Parallel Algorithms},
  pages     = {35--60},
  year      = {1993}
}

@inproceedings{
yang2024parallelizing,
title={Parallelizing Linear Transformers with the Delta Rule over Sequence Length},
author={Songlin Yang and Bailin Wang and Yu Zhang and Yikang Shen and Yoon Kim},
booktitle={Proceedings of the 38th Conference on Neural Information Processing Systems (NeurIPS)},
year={2024},
}

@article{reyna2023heart,
  author    = {Reyna, M. A. and Kiarashi, Y. and Elola, A. and Oliveira, J. and Renna, F. and Gu, A. and {Perez Alday}, E. A. and Sadr, N. and Sharma, A. and Kpodonu, J. and Mattos, S. and Coimbra, M. T. and Sameni, R. and Rad, A. B. and Clifford, G. D.},
  title     = {Heart murmur detection from phonocardiogram recordings: The George B. Moody PhysioNet Challenge 2022},
  journal   = {PLOS Digital Health},
  year      = {2023},
  volume    = {2},
  number    = {9},
}

@article{goldberger2000physiobank,
  title={PhysioBank, PhysioToolkit, and PhysioNet: Components of a new research resource for complex physiologic signals},
  author={Goldberger, Ary L and Amaral, Luis AN and Glass, Leon and Hausdorff, Jeffrey M and Ivanov, Plamen Ch and Mark, Roger G and Mietus, Joseph E and Moody, George B and Peng, Chung-Kang and Stanley, H Eugene},
  journal={Circulation},
  volume={101},
  number={23},
  pages={e215--e220},
  year={2000},
  publisher={Am Heart Assoc}
}

@inproceedings{deng2009imagenet,
  title     = {ImageNet: A Large-Scale Hierarchical Image Database},
  author    = {Deng, Jia and Dong, Wei and Socher, Richard and Li, Li-Jia and Li, Kai and Fei-Fei, Li},
  booktitle = {2009 IEEE Conference on Computer Vision and Pattern Recognition (CVPR)},
  pages     = {248--255},
  year      = {2009},
  organization = {IEEE}
}

@inproceedings{walker2022dual,
  title     = {Dual Bayesian ResNet: A Deep Learning Approach to Heart Murmur Detection},
  author    = {Walker, Benjamin and Krones, Felix and Kiskin, Ivan and Parsons, Guy and Lyons, Terry and Mahdi, Adam},
  booktitle = {Computing in Cardiology (CinC)},
  volume    = {49},
  year      = {2022},
}

@inproceedings{McDonald2022ParallelHSMM,
  author    = {Andrew McDonald and Mark J. F. Gales and Anurag Agarwal},
  title     = {Detection of Heart Murmurs in Phonocardiograms with Parallel Hidden Semi-Markov Models},
  booktitle = {Computing in Cardiology (CinC)},
  year      = {2022},
  volume    = {49},
}

@inproceedings{Xu2022HMSNet,
  author    = {Yujia Xu and Xinqi Bao and Hak{-}Keung Lam and Ernest N. Kamavuako},
  title     = {Hierarchical Multi-Scale Convolutional Network for Murmurs Detection on PCG Signals},
  booktitle = {Computing in Cardiology (CinC)},
  year      = {2022},
  volume    = {49},
}

@inproceedings{Lee2022FreqTime,
  author    = {Jungguk Lee and Taein Kang and Narin Kim and Soyul Han and Hyejin Won and Wuming Gong and Il{-}Youp Kwak},
  title     = {Deep Learning Based Heart Murmur Detection Using Frequency-time Domain Features of Heartbeat Sounds},
  booktitle = {Computing in Cardiology (CinC)},
  year      = {2022},
  volume    = {49},
}

@inproceedings{Lu2022LightweightRobust,
  author    = {Hui Lu and Julia Beatriz Yip and Tobias Steigleder and Stefan Grie{\ss}hammer and Naga Venkata Sai Jitin Jami and Bjoern Eskofier and Christoph Ostgathe and Alexander Koelpin},
  title     = {A Lightweight Robust Approach for Automatic Heart Murmurs and Clinical Outcomes Classification from Phonocardiogram Recordings},
  booktitle = {Computing in Cardiology (CinC)},
  year      = {2022},
  volume    = {49},
}

@inproceedings{sutskever2011generating,
  title        = {Generating Text with Recurrent Neural Networks},
  author       = {Sutskever, Ilya and Martens, James and Hinton, Geoffrey E.},
  booktitle    = {Proceedings of the 28th International Conference on Machine Learning (ICML)},
  year         = {2011},
}

@InProceedings{pmlr-v235-lizaire24a,
  title     = {A Tensor Decomposition Perspective on Second-order {RNN}s},
  author    = {Lizaire, Maude and Rizvi-Martel, Michael and Gamal, Marawan and Rabusseau, Guillaume},
  booktitle = {Proceedings of the 41st International Conference on Machine Learning (ICML)},
  year = {2024},
}

@inproceedings{fan-etal-2024-advancing,
    title = "Advancing Regular Language Reasoning in Linear Recurrent Neural Networks",
    author = "Fan, Ting-Han  and
      Chi, Ta-Chung  and
      Rudnicky, Alexander",
    booktitle = "Proceedings of the 2024 Conference of the North American Chapter of the Association for Computational Linguistics: Human Language Technologies (Volume 2: Short Papers)",
    year = "2024",
    pages = "45--53",
}

@inproceedings{siems2025deltaproductimprovingstatetrackinglinear,
      title={DeltaProduct: Improving State-Tracking in Linear RNNs via Householder Products}, 
      author={Julien Siems and Timur Carstensen and Arber Zela and Frank Hutter and Massimiliano Pontil and Riccardo Grazzi},
      booktitle={Proceedings of the 39th Conference on Neural Information Processing Systems (NeurIPS)},
      year={2025}
}

@inproceedings{yang2025gateddeltanetworksimproving,
title={Gated Delta Networks: Improving Mamba2 with Delta Rule},
author={Songlin Yang and Jan Kautz and Ali Hatamizadeh},
booktitle={Proceedings of the 13th International Conference on Learning Representations (ICLR)},
year={2025},
}

@inproceedings{Dao2024Transformers,
author = {Dao, Tri and Gu, Albert},
title = {Transformers are SSMs: generalized models and efficient algorithms through structured state space duality},
year = {2024},
booktitle = {Proceedings of the 41st International Conference on Machine Learning (ICML)},
}

@article{peng2025rwkv7gooseexpressivedynamic,
      title={RWKV-7 `Goose with Expressive Dynamic State Evolution}, 
      author={Bo Peng and Ruichong Zhang and Daniel Goldstein and Eric Alcaide and Xingjian Du and Haowen Hou and Jiaju Lin and Jiaxing Liu and Janna Lu and William Merrill and Guangyu Song and Kaifeng Tan and Saiteja Utpala and Nathan Wilce and Johan S. Wind and Tianyi Wu and Daniel Wuttke and Christian Zhou-Zheng},
      year={2025},
      journal={arXiv preprint arXiv:2503.14456},
}

@inproceedings{peng2021random,
  title={Random Feature Attention},
  author={Peng, Hao and Pappas, Nikolaos and Yogatama, Dani and Schwartz, Roy and Smith, Noah A and Kong, Lingpeng},
  booktitle={Proceedings of the 9th International Conference on Learning Representations (ICLR)},
  year={2021}
}

@inproceedings{zhang2024gated,
  title={Gated Slot Attention for Efficient Linear-Time Sequence Modeling},
  author={Zhang, Yuxuan and Yang, Shiliang and Zhu, Rong and Zhang, Yichong and Cui, Lei and Wang, Yongjing and Wang, Bin and Shi, Feng and Wang, Bing and Bi, Wei and Zhou, Ping and Fu, Guoxin},
  booktitle={Proceedings of the 38th Conference on Neural Information Processing Systems (NeurIPS)},
  year={2024}
}

@inproceedings{sun2025learninglearntesttime,
      title={Learning to (Learn at Test Time): RNNs with Expressive Hidden States}, 
      author={Yu Sun and Xinhao Li and Karan Dalal and Jiarui Xu and Arjun Vikram and Genghan Zhang and Yann Dubois and Xinlei Chen and Xiaolong Wang and Sanmi Koyejo and Tatsunori Hashimoto and Carlos Guestrin},
      year={2025},
      booktitle={Proceedings of the 42nd International Conference on Machine Learning (ICML)}
}

@inproceedings{behrouz2025titanslearningmemorizetest,
      title={Titans: Learning to Memorize at Test Time}, 
      author={Ali Behrouz and Peilin Zhong and Vahab Mirrokni},
      year={2025},
      booktitle={Proceedings of the 39th Conference on Neural Information Processing Systems (NeurIPS)}
}

@inproceedings{beck2024xlstm,
  title={xLSTM: Extended Long Short-Term Memory},
  author={Beck, Maximilian and P{\"o}ppel, Korbinian and Spanring, Markus and Auer, Andreas and Prudnikova, Oleksandra and Kopp, Michael and Klambauer, G{\"u}nter and Brandstetter, Johannes and Hochreiter, Sepp},
  booktitle={Proceedings of the 38th Conference on Neural Information Processing Systems (NeurIPS)},
  year={2024}
}

@InProceedings{pmlr-v119-katharopoulos20a,
  title     = {Transformers are {RNN}s: Fast Autoregressive Transformers with Linear Attention},
  author    = {Katharopoulos, Angelos and Vyas, Apoorv and Pappas, Nikolaos and Fleuret, Fran{\c{c}}ois},
  booktitle = {Proceedings of the 37th International Conference on Machine Learning (ICML)},
  year      = {2020},
}

@inproceedings{Giles1990HigherOrderRNN,
  author    = {C. L. Giles and G. Z. Sun and H. H. Chen and Y. C. Lee and D. Chen},
  title     = {Higher Order Recurrent Networks and Grammatical Inference},
  booktitle = {Proceedings of the 3rd International Conference on Neural Information Processing Systems (NeurIPS)},
  year      = {1989}
}

@inproceedings{schlag2021linear,
  author    = {Imanol Schlag and Kazuki Irie and J{\"u}rgen Schmidhuber},
  title     = {Linear Transformers Are Secretly Fast Weight Programmers},
  booktitle = {Proceedings of the 38th International Conference on Machine Learning (ICML)},
  year      = {2021},
}

@inproceedings{deletang2023neural,
  author       = {Gr{\'{e}}goire Del{\'{e}}tang and
                  Anian Ruoss and
                  Jordi Grau{-}Moya and
                  Tim Genewein and
                  Li Kevin Wenliang and
                  Elliot Catt and
                  Chris Cundy and
                  Marcus Hutter and
                  Shane Legg and
                  Joel Veness and
                  Pedro A. Ortega},
  title        = {Neural Networks and the Chomsky Hierarchy},
  booktitle    = {Proceedings of the 11th International Conference on Learning Representations (ICLR)},
  year         = {2023},
}

@inproceedings{movahedi2025fixedpointrnnsdiagonaldense,
      title={Fixed-Point RNNs: From Diagonal to Dense in a Few Iterations}, 
      author={Sajad Movahedi and Felix Sarnthein and Nicola Muca Cirone and Antonio Orvieto},
      year={2025},
      booktitle={Proceedings of the 39th Conference on Neural Information Processing Systems (NeurIPS)},
}

@inproceedings{Cirone2023neural,
  title        = {Neural signature kernels as infinite‐width‐depth‐limits of controlled ResNets},
  author       = {Cirone, N.~M. and Lemercier, M. and Salvi, C.},
  booktitle    = {Proceedings of the 40th International Conference on Machine Learning (ICML)},
  year         = {2023}
}

@inproceedings{grazzi2024unlockingstatetrackinglinearrnns,
      title={Unlocking State-Tracking in Linear RNNs Through Negative Eigenvalues}, 
      author={Riccardo Grazzi and Julien Siems and Jörg K. H. Franke and Arber Zela and Frank Hutter and Massimiliano Pontil},
  booktitle    = {Proceedings of the 13th International Conference on Learning Representations (ICLR)},
  year         = {2025},
}

@article{paszke2019pytorchimperativestylehighperformance,
      title={PyTorch: An Imperative Style, High-Performance Deep Learning Library}, 
      author={Adam Paszke and Sam Gross and Francisco Massa and Adam Lerer and James Bradbury and Gregory Chanan and Trevor Killeen and Zeming Lin and Natalia Gimelshein and Luca Antiga and Alban Desmaison and Andreas Köpf and Edward Yang and Zach DeVito and Martin Raison and Alykhan Tejani and Sasank Chilamkurthy and Benoit Steiner and Lu Fang and Junjie Bai and Soumith Chintala},
      year={2019},
      journal={arXiv preprint arXiv:1912.01703},
}

@article{cirone2025parallelflowparallelizinglineartransformers,
      title={ParallelFlow: Parallelizing Linear Transformers via Flow Discretization}, 
      author={Nicola Muca Cirone and Cristopher Salvi},
      year={2025},
      journal={arXiv preprint arXiv:2504.00492},
}

@ARTICLE{Li2019Symmetry,
  author={Li, Xingguo and Lu, Junwei and Arora, Raman and Haupt, Jarvis and Liu, Han and Wang, Zhaoran and Zhao, Tuo},
  journal={IEEE Transactions on Information Theory}, 
  title={Symmetry, Saddle Points, and Global Optimization Landscape of Nonconvex Matrix Factorization}, 
  year={2019},
  volume={65},
  number={6},
  pages={3489-3514},
}

@article{saarinen1993ill,
author = {Saarinen, S. and Bramley, R. and Cybenko, G.},
title = {Ill-Conditioning in Neural Network Training Problems},
journal = {SIAM Journal on Scientific Computing},
volume = {14},
number = {3},
pages = {693-714},
year = {1993},
}

@inproceedings{
  meng2019gsgd,
  title={{G-SGD}: Optimizing {ReLU} Neural Networks in Its Positively Scale-Invariant Space},
  author={Qi Meng and Shuxin Zheng and Huishuai Zhang and Wei Chen and Qiwei Ye and Zhi-Ming Ma and Nenghai Yu and Tie-Yan Liu},
  booktitle={Proceedings of the 7th International Conference on Learning Representations (ICLR)},
  year={2019},
}

@article{Perez_Arribas2018-xp,
  title    = "A signature-based machine learning model for distinguishing
              bipolar disorder and borderline personality disorder",
  author   = "Perez Arribas, Imanol and Goodwin, Guy M and Geddes, John R and
              Lyons, Terry and Saunders, Kate E A",
  journal  = "Transl Psychiatry",
  volume   =  8,
  number   =  1,
  year     =  2018,
}

@INPROCEEDINGS{Weixin2016,
  author={Yang, Weixin and Jin, Lianwen and Hao Ni and Lyons, Terry},
  booktitle={Proceedings of the 23rd International Conference on Pattern Recognition (ICPR)}, 
  title={Rotation-free online handwritten character recognition using dyadic path signature features, hanging normalization, and deep neural network}, 
  year={2016},
}

@inproceedings{wang19e_interspeech,
  author={Bo Wang and Maria Liakata and Hao Ni and Terry Lyons and Alejo J. Nevado-Holgado and Kate Saunders},
  title={{A Path Signature Approach for Speech Emotion Recognition}},
  year={2019},
  booktitle={Proc. Interspeech 2019},
  pages={1661--1665},
}

@article{graham2013sparse,
      title={Sparse arrays of signatures for online character recognition}, 
      author={Benjamin Graham},
      year={2013},
      journal={arXiv preprint arXiv:1308.0371},
}

@article{gyurkó2014extractinginformationsignaturefinancial,
      title={Extracting information from the signature of a financial data stream}, 
      author={Lajos Gergely Gyurkó and Terry Lyons and Mark Kontkowski and Jonathan Field},
      year={2014},
      journal={arXiv preprint arXiv:1307.7244},
}

@article{levin2016learningpastpredictingstatistics,
      title={Learning from the past, predicting the statistics for the future, learning an evolving system}, 
      author={Daniel Levin and Terry Lyons and Hao Ni},
      year={2016},
      journal={arXiv preprint arXiv:1309.0260},
}

@article{kermack1927contribution,
  title={A contribution to the mathematical theory of epidemics},
  author={Kermack, W.O. and McKendrick, A.G.},
  journal={Proceedings of the Royal Society of London. Series A},
  volume={115},
  number={772},
  pages={700--721},
  year={1927}
}

@article{solow1956contribution,
  title={A Contribution to the Theory of Economic Growth},
  author={Solow, Robert M.},
  journal={The Quarterly Journal of Economics},
  volume={70},
  number={1},
  pages={65--94},
  year={1956}
}

@book{lotka1925elements,
  title={Elements of Physical Biology},
  author={Lotka, Alfred J.},
  year={1925},
  publisher={Williams \& Wilkins Company}
}

@article{volterra1926fluctuations,
  title={Fluctuations in the abundance of a species considered mathematically},
  author={Volterra, Vito},
  journal={Nature},
  volume={118},
  pages={558--560},
  year={1926}
}

@article{Vauvelle2022NeuralSignatureEHR,
  author       = {Andre Vauvelle and Paidi Creed and Spiros Denaxas},
  title        = {Neural‑signature methods for structured EHR prediction},
  journal      = {BMC Medical Informatics and Decision Making},
  volume       = {22},
  number       = {1},
  year         = {2022},
}

@article{cohen2023subtle,
author = {Cohen, Samuel N and Foster, James and Foster, Peter and Lou, Hang and Lyons, Terry and Morley, Sam and Morrill, James and Ni, Hao and Palmer, Edward and Wang, Bo and Wu, Yue and Yang, Lingyi and Yang, Weixin},
title = {Subtle variations in sepsis-{III} definitions markedly affect predictive performance},
year = {2024},
journal = {Nature Scientific Reports},
volume={14},
number={1920}
}

@inproceedings{arribas2020sigsdes,
  title={Sig-SDEs model for quantitative finance},
  author={Perez Arribas, Imanol and Cristopher Salvi and Lukasz Szpruch},
  booktitle={ACM International Conference on AI in Finance},
  year={2020}
}

@inproceedings{lemercier2021siggpde,
  title={SigGPDE: Scaling Sparse Gaussian Processes on Sequential Data},
  author={Lemercier, Maud and Salvi, Cristopher and Cass, Thomas and Bonilla, Edwin V and Damoulas, Theodoros and Lyons, Terry},
  booktitle={International Conference on Machine Learning (ICML)},
  year={2021},
  organization={PMLR}
}

@article{horvath2023optimal,
  title={Optimal Stopping via Distribution Regression: A Higher Rank Signature Approach},
  author={Horvath, Blanka and Lemercier, Maud and Liu, Chong and Lyons, Terry and Salvi, Cristopher},
  journal={arXiv preprint arXiv:2304.01479},
  year={2023}
}

@inproceedings{manten2025signature,
  title={Signature Kernel Conditional Independence Tests in Causal Discovery for Stochastic Processes},
  author={Manten, Georg and Casolo, Cecilia and Ferrucci, Emilio and Mogensen, S{\o}ren Wengel and Salvi, Cristopher and Kilbertus, Niki},
  booktitle={Proceedings of the 13th International Conference on Learning Representations (ICLR)},
  year={2025},
}

@article{cohen2023nowcasting,
  title={Nowcasting with signature methods},
  author={Cohen, Samuel N. and Lui, Silvia and Malpass, Will and Mantoan, 
  Giulia and Nesheim, Lars and de Paula, \'{A}ureo and Reeves, Andrew and
  Scott, Craig and Small, Emma and Yang, Lingyi},
  journal={arXiv preprint arXiv:2305.10256},
  year={2023}
}

@article{cirone2025rough,
  title={Rough kernel hedging},
  author={Cirone, Nicola Muca and Salvi, Cristopher},
  journal={arXiv preprint arXiv:2501.09683},
  year={2025}
}

@article{salvi2023structure,
  title={A structure theorem for streamed information},
  author={Salvi, Cristopher and Diehl, Joscha and Lyons, Terry and Preiss, Rosa and Reizenstein, Jeremy},
  journal={Journal of Algebra},
  volume={634},
  pages={911--938},
  year={2023},
  publisher={Elsevier}
}

@article{shmelev2024sparse,
  title={Sparse Signature Coefficient Recovery via Kernels},
  author={Shmelev, Daniil and Salvi, Cristopher},
  journal={arXiv preprint arXiv:2412.08579},
  year={2024}
}

@inproceedings{cochrane2021sk,
  title={SK-Tree: a systematic malware detection algorithm on streaming trees via the signature kernel},
  author={Cochrane, Thomas and Foster, Peter and Chhabra, Varun and Lemercier, Maud and Lyons, Terry and Salvi, Cristopher},
  booktitle={2021 IEEE International Conference on Cyber Security and Resilience (CSR)},
  pages={35--40},
  year={2021},
  organization={IEEE}
}

@inproceedings{holberg2024exact,
  title={Exact Gradients for Stochastic Spiking Neural Networks Driven by Rough Signals},
  author={Holberg, Christian and Salvi, Cristopher},
  booktitle={Proceedings of the 38th Conference on Neural Information Processing Systems (NeurIPS)},
  year={2024}
}

@inproceedings{lemercier2021distribution,
  title={Distribution regression for sequential data},
  author={Lemercier, Maud and Salvi, Cristopher and Damoulas, Theodoros and Bonilla, Edwin and Lyons, Terry},
  booktitle={International Conference on Artificial Intelligence and Statistics (AISTATS)},
  year={2021},
  organization={PMLR}
}

@phdthesis{salvi2021rough,
  title={Rough paths, kernels, differential equations and an algebra of functions on streams},
  author={Salvi, Cristopher},
  year={2021},
  school={University of Oxford}
}

@article{salvi2021signature,
  title={The Signature Kernel is the solution of a {Goursat PDE}},
  author={Salvi, Cristopher and Cass, Thomas and Foster, James and Lyons, Terry and Yang, Weixin},
  journal={SIAM Journal on Mathematics of Data Science},
  volume={3},
  number={3},
  pages={873--899},
  year={2021},
  publisher={SIAM}
}

@inproceedings{choi2022STGNCDE,
  title        = {Graph Neural Controlled Differential Equations for Traffic Forecasting},
  author       = {Choi, Jeongwhan and Choi, Hwangyong and Hwang, Jeehyun and Park, Noseong},
  booktitle    = {Proceedings of the 36th AAAI Conference on Artificial Intelligence (AAAI-22)},
  year         = {2022},
}

@article{choi2023graphneuralroughdifferential,
  title        = {Graph Neural Rough Differential Equations for Traffic Forecasting},
  author       = {Choi, Jeongwhan and Park, Noseong},
  journal      = {ACM Transactions on Intelligent Systems and Technology},
  volume       = {14},
  number       = {4},
  year         = {2023},
}

@article{qin2023learningdynamicgraphembeddings,
  title={Learning dynamic graph embeddings with neural controlled differential equations},
  author={Qin, Tiexin and Walker, Benjamin and Lyons, Terry and Yan, Hong and Li, Haoliang},
  journal={IEEE Transactions on Pattern Analysis and Machine Intelligence},
  year={2025},
  publisher={IEEE}
}

@inproceedings{berndt2025permutation,
    author = {Torben Berndt and Benjamin Walker and Tiexin Qin and Jan Stühmer and Andrey Kormilitzin},
    title = {Permutation Equivariant Neural Controlled Differential Equations for Dynamic Graph Representation Learning},
    booktitle = {Proceedings of the 39th Conference on Neural Information Processing Systems (NeurIPS)},
    year = {2025},
}

@InProceedings{seedat2022continuous,
  title     = {Continuous‑Time Modeling of Counterfactual Outcomes Using Neural Controlled Differential Equations},
  author    = {Seedat, Nabeel and Imrie, Fergus and Bellot, Alexis and Qian, Zhaozhi and van der Schaar, Mihaela},
  booktitle = {Proceedings of the 39th International Conference on Machine Learning (ICML)},
  year      = {2022},
}

@article{lee2024gdflowanomalydetectionncdebased,
      title={GDFlow: Anomaly Detection with NCDE-based Normalizing Flow for Advanced Driver Assistance System}, 
      author={Kangjun Lee and Minha Kim and Youngho Jun and Simon S. Woo},
      year={2024},
      journal={arXiv preprint arXiv:2409.05346},
}

@inproceedings{seonkyu2024bridging,
author = {Lim, Seonkyu and Choi, Jeongwhan and Park, Noseong and Yoon, Sang-Ha and Kang, ShinHyuck and Kim, Young-Min and Kang, Hyunjoong},
year = {2024},
title = {Bridging Dynamic Factor Models and Neural Controlled Differential Equations for Nowcasting GDP},
booktitle = {Proceedings of the 33rd ACM International Conference on Information and Knowledge Management},
}

@article{morrill2022interpolation,
  title     = {On the Choice of Interpolation Scheme for Neural CDEs},
  author    = {James Morrill and Patrick Kidger and Lingyi Yang and Terry Lyons},
  journal   = {Transactions on Machine Learning Research},
  volume    = {2022},
  number    = {9},
  article   = {192},
  year      = {2022},
}

@inproceedings{zeng2025trajsurv,
  title     = {TrajSurv: Learning Continuous Latent Trajectories from Electronic Health Records for Trustworthy Survival Prediction},
  author    = {Sihang Zeng and Lucas Jing Liu and Jun Wen and Meliha Yetisgen and Ruth Etzioni and Gang Luo},
  booktitle = {Proceedings of the Machine Learning for Healthcare (MLHC) Conference},
  year      = {2025},
}

@inproceedings{Liao2021LogsigRNN,
  author       = {Shujian Liao and Terry Lyons and Weixin Yang and Kevin Schlegel and Hao Ni},
  title        = {Logsig‑RNN: a novel network for robust and efficient skeleton‑based action recognition},
  booktitle    = {Proceedings of the 32nd British Machine Vision Conference (BMVC 2021)},
  year         = {2021},
}

@inproceedings{MorenoPino2024RoughTransformers,
  author       = {Fernando Moreno‑Pino and Álvaro Arroyo and Harrison Waldon and Xiaowen Dong and Álvaro Cartea},
  title        = {Rough Transformers: Lightweight and Continuous Time‑Series Modelling through Signature Patching},
  booktitle    = {Proceedings of the 38th Conference on Neural Information Processing Systems (NeurIPS)},
  year         = {2024},
}

@book{ramsay2005functional,
  title     = {Functional Data Analysis},
  author    = {Ramsay, James O. and Silverman, Bernard W.},
  year      = {2005},
  edition   = {2},
  publisher = {Springer},
  address   = {New York},
}

@inproceedings{cuturi2017softdtw,
  title     = {Soft-{DTW}: a Differentiable Loss Function for Time-Series},
  author    = {Cuturi, Marco and Blondel, Mathieu},
  booktitle = {Proceedings of the 34th International Conference on Machine Learning (ICML)},
  year      = {2017},
}

@inproceedings{chen2025patchwise,
  title     = {Patch-wise Structural Loss for Time Series Forecasting},
  author    = {Chen, Yuxuan and Zhang, Jiahui and Li, Xin and others},
  booktitle = {Proceedings of the 42nd International Conference on Machine Learning (ICML)},
  year      = {2025},
}

@article{ferraty2007nonparametric,
author = {Ferraty, Frederic and Mas, André and Vieu, Philippe},
year = {2007},
pages = {267 - 286},
title = {Nonparametric regression on functional data: Inference and practical aspects},
volume = {49},
number = {3},
journal = {Australian \& New Zealand Journal of Statistics},
}

@article{Lyons1998,
author = {Lyons, Terry},
journal = {Revista Matemática Iberoamericana},
number = {2},
pages = {215-310},
title = {Differential equations driven by rough signals.},
volume = {14},
year = {1998},
}

@article{chevyrev2016characteristic,
  title     = {Characteristic functions of measures on geometric rough paths},
  author    = {Chevyrev, Ilya and Lyons, Terry J.},
  journal   = {Annals of Probability},
  volume    = {44},
  number    = {6},
  pages     = {4049--4082},
  year      = {2016},
}

@article{kiraly2019kernels,
  title     = {Kernels for sequentially ordered data},
  author    = {Király, Franz J. and Oberhauser, Harald},
  journal   = {Journal of Machine Learning Research},
  volume    = {20},
  number    = {31},
  pages     = {1--45},
  year      = {2019},
}

@article{cass2024weighted,
  title        = {Weighted Signature Kernels},
  author       = {Cass, Thomas and Lyons, Terry and Xu, Xingcheng},
  journal      = {Annals of Applied Probability},
  volume       = {34},
  number       = {1A},
  pages        = {585--626},
  year         = {2024},
}

@inproceedings{Cuturi2007AKF,
  title={A Kernel for Time Series Based on Global Alignments},
  author={Marco Cuturi and Jean-Philippe Vert and {\O}ystein Birkenes and Tomoko Matsui},
  booktitle={Proceedings of the IEEE International Conference on Acoustics, Speech and Signal Processing},
  year={2007},
}

@article{mikalsen2018time,
title = {Time series cluster kernel for learning similarities between multivariate time series with missing data},
journal = {Pattern Recognition},
volume = {76},
pages = {569-581},
year = {2018},
author = {Karl Øyvind Mikalsen and Filippo Maria Bianchi and Cristina Soguero-Ruiz and Robert Jenssen},
}

@article{wynne2022kernel,
author = {Wynne, George and Duncan, Andrew B.},
title = {A kernel two-sample test for functional data},
year = {2022},
issue_date = {January 2022},
volume = {23},
journal = {Journal of Machine Learning Research},
}

@article{nyquist1928certain,
  title={Certain topics in telegraph transmission theory},
  author={Nyquist, Harry},
  journal={Transactions of the American Institute of Electrical Engineers},
  volume={47},
  number={2},
  pages={617--644},
  year={1928},
  publisher={IEEE},
}

@article{shannon1949communication,
  title={Communication in the presence of noise},
  author={Shannon, Claude Elwood},
  journal={Proceedings of the IRE},
  volume={37},
  number={1},
  pages={10--21},
  year={1949},
  publisher={IEEE},
}

@article{Stieltjes1894,
  author = {Stieltjes, T. J.},
  title = {Recherches sur les fractions continues},
  journal = {Annales de la Faculté des Sciences de Toulouse pour les Sciences Mathématiques et les Sciences Physiques},
  volume = {8},
  number = {4},
  pages = {1--122},
  year = {1894}
}

@article{cass2024lecturenotesroughpaths,
      title={Lecture notes on rough paths and applications to machine learning}, 
      author={Thomas Cass and Cristopher Salvi},
      year={2024},
      journal={arXiv preprint arXiv:2404.06583},
}

@article{Strichartz1987,
  author       = {Strichartz, Robert S.},
  title        = {The Campbell--Baker--Hausdorff--Dynkin Formula and Solutions of Differential Equations},
  journal      = {Journal of Functional Analysis},
  year         = {1987},
  volume       = {72},
  number       = {2},
  pages        = {320--345},
}

@article{MoanNiesen2008,
  author  = {Moan, Per Christian and Niesen, Jitse},
  title   = {Convergence of the Magnus Series},
  journal = {Foundations of Computational Mathematics},
  year    = {2008},
  volume  = {8},
  pages   = {291--301},
}

@article{Magnus1954,
  author    = {Wilhelm Magnus},
  title     = {On the exponential solution of differential equations for a linear operator},
  journal   = {Communications on Pure and Applied Mathematics},
  volume    = {7},
  number    = {4},
  pages     = {649--673},
  year      = {1954},
}

@article{Heun1900,
  author       = {Heun, Karl},
  title        = {Neue Methode zur approximativen Integration der Differentialgleichungen einer unabhängigen Veränderlichen},
  journal      = {Zeitschrift für Mathematik und Physik},
  volume       = {45},
  pages        = {23--38},
  year         = {1900},
}

@book{AtkinsonHanStewart2009,
  author       = {Atkinson, Kendall E. and Han, Weimin and Stewart, David E.},
  title        = {Numerical Solution of Ordinary Differential Equations},
  publisher    = {John Wiley \& Sons},
  year         = {2009},
}

@article{tsitouras2011runge,
  title     = {Runge--Kutta pairs of order 5(4) satisfying only the first column simplifying assumption},
  author    = {Tsitouras, Charalampos},
  journal   = {Computers \& Mathematics with Applications},
  volume    = {62},
  number    = {2},
  pages     = {770--775},
  year      = {2011},
}

@article{cirone2025genusexpansionnonlinearrandom,
      title={Genus expansion for non-linear random matrix ensembles with applications to neural networks}, 
      author={Nicola Muca Cirone and Jad Hamdan and Cristopher Salvi},
      year={2025},
      journal={arXiv preprint arXiv:2407.08459},
}

@book{whittle1951hypothesis,
  author    = {Whittle, Peter},
  title     = {Hypothesis Testing in Time Series Analysis},
  year      = {1951},
  publisher = {Almqvist \& Wiksells},
}

@techreport{galtieri1964,
  author      = {Galtieri, C. A.},
  title       = {Problems of Estimation in Discrete-Time Processes},
  institution = {IBM Research Laboratory},
  address     = {San Jose, California},
  number      = {RJ-315},
  year        = {1964},
}

@inproceedings{astrom1965numerical,
  author    = {{\AA}str{\"o}m, Karl Johan and Bohlin, Torsten},
  title     = {Numerical Identification of Linear Dynamic Systems from Normal Operating Records},
  booktitle = {Proc. IFAC Conference on Self-Adaptive Control Systems},
  volume = {2},
  issue = {2},
  year      = {1965}
}

@article{astrom1971system,
  author  = {{\AA}str{\"o}m, K. J. and Eykhoff, P.},
  title   = {System Identification---A Survey},
  journal = {Automatica},
  volume  = {7},
  number  = {2},
  pages   = {123--162},
  year    = {1971},
}

@article{vanoverschee1994n4sid,
  author  = {Van Overschee, Peter and De Moor, Bart},
  title   = {N4SID: Subspace Algorithms for the Identification of Combined Deterministic-Stochastic Systems},
  journal = {Automatica},
  volume  = {30},
  number  = {1},
  pages   = {75--93},
  year    = {1994},
}

@article{kolmogorov1941,
  author  = {Kolmogorov, A. N.},
  title   = {Интерполирование и экстраполирование стационарных случайных последовательностей},
  journal = {Известия АН СССР. Серия математическая},
  volume  = {5},
  number  = {1},
  pages   = {3--14},
  year    = {1941}
}

@techreport{kolmogorov1962english,
  author      = {Kolmogorov, A. N.},
  title       = {Interpolation and Extrapolation of Stationary Random Sequences},
  institution = {RAND Corporation},
  number      = {RM-3090-PR},
  address     = {Santa Monica, CA},
  year        = {1962},
  note        = {English translation by Worthie L. Doyle and Ivan Selin}
}

@article{ohagan1978curve,
  author  = {O'Hagan, Adrian},
  title   = {Curve Fitting and Optimal Design for Prediction},
  journal = {Journal of the Royal Statistical Society. Series B (Methodological)},
  volume  = {40},
  number  = {1},
  pages   = {1--42},
  year    = {1978}
}

@inproceedings{williams1996gaussian,
  author    = {Williams, Christopher K. I. and Rasmussen, Carl Edward},
  title     = {Gaussian Processes for Regression},
  booktitle = {Advances in Neural Information Processing Systems 8},
  pages     = {514--520},
  publisher = {MIT Press},
  year      = {1996}
}

@book{rasmussen2006gaussian,
  author    = {Rasmussen, Carl Edward and Williams, Christopher K. I.},
  title     = {Gaussian Processes for Machine Learning},
  publisher = {MIT Press},
  address   = {Cambridge, MA},
  year      = {2006}
}

@inproceedings{kidger2021neural,
  author    = {Kidger, Patrick and Foster, James and Li, Xuechen and Lyons, Terry J},
  title     = {Neural SDEs as Infinite-Dimensional GANs},
  booktitle = {Proceedings of the 38th International Conference on Machine Learning (ICML)},
  year      = {2021},
}

@inproceedings{hess2024bncde,
  author    = {Hess, Konstantin and Melnychuk, Valentyn and Frauen, Dennis and Feuerriegel, Stefan},
  title     = {Bayesian Neural Controlled Differential Equations for Treatment Effect Estimation},
  booktitle = {Proceedings of the 12th International Conference on Learning Representations (ICLR)},
  year      = {2024},
}

@inproceedings{garnelo2018conditional,
  author    = {Garnelo, Marta and Rosenbaum, Dan and Maddison, Christopher J. and Ramalho, Tiago and Saxton, David and Shanahan, Murray and Teh, Yee Whye and Rezende, Danilo J. and Eslami, S. M. Ali},
  title     = {Conditional Neural Processes},
  booktitle = {Proceedings of the 35th International Conference on Machine Learning (ICML)},
  year      = {2018},
}

@InProceedings{cunningham2012gaussian,
  title = {Gaussian Processes for time-marked time-series data},
  author = {Cunningham, John and Ghahramani, Zoubin and Rasmussen, Carl},
  booktitle = {Proceedings of the Fifteenth International Conference on Artificial Intelligence and Statistics (AISTATS)},
  year = {2012},
}

@article{demsar2006statistical,
  author  = {Janez Dem\v{s}ar},
  title   = {Statistical Comparisons of Classifiers over Multiple Data Sets},
  journal = {Journal of Machine Learning Research},
  volume  = {7},
  year    = {2006}
}

@article{wilcoxon1945individual,
  author  = {Frank Wilcoxon},
  title   = {Individual Comparisons by Ranking Methods},
  journal = {Biometrics Bulletin},
  volume  = {1},
  number  = {6},
  year    = {1945},
}

@article{lou2024pathdevelopment,
  title   = {Path Development Network with Finite-dimensional Lie Group Representation},
  author  = {Lou, Hang and Li, Siran and Ni, Hao},
  journal = {Transactions on Machine Learning Research},
  year    = {2024},
}

@article{jiang2024gcndevlstm,
      title={GCN-DevLSTM: Path Development for Skeleton-Based Action Recognition}, 
      author={Jiang, Lei and Yang, Weixin and Zhang, Xin and Ni, Hao},
      year={2024},
      journal={arXiv preprint, arXiv.2403.15212},
}

\end{document}